\documentclass[10pt]{article} 
\usepackage[left=1in,top=1in,right=1in,bottom=1in,letterpaper]{geometry}

\usepackage[ansinew]{inputenc}
\usepackage{ upgreek }

\usepackage{multicol}
\usepackage{blindtext}
\usepackage{enumitem}
\usepackage{calc}
\usepackage{ifthen}
\usepackage{amsmath,amsthm,amsfonts,amssymb}
\usepackage{color,graphicx,overpic}
\usepackage{hyperref, thmtools}
\usepackage{bbm}
\usepackage{algorithm}
\usepackage{algpseudocode}
\usepackage{subcaption}
\usepackage{nicefrac}
\usepackage{nicematrix}
\usepackage{pgfplots}
\usepgfplotslibrary{groupplots}
\pgfplotsset{compat=1.17}
\usepackage{tikz}
\usepackage{adjustbox}
\usepackage{calc}

\newlength\TopY
\newlength\BotY
\newlength\TotalHeight

\newcommand{\eq}[1]{\begin{align}#1\end{align}}

\newcommand{\bhline}[1]{\noalign{\hrule height #1}}

\newcommand*\samethanks[1][\value{footnote}]{\footnotemark[#1]}

\usepackage{tcolorbox}  

\tcbset{
  colframe=black!60,
  colback=white,
  boxrule=0.5pt,
  sharp corners,
  left=5pt, right=5pt, top = 4pt, bottom = 4pt}

\usepackage{titlesec}
\titlespacing*{\paragraph}{0pt}{2.5pt plus 1pt minus 1pt}{1em}

\makeatletter
\newcommand{\leqnomode}{\tagsleft@true}
\newcommand{\reqnomode}{\tagsleft@false}
\makeatother

\usepackage{ upgreek }
\usepackage{graphicx} 
\usepackage{textcomp} 
\usepackage{amsfonts} 
\usepackage{amsmath}
\usepackage{amssymb}
\usepackage{yfonts}
\usepackage{mathtools}
\usepackage{mathrsfs}
\usepackage{latexsym}
\usepackage{array}
\usepackage{float}
\usepackage{calc}
\usepackage{wrapfig}
\usepackage{textcomp}
\usepackage{verbatim} 
\usepackage{hieroglf}
\usepackage{multirow} 
\usepackage{rotating} 
\usepackage{bm}
\usepackage{dsfont} 
\usepackage{etoolbox}

\usepackage{url}
\usepackage{color}
\usepackage{tikz}
\usetikzlibrary{shapes,arrows,positioning,calc}
\usepackage{float}
\usepackage{soul}
\usepackage{enumitem}
\usepackage{booktabs}

\usepackage{xcolor}
\usepackage{algorithm}
\usepackage{algpseudocode}
    
\usepackage[normalem]{ulem}

\newcommand{\tri}[1]{{\left\vert\kern-0.25ex\left\vert\kern-0.25ex\left\vert #1 
    \right\vert\kern-0.25ex\right\vert\kern-0.25ex\right\vert}}

\newcommand{\RN}[1]{%
  \textup{\uppercase\expandafter{\romannumeral#1}}%
}

\newcommand{\rnu}[1]{%
  \textup{\expandafter{\romannumeral#1}}%
}

\DeclareMathOperator*{\argmin}{arg\,min}

\def\E{\mathbb{E}}

\def\G{\mathbb{G}}

\def\N{\mathbb{N}}
\def\cN{\mathcal{N}}
\def\cP{\mathcal{P}}
\def\R{\mathbb{R}}
\def\Z{\mathbb{Z}}

\newcommand{\mpr}{\mathbb{P}}

\newcommand{\grad}{\nabla}

\makeatletter
\@addtoreset{equation}{section}
\makeatother

\hypersetup{
  colorlinks,
  linkcolor={red!50!black},
  citecolor={blue!50!black},
  urlcolor={blue!80!black}
}

\DeclarePairedDelimiter\ceil{\lceil}{\rceil}
\DeclarePairedDelimiter\floor{\lfloor}{\rfloor}
\DeclarePairedDelimiter\abs{\lvert}{\rvert}%
\DeclarePairedDelimiter\norm{\lVert}{\rVert}%

\newcommand*{\inner}[2]{\left  \langle #1,  #2 \right \rangle}

\newcommand*{\rom}[1]{\expandafter\@slowromancap\romannumeral #1@}

\def\tr{\text{Tr}}
\def\sym{\mathrm{Sym}}

\newcommand*{\indic}[1]{\mathbbm{1}_{  #1  }   }
\newcommand{\bs}[1]{\boldsymbol{#1}}

\def\cR{\mathcal{R}}
\def\c{\upkappa}
\def\rks{\mathsf{rk}_{\star}} 
\def\rk{\mathsf{rk}} 
\def\rs{r_s}
\def\ru{r_u}
\def\rus{r_{u_{\star}}}

\def\W{\bs{W}}
\def\Z{\bs{Z}}
\def\Th{\bs{\Theta}}
\def\lr{\upeta}
\def\normL{\norm{\bs{\Lambda}}_\mathrm{F}}

\def\V{\bs{V}}

\def\tL{\widetilde{\bs{\Lambda}}}
\def\Lae{\bs{\Lambda}_{\mathrm{e}}}
\def\La{\bs{\Lambda}}

\newcommand{\uV}[1]{\mspace{2.2mu}\underline{\mspace{-2.2mu}\bs{V}\mspace{-2.2mu}}_{#1} \mspace{2.2mu}}
\newcommand{\uVs}[1]{\mspace{2.2mu}\underline{\mspace{-2.2mu}\bs{V}\mspace{-2.2mu}}_{#1}^{\,2} \mspace{2.2mu}}
\newcommand{\uVc}[1]{\mspace{2.2mu}\underline{\mspace{-2.2mu}\bs{V}\mspace{-2.2mu}}_{#1}^{\,3} \mspace{2.2mu}}
\newcommand{\uVq}[1]{\mspace{2.2mu}\underline{\mspace{-2.2mu}\bs{V}\mspace{-2.2mu}}_{#1}^{\,4} \mspace{2.2mu}}
\newcommand{\oV}[1]{\bar{\bs{V}}_{#1}}
\newcommand{\uK}[1]{\underline{\bs{K}}_{#1}}
\newcommand{\oK}[1]{\bar{\bs{K}}_{#1}}
\newcommand{\uB}[1]{\underline{\bs{B}}_{#1}}
\newcommand{\oB}[1]{\bar{\bs{B}}_{#1}}
\newcommand{\uG}[1]{\underline{\bs{G}}_{#1}}
\newcommand{\oG}[1]{\bar{\bs{G}}_{#1}}

\newcommand{\unu}[1]{\underline{\bs{\nu}}_{#1}}

\newcommand{\oze}[1]{\bar{\bs{\zeta}}_{#1}}
\newcommand{\uze}[1]{\underline{\bs{\zeta}}_{#1}}

\def\Llf{\bs{\Lambda}_{\ell_1} }
\def\Llftop{\bs{\Lambda}_{\ell_1,11} }

\def\Lls{\bs{\Lambda}_{\ell_2} }
\def\Llstop{\bs{\Lambda}_{\ell_2,11} }

\def\Luf{\bs{\Lambda}_{u_1} }
\def\Luftop{\bs{\Lambda}_{u_1,11} }

\def\Lus{\bs{\Lambda}_{u_2} }

\def\A{\bs{A}}

\def\M{\bs{M}}
\def\G{\bs{G}}
\def\T{\bs{T}}
\def\F{\bs{F}}
\def\D{\bs{D}}
\def\S{\bs{S}}
\def\U{\bs{U}}
\def\hL{\hat{\bs{\Lambda}}}
\def\sTh{\mathsf{\Theta}}
\def\sM{\mathsf{M}}
\def\sG{\mathsf{G}}

\def\Lst{\grad_{\text{St}} \bs{L}_{t+1}}

\def\Pst{ \bs{\mathcal{P}}_{t+1}}

\newcounter{relctr} 
\newcommand\labelrel[2]{%
  \begingroup
    \refstepcounter{relctr}%
    \stackrel{\tiny{(\alph{relctr})}}{\mathstrut{#1}}%
    \originallabel{#2}%
  \endgroup
}
\AtBeginDocument{\let\originallabel\label} 
\AtBeginEnvironment{proof}{\setcounter{relctr}{0}}

\declaretheoremstyle[
  headfont=\color{blue}\normalfont\bfseries,
  bodyfont=\color{blue}\normalfont\itshape,
]{colored}

\declaretheoremstyle[
  headfont=\color{purple}\normalfont\bfseries,
  bodyfont=\color{purple}\normalfont\itshape,
]{pcolored}

\declaretheoremstyle[
  headfont=\color{violet}\normalfont\bfseries,
  bodyfont=\color{violet}\normalfont\itshape,
]{vcolored}

\declaretheoremstyle[
  headfont=\color{red}\normalfont\bfseries,
  bodyfont=\color{red}\normalfont\itshape,
]{rcolored}

\declaretheorem[
  name=Theorem,
]{theorem}

\declaretheorem[
  name=Corollary,
]{corollary}

\declaretheorem[
  name=Lemma,
]{lemma}

\declaretheorem[	
  name=Remark,
]{remark}

\declaretheorem[
  name=Proposition,
]{proposition}

\def\BibTeX{{\rm B\kern-.05em{\sc i\kern-.025em b}\kern-.08em
    T\kern-.1667em\lower.7ex\hbox{E}\kern-.125emX}}

\mathtoolsset{showonlyrefs}
\title{Learning quadratic neural networks in high dimensions:\\ SGD dynamics and scaling laws}

\author{
G\'erard Ben Arous\thanks{New York University. \texttt{benarous@cims.nyu.edu}.},\,\,
Murat A. Erdogdu\thanks{University of Toronto and Vector Institute. \texttt{\{erdogdu,vural\}@cs.toronto.edu}.},\,\,
Nuri Mert Vural\samethanks[2],\,\,
Denny Wu\thanks{New York University and Flatiron Institute. \texttt{dennywu@nyu.edu}.}
\vspace{-3mm}
}

\begin{document}
\maketitle


\begin{abstract}

We study the optimization and sample complexity of gradient-based training of a two-layer neural network with quadratic activation function in the high-dimensional regime, where the data is generated as $f_*(\bs{x}) \propto \sum_{j=1}^{r}\lambda_j \sigma\left(\langle \boldsymbol{\theta_j}, \boldsymbol{x}\rangle\right), \boldsymbol{x} \sim \mathcal{N}(0,\boldsymbol{I}_d)$, $\sigma$ is the 2nd Hermite polynomial, and $\lbrace\boldsymbol{\theta}_j \rbrace_{j=1}^{r} \subset \mathbb{R}^d$ are orthonormal signal directions. We consider the extensive-width regime $r \asymp d^\beta$ for $\beta \in [0, 1)$, and assume a power-law decay on the (non-negative) second-layer coefficients $\lambda_j\asymp j^{-\alpha}$ for $\alpha \geq 0$. 
We present a sharp analysis of the SGD dynamics in the feature learning regime, for both the population limit and the finite-sample (online) discretization, and derive scaling laws for the prediction risk that highlight the power-law dependencies on the optimization time, sample size, and model width. Our analysis combines a precise characterization of the associated matrix Riccati differential equation with novel matrix monotonicity arguments to establish convergence guarantees for the infinite-dimensional effective dynamics. 

\end{abstract}

\allowdisplaybreaks
\section{Introduction} 

We study the problem of learning a two-layer neural network (NN) with quadratic activation on isotropic Gaussian data. The target function (or the ``teacher'' model) is defined as
\begin{equation}\textstyle
y = f_*(\bs{x}) + \epsilon   ~~\text{with}~~  
 f_*(\bs{x}) =
\frac{1}{\norm{\bs{\Lambda}}_\text{F}} \sum_{j=1}^{r} \!
\lambda_j \sigma\left(\langle \boldsymbol{\theta}_j, \boldsymbol{x}\rangle\right) ~\text{and}~  \boldsymbol{x} \sim \mathcal{N}(0,\boldsymbol{I}_d),    \label{eq:target}
\end{equation}
where $\sigma(z) = z^2 - 1$ is the 2nd Hermite polynomial; $\epsilon$ is zero-mean, independent sub-Gaussian noise; $\{ \boldsymbol{\theta}_j \}_{j=1}^{r} \subset \mathbb{R}^d$ are orthogonal signal directions (index features), $\lambda_1 > \lambda_2 > \dots > \lambda_{r} > 0$ are their respective contributions, and $\bs{\Lambda} = \text{diag}(\lambda_1, \cdots, \lambda_r)$ collects the second-layer coefficients. The normalization in front of the sum ensures that the output magnitude remains constant. 
Our goal is to learn this target network using a ``student'' two-layer neural network with quadratic activation and $\rs$ neurons, trained via a gradient-based optimization algorithm. This setting encompasses several well-known problems:
\begin{itemize}[leftmargin=*,itemsep=0.6mm,topsep=0.6mm]
    \item \textit{Phase retrieval ($r =1$)}. The problem of learning one quadratic neuron (i.e., phase retrieval) has been studied extensively~\cite{fienup1982phase,chen2015solving,tan2019online}. The quadratic $\sigma$ has information exponent $k=2$ (defined as the index of the lowest non-zero Hermite coefficient~\cite{dudeja2018learning,benarous2021online}). This entails that randomly initialized parameters are close to a saddle point in high dimensions; hence the SGD dynamics exhibit a plateau (``search'' phase) of length $\log d$ before the loss decreases sharply (``descent'' phase). 
    \item \textit{Multi-spike PCA ($r = \Theta_d(1)$)}. The target function \eqref{eq:target} is a subclass of Gaussian multi-index models, for which various algorithms have been proposed for the finite-rank case $\rs = \Theta_d(1)$ \cite{chen2020learning,damian2022neural,bietti2023learning}. The setting also closely relates to the multi-spike PCA problem, for which online SGD \cite{arous2024high} and other streaming algorithms has been studied \cite{oja1985stochastic,jain2016streaming,allen2017first}. 
    
    \item \textit{Linear-width quadratic NN $(r \asymp d)$.} 
    The regime where the teacher width $\rs$ grows proportionally with dimensionality $d$ has also been studied, typically in the well-conditioned setting (e.g., identical $\lambda_j$'s). Recent works characterized the objective landscape \cite{soltanolkotabi2018theoretical,du2018power,venturi2019spurious,gamarnik2019stationary,ghorbani2019limitations}, optimization dynamics \cite{sarao2020optimization,martin2023impact}, and statistical efficiency \cite{maillard2024bayes,erba2025nuclear}.
\end{itemize}

In this work we focus on the ``extensive-rank'' regime where $r \asymp d^\beta$ for $\beta \in (0, 1)$ and $\rs \asymp d^{\gamma}$ for $\gamma \in [0,1)$, and place a power-law assumption on the second-layer coefficients: $\lambda_j\asymp j^{-\alpha}$ for $\alpha \geq 0$. Our setting is motivated by the following lines of research. 

\begin{figure}[t]
  \centering
  \begin{subfigure}{0.48\textwidth}
    \centering
    \scalebox{0.46}{
\begin{tikzpicture}
  \begin{groupplot}[
    group style={
      group name=taskgrid,
      group size=1 by 5,
      vertical sep=0.55cm
    },
    width=5cm,
    height=3.2cm,
    xmin=0, xmax=5,
    ymin=0, ymax=1.05,
    axis lines=left,
    xtick=\empty,
    ytick={0,0.5,1},
    grid=both
  ]
    \nextgroupplot[
      ylabel={\Large Task 1},
      y label style={rotate=270, anchor=east, xshift=2pt}
    ]
      \addplot[line width=1.3pt,blue,domain=0:1]{1};
      \addplot[line width=1.3pt,blue,domain=1:5]{0};
      \addplot[line width=1.3pt,blue] coordinates{(1,1)(1,0)};
      
    \nextgroupplot[
      ylabel={\Large Task 2},
      y label style={rotate=270, anchor=east, xshift=2pt}
    ]
      \addplot[line width=1.3pt,blue,domain=0:2]{0.5};
      \addplot[line width=1.3pt,blue,domain=2:5]{0};
      \addplot[line width=1.3pt,blue] coordinates{(2,0.5)(2,0)};
      \coordinate (task2anchor) at (rel axis cs:1.02,0.5);

    \nextgroupplot[
      ylabel={\Large Task 3},
      y label style={rotate=270, anchor=east, xshift=2pt}
    ]
      \addplot[line width=1.3pt,blue,domain=0:3]{0.25};
      \addplot[line width=1.3pt,blue,domain=3:5]{0};
      \addplot[line width=1.3pt,blue] coordinates{(3,0.25)(3,0)};

    \nextgroupplot[hide axis]
      \node at (rel axis cs:0.5,0.5){\Huge$\vdots$};

    \nextgroupplot[
    ylabel={\Large Task $r$},
      y label style={rotate=270, anchor=east, xshift=2pt}, xlabel={\Large Time $t$}]
      \addplot[line width=1.3pt,blue,domain=0:4.5]{0.1};
      \addplot[line width=1.3pt,blue,domain=4.5:5]{0};
      \addplot[line width=1.3pt,blue] coordinates{(4.5,0.1)(4.5,0)};
  \end{groupplot}

  \draw[very thick,->,>=latex]
    ($(task2anchor)+(0.25cm,-2cm)$)
    -- ($(task2anchor)+(1.75cm,-2cm)$)
    node[midway,above,yshift=4pt,font=\Large]{Sum};

  \pgfextracty{\TopY}{\pgfpointanchor{taskgrid c1r1}{north east}}
  \pgfextracty{\BotY}{\pgfpointanchor{taskgrid c1r5}{south east}}
  \setlength{\TotalHeight}{\TopY}
  \addtolength{\TotalHeight}{-\BotY}

  \node[anchor=north west,inner sep=0pt]
    at ($(taskgrid c1r1.north east)+(2cm,0)$)
  {%
    \adjustbox{height=1.05\TotalHeight}{\input{loglog.tex}}%
  };
\end{tikzpicture}
        }
    \caption{Additive model hypothesis for scaling laws.}
    \label{fig:left}
  \end{subfigure}
   ~~~~~~\begin{subfigure}{0.42\textwidth}
   \centering
     \includegraphics[width=0.95\textwidth]{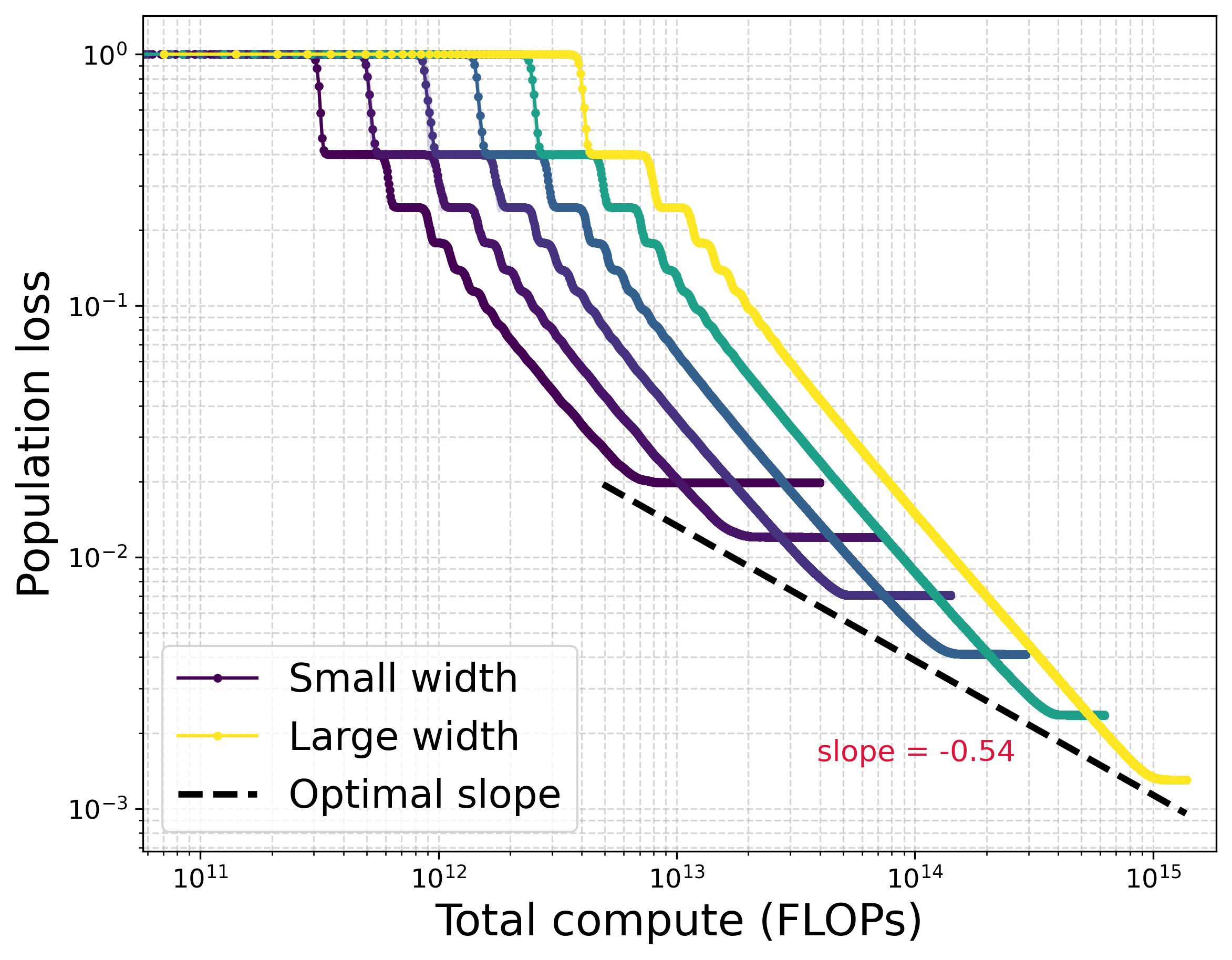}
     \caption{SGD risk curves for quadratic NN.}
    \label{fig:right}
    \end{subfigure}
\caption{\small $(a)$ Illustration of the additive model hypothesis, i.e., sum of emergent learning curves at different timescales yields a power law in the cumulative loss. $(b)$ Population loss vs.~compute for two-layer quadratic NNs trained with online SGD with batch size $d$ on squared loss. We set $d=3200$, and for the teacher model $r=2400$, $\alpha=1$. }
  \label{fig:intro}
\end{figure}

\paragraph{Neural scaling laws \& emergence.} 
Recent empirical studies on large language models (LLMs) reveal that increasing the model or training data size often results in a predictable, power-law decrease in the loss known as \textit{neural scaling laws} \cite{hestness2017deep,kaplan2020scaling,hoffmann2022training}. While such scaling of generalization error has been derived for sketched linear models \cite{maloney2022solvable,bordelon2024dynamical,paquette2024phases,lin2024scaling,defilippis2024dimension}, these analyses assume random projection with no \textit{feature learning}, and hence cannot capture the NN's ability to learn useful features~\cite{girshick2014rich,devlin2018bert} that adapt to the underlying data structure. We aim to investigate a setting where the training of a nonlinear NN beyond the ``lazy'' regime exhibits a nontrivial scaling law.

Feature learning in neural networks is often studied theoretically through the learning of \emph{multi-index models}, where the target function depends on a small number of latent directions (see \cite{bruna2025survey} and references therein). For these low-dimensional targets, it is known that the training dynamics typically exhibit \emph{emergent} (or staircase-like) behavior --- long plateaus followed by sharp drops in loss \cite{benarous2022high,abbe2023sgd}. 
To reconcile this emergent loss curve with smooth power-law decay, recent works hypothesized that the pretraining objective can be decomposed into a sum of losses on individual tasks \cite{michaud2024quantization,nam2024exactly}, the learning of each exhibits a sharp transition, and the superposition of numerous emergent risk curves at different timescales yields a power-law scaling of the cumulative loss (see Figure~\ref{fig:intro}(a)). In this context, the two-layer network \eqref{eq:target} can be viewed as a sum of single-index phase retrieval tasks, where the length of each $\sim\log d$ plateau in the risk trajectory can be modulated by the second-layer coefficient $\lambda_j$. This motivates the following question:

\begin{center}
\textit{\textbf{Q1:} Does gradient-based training of a two-layer quadratic network yield power-law loss scaling, when the target function is an additive model with varying second-layer coefficients $\{\lambda_j\}_{j=1}^{r}$?}
\end{center}
In Figure~\ref{fig:intro}(b) we empirically observe the affirmative: when the target function has smoothly decaying second-layer weights, online SGD training yields a power-law risk curve that resembles the scaling laws in \cite{kaplan2020scaling,hoffmann2022training}. 
The goal of this work is to rigorously establish such scaling laws.

\paragraph{Learning extensive-width neural networks.} Prior works on multi-index models have shown that when $r =\Theta_d(1)$, gradient-based training succeeds with polynomial sample complexity depending on properties of the link function \cite{abbe2022merged,damian2022neural,bietti2022learning}. 
The ``extensive-rank'' regime where $r \asymp d^\beta$ for $\beta>0$ is relatively under-explored (except for the linear width regime $r \asymp d$ \cite{martin2023impact,maillard2024bayes}); this setting is arguably closer to the practical neural network training (compared to the narrow-width setting), and also bears connections to several observations in the LLM literature such as \textit{superposition} \cite{elhage2022toy} and \textit{skill localization} \cite{dai2021knowledge,wang2022finding,panigrahi2023task}, where the model simultaneously acquires a large number of ``skills'' during pretraining (see e.g., \cite{oko2024learning}). 

The learning dynamics of \eqref{eq:target} with divergingly many neurons is challenging to analyze primarily due to the fact that the effective dynamics may not be captured by a finite set of \textit{summary statistics} \cite{benarous2022high} (as in the finite-$r$ case).  Recent works~\cite{oko2024learning,simsek2024learning} addressed this challenge by assuming that the activation $\sigma$ has information exponent $k \geq 3$, which allows the learning dynamics to decouple across feature directions. However, the case $k \le 2$, which includes the quadratic activation studied in this work, remained open: existing analyses either assumed ``isotropic'' feature contributions ($\lambda_1 = \lambda_{r}$) \cite{ren2024learning,simsek2024learning}, or established a computational complexity for SGD that scales with $d^{\Theta(\lambda_1 / \lambda_{r})}$ \cite{li2020learning}, which leads to pessimistic \textit{exponential} dimension dependency in the power-law setting we consider. We therefore ask the following question.
\begin{center}
    \textit{\textbf{Q2:} Can we establish optimization and sample complexity of learning an extensive-width \\ quadratic neural network \eqref{eq:target} with anisotropic, power-decaying feature contributions? } 
\end{center}

\vspace{-3.2mm}

\subsection{Our Contributions}  

We analyze the risk trajectory of learning \eqref{eq:target} with both gradient flow on the mean squared error (MSE) loss and its online SGD discretization on Stiefel manifold, covering the extensive-width and power-law settings. We derive scaling laws for feature recovery and population risk as a function of teacher and student network widths $\rs,r$, the decay exponent $\alpha$, the optimization time, and the sample size (for the discretized dynamics). Our contributions are summarized as follow (see also Table \ref{tab:summary}). 

\begin{enumerate}[leftmargin=*,topsep=0.75mm]
    \item In Section \ref{sec:continuous}, we analyze the population gradient flow and tightly characterize the loss decay with respect to time and the student width $r_s$. We show that the signal directions are recovered sequentially, and the population MSE follows a smooth power law specified by the decay rate $\alpha>0$. 
    \item  In Section \ref{sec:discrete}, we consider the online stochastic gradient descent (SGD) dynamics on the Stiefel manifold and derive scaling laws with respect to sample size. When specializing to the isotropic setting $\alpha=0$, our sample complexity improves upon \cite{ren2024learning} in the extensive-width setting and matches the information theoretic limit (in terms of $d,r$ dependence) up to polylogarithmic factors. 
\end{enumerate}

The following technical challenges in the extensive-width regime are central to our analysis:  
\begin{itemize}[leftmargin=*,topsep=0.75mm]
\item \textit{Coupled population dynamics.} As $r, r_s \to \infty$, we must track infinitely many overlapping student and teacher neurons. \cite{oko2024learning,simsek2024learning} assumed high information exponent $k > 2$, to decouple the dynamics into $r$ independent single-index models, but such property does not hold in our quadratic case ($k = 2$). We address this by leveraging the closed-form solution of the quadratic problem \cite{martin2023impact}, which satisfies a \textit{Matrix Riccati ODE}. A key ingredient in our analysis is its \textit{monotonicity with respect to its initialization}, illustrated in Figures~\ref{fig:nonmonotone}(a),  which enables sharp risk bounds via comparisons to decoupled models.
\item \textit{Operator norm discretization error.} Prior works~\cite{benarous2021online, bietti2023learning, arous2024high} focused on finite-$r$ settings, where Frobenius norm control of the SGD noise was sufficient and natural: it allows bounding error direction-wise without incurring additional dimension dependence. However, in the extensive-width regime, such bounds become pessimistic and lead to suboptimal $r$-dependent rates. Hence we need to establish \emph{operator norm} concentration around the population dynamics.
\item \textit{Matrix-monotone comparison framework.}  
To control discretization error in operator norm, we extend the monotonicity-based argument from the first item to discrete time and introduce a novel comparison-based discretization technique. Our approach constructs matrix-valued reference sequences corresponding to decoupled dynamics that tightly bound the discrete evolution from above and below. This yields sharp operator norm control even when the true trajectories are non-monotone (see Figure~\ref{fig:nonmonotone}), as the analysis avoids relying on the trajectory itself by comparing against simpler bounding sequences.  
\end{itemize}

\begin{table*}[t]  
\captionsetup{singlelinecheck=off}
\begin{center}
\renewcommand{\arraystretch}{1.5}
\scalebox{0.95}{
\begin{NiceTabular}{ccccc}[hlines]
\bhline{1.5pt}
Algorithm & Decay rate ($\lambda_j$) & Risk scaling law & Result \\
\bhline{1.5pt}
\Block{2-1}{Gradient flow} 
& $\alpha>0.5$ & $\bar{t}^{- \frac{2 \alpha - 1}{\alpha}} +  \rs^{- (2 \alpha  - 1)}$  & \Block{2-1}{Theorem \ref{thm:gfresult}}  \\
&  $\alpha<0.5$ &  $(1 - \bar{t}^{\frac{1 - 2\alpha}{\alpha}} )_+  + (1 - (\rs/r)^{1 - 2\alpha})_+$   & \\
\bhline{1pt}                               
\Block{2-1}{Online SGD (Stiefel)} 
& $\alpha>0.5$ & $(\eta\bar{t})^{- \frac{2 \alpha - 1}{\alpha}} +  \rs^{- (2 \alpha  - 1)}$ & \Block{2-1}{Theorem \ref{thm:sgdresult}}  \\
& $\alpha<0.5$ & $(1 - (\eta\bar{t})^{\frac{1 - 2\alpha}{\alpha}} )_+  + (1 - (\rs/r)^{1 - 2\alpha})_+$ & \\
\bhline{1pt}  
\end{NiceTabular} 
}
\renewcommand{\arraystretch}{.1}

 \vspace{-0.5mm}  
\caption[]{\small Scaling laws for learning quadratic neural network \eqref{eq:target} using population gradient flow and its online SGD discretization. We omit constant factors in the risk scaling for ease of presentation. 
\begin{itemize}[leftmargin=*]
\item In $\alpha > 0.5$, for population gradient flow, $\bar{t}\sim t\cdot\log d$ is the rescaled time; for online SGD, $\bar{t}\sim t\cdot\log d$ where $t$ is the number of gradient steps, which is equal to the sample size, and $\eta\sim 1/(d \ \mathrm{poly log} (d))$ is the step size. 
\item In $\alpha< 0.5$, for population gradient flow, $\bar{t}\sim t\cdot r \log d$ is the rescaled time; for online SGD, $\bar{t}\sim t \cdot r \log d$ where $t$ is the number of gradient steps    and $\eta\sim 1/(dr^\alpha \ \mathrm{poly log} (d))$ is the step size. 
\end{itemize}
}
\label{tab:summary} 
\end{center}
\vspace{-2mm}
\end{table*} 

\subsection{Additional Related Works}

\paragraph{Learning multi-index models with SGD.} 
When $r =1$, the target is a \textit{single-index model} with quadratic link function. The SGD learning of single-index models has been extensively studied in the feature learning literature \cite{benarous2021online,ba2022high,mousavi2022neural,ba2023learning,mousavi2023gradient,moniri2023theory,mahankali2024beyond,berthier2024learning,damian2024smoothing,glasgow2025propagation}; while this model has $d$ parameters to be estimated, the quadratic link (with information exponent $k=2$) incurs an additional $\log d$ factor in the complexity of online SGD. More generally, the setting where $r =\Theta_d(1)$ is covered by recent analyses of \textit{multi-index models} \cite{abbe2022merged,abbe2023sgd,bietti2023learning,dandi2023two,collins2023hitting,arous2024high,vural2024pruning,mousavi2024learning}; however, these
learning guarantees for multi-index models typically yield superpolynomial complexity when the target function is rank-extensive. 
The sample complexity of gradient-based learning is also connected to statistical query lower bounds \cite{damian2024computational,dandi2024benefits,lee2024neural,arnaboldi2024repetita}.  

\paragraph{Quadratic NNs and additive models.} Prior theoretical works on learning two-layer neural network with quadratic activation function have studied the loss landscape \cite{soltanolkotabi2018theoretical,du2018power,venturi2019spurious,gamarnik2019stationary,ghorbani2019limitations} and the optimization dynamics 
\cite{sarao2020optimization,arnaboldi2023escaping,martin2023impact,ren2024learning}. 
While existing optimization and statistical guarantees may cover the extensive-width regime (see e.g., \cite{du2018power,martin2023impact,ren2024learning}), to our knowledge, precise scaling laws have not been established in our extensive-rank and power-law setting.
\eqref{eq:target} is also an instance of the \textit{additive model} \cite{stone1985additive,hastie1987generalized,bach2017breaking} where the individual functions are given as (orthogonal) single-index models with \textit{unknown} index features. For this model, \cite{oko2024learning,simsek2024learning} established learning guarantees in the well-conditioned regime, under the assumption that the link function $\sigma$ has information exponent $k>2$. 

\section{Background and Problem Setting}
\label{sec:background}

\subsection{Student-teacher Setting}
 
\paragraph{Teacher Network.}  
We consider the task of learning a teacher network with a quadratic (second-order Hermite) activation function written as
\eq{
y = f_*(\bs{x}) + \epsilon ~~ \text{with} ~~  f_*(\bs{x}) \coloneqq \frac{1}{ \norm{\bs{\Lambda}}_\text{F}} \sum_{j = 1}^{r} \lambda_j  \big( \inner{\bs{\theta}_j}{\bs{x}}^2  -   1 \big)  ~ \text{and}  ~ \bs{x} \sim \cN(0, \bs{I}_d),  \label{eq:target2}
}
where $\bs{x} \in \R^d$ is the input;  $\epsilon$ is zero-mean, independent sub-Gaussian noise; $r$ is the teacher network width; and $\{\bs{\theta}_j\}_{j=1}^{r} \subset \mathbb{R}^d$ is an orthonormal set of unknown signal vectors. We collect these as columns of the matrix $\bs{\Theta} \in \mathbb{R}^{d \times r}$. The contributions of these vectors are determined by the unknown second-layer coefficients $\lambda_1 > \lambda_2 > \dots > \lambda_{r} > 0$ with a power-law decay $\lambda_j  \asymp j^{-\alpha}$ for $\alpha\geq 0$, and $\bs{\Lambda}$ is a diagonal matrix whose $j$-th diagonal entry is $\lambda_j$. 
The normalization in front of summation ensures $\E[y^2]$ is constant. We assume $\epsilon$ is sub-Gaussian for mathematical convenience.
We focus on the regime where $r \asymp d^\beta$ for $\beta \in (0,1)$.

\begin{remark}
The orthogonality of $\{\bs{\theta}_j\}_{j=1}^r$ can be assumed without loss of generality: consider teacher models in the form of \eqref{eq:target2} with arbitrary first-layer weights $\bs{\Theta}$ and normalization $\E[ f_*(\bs{x}) ]=0$, the output can be written as $ f_*(\bs{x}) \propto\mathrm{Tr}\big( \bs{\Theta}\bs{\Lambda}\bs{\Theta}^\top  \bs{x}\bs{x}^\top   \big) + \mathrm{cst}$; hence we may redefine $(\lambda_j,\bs{\theta}_j)$ via the spectral decomposition. 
\end{remark}

\paragraph{Student Network.} We learn the target model with a quadratic student network defined as  
\eq{
\hat y(\bs{x}, \bs{W}) = \frac{1}{ \sqrt{\rs}} \sum_{j = 1}^{\rs}  \inner{\bs{w}_j}{\bs{x}}^2 - \norm{\bs{w}_j}_2^2,
\label{eq:student}
}
where $\rs$ is the width of the student network, and $\{\bs{w}_j\}_{j=1}^{\rs} \subset \mathbb{R}^d$ denotes the set of trainable weights. We collect these weights as the columns of the matrix $\bs{W} \in \mathbb{R}^{d \times \rs}$, and omit the dependence on $\bs{x}$ in $\hat y(\bs{x}, \bs{W})$ when clear from the context. 
Note that the norm subtraction ensures $\mathbb{E}_{\bs{x}}[\hat{y}(\bs{x}, \bs{W})] = 0$. 
We may equivalently write the student network as $\hat y(\bs{x}, \bs{W}) = \frac{1}{ \sqrt{\rs}} \sum_{j = 1}^{\rs}  \norm{\bs{w}_j}_2^2 \cdot (\inner{\bar{\bs{w}}_j}{\bs{x}}^2 - 1)$ where $\bar{\bs{w}}_j$ is unit-norm;
since our student does not have trainable second-layer, 
the norm component $\norm{\bs{w}_j}_2^2$ allows the model to adapt to 
the target second-layer $\lambda_j$; this homogeneous parameterization has been studied in prior works 
\cite{chizat2020implicit,ge2021understanding}.

\subsection{Training Objective}
Training constitutes to minimizing the squared loss; we define the instantaneous loss on $(\bs{x},y)$ as
\eq{\textstyle
\mathcal{L}(\bs{W}; (\bs{x}, y)) \coloneqq \frac{1}{16} \big(y - \hat y(\bs{x}, \bs{W}) \big)^2, 
}
where the prefactor is included for notational convenience in the gradient computation. 
We omit the dependence on $(\bs{x}, y)$ when clear from context. The population risk can be written as
\eq{\textstyle
R(\bs{W}) \coloneqq \mathbb{E}_{(\bs{x}, y)}[ \mathcal{L}(\bs{W}) ] = \frac{1}{8} \big \lVert \frac{1}{\sqrt{\rs}} \bs{W} \bs{W}^\top - \frac{1}{\normL} \bs{\Theta} \bs{\Lambda} \bs{\Theta}^\top \big \rVert_\text{F}^2 + \tfrac{1}{16}\E[\epsilon^2].
\label{eq:poprisk}
}

\paragraph{Alignment.}
Observe that the student network is invariant to right-multiplication of its weight matrix by an orthonormal matrix, i.e., $\hat y(\bs{x}, \bs{W}) = \hat y(\bs{x}, \bs{W} \bs{O})$ for any $\bs{O} \in \R^{\rs \times \rs}$ with $\bs{O}^\top \bs{O} = \bs{I}$. Consequently, any notion of alignment that depends on individual directions in $\bs{W}$ may not be informative. To capture directional learning in a way that respects this symmetry, we define alignment in terms of the subspace spanned by the student weights. We formalize this using the polar decomposition:
\eq{
\bs{W} \coloneqq \bs{U} \bs{Q}^{1/2}, \quad \text{where} \quad \bs{Q} \coloneqq \bs{W}^\top \bs{W} \quad \text{and} \quad \bs{U}^\top \bs{U} = \bs{I}_{\rs}. \label{def:polarcoord}
}
Here, $\bs{Q}$ denotes the radial component of the student weights, while $\bs{U}$ is an orthonormal matrix that encodes their directional component. We quantify the alignment between the student network and the $j$th teacher feature by the squared norm of the projection of $\bs{\theta}_j$ onto the column space of $\bs{W}$:
\eq{
\mathrm{Alignment}(\bs{W}, \bs{\theta}_j) \coloneqq \norm{ \bs{U}^\top \bs{\theta}_j }_2^2. \label{def:alignment}
}
$\mathrm{Alignment}(\bs{W}, \bs{\theta}_j)$ takes values in the interval $[0,1]$;
it is 0 if $\bs{\theta}_j$ is orthogonal to $\bs{W}$ (no alignment),
while it is 1 if $\bs{\theta}_j$ is in the column space of $\bs{W}$ (perfect alignment)\footnote{The definition in \eqref{def:alignment} may fail to converge to $1$ when $\alpha = 0$ and $\rs < r$, due to rotational symmetry in the teacher network. In this case, a more suitable notion of alignment can be defined using the principal angles between the subspaces spanned by $\bs{W}$ and $\bs{\Theta}$, which provides a rotation-invariant characterization of directional overlap. Specifically, for $\alpha = 0$, we define $\mathrm{Alignment}(\bs{W}, \bs{\theta}_j)$ as the $j$th largest eigenvalue of the matrix $\bs{\Theta}^\top \bs{U} \bs{U}^\top \bs{\Theta}$. \vspace{-3mm}}.

\section{Continuous Dynamics: Population Gradient Flow}
\label{sec:continuous}

We first analyze the continuous-time population gradient flow dynamics for \eqref{eq:poprisk}, given as
\eq{
\partial_t \bs{W}_t = - \grad R(\bs{W}_t),  ~~ \text{where} ~  \bs{W}_{0} \in \R^{d \times \rs}, ~  \bs{W}_{0,ij} \sim_{iid} \cN \left(0, 1/d \right), \tag{GF} \label{eq:gf} 
}
and the population gradient reads
\eq{
\grad R(\bs{W}_t) =-\frac{1}{2 \sqrt{\rs} \normL} \left( \Th \bs{\Lambda}\Th^\top - \frac{\normL}{\sqrt{\rs}} \W_t \W_t^\top   \right) \W_t.
}
For notational convenience, we define
\eq{ 
\mathcal{R}(t) \! \coloneqq \!   \textstyle\big \lVert \frac{1}{\sqrt{\rs}} \bs{W}_t \bs{W}_t^\top - \frac{1}{\normL} \bs{\Theta} \bs{\Lambda} \bs{\Theta}^\top \big \rVert_\text{F}^2 ~~ \text{and} ~~ \mathcal{A}(t, \bs{\theta}_j) \! \coloneqq \! \mathrm{Alignment}(\bs{W}_t, \bs{\theta}_j).
}
The following theorem sharply characterizes the timescale for alignment and the limiting risk curve. 

\begin{theorem}
\label{thm:gfresult}
Let $\lambda_j = j^{-\alpha}$ and $r \asymp d^\beta$ for some $\alpha \geq 0$ and $\beta \in (0,1)$. Consider the regime
\eq{\label{eq:rs-scale}
\begin{cases}
\frac{\rs}{r} \to \varphi \in (0,\infty) \hspace*{-0.5em} & \text{and } ~ d \geq \Omega_{\alpha, \beta, \varphi}(1), \quad  \text{if } \alpha \in [0, 0.5), \\[0.75ex]
\rs \asymp 1,                             & \text{and } ~ d \geq \Omega_{\alpha,\rs}(1), \quad \text{~if } \alpha > 0.5.
\end{cases}  
}
Define the effective student width and effective timescale as
\eq{
r_{\mathrm{eff}} \coloneqq \begin{cases}
\lfloor \rs (1 - \log^{\nicefrac{-1}{8}} \! d) \wedge r \rfloor, & \text{if } \alpha \in [0,0.5) \\
\rs, & \text{if } \alpha > 0.5.
\end{cases} ~~ \text{and} ~~ 
\mathsf{T}_{\mathrm{eff}}   \coloneqq \sqrt{\rs} \norm{\bs{\Lambda}}_\mathrm{F} \log \nicefrac{d}{\rs}.
}
Then, the population \eqref{eq:gf} dynamics satisfy the following   with probability $1 - o(1/d^2) - \Omega(1/\rs^2)$:
\begin{enumerate}[leftmargin = *, itemsep=0mm]
    \item \textbf{Alignment:} For $j \leq r_{\mathrm{eff}}$  and $t > 0$ satisfying $t \asymp r^\alpha$ when $\alpha  \in [0,0.5)$ and $t \asymp 1$ when $\alpha > 0.5$, we have
        \eq{
            \mathcal{A}  \big(t  \mathsf{T}_{\mathrm{eff}}, \bs{\theta}_j \big) = \mathbbm{1}\{t \geq  \tfrac{1}{\lambda_j} \} + o_d(1). \label{eq:GF-alignment}
        }
    \item  \textbf{Risk curve:} Under the same time scaling,  
        \eq{
        \mathcal{R} \big(t  \mathsf{T}_{\mathrm{eff}} \big) = 1 -\frac{1}{\normL^2} \sum_{j = 1}^{r_{\mathrm{eff}}} \lambda_j^2   \mathbbm{1}\{t  \geq \tfrac{1}{\lambda_j}\}    + o_d(1). \label{eq:GF-risk}
        }
\end{enumerate}
\end{theorem}

\begin{remark}
We make the following remarks about our result in Theorem \ref{thm:gfresult}:
\begin{itemize}[leftmargin=*]
\item  The spectral decay $\alpha$ determines both the choice of student width $\rs$ and the learning timescale in Theorem \ref{thm:gfresult}. Specifically, when $\alpha > 1/2$ (i.e., light-tailed regime), the coefficients $\{ \lambda_j \}_{j = 1}^r$ are square-summable, making the teacher model effectively finite-dimensional. Hence only finitely many directions need to be learned to achieve small loss, and a timescale of order $\log d$ and finite-width student suffices. In contrast, for the heavy-tailed regime $\alpha < 1/2$, we need to recover linear-in-$r$ directions, which require proportional student width $r_s / r \to \varphi$ and a longer timescale $r \log d$. 
This difference will be made explicit in Corollary~\ref{cor:asympriskcont}. 
\item Theorem \ref{thm:gfresult} verifies the \textit{additive model} hypothesis \cite{michaud2024quantization} for quadratic neural networks in the feature learning regime; specifically, \eqref{eq:GF-alignment} identifies sharp transition time in alignment between student weights and the $j$-th teacher direction, and \eqref{eq:GF-risk} suggests that the cumulative loss can be decomposed into individual emergent risk curves where the timescale is decided by the signal strength $\lambda_j$. 
\end{itemize}

\end{remark}

\paragraph{Neural scaling laws.}
As a corollary of Theorem \ref{thm:gfresult}, we obtain the following risk characterization.
\begin{corollary}
\label{cor:asympriskcont}
By Theorem \ref{thm:gfresult}, the asymptotic risk of \eqref{eq:gf} is given as follows: 
\begin{itemize}[leftmargin = *, itemsep=0mm]
    \item   Heavy-tailed regime ($\alpha \in [0,0.5)$): Almost surely, for all $t > 0$
    $$\mathcal{R}(t r \log d)
    \xrightarrow
    {d \to \infty}
    \big(1 - C t^{\frac{1 - 2 \alpha}{\alpha}}  \big)_+  \vee \big(1 - \varphi^{1 - 2\alpha}\big)_+.$$ 
    \item Light-tailed regime ($\alpha > 0.5$):  With probability  $1 - \Omega(1/\rs)$, for all $t > 0$, the risk $\mathcal{R}(t \log d)$ converges as $d \to \infty$ to a deterministic limit satisfying
    \eq{
    \mathcal{R}(t \log d) \xrightarrow{d \to \infty}
    \Theta\left( t^{- \frac{2\alpha - 1}{\alpha}} + \rs^{-(2\alpha - 1)} \right).
    }
\end{itemize}
\end{corollary}
Corollary \ref{cor:asympriskcont} shows that, over appropriate timescales, the cumulative effect of these emergent transitions yields a smoothly decaying risk curve. Intuitively speaking, the power-law exponent arises from the Riemann integral approximation of the infinite sum \eqref{eq:GF-risk} -- see Appendix~\ref{app:scaling-proof} for details. 

The asymptotic risk behavior in Corollary \ref{cor:asympriskcont} is visualized in Figure \ref{fig:asymptoticriskgf}  (see also Figure~\ref{fig:intro}(b) for empirical simulation). The figure illustrates how the sharp, step-like emergent curve at $\alpha = 0$ (as observed in earlier works on multi-index learning \cite{benarous2021online,abbe2023sgd}) gradually transitions into a smooth curve as $\alpha$ increases. Notably, in the light-tailed regime $\alpha>1/2$, our risk curve resembles the neural scaling laws in \cite{kaplan2020scaling,hoffmann2022training} which takes the form of $\cR \sim 1/(\text{Data size})^a + 1/(\text{Model size})^{b}$, where the data size can be connected to optimization time under the one-pass discretization, which we analyze in the ensuing section.

\begin{figure}[t]
\vspace{-5mm}
    \centering
    \begin{subfigure}[t]{0.56\textwidth}
        \centering
        \includegraphics[width=\textwidth]{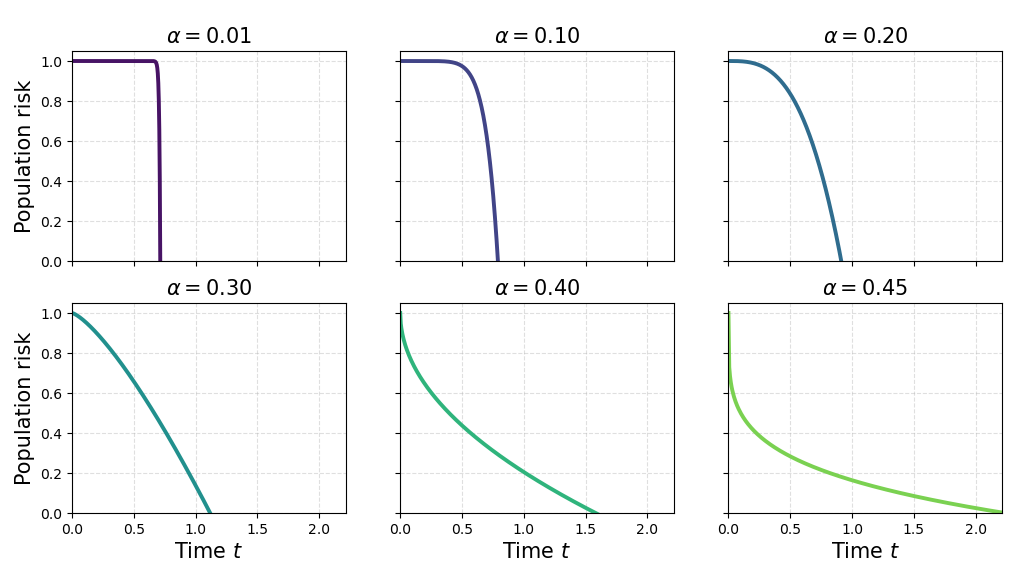}
        \caption{Heavy-tailed regime $(\alpha<1/2)$.}
    \end{subfigure} 
    \begin{subfigure}[t]{0.345\textwidth}
        \centering
        \includegraphics[width=\textwidth]{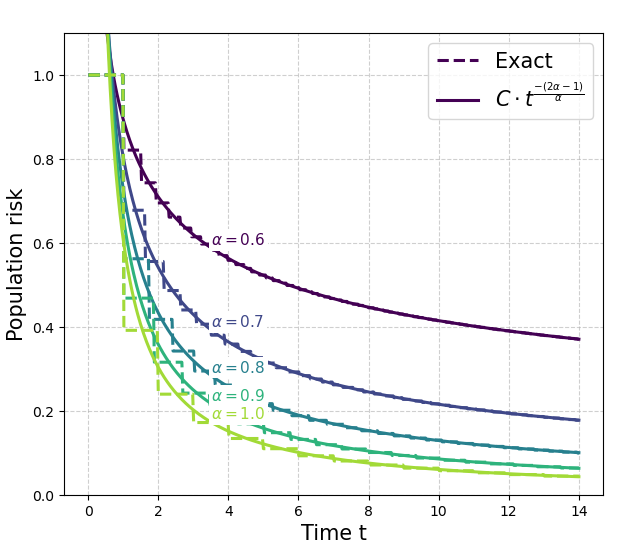}
        \caption{Light-tailed regime $(\alpha>1/2)$.}
    \end{subfigure}
    \caption{\small Illustration of the limiting risk trajectories and scaling behavior given in Corollary \ref{cor:asympriskcont}.}
    \label{fig:asymptoticriskgf}
\end{figure}

\section{Discrete Dynamics: Online Stochastic Gradient Descent}
\label{sec:discrete}
Now we analyze the finite-sample, discrete-time counterpart of the population dynamics \eqref{eq:gf} and establish computational and statistical guarantees. 
We first discretize the directional component of the dynamics via online SGD with Stiefel constraint (see Proposition \ref{prop:sgdrotation}), and then introduce a fine-tuning step with negligible statistical and computational cost to fit the radial component; this mirrors the layer-wise training paradigm commonly used in theoretical analyses of gradient-based feature learning \cite{abbe2022merged,damian2022neural,ba2022high,barak2022hidden}.  
The procedure is summarized in Algorithm~\ref{alg:sgd}.  

\vspace{-1.mm}

\begin{algorithm}[H]
\caption{Online Stochastic Gradient Descent (Stiefel)
\label{alg:sgd}}
\begin{algorithmic}[1] 
\setlength{\itemsep}{0.25ex}
\For{$t = 1,2, \dots$}
    \State $\bs{\widetilde{W}}_{t} = \bs{W}_{t-1} - \eta \nabla_{\mathrm{St}} \mathcal{L}(\bs{W}_{t-1})$  
    \State $\bs{W}_{t} = \bs{\widetilde{W}}_{t} \left( \bs{\widetilde{W}}_{t}^\top \bs{\widetilde{W}}_{t} \right)^{-1/2}$ \Comment{Feature learning}
\EndFor \vspace{0.4em}
\State \label{line:finetuning} $\bs{W}^{\mathrm{final}}_{t} \! = \! \bs{W}_t \bs{\Omega}_*$  where $ \bs{\Omega}_* =  {\displaystyle \argmin_{\bs{\Omega}  \in \R^{\rs \times \rs}} } \sum_{j = 1}^{N_{\mathrm{Ft}}} \mathcal{L}\big(\W_t \bs{\Omega}; (\bs{x}_{t+j}, y_{t+j}) \big)$ \Comment{Fine-tuning}
\end{algorithmic}  
\end{algorithm}
\vspace{-2mm}

In the feature learning step, we update the first-layer weights $\W_t$ to recover the subspace spanned by the teacher directions. To this end, we use online SGD on Stiefel manifold~\cite{arous2024high} with polar retraction.  The Riemannian gradient on the Stiefel manifold is given by:
\eq{\textstyle
\nabla_{\mathrm{St}} \mathcal{L}(\bs{W}_{t-1}) & \coloneqq \nabla  \mathcal{L}(\textstyle\bs{W}_{t-1})- \frac{1}{2}  \bs{W}_{t-1} \left( \bs{W}_{t-1}^\top  \nabla  \mathcal{L}(\bs{W}_{t-1}) + \nabla  \mathcal{L}(\bs{W}_{t-1})^\top  \bs{W}_{t-1}  \right), 
}
where the instantaneous loss is defined for  the sample $(\bs{x}_{t}, y_{t})$.
Since the goal is to ensure subspace alignment as in \eqref{def:alignment}, the overlap of individual student-teacher weights is not relevant during this phase.

After the feature learning phase, we perform a fine-tuning step to rotate $\bs{W}_t$ so that each $\bs{w}_j$ aligns with the corresponding teacher direction $\bs{\theta}_j$. This is achieved by solving an empirical risk minimization problem over $N_{\mathrm{Ft}}$ fresh samples.  The optimal fine-tuning matrix $\bs{\Omega}_*$ admits a closed-form solution that is also numerically easy to compute. Importantly, the computational and statistical complexity of this step scales only quadratically with the student width $\rs$, which is negligible compared to the cost of feature learning. 
The derivation and complexity analysis for this phase are provided in Appendix~\ref{sec:finetuning}.

\begin{remark}
Recall that the stage-wise training procedure is not required in our continuous-time analysis in Section~\ref{sec:continuous}. This is because we employ a Stiefel gradient similar to \cite{bietti2023learning,arous2024high} -- which alone cannot fit the radial component -- to simplify the discretization analysis. 
We conjecture that a standard Euclidean discretization of \eqref{eq:gf} can also achieve the same risk scaling; see Figure~\ref{fig:intro}(b) for empirical evidence. 
\end{remark}

We define the population risk of the output of Algorithm~\ref{alg:sgd}, the alignment with a teacher direction $\bs{\theta}_j$, and the optimal risk achievable by a student neural network with width $\rs$ respectively as
\eq{\textstyle
\mathcal{R}(t) \coloneqq R(\W_t^{\mathrm{final}}), \quad
\mathcal{A}(t, \bs{\theta}_j) \coloneqq \mathrm{Alignment}(\bs{W}_t, \bs{\theta}_j), \quad
\mathcal{R}_{\mathrm{opt}} \coloneqq \frac{1}{\normL^2} \sum_{j = (\rs \wedge r) + 1}^{r} \lambda_j^2.
}
Intuitively, $\cR_{\mathrm{opt}}$ is the risk achieved by exactly fitting the top $r_s\le r$ components of the teacher model. 
Note that the alignment $\mathcal{A}(t, \bs{\theta}_j)$ depends only on the directional component of $\bs{W}_t$; thus, this quantity remains unchanged during fine-tuning. 
The following theorem characterizes the alignment and risk curve for the discrete-time Algorithm~\ref{alg:sgd}.

\begin{theorem}
\label{thm:sgdresult}
Let the parameters $\{\lambda_j\}_{j=1}^r$, $r$,$r_s$, $r_{\mathrm{eff}}$ and $\mathsf{T}_{\mathrm{eff}}$,
and the scaling regime~\eqref{eq:rs-scale}
be as in Theorem~\ref{thm:gfresult}.
Suppose the student weights are initialized uniformly on the Stiefel manifold, and that the step size $\eta$ and fine-tuning sample size  $N_{\mathrm{Ft}}$ satisfy
\eq{
\eta  \asymp \frac{1}{d} \begin{cases}
\frac{1}{r^{\alpha} \log^{C_{\alpha}} (1 + d/\rs)}, & \alpha \in [0,0.5) \\[0.4em] 
\frac{1}{\log^{C_{\alpha}} \! d}, & \alpha > 0.5    
\end{cases}  ~~ \text{and} ~~ N_{\textrm{Ft}} \asymp \rs^2 \log^5 d,
}
for some constant $C_\alpha > 0$ depending only on $\alpha$.
Then with probability $1 - o_d(1/d^2) - \Omega(1/\rs^2)$, 
\begin{enumerate}[leftmargin = *]
    \item \textbf{Runtime and sample complexity:} If 
    \eq{
    T \geq  \begin{cases}
        d r^{1+\alpha} \log^{C_{\alpha} +1}(1 + d/\rs), & \alpha \in [0,0.5) \\
        d \log^{C_{\alpha} +1} d, & \alpha > 0.5.
    \end{cases} \label{eq:samplecomplexity}
    }
    we have
    $
    \mathcal{R}(T) =  \mathcal{R}_{\mathrm{opt}} + o_d(1).
    $
    \item \textbf{Alignment and Risk curve:} For $t > 0$ satisfying $t \asymp r^\alpha/\eta$ when $\alpha  \in [0,0.5)$ and $ t \asymp 1/\eta$ when $\alpha > 0.5$, 
    \[
    \bullet~ \mathcal{A}\big(t \mathsf{T}_{\mathrm{eff}}, \bs{\theta}_j\big) 
    = \mathbbm{1}\{ \eta t \geq \tfrac{1}{\lambda_j}\} + o_d(1).
     \quad\bullet~ \mathcal{R}\big(t \mathsf{T}_{\mathrm{eff}}\big) 
    = 1 - \frac{1}{\normL^2} \sum_{j = 1}^{r_{\mathrm{eff}}} \lambda_j^2 \, \mathbbm{1}\{ \eta t \geq \tfrac{1}{\lambda_j} \} + o_d(1). 
    \]
\end{enumerate}
\end{theorem}
\begin{remark}
We make the following remarks on the sample complexity. 
\begin{itemize}[leftmargin=*]
    \item The bound in \eqref{eq:samplecomplexity} implies a complexity of  $n\asymp T \! \simeq \!  d r^{1+\alpha} \, \mathrm{polylog}(1 + d / r_s)$  in the heavy-tailed case, and $T \simeq d \, \mathrm{polylog}(d)$ in the light-tailed case. Note that due to the one-pass nature of the algorithm, the runtime and sample complexity are identical (up to the negligible fine-tuning step). 
    \item In the light-tailed regime ($\alpha > 1/2$), the required sample size $n \simeq d \, \mathrm{polylog}(d)$ is information theoretically optimal up to logarithmic factors. Note that kernel methods and neural networks in the lazy regime \cite{jacot2018neural,chizat2018note} requires $n\gtrsim d^2$ samples to learn a quadratic target function; thus our sample complexity bound illustrates the benefit of feature learning. 
    \item In the heavy-tailed regime ($\alpha<1/2$), we obtain (nearly) information theoretically optimal sample complexity when $\alpha=0$ (see discussion below). For the intermediate regime $\alpha\in (0,1/2)$, we conjecture that the optimal sample complexity is $T \simeq d r$, which implies our current bound is suboptimal by a factor of $r^{\alpha}$. 
\end{itemize}
\end{remark}

\paragraph{Isotropic Setting $(\alpha=0)$.} 
In the isotropic case, where the goal is to estimate the $r$-dimensional subspace spanned by the teacher weights, the above theorem yields a sample and runtime complexity $n\asymp T \asymp d r \, \mathrm{polylog}(1 + d / r_s)$. This interpolates between the $n \simeq d \, \mathrm{polylog}(d)$ rate for phase retrieval $r = 1$~\cite{tan2019online,benarous2021online}, and $n \simeq d^2$  as $r \to d$, which matches the sample complexity in the linear-width regime \cite{maillard2024bayes,erba2025nuclear}. 
Notably, our $r$-dependence improves upon the recent work of \cite{ren2024learning}, which established a sufficient sample size of $n \gtrsim d \, \mathrm{poly}(r)$ for a similar quadratic setting.  
We expect our result to be optimal up to polylogarithmic factors due to the intrinsic $dr$-dimensional nature of the subspace recovery problem.

\paragraph{Scaling laws in discrete time.} As indicated by the alignment and risk expressions in Theorem~\ref{thm:sgdresult}, a sufficiently small learning rate $\eta$ ensures that running online SGD for $t$ steps closely tracks the population gradient flow trajectory \eqref{eq:gf} at time $\eta t$, exhibiting the same scaling behavior. The following corollary formalizes the discrete-time counterpart of Corollary~\ref{cor:asympriskcont}.

\begin{corollary}
\label{cor:asympriskdis}
We consider $\eta t  \xrightarrow{d \to \infty} t_{\mathrm{c}} > 0$.  By Theorem \ref{thm:sgdresult}, we have
\begin{itemize}[leftmargin = *,noitemsep]
    \item   Heavy-tailed case ($\alpha \in [0,0.5)$): Almost surely,  
    $$\mathcal{R}(t r \log d)
    \xrightarrow{d \to \infty}
    \big(1 - C  t_c^{\frac{1 - 2 \alpha}{\alpha}}  \big)  \vee (1 - \varphi^{1 - 2\alpha})_+.$$ 
    \item Light-tailed case ($\alpha > 0.5$): With probability  $1 - \Omega(1/\rs)$,   $\mathcal{R}(t \log d)$ has an asymptotic limit satisfying 
    $$\mathcal{R}(t \log d)
    \xrightarrow{d \to \infty}
    \Theta  \big(   t_c^{- \frac{2 \alpha - 1}{\alpha}} + \rs^{- (2 \alpha  - 1)}  \big).$$
\end{itemize}
\end{corollary}

\section{Overview of Proof Techniques}
\label{sec:sketch}

To avoid notational confusion between discrete-time and continuous-time dynamics, we adopt the following convention throughout this section. Subscripts (e.g., $\W_t$) denote discrete-time quantities, while parentheses (e.g., $\W(t)$) denote continuous-time trajectories. Specifically, $\{\W_t\}_{t \in \mathbb{N}}$ refers to the iterates of online SGD; $\{\W(t)\}_{t \geq 0}$ denotes the continuous-time gradient flow governed by~\eqref{eq:gf}.

Since our proof strategy heavily relies on the matrix (Loewner) order for symmetric matrices, we introduce the following notations. For symmetric matrices $\bs{G}_1, \bs{G}_2 \in \R^{d \times d}$, we write $\bs{G}_1 \prec \bs{G}_2$ (respectively, $\bs{G}_1 \preceq \bs{G}_2$) if $\bs{G}_2 - \bs{G}_1$ is positive definite (respectively, positive semidefinite). The reverse relations are denoted by $\bs{G}_1 \succ \bs{G}_2$ and $\bs{G}_1 \succeq \bs{G}_2$. Figure~\ref{fig:nonmonotone}(a) illustrates this ordering by comparing the level sets of the quadratic forms $\{\bs{v} : \bs{v}^\top \G_i \bs{v} = 1 \}$ for $i = 1, 2$. In particular, $\G_1 \preceq \G_2$ implies that the level sets of $\G_2$ are strictly contained within those of $\G_1$, as shown by the dashed ellipses.

\vspace{-2.5mm}

\subsection{Proof Sketch of Theorem \ref{thm:gfresult}}
\vspace{-0.5mm}

We first observe that both the population risk $R\big(\W(t)\big)$ and the alignment $\mathrm{Alignment}\big(\W(t), \bs{\theta}_j\big)$ depend on $\W(t)$ through two Gram matrices:  the weight Gram matrix $\G_W(t) \coloneqq \W(t)\W(t)^\top$, and the  alignment Gram matrix $\G_U(t) \coloneqq \Th^\top \U(t)\U(t)^\top \Th$, where $\{\U(t)\}_{t \geq 0}$ denotes the directional component of $\W(t)$, as defined in~\eqref{def:polarcoord}. The proof proceeds by analyzing the evolution of these matrices, each governed by an autonomous ODE; in particular, a matrix Riccati differential equation.   
\begin{proposition}
\label{prop:grammatrixodes}
   The Gram matrices defined above satisfy the following matrix Riccati ODEs:
  \begin{itemize}[leftmargin=*]
      \item  \textbf{Weight Gram matrix:}  $\partial_t \G_{W}(t) \! = \! \frac{0.5}{\normL \sqrt{\rs} }   \Big(\Th \bs{\Lambda} \Th^\top  \G_{W}(t)  +  \G_{W}(t) \Th \bs{\Lambda} \Th^\top - \frac{2 \normL}{\sqrt{\rs}}  \G^2_{W}(t) \Big)$.  \vspace{-1mm}
      \item  \textbf{Alignment Gram matrix:} $\partial_t \G_{U}(t) = \frac{0.5}{\normL \sqrt{\rs} } \big(  \bs{\Lambda}    \G_{U}(t)  + \G_{U}(t)  \bs{\Lambda}  - 2  \G_{U}(t)  \bs{\Lambda} \G_{U}(t) \big)$.
  \end{itemize}
\end{proposition}

Both equations in Proposition~\ref{prop:grammatrixodes} take the form of matrix Riccati ODEs~\cite{Bittanti1991TheRE}, whose structural properties play a central role in the proof. To illustrate the core idea, we focus on the alignment dynamics.   For simplicity, we write $\G(t) \coloneqq\G_{U}(t)$  and consider 
\eq{
\partial_t \G(t) =  \frac{0.5}{\normL \sqrt{\rs} }  \left( \bs{\Lambda} \G(t) +    \G(t) \bs{\Lambda}  - 2 \G(t)  \bs{\Lambda} \G(t) \right). \label{eq:contriccati}
}
Note that $\mathrm{Alignment}\big(\W(t), \bs{\theta}_j\big)$ corresponds to the $j$\textsuperscript{th} diagonal entry of $\G(t)$. To characterize its trajectory, we leverage the monotonicity of the matrix Riccati flow with respect to its initialization, i.e., if $\bs{G}_0^+ \! \succeq \! \bs{G}_0^-$,  the corresponding solutions satisfy $\G(t, \bs{G}_0^+) \succeq \G(t, \bs{G}_0^-)$ for all $t \geq 0$, where $\G(t, \bs{G}_0)$ denotes the solution to~\eqref{eq:contriccati} with initial condition $\bs{G}_0$. Our proof strategy builds on this monotonicity and proceeds as follows:
\begin{enumerate}[leftmargin=*,]
\item \textbf{Diagonalization \& decoupling.} If $\G_0$ is diagonal, the solution $\{ \G(t) \}_{t \geq 0}$ remains diagonal under~\eqref{eq:contriccati}, reducing the dynamics to independent scalar ODEs that govern each diagonal entry. Moreover, each scalar ODE admits a closed-form solution, allowing us to track the evolution of individual alignment terms. 
\item \textbf{Asymptotic characterization.} For general $\bs{G}_0$, we construct diagonal matrices $\bs{G}_0^+ \succeq \bs{G}_0 \succeq \bs{G}_0^-$. By monotonicity, the corresponding trajectories upper and lower bound $\{ \G(t) \}_{t \geq 0}$. These bounding systems are diagonal and decoupled, and as $d \to \infty$, their trajectories converge to the same limit.  
\end{enumerate}

We apply this strategy in Appendix \ref{sec:thmgfresult} to derive the exact asymptotics stated in Theorem~\ref{thm:gfresult}.

\begin{remark}
We remark a conceptual point about the monotonicity of Riccati flow: {while the Riccati flow is monotone with respect to its initialization, this does not imply that its solution is monotone in time.} That is, the trajectory $\G(t)$ may not evolve monotonically in matrix order, even though a larger initialization yields a trajectory that remains above that of a smaller one for all $t \geq 0$. This distinction is illustrated in Figure~\ref{fig:nonmonotone}.
\end{remark}

\begin{figure}[t]
    \centering
    \begin{subfigure}[t]{0.48\textwidth}
        \centering
        \includegraphics[width=0.9\textwidth]{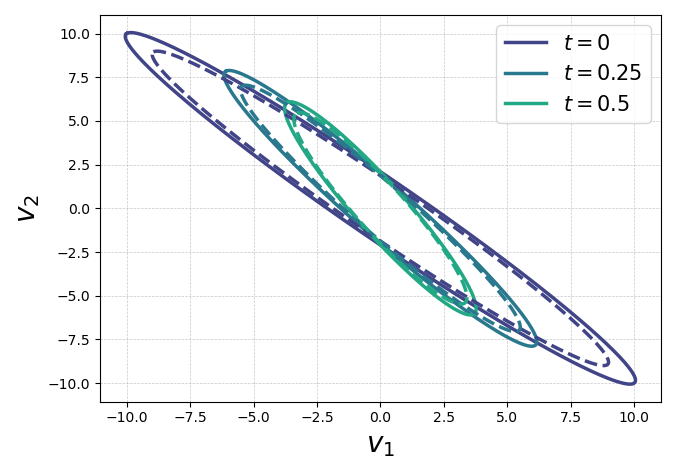}
        \vspace{-2mm}
        \caption{Trajectory in matrix order.}
    \end{subfigure} 
    \begin{subfigure}[t]{0.48\textwidth}
        \centering
        \includegraphics[width=0.9\textwidth]{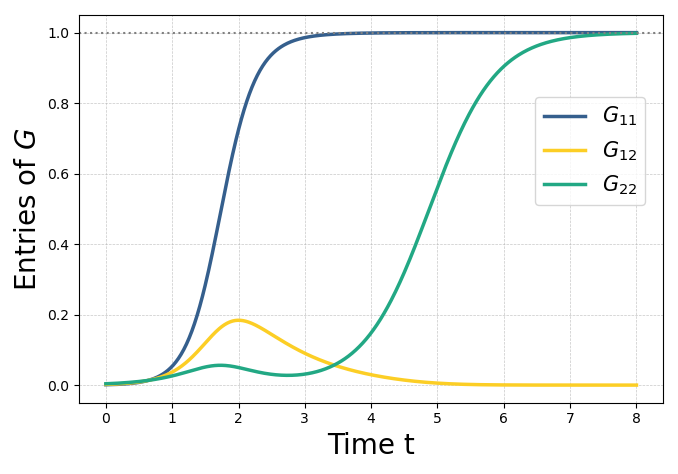}
        \vspace{-2mm}
        \caption{Trajectory of entries of the Gram matrix.}
    \end{subfigure} 
\caption{\small Solutions of the matrix Riccati ODE in \eqref{eq:contriccati} with $\lambda_1 = 2$, $\lambda_2 = 1$, $\rs = 2$.
\emph{(a)} To visualize the dynamics under matrix order, we plot the level sets of $\G(t)$ at times $t \in \{0, 0.25, 0.5\}$ for two initializations: $\G(0)$ (solid) and a scaled version $1.25\,\G(0)$ (dashed). The dashed ellipses remain enclosed within the solid ones at all times, illustrating monotonicity of the Riccati flow \textit{with respect to initialization}. However, note that $\G(t)$ is not monotone in Loewner order over time, as seen from the lack of nesting among the solid ellipses. 
\emph{(b)} Entry-wise evolution of $\G(t)$ under a random initialization with $d = 1024$. The diagonal entry $\G_{22}(t)$ exhibits non-monotonic behavior, illustrating that the solution trajectory $\G(t)$ need not be monotone in time; the off-diagonal entry $\G_{12}(t)$ is also shown for reference.}
\label{fig:nonmonotone}
\end{figure}
 
\subsection{Proof Sketch of Theorem \ref{thm:sgdresult}}
\subsubsection*{Extending Monotonicity Arguments to Discrete Dynamics}

We begin by observing that online SGD on the Stiefel manifold approximates the directional component of the continuous-time gradient flow, with stochastic gradients arising from online sampling. This connection becomes apparent when comparing the discrete and continuous dynamics:
\begin{alignat}{2}
\textbf{SGD on Stiefel:}  ~
& \bs{\widetilde{W}}_{t} \! = \! \bs{W}_{t-1} \! - \! \eta \nabla_{\mathrm{St}} \mathcal{L}(\bs{W}_{t-1})  ~~ \Rightarrow   ~~
&& \textbf{GF on Stiefel:}  ~
  \partial_t \widehat{\W}(t) \! = \! -\nabla_{\mathrm{St}} R\big( \widehat{\W}(t) \big).   \label{eq:continousdirectional} \\
& \bs{W}_{t} \! = \!  \bs{\widetilde{W}}_{t} \left( \bs{\widetilde{W}}_{t}^\top \bs{\widetilde{W}}_{t} \right)^{-1/2}
\ && 
\end{alignat}
The proposition below formalizes the idea that online SGD approximates the directional dynamics of the continuous gradient flow at the population level. For the statement, recall that $\bs{U}(t)$ denotes the directional component of the gradient flow solution $\bs{W}(t)$ from \eqref{eq:gf}, as defined in \eqref{def:polarcoord}. 
\begin{proposition}
\label{prop:sgdrotation}
Let $\widehat{\bs{W}}(t)$ be the solution to the continuous-time gradient flow on the Stiefel manifold defined in \eqref{eq:continousdirectional}, initialized with $\bs{U}(0)$. Then for all $t \geq 0$, the column spaces of $\widehat{\bs{W}}(t)$ and $\bs{U}(t)$ coincide.
\end{proposition}

This result justifies studying the online SGD on Steifel manifold via the directional dynamics of \eqref{eq:gf}.
 To this end, we introduce the discrete analog of $\G(t)$ above as $\G_t = \Th^\top \W_t \W_t^\top \Th$. Extending the analysis to discrete time is non-trivial due to the \emph{loss of monotonicity} in the Euler discretization of the Riccati dynamics~\eqref{eq:contriccati}. In particular, the update
\eq{
\underbrace{ 
\G_{t} \! = \!  \G_{t-1}  \! + \!  \tfrac{0.5 \eta}{\normL \sqrt{\rs}} \left(  \bs{\Lambda} \G_{t-1} \! +  \!\G_{t-1}\bs{\Lambda}   -  2 \G_{t-1} \bs{\Lambda}  \G_{t-1}\right)   }_{\text{non-monotone dynamics}}
+ ~\text{(2nd-order terms and noise)}   \hspace{1.2em}
\label{eq:gndiscrete}
}
no longer preserves the matrix order structure crucial to the continuous-time argument.

To overcome this, we construct an auxiliary discrete system that approximates~\eqref{eq:gndiscrete} up to second-order terms while preserving monotonicity. Specifically, we define the map
\eq{
\G(\G_t ,\eta)   \coloneqq    \bs{G}_t   -  \frac{\eta}{2} (2 \bs{G}_t - \bs{I}_r) \bs{\Lambda}  (2 \bs{G}_t - \bs{I}_r)  \left(   \bs{I}_r \! + \!  \eta \bs{\Lambda} (2 \bs{G}_t - \bs{I}_r )   \right)^{-1}   +   \eta \bs{\Lambda} \hspace{1em}\label{eq:gdnricc}
}
which matches~\eqref{eq:gndiscrete} up to second-order terms. Indeed, expanding the inverse term gives 
\eq{
\textstyle
\G(\G_t , \eta)  & = \underbrace{ \bs{G}_t  - \frac{\eta}{2}  (2 \bs{G}_t - \bs{I}_r) \bs{\Lambda}  (2 \bs{G}_t - \bs{I}_r) + \eta \bs{\Lambda} }_{  = \bs{G}_t +  \eta \left(  \bs{\Lambda} \bs{G}_t + \bs{G}_t\bs{\Lambda}  - 2 \bs{G}_t \bs{\Lambda}  \bs{G}_t\right) } + ~ \text{2nd-order terms}.
}
The key advantage of the iteration~\eqref{eq:gdnricc} is that it preserves matrix order:
\begin{proposition}
\label{prop:monotoneupdatetext}
    For $  \eta > 0$, if $\bs{G}^{+}_t \succeq \bs{G}^{-}_t \succeq 0$, we have   $\G(\G^+_t, \eta) \succeq \G(\G^-_t, \eta)$.
\end{proposition}
We use this to bound the non-monotone dynamics~\eqref{eq:gndiscrete} via monotone iterates. Roughly,  we show that for small enough step size 
$\eta$,  the following holds: 
\eq{
\G \big(\G_{t-1},   (1 + \varepsilon) \lr \big) +  \text{ Noise}   \succeq  \bs{G}_{t}   \succeq  \G \big(\G_{t-1},  (1 - \varepsilon) \lr \big) +   \text{ Noise}  \label{eq:boundingdiscrete} 
}
for some $\varepsilon = o_d(1)$,  where we denote the effective learning rate in~\eqref{eq:gndiscrete} with $\lr = \tfrac{\eta}{\sqrt{\rs} \normL}$.  We then follow the same bounding argument used in the continuous case by defining the upper and lower reference sequences, 
\eq{
\G_{t}^\pm  = \G \big(\G^{\pm}_{t-1},  (1 \pm \varepsilon)\lr \big) + \text{Noise}, ~~ \text{where} ~~ \G_0^+ \succeq \G_0 \succeq \G_0^- \succeq 0,
}
show that $\G_t^+ \succeq \G_t \succeq \G_t^-$ for all $t \in \N$.
Finally, by choosing $\G^{\pm}_0$ to be diagonal, the bounding dynamics reduce to decoupled scalar recursions, which can be analyzed explicitly. This allows us to establish concentration of the original iterates $\{ \G_{t} \}_{t \in \N}$ around the bounding sequences, leading to operator-norm convergence of the discrete-time dynamics to their continuous-time counterparts. See Appendix~\ref{sec:defs} for full argument.

\subsubsection*{Risk Decomposition for Fine Tuning}

The fine-tuning step relies on the following decomposition of the population risk:

\begin{proposition}
\label{prop:riskdecomptext}
    For any $\bs{\Omega} \in \R^{\rs \times \rs}$, the population risk defined in \eqref{eq:poprisk} can be written as:
    \eq{
R(\W_t \bs{\Omega}) =  \frac{1}{\rs}  \norm*{ \bs{\Omega} \bs{\Omega}^\top - \tfrac{\sqrt{\rs}}{\normL}  \W_t^\top \Th \bs{\Lambda} \Th^\top \W_t }_F^2 + \frac{1}{\normL^2} \Big( \normL^2 -   \norm{\bs{\Lambda}^{\frac{1}{2}}  \G_t \bs{\Lambda}^{\frac{1}{2}}}_F^2   \Big) ,
}
where  $\G_t = \Th^\top \W_t \W_t^\top \Th$ is the discrete alignment Gram matrix defined in the previous part.
\end{proposition}

We observe  that both the second term and the matrix $\bs{W}_t^\top \bs{\Theta} \bs{\Lambda} \bs{\Theta}^\top \bs{W}_t$ are independent of $\bs{\Omega}$. Hence, the fine-tuning step reduces to a least squares problem in the matrix $\bs{\Omega} \bs{\Omega}^\top$ in population, which is approximated via empirical risk minimization over a fresh batch of samples. By standard concentration arguments, a sample size of $N_{\mathrm{Ft}} \geq \rs^2 \mathrm{polylog} d$ suffices to ensure that the empirical minimizer approximates the population solution with high probability. Full details are provided in Appendix~\ref{sec:finetuning}.

\section{Conclusion}
\label{sec:concl}

In this work, we presented a comprehensive theoretical analysis of gradient-based learning in high-dimensional, extensive-width two-layer neural networks with quadratic activation.  We established precise scaling
laws that characterize both the population gradient flow and its empirical, discrete-time approximation.
These results demonstrate how anisotropic signal strengths in the target function fundamentally shapes the
convergence behavior and sample efficiency of gradient-based learning. 

\paragraph{Beyond quadratic activations.} 
An immediate direction for future research is to extend our analysis to more general activation functions. Link functions with higher information exponent is studied in a companion work \cite{ren2025emergence}, where the precise risk scaling is established by exploiting a decoupling structure that is unique to the information exponent $k>2$ setting. 
Importantly, many commonly-used activation functions (ReLU, GeLU, etc.) have information exponent $k=1$ and also contain a nonzero $\mathrm{He}_2$ component. For such nonlinearities, we conjecture that SGD dynamics exhibits a multi-phase risk curve (analogous to the incremental learning phenomenon in \cite{abbe2023sgd,bietti2023learning}), where the higher Hermite modes affects the learning dynamics after the low-order terms are learned. In Figure \ref{fig:scalinglaw-relu} we report the SGD risk curves for ReLU networks, in which we observe $(i)$ an initial loss drop driven by the $\mathrm{He}_1$ component (which finds a degenerate rank-1 subspace), followed by $(ii)$ a power-law decay phase driven by the quadratic $\mathrm{He}_2$ component where the empirical scaling exponent align closely with our theoretical predictions, and finally $(iii)$ a slope change late in training likely due to higher Hermite terms (in Figure~\ref{fig:scalinglaw-mixture} we confirm that this ``late'' phase is absent if we remove these higher-order components). Understanding such complex multi-phase learning dynamics remains an interesting challenge for future work.

\begin{figure}[!htb]
  \begin{subfigure}{0.45\textwidth}
    \centering
    \includegraphics[width=0.95\textwidth]{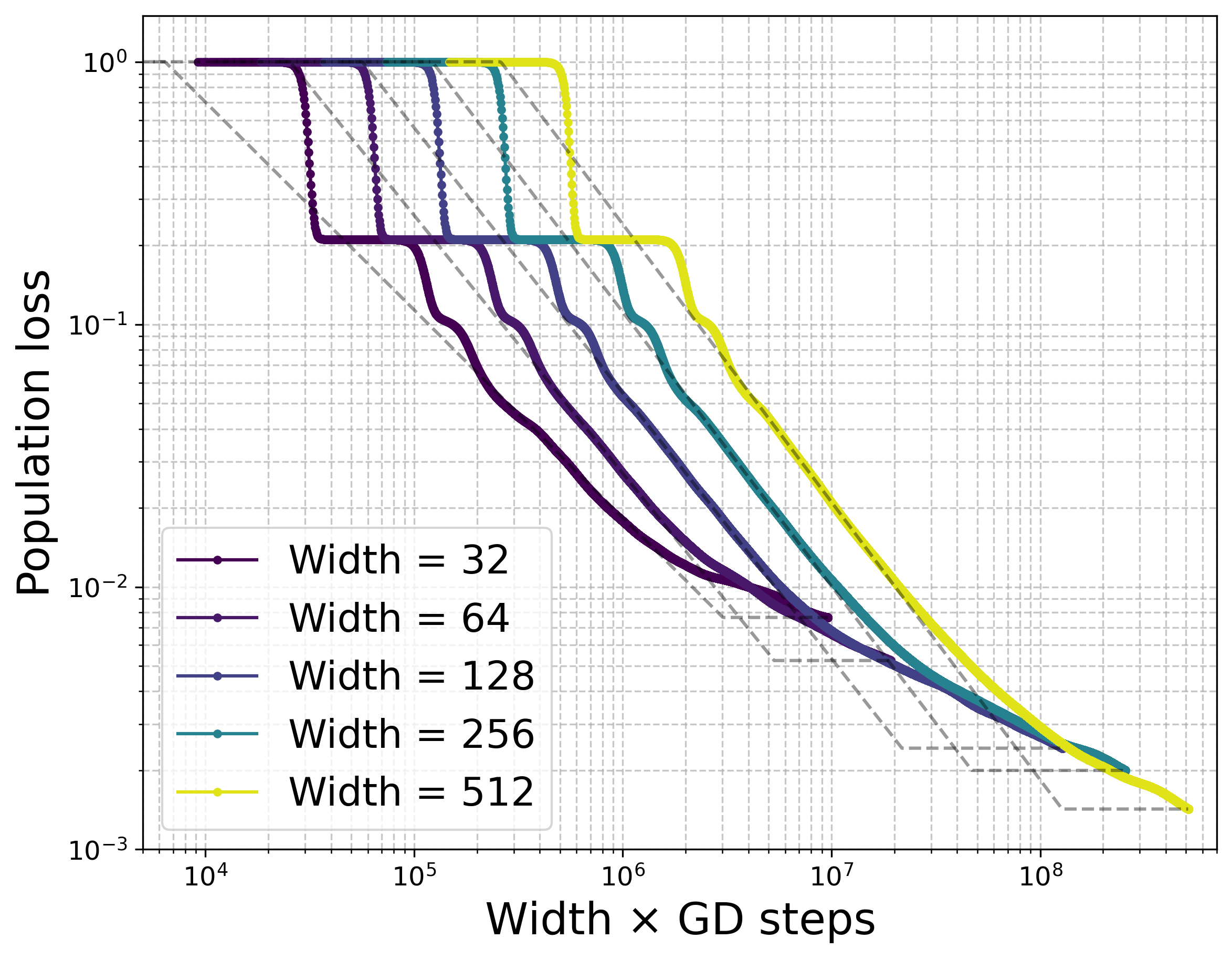}
    \caption{$\alpha = 1$. Ideal exponent: $-1$; empirical:  $-1.01$.\!\!}
    \label{fig:relua10}
  \end{subfigure}
   ~~~~~~\begin{subfigure}{0.45\textwidth}
   \centering
     \includegraphics[width=0.95\textwidth]{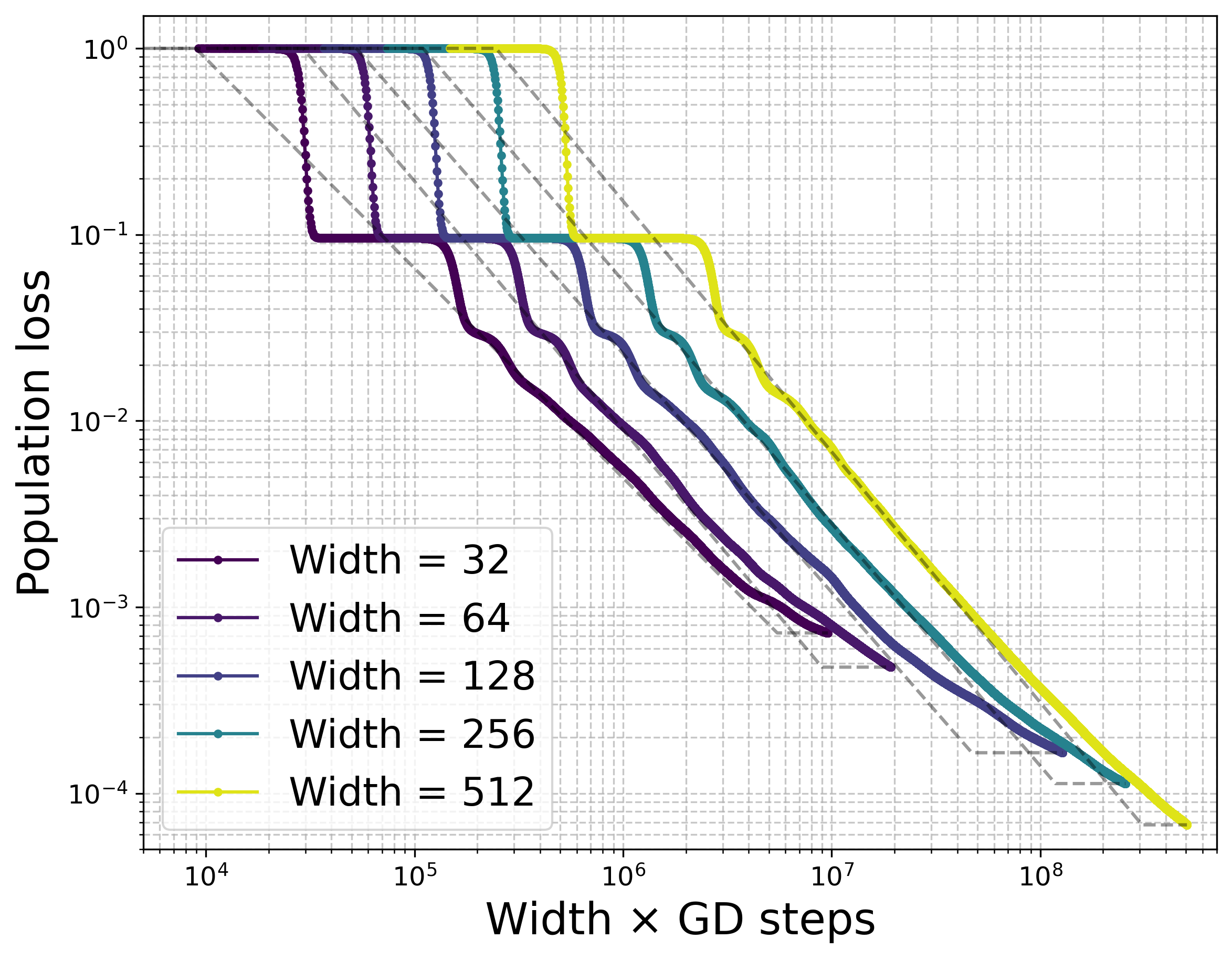}
     \caption{$\alpha = \frac{3}{2}$. Ideal exponent; $-\frac{4}{3}$, empirical:: $- 1.26$.\!\!}
    \label{fig:relua15}
    \end{subfigure} 
\caption{\small  Population loss vs.~compute for two-layer ReLU network (power-law second-layer with exponent $\alpha$) trained with population gradient descent. The student network adopts the 2-homogeneous parameterization as in \eqref{eq:student}. Observe that after the initial loss drop due to the $\mathrm{He}_1$ component, the risk curves follow a power-law scaling where the exponent (dashed lines) nearly matches our theoretical prediction for the quadratic setting $\frac{1-2\alpha}{\alpha}$.}
  \label{fig:scalinglaw-relu} 
\end{figure}

\bigskip

\subsection*{Acknowledgment}

The authors thank Florent Krzakala, Jason D.~Lee, and Lenka Zdeborov{\'a} for discussion and feedback. 
The research of GBA was supported in part by NSF grant 2134216. 
MAE was partially supported by the
NSERC Grant [2019-06167], the CIFAR AI Chairs program, and the CIFAR Catalyst grant. 
Part of this work was completed when NMV interned at the Flatiron Institute. 

\bigskip

{


\bibliography{ref}  
\bibliographystyle{alpha} 

\newpage

}

\appendix

\tableofcontents

\newpage

\allowdisplaybreaks

\section{Additional Figures}
\label{app:experiment}

\begin{figure}[!hbt]
    \centering
    \begin{subfigure}[t]{0.48\textwidth}
        \centering
        \includegraphics[width=0.9\textwidth]{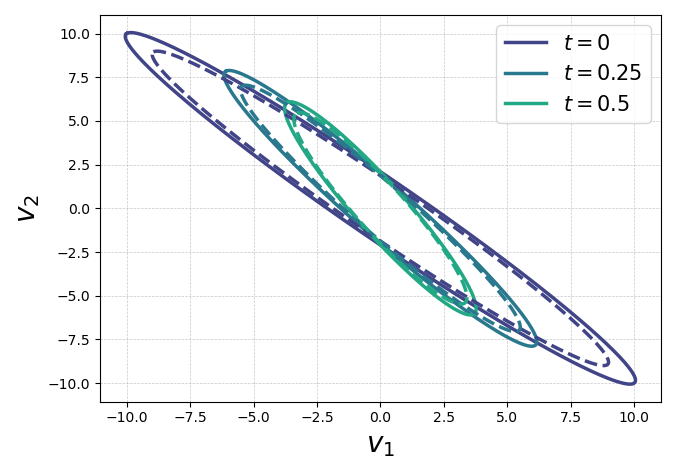}
        \vspace{-2mm}
        \caption{Trajectory in matrix order.}
    \end{subfigure} 
    \begin{subfigure}[t]{0.48\textwidth}
        \centering
        \includegraphics[width=0.9\textwidth]{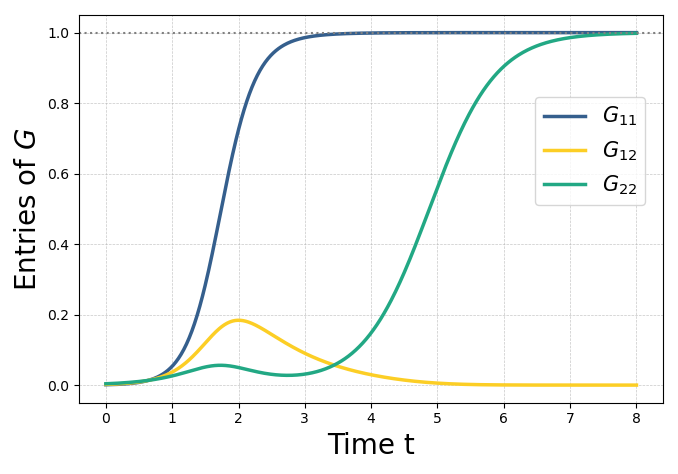}
        \vspace{-2mm}
        \caption{Trajectory of entries of the Gram matrix.}
    \end{subfigure} 
\caption{\small Solutions of the matrix Riccati ODE in \eqref{eq:contriccati} with $\lambda_1 = 2$, $\lambda_2 = 1$, $\rs = 2$.
\emph{(a)} To visualize the dynamics under matrix order, we plot the level sets of $\G(t)$ at times $t \in \{0, 0.25, 0.5\}$ for two initializations: $\G(0)$ (solid) and a scaled version $1.25\,\G(0)$ (dashed). The dashed ellipses remain enclosed within the solid ones at all times, illustrating monotonicity of the Riccati flow \textit{with respect to initialization}. However, note that $\G(t)$ is not monotone in Loewner order over time, as seen from the lack of nesting among the solid ellipses. 
\emph{(b)} Entry-wise evolution of $\G(t)$ under a random initialization with $d = 1024$. The diagonal entry $\G_{22}(t)$ exhibits non-monotonic behavior, illustrating that the solution trajectory $\G(t)$ need not be monotone in time. 
}
\label{fig:nonmonotone}
\end{figure}

\begin{figure}[!htb]
  \centering
   \begin{subfigure}{0.45\textwidth}
    \centering
    \includegraphics[width=0.95\textwidth]{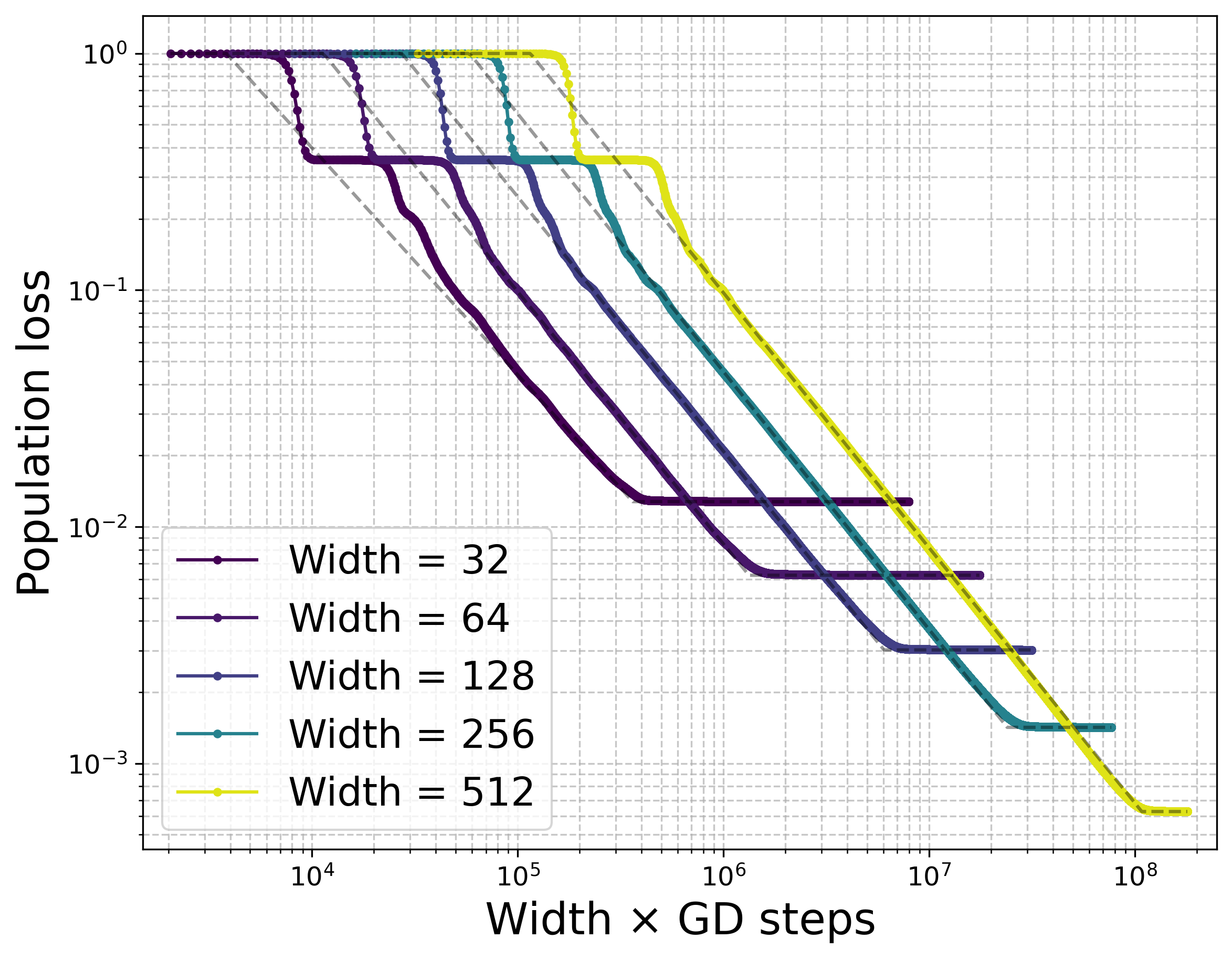}
    \caption{$\alpha = 1$. Ideal: $-1$, empirical slope:  $-1.08$.\!\!}
    \label{fig:h12a1}
  \end{subfigure}
   ~~~~~~\begin{subfigure}{0.45\textwidth}
   \centering
     \includegraphics[width=0.95\textwidth]{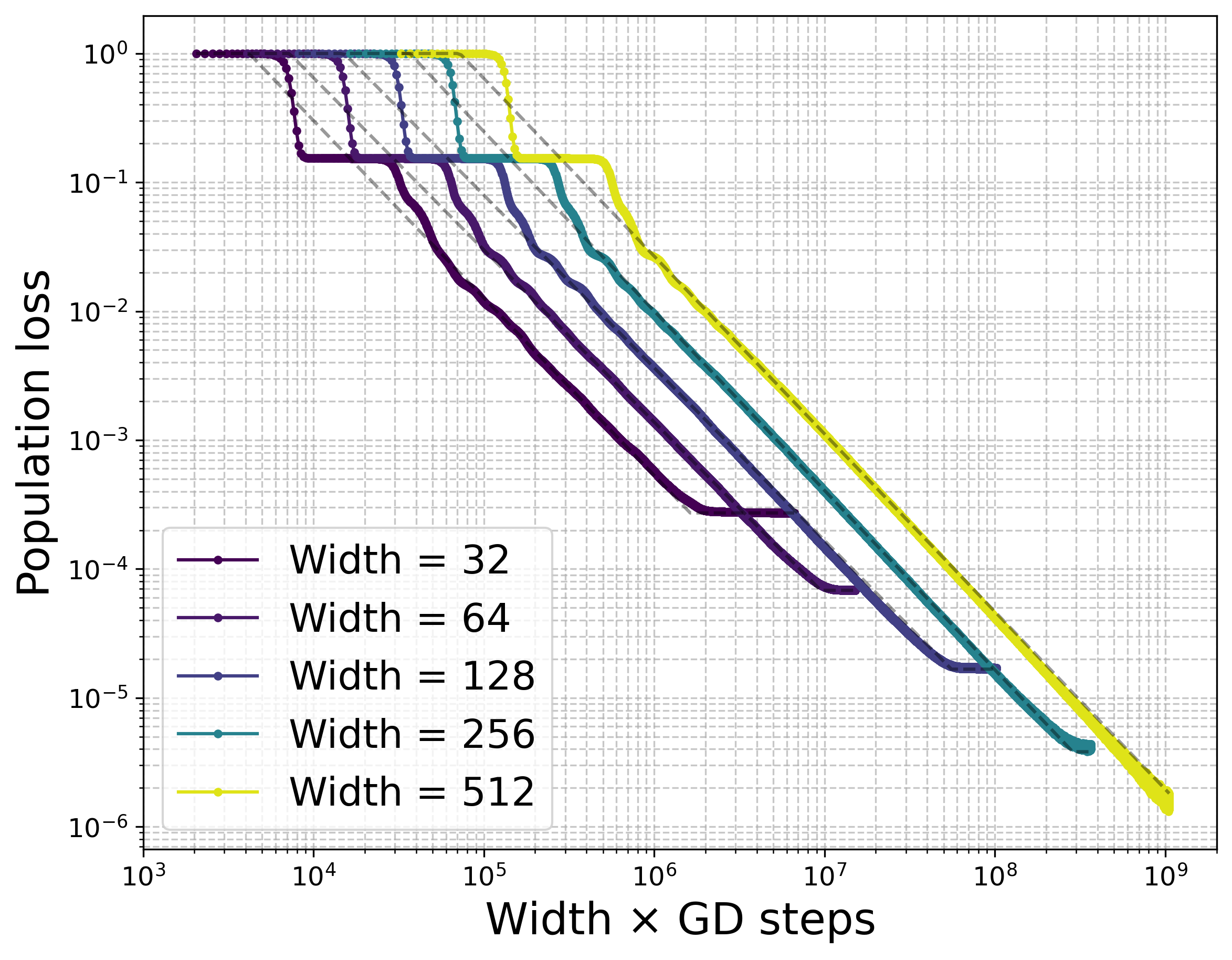}
     \caption{$\alpha = \frac{3}{2}$. Ideal: $- \frac{4}{3}$, empirical slope:  $-1.38$.\!\!}
    \label{fig:h12a15}
    \end{subfigure}

\caption{\small Population loss vs.~compute for two-layer neural network with activation function $\sigma \propto \mathrm{He}_1 + \mathrm{He}_2$, trained with population gradient descent. The student network adopts the 2-homogeneous parameterization as in \eqref{eq:student}. Observe that after the initial loss drop due to the $\mathrm{He}_1$ component, the risk curves exhibit a power-law scaling where the exponent (dashed lines) nearly matches our theoretical prediction for the quadratic setting $\frac{1-2\alpha}{\alpha}$; and unlike the ReLU setting (Figure \ref{fig:scalinglaw-relu}), the loss immediately plateaus after the power-law phase. } 
  \label{fig:scalinglaw-mixture} 
\end{figure}

\paragraph{Experiment Setting.}  In Figures \ref{fig:scalinglaw-relu}  and \ref{fig:scalinglaw-mixture}, we plot the mean squared error loss for gradient descent with a constant step size on the population loss, using activations $\sigma = \mathrm{ReLU}$ and $\sigma = \mathrm{He}_1 + \mathrm{He}_2$. The teacher model has orthogonal first-layer neurons and power-law decay in the second-layer coefficients with $\alpha \in \{1.0, 1.5\}$. Both teacher and student networks use the same activation function, which we normalize to have zero-mean an unit $L^2$ norm. The student network uses the $2$-homogeneous parameterization:
\eq{
\hat y(\bs{W}) = \frac{1}{\sqrt{\rs}} \sum_{ i = 1}^{\rs} \norm{\bs{w}_i}_2^2 \cdot \sigma(\inner{\bs{w}_i}{\bs{x}}) ~~ \text{where} ~~ \sigma \in  \{ \textstyle \frac{\mathrm{ReLU} - 1/\sqrt{2\pi}}{0.5},  \frac{\mathrm{He}_1 + \mathrm{He}_2}{3}  \}. 
}
We set dimension $d = 5000$, number of teacher neurons $r = 2400$, student widths $\rs \in \{32, 64, 128, 256, 512\}$, and learning rate $\eta =  0.5/\sqrt{r}$.   To estimate the scaling exponents, we first identify the range of compute exhibiting a linear trend by visual inspection, and then fit the exponent via least squares. The dashed lines in the plots correspond to these fitted lines, and the reported empirical exponents represent the median values across different student widths. 

\bigskip

\section{Preliminaries for Proofs}
 \label{sect:prelimproof}
\smallskip

\textbf{Proof organization.}  Section~\ref{sect:prelimproof} introduces the notations and definitions used throughout the paper. In Section~\ref{app:background}, we provide a brief review of matrix Riccati ODEs and difference equations, along with the necessary supporting statements. The main results are proved in Section~\ref{sec:mainproofs}. In Section~\ref{sec:finetuning} we discuss the fine-tuning phase for the discretized algorithm. Additional proofs related to online SGD and auxiliary lemmas are deferred to Sections~\ref{sec:onlinesgd} and~\ref{sec:auxstats}, respectively.

\smallskip
\textbf{Notation and Definitions.}
We use $[n] \coloneqq \{1, 2, \ldots, n\}$ to denote the first $n$ natural numbers. The Euclidean inner product and norm are denoted by $\inner{\cdot}{\cdot}$ and $\norm{\cdot}_2$, respectively.  For matrices, $\norm{\cdot}_2$ and $\norm{\cdot}_\text{F}$ denote the operator norm and Frobenius norm.  The positive part is denoted by $(x)_+ \coloneqq \max\{x, 0\}$. We write $f_d = o_d(1)$ if $f_d \to 0$ as $d \to \infty$, and $f_d \ll g_d$ if $f_d/g_d \to 0$. We use $O(\cdot)$ or $\Omega(\cdot)$ to suppress constants in upper and lower bounds respectively, and we use subscript to indicate parameter dependence, e.g., $O_\alpha(\cdot)$.

\smallskip
 The symmetric part of a square matrix $\bs{M} \in \R^{d \times d}$ is given by $\sym(\bs{M}) \coloneqq \frac{1}{2}(\bs{M} + \bs{M}^\top)$. For symmetric matrices $\bs{A}, \bs{B} \in \R^{d \times d}$, we write $\bs{A} \prec \bs{B}$ (or $\bs{A} \preceq \bs{B}$) if $\bs{B} - \bs{A}$ is positive definite (or positive semidefinite). 
 Moreover, if $\bs{A}$ and $\bs{B}$ are mutually diagonalizable, we write $\bs{A} \bs{B}^{-1} = \frac{\bs{A}}{\bs{B}}$.
 
\smallskip
We follow the convention that subscripts (e.g., $\W_t$) refer to discrete-time quantities, and parentheses (e.g., $\W(t)$) refer to continuous-time quantities. The overlap matrices of interest are defined as
\eq{
\underbrace{ \G_W(t) \coloneqq \W(t) \W(t)^\top }_{\text{Weight Gram matrix}}, \quad 
\underbrace{\G_U(t) \coloneqq \Th^\top \bs{U}(t) \bs{U}(t)^\top \Th }_{\text{Alignment Gram matrix}}, \quad 
\underbrace{\G_t \coloneqq \Th^\top \W_t \W_t^\top \Th. }_{\text{Discrete alignment Gram matrix}}
}

Let $\Z \in \R^{d \times \rs}$ be a Gaussian matrix with i.i.d.\ entries distributed as $\mathcal{N}(0,1/d)$. We define $\Z_{1:m}  \in \R^{m \times \rs}$ as the submatrix formed by the first $m$ rows:
\eq{
\Z = \begin{bmatrix}
\Z_{1:m} \\
\Z_{\text{rest}}
\end{bmatrix}.
}
Without loss of generality, we assume the teacher directions coincide with the standard basis vectors, i.e., $\bs{\theta}_j = \bs{e}_j$. With this, the initialization satisfies:
\eq{
\G_W(0) = \Z \Z^\top, \qquad 
\G_U(0) = \G_0 = \Z_{1:r} (\Z^\top \Z)^{-1} \Z_{1:r}^\top.  \label{def:init}
}
We start with characterizing ``good events'' for initial matrices given by the following lemma:
\begin{lemma}
\label{lem:goodevents}
\hfill \\[0.6em]
\textbf{Both cases ($\alpha \in [0,0.5) \cup (0.5,\infty)$).}  For $d \geq \Omega(1)$, the following holds:
\begin{enumerate}[label=(E.\arabic*)]
\item   \label{event:Zbound}    $ \tfrac{1}{1.05} \leq \lambda_{\mathrm{min}} (\Z^\top \Z  ) \leq \lambda_{\mathrm{max}} (  \Z^\top \Z) \leq 1.05$.
\item   \label{event:Gbound}  $  1.05 \Z_{1:r} \Z_{1:r}^\top \succeq \G_U(0) = \G_0 \succeq       \frac{1}{1.05} \Z_{1:r} \Z_{1:r}^\top$.
\end{enumerate}
\textbf{Heavy-tailed case ($\alpha \in [0,0.5)$).}   For $d \geq \Omega_{\varphi}(1)$, the following holds:
\begin{enumerate}[label=(H.\arabic*)]
\item \label{event:htZlb}  For  $m  \leq  \rs (1 -   \log^{\nicefrac{-1}{2}} d) \wedge r$ uniformly,  we consider $\lambda_{\mathrm{min}} (  \Z_{1:m} \Z_{1:m}^\top ) \geq  \tfrac{\rs}{5 d} \big( 1 -  \tfrac{m}{\rs}   \big)^2 $.
\item \label{event:htZlbiso}  For all $m  \leq  \rs (1 -   \log^{\nicefrac{-1}{2}} d) \wedge r$ uniformly,  $\lambda_{m}(    \Z_{1:r} \Z_{1:r}^\top ) \geq  \tfrac{\rs}{5 d} \big( 1 -     \tfrac{m}{\rs}    \big)^2$.
\item \label{event:htZub}    $\lambda_{\mathrm{max}} ( \Z_{1:r} \Z_{1:r}^\top ) \leq \frac{2\rs}{d} \big(1 + \frac{1}{\sqrt{\varphi}} \big)^2$.
\end{enumerate}
\textbf{Light-tailed case ($\alpha \in [0,0.5)$).}   For $d \geq \Omega(1)$, the following holds:
\begin{enumerate}[label=(L.\arabic*)]
\item   \label{event:ltZlb}  $\tfrac{1}{\rs^5 d}  \leq  \lambda_{\mathrm{min}} (  \Z_{1:\rs} \Z_{1:\rs}^\top )$
\item   \label{event:ltZub}   For   $m  \in \{ 1, 2, \cdots, 5 \rs,  \ceil{\log^{2.5} d},\ceil{\log^{6} d}, r \}$ uniformly,  $ \lambda_{\mathrm{max}} ( \Z_{1:m} \Z_{1:m}^\top ) \leq  \tfrac{5(\rs\vee m)}{d}$
\end{enumerate}
We define
\eq{
\mathcal{G}_{\text{init}} \equiv \begin{cases}
  \ref{event:Zbound}  \cap \ref{event:Gbound} \cap   \ref{event:htZlb}  \cap  \ref{event:htZlbiso}   \cap   \ref{event:htZub} , & \alpha \in [0,0.5) \\
    \ref{event:Zbound}  \cap \ref{event:Gbound} \cap   \ref{event:ltZlb} \cap    \ref{event:ltZub} , & \alpha \in > 0.5.
\end{cases}
}
We have 
\eq{
\mpr[ \mathcal{G}_{\text{init}} ] \geq \begin{cases}
 1 - 3  \rs \exp \big( \tfrac{- \rs}{2 \log^2  d} \big)  , & \alpha \in [0,0.5) \\
1 - \Omega(1/\rs^2), & \alpha  > 0.5.
\end{cases}
}
\end{lemma}
 
 \begin{proof}
We will use the following:
\begin{enumerate}[label=(S.\arabic*), leftmargin=*]
\item   \label{cond:lowerbound} By \cite[Corollary 5.35]{vershynin2010introduction},   for $m \leq \rs$ and $\sqrt{\rs} - \sqrt{m} \geq t > 0$ 
\eq{
& \mpr \left[      \lambda_{\mathrm{min}} \left(  \bs{Z}_{1:m} \bs{Z}_{1:m}^\top  \right) \! \geq \!   \tfrac{\rs}{d}  \left( 1 \! - \! \sqrt{  \tfrac{m}{\rs}  } \! - \! \tfrac{t}{\sqrt{\rs}} \right)^2 \right] \geq 1 - 2 e^{- t^2}, 
}
and   for  $m \geq \rs$ and $\sqrt{m} - \sqrt{\rs} \geq t > 0$ 
\eq{
& \mpr \left[      \lambda_{\mathrm{min}} \left( \bs{Z}_{1:m}^\top  \bs{Z}_{1:m}   \right) \! \geq \!   \tfrac{m}{d}  \left( 1 \! - \! \sqrt{  \tfrac{\rs}{m}  } \! - \! \tfrac{t}{\sqrt{m}} \right)^2 \right] \geq 1 - 2 e^{- t^2}.  
}

\item  \label{cond:upperbound}	By \cite[Corollary 5.35]{vershynin2010introduction},  for any fixed $m$  
\eq{
\mpr \left[  \tfrac{m}{d} \left( 1 + \sqrt{  \tfrac{\rs}{m}  } + \tfrac{t}{\sqrt{m}} \right)^2 \geq \lambda_{\mathrm{max}} \left(  \bs{Z}_{1:m}^\top   \bs{Z}_{1:m} \right)   \right] \geq 1 - 2 e^{- t^2}.
}
\item  \label{cond:squarematrix} By \cite[Theorem 5.38]{vershynin2010introduction}, there exists $C, c > 0$ such that
\eq{
\mpr \left[     \lambda_{\mathrm{min}} \left(  \bs{Z}_{1:\rs} \bs{Z}_{1:\rs}^\top  \right)\geq \tfrac{  \varepsilon^2}{4 d \rs}  \right] \geq 1 -  C \varepsilon -  e^{- c \rs}.
}
\end{enumerate}
For the heavy tailed case,  we  consider $d$ is large enough to guarantee  $\abs{\tfrac{\rs}{r} - \varphi} \leq \frac{\varphi}{2}.$ We have
\begin{itemize}[leftmargin=*]
\item  By using \ref{cond:lowerbound} and \ref{cond:upperbound} with $m = d$,  $t = \sqrt{ \tfrac{d}{\log d} }$, we can show that $\mpr[  \ref{event:Zbound} ] \geq 1 - e^{ \frac{ - d}{\log d}}$ for $d \geq \Omega(1)$.
 \item By   \eqref{def:init} and   \ref{event:Zbound} ,    \ref{event:Gbound}  follows.
 \item For   \ref{event:htZlb}, by using \ref{cond:lowerbound} with $t = \frac{\sqrt{\rs} - \sqrt{m}}{\sqrt{\log d}} \geq  \sqrt{  \tfrac{\rs }{2 \log^2 d} }$, we have with probability $1 - 2 \rs \exp \big( \tfrac{- \rs}{2 \log^2  d} \big)$,   for $m \leq \rs (1 -   \log^{\nicefrac{-1}{2}} d) \wedge r$ uniformly:
\eq{
 \lambda_{\mathrm{min}} \left(  \bs{Z}_{1:m} \bs{Z}_{1:m}^\top  \right) >   \tfrac{\rs}{d} \left( 1 - \tfrac{1}{\sqrt{\log d}} \right)^2    \left( 1  - \sqrt{  \tfrac{m}{\rs}  }  \right)^2 \geq  \tfrac{\rs}{5 d}  \left( 1  -   \tfrac{m}{\rs}   \right)^2 . 
}
Therefore,   $\mpr[  \ref{event:htZlb} ] \geq 1 - 2 \rs \exp \big( \tfrac{- \rs}{2 \log  d} \big)$.
\item By Cauchy's eigenvalue interlacing theorem,   $\lambda_{m}(    \Z_{1:r} \Z_{1:r}^\top ) \geq  \lambda_{\mathrm{min}} \left(  \bs{Z}_{1:m} \bs{Z}_{1:m}^\top  \right)$. Therefore, by  \ref{event:htZlb} ,  \ref{event:htZlbiso} follows. 
\item For \ref{event:htZub} , by using $m = r$ and $t = 0.4 \sqrt{\rs}$ in   \ref{cond:upperbound}, we have 	 $\mpr[  \ref{event:htZub}  ] \geq 1 - 2 e^{- 0.16 \rs}$.
\item  For   \ref{event:ltZlb}, by using \ref{cond:squarematrix} with $\varepsilon = \frac{2}{\rs^2}$,  we have  $\mpr  [   \ref{event:ltZlb} ] \geq 1 - \Omega(1/\rs^2)$.
\item  For   \ref{event:ltZub},  by using  \ref{cond:upperbound}  with $t = 0.4 \sqrt{\rs}$, we have with probability $1 -  (10 \rs + 6) e^{-0.16 \rs}$ for $m \in [5\rs] \cup \{\ceil{\log^{2.5} d} ,\ceil{\log^{6} d}, r \}$ uniformly:
\eq{
 \lambda_{\mathrm{max}} ( \Z_{1:m} \Z_{1:m}^\top ) \leq   \tfrac{\rs}{d} \left( 1.4 + \sqrt{  \tfrac{m}{\rs}  }  \right)^2 \leq \frac{5 (\rs \vee m)}{d}.
}
\end{itemize}
By union bound, we have the result.
 \end{proof}

\section{Background: Matrix Riccati Dynamical Systems}
\label{app:background}
We begin by reviewing Riccati dynamical systems in both continuous and discrete time, establishing the necessary background for the arguments that follow. For the following, we define
\eq{
 \Lae \coloneqq \begin{bmatrix}
 \bs{\Lambda}  & 0 \\
0 & 0
\end{bmatrix} ~~ \text{and} ~~   \tL \coloneqq 
 \tfrac{\sqrt{\rs}}{\normL}   \Lae.
}
For notational convenience,     we adapt the abuse of notation:
\eq{
\frac{ \Lae}{\bs{I}_d - \exp (   - t   \Lae)} = \lim_{\varepsilon \to 0}  \frac{ (    \Lae  + \varepsilon \bs{I}_d)}{\bs{I}_d - \exp (  - t (  \Lae  + \varepsilon \bs{I}_d) )}  =   \begin{bmatrix}
\frac{        \bs{\Lambda}  }{\bs{I}_r - \exp (    - t      \bs{\Lambda}  )} & 0 \\ 
0 &   \frac{1}{t} \bs{I}_{d - r}
\end{bmatrix}.
}
\subsection{Continous-time Matrix Riccati ODE}
In this paper, we study continuous-time matrix Riccati differential equations of the following form:
\begin{itemize}[leftmargin=*]
\item Weight Gram matrix: For  $T_W = \rs$
\eq{
\partial_t  \G_{W}(t) =  \frac{0.5}{T_W}   \Big( \tL \G_{W}(t)  + \G_{W}(t)  \tL -  2   \G^2_{W}(t)     \Big). \label{def:weightGramode}
}
\item   Alignment Gram matrix:  For $T_U = \normL \sqrt{\rs}$, 
\eq{
\partial_t  \G_{U}(t) =  \frac{0.5}{T_U}   \Big( \La \G_{U}(t)  + \G_{U}(t)   \La -  2   \G_{U}(t)  \La \G_{U}(t)     \Big). \label{def:alignmentGramode}
}
\end{itemize}
 
For $\alpha = 0$, we assume that the ODEs are expressed in the eigenbasis of $\G_W(0)$ or $\G_U(0)$, ensuring that the trajectories remain diagonal. The solutions of these ODEs are characterized in the following statement:

\begin{lemma} 
\label{lem:riccsolution}
\eqref{def:weightGramode}  and \eqref{def:alignmentGramode} admit the following solutions:
\eq{
\bs{G}_W(t) & = 
    \frac{  \tL}{\bs{I}_d   -   \exp  (  \nicefrac{-  t \tL}{T_W}  )}  \\
    & -  \frac{     \tL  \exp (  \nicefrac{- 0.5 t \tL}{T_W} ) }{\bs{I}_d  -  \exp (  \nicefrac{-  t \tL}{T_W} )}     \bigg (\G_{W}(0)  +   \frac{    \tL  \exp  (  \nicefrac{-   t \tL}{T_W} ) }{\bs{I}_d \! - \! \exp  (  \nicefrac{-  t \tL}{T_W} )}    \bigg )^{-1}     \frac{  \tL  \exp (  \nicefrac{- 0.5 t \tL}{T_W} ) }{\bs{I}_d \! - \! \exp (    \nicefrac{- t \tL}{T_W}  )}  \\[1em]
    \bs{G}_U(t) &  = 
    \frac{  \bs{I}_r }{\bs{I}_r - \exp  (  \nicefrac{-  t \La}{T_U}  )} \\
    & -  \frac{      \exp (  \nicefrac{- 0.5 t \La}{T_U} ) }{\bs{I}_r - \exp (  \nicefrac{-  t \La}{T_U} )}     \bigg (\G_{U}(0)  +     \frac{      \exp (  \nicefrac{-   t \La}{T_U} ) }{\bs{I}_r - \exp (  \nicefrac{-  t \La}{T_U} )}      \bigg )^{-1}   \frac{      \exp (  \nicefrac{- 0.5 t \La}{T_U} ) }{\bs{I}_r - \exp (  \nicefrac{-  t \La}{T_U} )}    
}
Moreover,    $( \G_{W}(t) )_{t \geq 0}$ and   $(\G_{U}(t) )_{t \geq 0}$ are monotone with respect to  $\G_{W}(0)  \succeq 0$ and $\G_{U}(0)  \succeq 0$ respectively.
\end{lemma}

\begin{proof}
 One can check by direct differentiation that the given closed-form expressions satisfy the ODEs above. The uniqueness of the solutions follow the local Lipschitzness of the drifts. Monotonicity is a consequence of Proposition \ref{prop:monotonecharac}.
\end{proof}

\subsection{Discrete-time Matrix Riccati Difference Equations}
In this section, we will study a particular discretization of Alignment Gram matrix ODE, given as
\eq{
\G_{t+1} = \G_t   - \frac{\eta}{2}  (2 \G_t - \bs{I}_r) \La  (2 \G_t - \bs{I}_r)  \big(\bs{I}_r + \eta \La  (2 \G_t - \bs{I}_r) \big)^{-1} + \eta   \bs{\Lambda}. \label{eq:discreteupdate}
}
For convenience, we will make a change of variable and define $\V_t \coloneqq   2 \La^{\frac{1}{2}} \G_t  \La^{\frac{1}{2}}  - \La$. We write  \eqref{eq:discreteupdate} in terms of $\V_t$ as follows:
\eq{
\V_{t+1} =   \V_t   - \eta \V_t^2  \big(\bs{I}_r + \eta   \V_t \big)^{-1} + \eta   \bs{\Lambda}^2. \label{eq:discreteupdat2}
}
We characterize the dynamics of $(\V_t)_{t \in \N}$ as follows:

\begin{lemma}
\label{lem:discretericcati2}
We consider
\eq{
\begin{bmatrix} 
\bs{X}_{t+1,1} \\
\bs{X}_{t+1,2}
\end{bmatrix} =  \begin{bmatrix} 
\bs{X}_{t, 1}  \\
\bs{X}_{t, 2} 
\end{bmatrix}  + \eta \bs{H}
\begin{bmatrix} 
\bs{X}_{t, 1} \\
\bs{X}_{t, 2}
\end{bmatrix} ~~ \text{where} ~~
\begin{bmatrix} 
\bs{X}_{0,1} \\
\bs{X}_{0,2}
\end{bmatrix} =
\begin{bmatrix} 
\bs{I}_r \\
\bs{V}_0
\end{bmatrix}   ~ \text{and} ~ \bs{H} \coloneqq \begin{bmatrix}
0 &  \bs{I}_r \\ 
 \bs{\Lambda}^2 &   \eta  \bs{\Lambda}^2
\end{bmatrix}.
}

The following hold for all $n \in \N$:
\begin{enumerate}[label=(R.\arabic*), leftmargin=*]
\item \label{res:posdef} We have
\eq{
\begin{bmatrix}
\bs{A}_{t,11} & \bs{\Lambda}^{-1} \bs{A}_{t,12} \\
 \bs{\Lambda}  \bs{A}_{t,12} & \bs{A}_{t,22}
\end{bmatrix} \coloneqq (\bs{I}_{2r} + \eta   \bs{H})^t.    \label{eq:matrixpowerform}
} 
where $\bs{A}_{t,11}$,  $\bs{A}_{t,12}$,  $\bs{A}_{t,22}$ are positive definite diagonal matrices.
\item \label{res:symdet} $\bs{A}_{t,11} + \eta \bs{\Lambda} \bs{A}_{t,12} = \bs{A}_{t,22}$ and $\bs{A}_{t,22} \bs{A}_{t,11} -  \bs{A}_{t,12}^2  = \bs{I}_r$.
\item  \label{res:ratiolb} For $\eta \leq 1$, we have  $\bs{A}_{t,11} \bs{A}_{t,12}^{-1} \succ \big( \bs{I}_r + \tfrac{\eta^2}{4} \bs{\Lambda}^2 \big)^{1/2} - \tfrac{\eta}{2}  \bs{\Lambda}$ and $\bs{A}_{t,22} \bs{A}_{t,12}^{-1} \succ \big( \bs{I}_r + \tfrac{\eta^2}{4} \bs{\Lambda}^2 \big)^{1/2} + \tfrac{\eta}{2}  \bs{\Lambda}$.
\item  \label{res:ratiobound} If $\norm{\eta \bs{\Lambda} }_2 < 1$,  
\eq{
\bs{A}_{t,22} \bs{A}_{t,12}^{-1} \succ \frac{(\bs{I}_r + \eta \bs{\Lambda})^t +(\bs{I}_r - \eta \bs{\Lambda})^t  }{(\bs{I}_r + \eta \bs{\Lambda})^t  - (\bs{I}_r - \eta \bs{\Lambda})^t  } \succeq \bs{A}_{t,11} \bs{A}_{t,12}^{-1}.
}
\end{enumerate}
Moreover, if $\bs{X}_{t,1}$ and $\bs{X}_{t+1,1}$ are invertible:
\begin{enumerate}[label=(R.\arabic*), leftmargin=*]
\setcounter{enumi}{6}
\item  \label{res:update}  For $\bs{V}_{t+1} \coloneqq  \bs{X}_{t+1,2} \bs{X}_{t+1,1}^{-1}$, and $\bs{V}_t \coloneqq  \bs{X}_{2,t} \bs{X}_{t,1}^{-1}$, we have
\eq{
\bs{V}_{t+1}  = \bs{V}_t - \eta  \bs{V}_t^2   \left(\bs{I}_r +   \eta \bs{V}_t \right )^{-1} +  \eta  \bs{\Lambda}^2. 
}
\end{enumerate}
\end{lemma}

\begin{proof}[Proof of Lemma \ref{lem:discretericcati2}]
We have
\eq{
\bs{I}_{2r} + \eta   \bs{H}  = \begin{bmatrix}
\bs{I}_r & \eta \bs{I}_r \\
\eta \bs{\Lambda}^2 & \bs{I}_r + \eta^2 \bs{\Lambda}^2
\end{bmatrix}. \label{def:updmatrix}
}
Let
\eq{
(\bs{I}_{2r} + \eta   \bs{H})^t \eqqcolon \begin{bmatrix}
\bs{\tilde{A}}_{t,11} & \bs{\tilde{A}}_{t,12} \\
\bs{\tilde{A}}_{21,t} & \bs{\tilde{A}}_{t,12}
\end{bmatrix}.  \label{eq:matrixpower}
}
Since each submatrix  in \eqref{def:updmatrix} is diagonal positive definite,  the matrices in  \eqref{eq:matrixpowerform} are also diagonal positive definite.  To prove \ref{res:posdef} and the first part of \ref{res:symdet}  we use proof by induction.  We assume   $\bs{\tilde{A}}_{t,11} + \eta \bs{\tilde{A}}_{21, t} = \bs{\tilde{A}}_{t,12}$ and  $\bs{\tilde{A}}_{21, t} \bs{\tilde{A}}^{-1}_{12, t} = \bs{\Lambda}^2$.
We have
\eq{
 \bs{\tilde{A}}_{12, t+1} \! = \! \bs{\tilde{A}}_{t,12}  + \eta   \bs{\tilde{A}}_{t,12}
  \labelrel={ricc:eqq0}   \bs{\tilde{A}}_{t,12} \!  + \!  \eta  ( \bs{\tilde{A}}_{t,11} + \eta \bs{\tilde{A}}_{21, t}   )  
& \labelrel={ricc:eqq1}   \eta \bs{\tilde{A}}_{t,11} + \left( \bs{I}_r + \eta^2   \bs{\Lambda}^2 \right)     \bs{\tilde{A}}_{t,12}  \\
& \labelrel={ricc:eqq2}   \eta \bs{\tilde{A}}_{t,11} + \bs{\Lambda}^{-2} \!  \left( \bs{I}_r + \eta^2   \bs{\Lambda}^2 \right)     \bs{\tilde{A}}_{21, t}     \\
& =   \bs{\Lambda}^{-2}   \bs{\tilde{A}}_{21, t+1} .
}
where  \eqref{ricc:eqq0}  follows the first assumption, \eqref{ricc:eqq1}  and \eqref{ricc:eqq2} follow the second assumption. 
Moreover,
\eq{
\bs{\tilde{A}}_{11, t+1} + \eta \bs{\tilde{A}}_{21, t+1} & =  \bs{\tilde{A}}_{t,11} + \eta \bs{\tilde{A}}_{21, t} + \eta^2 \bs{\Lambda}^2    \bs{\tilde{A}}_{t,11} + \eta (\bs{I}_r + \eta^2 \bs{\Lambda}^2)  \bs{\tilde{A}}_{21, t}  \\
& =  (\bs{I}_r + \eta^2 \bs{\Lambda}^2)  (  \bs{\tilde{A}}_{t,11} + \eta \bs{\tilde{A}}_{21, t}  ) + \eta  \bs{\tilde{A}}_{21, t}  \\
&  \labelrel={ricc:eqq3}  (\bs{I}_r + \eta^2 \bs{\Lambda}^2)   \bs{\tilde{A}}_{t,12}  +   \eta \bs{\Lambda}^2  \bs{\tilde{A}}_{t,12}   
  =   \bs{\tilde{A}}_{22, t+1}. 
}
where \eqref{ricc:eqq3}  follows the first and second assumptions. For the second part of  \ref{res:symdet}  we again use proof by induction.  We assume 
$\bs{A}_{t,22} \bs{A}_{t,11} -  \bs{A}_{t,12}^2  = \bs{I}_r$.  We have
\eq{
  \bs{\tilde{A}}_{11, t+1}   \bs{\tilde{A}}_{22, t+1}   -  \bs{\tilde{A}}_{12, n+1}   \bs{\tilde{A}}_{21, t+1}  
  & = \Big(   \bs{\tilde{A}}_{t,11} + \eta  \bs{\tilde{A}}_{21, t}  \Big) \Big( \eta \bs{\Lambda}^2  \bs{\tilde{A}}_{t,12}   + \left( \bs{I}_r + \eta^2 \bs{\Lambda}^2 \right)  \bs{\tilde{A}}_{t,12} \Big) \\
  & - \Big( \eta \bs{\Lambda}^2  \bs{\tilde{A}}_{t,11} +   \left( \bs{I}_r + \eta^2 \bs{\Lambda}^2 \right)  \bs{\tilde{A}}_{21, t}   \Big) \Big( \bs{\tilde{A}}_{t,12}  + \eta \bs{\tilde{A}}_{t,12} \Big) \\
  & =    \bs{\tilde{A}}_{t,11}   \bs{\tilde{A}}_{t,12}  -  \bs{\tilde{A}}_{t,12}    \bs{\tilde{A}}_{21, t} = \bs{I}_r.
}
For   \ref{res:ratiolb}, by using  \ref{res:symdet},  we have
\eq{
 \bs{A}_{t,11} \left(  \bs{A}_{t,11} + \eta \bs{\Lambda} \bs{A}_{t,12}    \right)  -  \bs{A}_{t,12}^2  = \bs{I}_r &  \Rightarrow \big(\bs{A}_{t,11} \bs{A}^{-1}_{t,12} \big)^2 + \eta \bs{\Lambda}  \big(\bs{A}_{t,11} \bs{A}^{-1}_{t,12} \big) - \bs{I}_r \succ 0 \\
&  \Rightarrow \bs{A}_{t,11} \bs{A}^{-1}_{t,12} \succ  \left( \bs{I}_r + \frac{\eta^2}{4}  \bs{\Lambda}  \right)^{1/2}  -  \frac{\eta}{2} \bs{\Lambda}.
}
The second part follows   \ref{res:symdet}.  For \ref{res:ratiobound}, we recall that
\eq{
& \bs{A}_{t+1,12} = \bs{A}_{t,12} + \eta \bs{\Lambda} \bs{A}_{t,22}  =  (\bs{I}_r + \eta^2 \bs{\Lambda}^2)\bs{A}_{t,12} + \eta \bs{\Lambda} \bs{A}_{t,11}  \label{eq:riccfact1} \\
& \bs{A}_{t+1,22} = \eta   \bs{\Lambda}  \bs{A}_{t,12} + \left( \bs{I}_r + \eta^2 \bs{\Lambda}^2 \right)  \bs{A}_{t,22} \succ \eta   \bs{\Lambda}  \bs{A}_{t,12} + \bs{A}_{t,22}.  \label{eq:riccfact2}  
}
We use proof by induction.  Suppose the lower bound for $\tfrac{\bs{A}_{t,22}}{\bs{A}_{t,12}}$ holds.  We have
\eq{
\frac{\bs{A}_{t+1,22}}{\bs{A}_{t+1,12}} & \labelrel\succ{ricc:ineqq4} \frac{\eta   \bs{\Lambda}  \bs{A}_{t,12} + \bs{A}_{t,22} }{\bs{A}_{t,12} + \eta \bs{\Lambda} \bs{A}_{t,22}}   \\
 &  \labelrel\succ{ricc:ineqq5} \left( \eta   \bs{\Lambda}  +  \frac{(\bs{I}_r + \eta \bs{\Lambda})^t +(\bs{I}_r - \eta \bs{\Lambda})^t  }{(\bs{I}_r + \eta \bs{\Lambda})^t  - (\bs{I}_r - \eta \bs{\Lambda})^t } \right) \left(  \bs{I}_r + \eta \bs{\Lambda}   \frac{(\bs{I}_r + \eta \bs{\Lambda})^t +(\bs{I}_r - \eta \bs{\Lambda})^t  }{(\bs{I}_r + \eta \bs{\Lambda})^t  - (\bs{I}_r - \eta \bs{\Lambda})^t  } \right)^{-1}   \\[0.5em]
& =  \frac{(\bs{I}_r + \eta \bs{\Lambda})^{t+1}  +(\bs{I}_r - \eta \bs{\Lambda})^{t+1}  }{(\bs{I}_r + \eta \bs{\Lambda})^{t+1}  - (\bs{I}_r - \eta \bs{\Lambda})^{t+1}  }.
}
where \eqref{ricc:ineqq4} follows \eqref{eq:riccfact2} and  \eqref{ricc:ineqq5} follows the induction hypothesis with that  $x \to \frac{x + \eta \lambda}{1+ \eta \lambda x}$ is monotonic increasing for $\eta \lambda < 1$.  For the upper bound,   suppose the lower bound for $\tfrac{\bs{A}_{t,11}}{\bs{A}_{t,12}}$ holds. We have
\eq{
\frac{\bs{A}_{t+1,11}}{\bs{A}_{t+1,12}} & \labelrel\preceq{ricc:ineqq6} \frac{\bs{A}_{t,11} + \eta   \bs{\Lambda}  \bs{A}_{t,12}}{\bs{A}_{t,12} + \eta \bs{\Lambda} \bs{A}_{t,11}}
 \\
 &  \labelrel\preceq{ricc:ineqq7} \left( \eta   \bs{\Lambda}  +  \frac{(\bs{I}_r + \eta \bs{\Lambda})^t +(\bs{I}_r - \eta \bs{\Lambda})^t  }{(\bs{I}_r + \eta \bs{\Lambda})^t  - (\bs{I}_r - \eta \bs{\Lambda})^t  } \right) \left( \bs{I}_r + \eta \bs{\Lambda}   \frac{(\bs{I}_r + \eta \bs{\Lambda})^t +(\bs{I}_r - \eta \bs{\Lambda})^t  }{(\bs{I}_r + \eta \bs{\Lambda})^t  - (\bs{I}_r - \eta \bs{\Lambda})^t  } \right)^{-1}   \\[0.5em]
& =  \frac{(\bs{I}_r + \eta \bs{\Lambda})^{t+1}  +(\bs{I}_r - \eta \bs{\Lambda})^{t+1}  }{(\bs{I}_r + \eta \bs{\Lambda})^{t+1}  - (\bs{I}_r - \eta \bs{\Lambda})^{t+1}  }.
}
where \eqref{ricc:ineqq6} follows \eqref{eq:riccfact1},  and \eqref{ricc:ineqq7} follows the induction hypothesis. 

Lastly if  $\bs{X}_{t,1}$ and  $\bs{X}_{t+1,1}$ are invertible,
\eq{
 \bs{X}_{t+1,2} \bs{X}^{-1}_{t+1,1} 
 & = \left(  \bs{X}_{t,2} \bs{X}^{-1}_{t,1} + \eta^2 \bs{\Lambda}^2  \bs{X}_{t,2} \bs{X}^{-1}_{t,1}   + \eta \bs{\Lambda}^2 \right) \left( \bs{I}_r + \eta  \bs{X}_{t,2} \bs{X}^{-1}_{t,2}  \right)^{-1} \\
& =  \bs{X}_{t,2} \bs{X}^{-1}_{t,2}  \left( \bs{I}_r + \eta \bs{X}_{t,2} \bs{X}^{-1}_{t,2} \right)^{-1} + \eta \bs{\Lambda}^2 \\
& =    \bs{X}_{t,2} \bs{X}^{-1}_{t,2}   - \eta   \bs{X}_{t,2} \bs{X}^{-1}_{t,1}    \bs{X}_{t,2} \bs{X}^{-1}_{t,1}    \left( \bs{I}_r + \eta \bs{X}_{t,2} \bs{X}^{-1}_{t,1}  \right)^{-1} + \eta \bs{\Lambda}^2.
}
\end{proof}

\begin{corollary}
\label{cor:discretericcatidyn}
For  $\bs{V}_0 = 2 \La_2^{\frac{1}{2}}  \bs{G}_0  \La_2^{\frac{1}{2}}  - \La_1$, we define
\eq{
\bs{V}_{t+1} = \bs{V}_t - \eta \bs{V}_t^2 \left( \bs{I}_r +  \eta  \bs{V}_t\right)^{-1} + \eta \hL^2.  \label{eq:discretericcatisym}
}
If $\La_1$, $\La_2$ and  $\hL$ are mutually  diagonalizable,  and $\bs{X}_{1,t}$ is invertible for $t \leq t^* \in \N$,   we have for   $t \leq t^*$
\eq{
\bs{G}_t &  = \frac{ \frac{\La_1}{\La_2} + \frac{\bs{A}_{t,22}}{\bs{A}_{t,12}}  \frac{  \hL  }{  \La_2 }   }{2}   - \frac{1}{4}      \frac{ \bs{A}^{-1}_{t,12}  \hL}{  \La_2 }   \left( \frac{  \frac{ \hL }{  \La_2 }   \frac{ \bs{A}_{t,11} }{\bs{A}_{t,12}} -  \frac{\La_1}{\La_2} }{2} +      \bs{G}_0  \right)^{-1}    \frac{ \hL \bs{A}^{-1}_{t,12}  }{  \La_2 },   
}
where  $\bs{A}_{11 t}$, $\bs{A}_{t,12}$, $\bs{A}_{t,22}$ are defined with $\hL$.
\end{corollary}

\begin{proof}
By using      \eqref{eq:matrixpowerform},  we can write that
\eq{
 \bs{V}_t  
& = \left(  \hL  \bs{A}_{t,12} +  \bs{A}_{t,22} \bs{V}_0   \right) \left(\hL \bs{A}^{-1}_{t,12}  \bs{A}_{t,11}  +   \bs{V}_0 \right)^{-1} \hL  \bs{A}^{-1}_{t,12} \\
& =  \hL  \bs{A}_{t,12}   \! \left( \hL \bs{A}^{-1}_{t,12} \bs{A}_{t,11}  +    \bs{V}_0 \right)^{-1}   \hL  \bs{A}^{-1}_{t,12}  \\
& \quad  +     \bs{A}_{t,22} \Big( \bs{I}_r  -    \bs{A}_{t,11} \bs{A}^{-1}_{t,12}  \hL  \left(  \hL \bs{A}^{-1}_{t,12} \bs{A}_{t,11}   +     \bs{V}_0 \right)^{-1} \! \Big) \hL  \bs{A}^{-1}_{t,12} \\
& =   \bs{A}_{t,22} \bs{A}^{-1}_{t,12}   \hL   -    \bs{A}^{-1}_{t,12} \hL   \left(  \hL  \bs{A}^{-1}_{t,12} \bs{A}_{t,11} +  \bs{V}_0 \right)^{-1}   \hL  \bs{A}^{-1}_{t,12}.
}
Therefore,
\eq{
\bs{G}_t &  = \frac{ \frac{\La_1}{\La_2} + \frac{\bs{A}_{t,22}}{\bs{A}_{t,12}}  \frac{  \hL  }{  \La_2 }   }{2}   - \frac{1}{4}      \frac{ \bs{A}^{-1}_{t,12}  \hL}{  \La_2 }   \left( \frac{  \frac{ \hL }{  \La_2 }   \frac{ \bs{A}_{t,11} }{\bs{A}_{t,12}} -  \frac{\La_1}{\La_2} }{2} +      \bs{G}_0  \right)^{-1}    \frac{ \hL \bs{A}^{-1}_{t,12}  }{  \La_2 }   . 
}
\end{proof}
 
\begin{proposition}
\label{prop:monotoneiteration}
For some  symmetric matrix $\bs{S}$, we consider
\eq{
\bs{V}_{1} =  \bs{V}_{0} +  \eta  \bs{S} -  \eta  \bs{V}_{0}^2 \left( \bs{I}_r + \eta  \bs{V}_{0}  \right)^{-1}. \label{eq:monotoneiteration}
}
If   $\bs{V}^{+}_{0} \succeq  \bs{V}_{0} \succ \tfrac{-1}{\eta} \bs{I}_r$,  we have  $ \bs{V}^{+}_{1} \succeq  \bs{V}_{1}$, where  $\bs{V}^{+}_{1}$ is the next iterate if we use  $\bs{V}^{+}_{0} $ in \eqref{eq:monotoneiteration}.
\end{proposition}

\begin{proof}
We have
\eq{
\bs{V}_{1}  =  \frac{1}{\eta} \Big( \bs{I}_r -  \left( \bs{I}_r + \eta  \bs{V}_{0}  \right)^{-1} \Big) +  \eta  \bs{S}.
}
The statement follows by  Proposition \ref{prop:monotonecharac}.
\end{proof}

\section{Proofs for Main Results}
\label{sec:mainproofs}
\subsection{Proof of Propositions \ref{prop:sgdrotation} and \ref{prop:monotoneupdatetext}}
For Proposition \ref{prop:sgdrotation}, we observe that
\eq{
\widehat{\G}(t) \coloneqq\widehat{\W}(t) \widehat{\W}(t)^\top  ~~ \text{and} ~~ \widetilde{\G}(t) \coloneqq \bs{U}(t) \bs{U}(t)^\top = \W(t) ( \W(t)^\top \W(t) )^{-1} \W(t)^\top  
}
have the exact same dynamics. Therefore the statement follows. Proposition \ref{prop:monotoneupdatetext} follows Proposition \ref{prop:monotonecharac}.

\subsection{Decomposition of the population risk}


The fine-tuning step relies on the following decomposition of the population risk:

\begin{proposition}
\label{prop:riskdecomptext}
    For any $\bs{\Omega} \in \R^{\rs \times \rs}$, the population risk defined in \eqref{eq:poprisk} can be written as:
    \eq{
R(\W_t \bs{\Omega}) =  \frac{1}{\rs}  \norm*{ \bs{\Omega} \bs{\Omega}^\top - \tfrac{\sqrt{\rs}}{\normL}  \W_t^\top \Th \bs{\Lambda} \Th^\top \W_t }_F^2 + \frac{1}{\normL^2} \Big( \normL^2 -   \norm{\bs{\Lambda}^{\frac{1}{2}}  \G_t \bs{\Lambda}^{\frac{1}{2}}}_F^2   \Big) ,
}
where  $\G_t = \Th^\top \W_t \W_t^\top \Th$ is the discrete alignment Gram matrix defined in the previous part.
\end{proposition}

We observe that both the second term and the matrix $\bs{W}_t^\top \bs{\Theta} \bs{\Lambda} \bs{\Theta}^\top \bs{W}_t$ are independent of $\bs{\Omega}$. Hence, the fine-tuning step reduces to a least squares problem in the matrix $\bs{\Omega} \bs{\Omega}^\top$ in population, which is approximated via empirical risk minimization over a fresh batch of samples. By standard concentration arguments, a sample size of $N_{\mathrm{Ft}} \geq \rs^2 \mathrm{polylog} d$ suffices to ensure that the empirical minimizer approximates the population solution with high probability. 

\begin{proof}
We begin by noting that $\bs{W}_t$ is an orthonormal matrix. Using this, we can express the population risk as:
\eq{ 
& R(\bs{W}_t \bs{\Omega})  =   \Big \lVert \frac{1}{\sqrt{\rs}} \bs{W}_t \bs{\Omega} \bs{\Omega}^\top \bs{W}_t^\top - \frac{1}{\normL} \bs{\Theta} \bs{\Lambda} \bs{\Theta}^\top \Big \rVert_\text{F}^2 \\
& = \!  \frac{1}{\normL^2}  \! \left( \! \normL^2  \! + \!  \frac{\normL^2}{\rs} \norm{  \bs{\Omega} \bs{\Omega}^\top }_F^2 \! -  \!  \frac{2 \normL}{\sqrt{\rs}} \tr(\bs{\Omega} \bs{\Omega}^\top  \bs{W}_t^\top \bs{\Theta} \bs{\Lambda} \bs{\Theta}^\top \W_t ) \! \pm \!  \norm{ \bs{W}_t^\top \bs{\Theta} \bs{\Lambda} \bs{\Theta}^\top \W_t}_F^2 \! \right) \\
& =  \frac{1}{\normL^2} \left(  \normL^2 -   \norm{ \bs{W}_t^\top \bs{\Theta} \bs{\Lambda} \bs{\Theta}^\top \W_t}_F^2 + \norm*{ \frac{\normL}{\sqrt{\rs}} \bs{\Omega} \bs{\Omega}^\top  -  \bs{W}_t^\top \bs{\Theta} \bs{\Lambda} \bs{\Theta}^\top \W_t }_F^2\right) 
}
By observing that  $\norm{ \bs{W}^\top \bs{\Theta} \bs{\Lambda} \bs{\Theta}^\top \W}_F^2 = \norm{\La^{\frac{1}{2}} \G_t \La^{\frac{1}{2}}}_F^2$, we have the statement.  
\end{proof}

\subsection{Proof of Theorem \ref{thm:gfresult}}
\label{sec:thmgfresult}
We let 
\eq{
t_{\mathrm{sc}} \coloneqq t \sqrt{\rs} \normL, \quad \mathsf{\kappa}_{\mathrm{eff}}  \coloneqq \begin{cases}
r^{\alpha}, & \alpha \in [0,0.5) \\
1, & \alpha > 0.5
\end{cases}, \qquad   
\mathsf{T}_{\mathrm{eff}} \coloneqq  \mathsf{\kappa}_{\mathrm{eff}}  \sqrt{\rs} \normL \log \nicefrac{d}{\rs}.  \label{def:contimescale}
}
and
\eq{
& \ru \coloneqq \begin{cases}
r, & \alpha \in [0,0.5) \\
\ceil{\log^{2.5} d}, & \alpha > 0.5
\end{cases}  \qquad 
\rus\coloneqq \begin{cases}
\floor{ \rs (1 - \log^{\nicefrac{-1}{8}} d) \wedge r} , & \alpha \in [0,0.5) \\
\rs, & \alpha > 0.5.
\end{cases}   \label{def:effwidthcont}
}
In the following part, we will establish the high-dimensional limit of the risk curve and the alignment.
\subsubsection{High-dimensional limit for the alignment}
By Lemma \ref{lem:riccsolution},  we have
\eq{
\bs{G}_U(t_{\mathrm{sc}}) \! =
 \!    \frac{  \bs{I}_r }{\bs{I}_r - \exp  (  -  t \La  )}   - \frac{      \exp (  - 0.5 t \La  ) }{\bs{I}_r - \exp ( -  t \La  )}     \bigg (\G_{U}(0) +   \frac{      \exp ( -   t \La) }{\bs{I}_r - \exp (  -  t \La )}      \bigg )^{-1} \!\!\!  \frac{      \exp ( - 0.5 t \La  ) }{\bs{I}_r - \exp (  -  t \La  )}.      \label{eq:alignment} 
} 
We define the block matrix forms
\eq{
& \bs{G}_U(t) \eqqcolon \begin{bmatrix}
\bs{G}_{U,11}(t) &  \bs{G}_{U,12}(t) \\
\bs{G}_{U,12}^\top(t) & \bs{G}_{U,22}(t)
\end{bmatrix},~~
\La  = \begin{bmatrix}
\La_{\mathrm{eff}} & 0 \\
0 & \La_{22}
\end{bmatrix} , ~~
\bs{\Lambda}_{\mathrm{e},11} \coloneqq \La_{\mathrm{eff}} , ~~ 
\bs{\Lambda}_{\mathrm{e},22}\coloneqq   \begin{bmatrix}
  \La_{22}  & 0 \\
0  &  0
\end{bmatrix},
\label{def:defalignment}
}
where   $\bs{G}_{U,11}(t) ,  \La_{\mathrm{eff}}  \in \R^{\rus \times \rus }$.  The following statement characterizes the time-scales for the alignment terms.  

\begin{proposition}
 \label{prop:alignmenthighdimensional}
$ \mathcal{G}_{\text{init}}$ implies that  $\mathcal{A}(t   \mathsf{T}_{\mathrm{eff}}, \bs{\theta}_j)  = \mathbbm{1} \{ t \mathsf{\kappa}_{\mathrm{eff}}  \geq \tfrac{1}{\lambda_j} \} + o_d(1)$ for $t \not = \lim_{d \to \infty} \frac{1}{\lambda_j \mathsf{\kappa}_{\mathrm{eff}}}$ and $j \leq \rus$.
\end{proposition}

\begin{proof}
For $\alpha = 0$, since the trajectory stays diagonal and the diagonal entries are monotonically increasing,  by using the events \ref{event:Gbound} and \ref{event:htZlbiso} with Lemma \ref{lem:asympkernel} we have the result. 

\smallskip
In the following, we will prove the result for $\alpha > 0$.    By  using Proposition \ref{prop:matrixyounginequality} with   \ref{event:ltZub}  
\eq{
\bs{G}_{U}(0) \preceq
\begin{bmatrix}
2.1 \Z_{1:\rus} \Z_{1:\rus}^\top   & 0 \\
0 & 2.1 \Z_2 \Z_2^\top
\end{bmatrix},
}
where
\eq{
2.1\lambda_{\mathrm{max}} ( \Z_{1:\rus} \Z_{1:\rus}^\top ) \leq  \begin{cases}
 5 (1 + \frac{1}{\sqrt{\varphi}} )^2, & \alpha \in [0,0.5) \\
15,  & \alpha > 0.5.
\end{cases} 
}
Therefore,
\eq{
\bs{G}_{U,11}(t_{\mathrm{sc}}) &  \preceq    
    \frac{ \bs{I}_{\rus} }{\bs{I}_{\rus} \! \!  - \! \exp   (   -  t  \La_{\mathrm{eff}}   )}    \\[0.3em]
    & -  \!   \frac{   \exp   (   -  0.5 t \La_{\mathrm{eff}}  )  }{\bs{I}_{\rus} \! \! -  \!  \exp   (   -  t \La_{\mathrm{eff}} )  }     \bigg (    \frac{O(\rs)}{d} \bs{I}_{\rus}  \!\! + \!   \frac{    \exp   (   -  t \La_{\mathrm{eff}}   )   }{\bs{I}_{\rus} \!\! -  \!  \exp   (   -  t  \La_{\mathrm{eff}}   )  }    \bigg )^{-1} \! \!\! \frac{  \exp   (   -  0.5 t  \La_{\mathrm{eff}}    ) }{\bs{I}_{\rus} \! \! - \!   \exp   (   -  t \La_{\mathrm{eff}}    ) } .
}
Therefore,  by Proposition \ref{prop:asymptoticlimit}, for $j \leq \rus$,
\eq{
\mathcal{A}(t \mathsf{T}_{\mathrm{eff}}, \bs{\theta}_j) \leq  \frac{1}{1 + \Big(\frac{d}{\rs} \frac{1}{\log^3 d} - 1\Big) \frac{d}{\rs}^{- t  \mathsf{\kappa}_{\mathrm{eff}}  j^{-\alpha}}}  = \mathbbm{1} \{ t \mathsf{\kappa}_{\mathrm{eff}}  \geq \tfrac{1}{\lambda_j} \} + o_d(1).
}
Moreover, for $t \leq  (\rus +1)^{\alpha} \log \frac{d}{\rs}$,  by using the events \ref{event:htZub} and \ref{event:ltZub}, we have
\eq{
\Z_2^\top \exp(t \La_{22}) \Z_2 \preceq \begin{cases}
O_{\varphi}(1) \bs{I}_{\rs}, & \alpha \in (0,0.5) \\
O(\log^{2.5} d)  \bs{I}_{\rs}, & \alpha > 0.5.
\end{cases}
}
Therefore, for  $t \leq  (\rus +1)^{\alpha} \log \frac{d}{\rs}$,  we have  $\bs{G}_{U,11}(t_{\mathrm{sc}})  \succeq   \uG{}(t)$ where
\eq{
 \uG{}(t) &  \coloneqq
    \frac{ \bs{I}_{\rus} }{\bs{I}_{\rus}  \! \!- \! \exp   (   -  t  \La_{\mathrm{eff}}   )}  \\[0.3em]
    & - \!  \frac{   \exp   (   -  0.5 t  \La_{\mathrm{eff}}    )  }{\bs{I}_{\rus} \! \! - \!   \exp   (   -  t  \La_{\mathrm{eff}}  )  }     \bigg (  \frac{\rs}{  d} \frac{ O(1) }{\log^4 d} \bs{I}_{\rus}  \!\!   + \!   \frac{    \exp   (   -  t  \La_{\mathrm{eff}}  )   }{\bs{I}_{\rus}\!  \! -    \! \exp   (   -  t   \La_{\mathrm{eff}}   )  }    \bigg )^{-1} \!\!\! \frac{  \exp   (   -  0.5 t  \La_{\mathrm{eff}}    ) }{\bs{I}_{\rus}  \! \! -  \!  \exp   (   -  t  \La_{\mathrm{eff}}    ) },
}
which implies that for  $t < (\rus +1)^{\alpha} \log \frac{d}{\rs}$  
\eq{
\mathcal{A}(t \mathsf{T}_{\mathrm{eff}}, \bs{\theta}_j) \geq  \frac{1}{1 +   O(\log^4 d)  \frac{d}{\rs}^{1 - t   \mathsf{\kappa}_{\mathrm{eff}}   j^{-\alpha}}}  = \mathbbm{1} \{ t  \mathsf{\kappa}_{\mathrm{eff}} \geq \tfrac{1}{\lambda_j} \} + o_d(1).
}
To extend the lower bound for $t >   (\rus +1)^{\alpha} \log \frac{d}{\rs}$ , let us define 
\eq{
t_0 \coloneqq  (\rus +1)^{\alpha} \log \tfrac{d}{\rs} ~~ \text{and} ~~ \La^{-}_{\mathrm{eff}} \coloneqq \La_{\mathrm{eff}} -  (\rus +1)^{- \alpha} \bs{I}_{\rus}.
}
We have for $t > t_0$,
\eq{
\partial_t  \G_{U, 11}(t) & =   \frac{0.5}{T_U}   \Big(  \La^{-}_{\mathrm{eff}}  \G_{U,11}(t)  + \G_{U,11}(t)  \La^{-}_{\mathrm{eff}} -  2   \G_{U,11}(t)   \La^{-}_{\mathrm{eff}}  \G_{U,11}(t)     \Big)  \\
& + \underbrace{ \frac{1}{T_U}  \Big( (\ru+1)^{-\alpha}   \G_{U, 11}(t)  ( \bs{I}_{\rus} - \G_{U, 11}(t)  )  -  \G_{U,12}(t) \La_{22} \G^\top_{U,12}(t)   \Big). }_{\succeq \frac{(\ru+1)^{-\alpha}}{T_U}   \Big(   \G_{U, 11}(t)  \big( \bs{I}_{\rus} - \G_{U, 11}(t) \big )  -  \G_{U,12}(t)  \G^\top_{U,12}(t)   \Big) \succeq 0}
} 
Therefore,  for $t > t_0$, by monotonicity and \cite[Theorem 38]{Barilari2014comparison}, $\bs{G}_{U,11}(t_{\mathrm{sc}})   \succeq \uG{}(t) \succeq \uG{}(t_0)$, where
\eq{
 & \uG{}(t) =  
    \frac{ \bs{I}_{\rus} }{\bs{I}_{\rus} - \exp   (   -  (t - t_0)   \La^{-}_{\mathrm{eff}}  )}   \\
   &   - \frac{   \exp   (   -  0.5 (t - t_0)    \La^{-}_{\mathrm{eff}} )  }{\bs{I}_{\rus} \!\!\!\! -   \exp   (   -  (t - t_0)  \La^{-}_{\mathrm{eff}}    )  }     \bigg (  \underline{\bs{G} }(t_0)  \! + \!   \frac{    \exp   (   -  (t - t_0)     \La^{-}_{\mathrm{eff}} )   }{\bs{I}_{\rus} \!\!\!\! -    \exp   (   -  (t - t_0)     \La^{-}_{\mathrm{eff}}   )  }    \bigg )^{-1} \!\!\!\! \! \frac{  \exp   (   -  0.5 (t - t_0)   \La^{-}_{\mathrm{eff}}  ) }{\bs{I}_{\rus} \!\!\!\! -   \exp   (   - (t - t_0)     \La^{-}_{\mathrm{eff}}     ) }.  
}
Therefore, the result extends to   $t > t_0$ as well.
\end{proof}
\subsubsection{High-dimensional limit for the risk curve}
For  $\mathsf{Err}(t) \coloneqq  \normL \big( \tfrac{\Lae }{\normL}  \! - \! \tfrac{\bs{G}_{W}(t)}{\sqrt{\rs}} \big)$,  by Lemma \ref{lem:riccsolution},  we have
\eq{
 \mathsf{Err}(t_{\mathrm{sc}})  &   \! = \!     \frac{  -  \Lae \exp  ( - t \Lae ) }{\bs{I}_d - \exp  (  - t  \Lae )} \\
 & \! + \! \frac{    \Lae  \exp (  - 0.5 t \Lae  ) }{\bs{I}_d - \exp  ( - t \Lae )}     \bigg(  \frac{\normL}{\sqrt{\rs}} \G_{W}(0) \! + \!  \frac{    \Lae \exp ( - t \Lae  ) }{\bs{I}_d - \exp  (  - t  \Lae )}    \bigg)^{-1} \!\!   \frac{  \Lae  \exp  (   - 0.5 t \Lae  ) }{\bs{I}_d - \exp  (   -   t \Lae  )},   \label{eq:riskdifference} 
}

 We define the block matrix forms
\eq{
\bs{G}_W(t) = \begin{bmatrix}
\bs{G}_{W,11}(t) &  \bs{G}_{W,12}(t) \\
\bs{G}_{W,12}^\top(t) & \bs{G}_{W,22}(t)
\end{bmatrix}, ~~  
\La  = \begin{bmatrix}
\La_{\mathrm{eff}} & 0 \\
0 & \La_{22}
\end{bmatrix}, ~~
\bs{\Lambda}_{\mathrm{e},11} \coloneqq \La_{\mathrm{eff}} , ~~
\bs{\Lambda}_{\mathrm{e},22}\coloneqq   \begin{bmatrix}
  \La_{22}  & 0 \\
0  &  0
\end{bmatrix},
\label{def:defrisk}
}
where   $\bs{G}_{W,11}(t) ,  \La_{\mathrm{eff}}  \in \R^{\ru \times \ru }$. 
Our proof strategy is as follows: In Proposition \ref{prop:riskoffdiagonal},    we show that the off-diagonal and lower-right terms in \eqref{eq:riskdifference} does not contribute to the high-dimensional limit. Then,  in  Proposition \ref{prop:riskhighdimensional},  we characterize the limit of the left-top terms.  Finally, in Proposition \ref{prop:riskmain}, we prove the  asymptotic behaviour of the risk curve.

\begin{proposition}
\label{prop:riskoffdiagonal}
$\mathcal{G}_{\text{init}}$ implies that  $\norm*{ \bs{G}_{W,12}(t  \mathsf{T}_{\mathrm{eff}} ) }_F^2  = o_d(\rs)$   and  $\norm*{ \bs{G}_{W,22}(t \mathsf{T}_{\mathrm{eff}} ) }_F^2 = o_d(\rs)$.
\end{proposition}

\begin{proof}
We let
\eq{
\D_1 \coloneqq \frac{  \bs{\Lambda}_{\mathrm{e},11} \exp(- t  \bs{\Lambda}_{\mathrm{e},11} )}{\bs{I}_{\ru} -\exp(- t  \bs{\Lambda}_{\mathrm{e},11}) }, \quad   \D_2 \coloneqq     \frac{  \bs{\Lambda}_{\mathrm{e},22} \exp(- t  \bs{\Lambda}_{\mathrm{e},22} )}{\bs{I}_{d - \ru} -\exp(- t  \bs{\Lambda}_{\mathrm{e},22}) },   \quad  \Z \coloneqq \begin{bmatrix}
\Z_{1:\ru} \\
\Z_2
\end{bmatrix}.
}
and
\eq{
\begin{bmatrix}
\bs{S}_{11} & \bs{S}_{12} \\
\bs{S}_{12}^\top & \bs{S}_{22}
\end{bmatrix} \coloneqq    \Big(   \begin{bmatrix}
 \frac{\normL}{\sqrt{\rs}}  \Z_{1:\ru}  \Z_{1:\ru} ^\top  + \D_1 &  \frac{\normL}{\sqrt{\rs}} \Z_{1:\ru}  \Z_2^\top \\
 \frac{\normL}{\sqrt{\rs}}\Z_2 \Z_{1:\ru} ^\top &   \frac{\normL}{\sqrt{\rs}}  \Z_2 \Z_2^\top   + \D_2
\end{bmatrix}   \Big)^{-1} .
}
where
\eq{
\bs{S}_{11} & =  \left(   \D_1 +   \Z_{1:\ru}   \left( \frac{\sqrt{\rs}}{\normL} \bs{I}_{\rs} +    \Z_2^\top  \D_2^{-1}  \Z_2 \right)^{-1}  \Z_{1:\ru} ^\top  \right)^{-1}, \\
\bs{S}_{12} & =  - \left(   \D_1 +   \Z_{1:\ru}  \left( \frac{\sqrt{\rs}}{\normL} \bs{I}_{\rs} +    \Z_2^\top  \D_2^{-1}     \Z_2 \right)^{-1}   \Z_{1:\ru} ^\top  \right)^{-1} \Z_{1:\ru}  \Z_2^\top  \left(  \Z_2 \Z_2^\top   +  \D_2 \right)^{-1},  \\
\bs{S}_{22} & =   \left( \D_2    + \Z_2  \left( \frac{\sqrt{\rs}}{\normL} \bs{I}_{\rs} +    \Z_{1:\ru} ^\top   \D_1^{-1}    \Z_{1:\ru}  \right)^{-1}  \Z_2^\top  \right)^{-1}.
}
\paragraph{Off-diagonal terms:} By Proposition  \ref{prop:offdiagonal}
\eq{
 & \tilde{\bs{G}}_{W,12}(t)   \coloneqq  \frac{  \bs{\Lambda}_{\mathrm{e},11} \exp(-0.5 t  \bs{\Lambda}_{\mathrm{e},11} )}{\bs{I}_{\ru} -\exp(- t  \bs{\Lambda}_{\mathrm{e},11}) }  \bs{S}_{12}   \frac{  \bs{\Lambda}_{\mathrm{e},22} \exp(- 0.5 t  \bs{\Lambda}_{\mathrm{e},11} )}{\bs{I}_{d - \ru} -\exp(- t  \bs{\Lambda}_{\mathrm{e},22}) } \\
 &    =   \exp( 0.5 t  \bs{\Lambda}_{\mathrm{e},11} )  \Z_{1:\ru} \left(  \tfrac{\sqrt{\rs}}{\normL} \bs{I}_{\rs} +    \Z_2^\top  \D_2^{-1}  \Z_2 +  \Z_{1:\ru}^\top  \D_1^{-1}  \Z_{1:\ru}    \right)^{-1} \Z_2^\top   \exp( 0.5 t  \bs{\Lambda}_{\mathrm{e},22} )  . 
}
We observe that  $ \lambda_{\mathrm{max}}(\D_2 ) \preceq \frac{1}{t} $ and  $\frac{\sqrt{\rs}}{\normL} \asymp  \mathsf{\kappa}_{\mathrm{eff}}$ .  By using   Proposition \ref{prop:offdiagonalfrob} with $\mathcal{G}_{\text{init}}$ and $\tilde{t} \coloneqq  t  \mathsf{\kappa}_{\mathrm{eff}} \log \tfrac{d}{\rs}$, we write
\eq{
  &  \frac{1}{\rs} \norm*{ \bs{G}_{W,12}(t  \mathsf{T}_{\mathrm{eff}} ) }_F^2 = \frac{1}{\normL^2} \norm{   \tilde{\bs{G}}_{W,12}(   \tilde{t} )  }_F^2\\
 & \leq  \frac{1}{\normL^2}\frac{O(1)}{  ( \mathsf{\kappa}_{\mathrm{eff}} +  \tilde{t} )^2 }   \sum_{i = 1}^{\ru \wedge \rs}     \left(  \lambda_{\mathrm{max}}( \Z_{1:\ru} \Z_{1:\ru}^\top )  \exp(\tilde{t}   \lambda_i) \wedge   ( \mathsf{\kappa}_{\mathrm{eff}} +  \tilde{t} )   \frac{ \lambda_i     \exp( \tilde{t}  \lambda_i) }{    \exp( \tilde{t}  \lambda_i) -   1 } \right). ~~~~  \label{eq:offdiagonalbound}
}
For the heavy-tailed case ($\alpha \in [0,0.5)$),   
\eq{
\eqref{eq:offdiagonalbound}  \leq \frac{O_{\alpha, \varphi, \beta}(1)}{ r \log^2 d} \sum_{ i \leq r}   \mathbbm{1}\{\tilde{t} \lambda_i \leq \log \tfrac{d}{\rs} \} +   \frac{O_{\alpha, \varphi, \beta}(1)}{ r^{1 - \alpha} \log d}  \sum_{i \leq r} \lambda_i  \mathbbm{1} \{\tilde{t} \lambda_i > \log \tfrac{d}{\rs} \}  = o_d(1).
}
For the light-tailed case ($\alpha > 0.5$),    
\eq{
\eqref{eq:offdiagonalbound}  \leq \frac{O_{\alpha, \rs, \beta}(1)}{  \log^2 d} \sum_{ i \leq \rs}   \mathbbm{1}\{\tilde{t} \lambda_i \leq \log \tfrac{d}{\ru \rs} \} +   \frac{O_{\alpha, \varphi, \beta}(1)}{  \log d}  \sum_{i \leq \rs} \lambda_i  \mathbbm{1} \{\tilde{t} \lambda_i > \log \tfrac{d}{\rs \ru}  \}  = o_d(1).
}
\paragraph{Lower-right terms:}  By using Matrix-Inversion lemma, we have
\eq{
- \D_2 + \D_2 \S_{22} \D_2 = - \Z_2 \left(  \frac{\sqrt{\rs}}{\normL} \bs{I}_{\rs} +    \Z_{1:\ru}^\top   \D_1^{-1}     \Z_{1:\ru} +  \Z_2^\top   \D_2^{-1}     \Z_2    \right)^{-1} \Z_2^\top.
}
We observe that  $ \lambda_{\mathrm{max}}(\D_2 ) \preceq \frac{1}{t} $ and  $\frac{\sqrt{\rs}}{\normL} \asymp  \mathsf{\kappa}_{\mathrm{eff}}$.   By using $\tilde{t} \coloneqq  t  \mathsf{\kappa}_{\mathrm{eff}} \log \tfrac{d}{\rs}$, we have
\eq{
  \frac{1}{\sqrt{\rs}}  \bs{G}_{W,22}(t \mathsf{T}_{\mathrm{eff}} )  & = \frac{1}{\normL}   \Z_2 \left(  \frac{\sqrt{\rs}}{\normL} \bs{I}_{\rs} +    \Z_{1:\ru}^\top   \D_1^{-1}     \Z_{1:\ru} +  \Z_2^\top   \D_2^{-1}     \Z_2    \right)^{-1} \Z_2^\top  \\
 &  \preceq \frac{O(1)}{\normL}    \Z_2 \left(   \mathsf{\kappa}_{\mathrm{eff}}    \bs{I}_{\rs}  +  \tilde{t} \Z_2^\top    \Z_2    \right)^{-1} \Z_2^\top. \label{eq:lowerrightbound}
}
For the heavy-tailed case ($\alpha \in [0,0.5)$),   
\eq{
\norm{ \eqref{eq:lowerrightbound} }_F^2 \leq  \frac{O_{\alpha, \varphi, \beta}(1)  }{ r^{- 2 \alpha}\log^2 d} \frac{1}{t^2 r^{2 \alpha} \log^2 d} = o_d(1).
}
For the light-tailed case ($\alpha > 0.5$),    
\eq{
\norm{ \eqref{eq:lowerrightbound} }_F^2 \leq   O_{\alpha, \rs, \beta}(1) \frac{1}{t^2 \log^2 d}  = o_d(1).
}
\end{proof}

\begin{proposition}
\label{prop:riskhighdimensional}
For some $c >0$, let    $\G(0) \coloneqq  \frac{c}{t} \Z_{1:\ru} \Z_{1:\ru}^\top$ and
\eq{
t \in \begin{cases}
\big(0,  \tfrac{(\rus+1)^{\alpha}}{\kappa_{\text{eff}}} \big), &\hspace{-0.3em} \alpha \in [0,0.5) \\[0.5em]
\big(0,  \tfrac{(\rus+1)^{\alpha}}{\kappa_{\text{eff}}} \big) \setminus \{ j^{\alpha}: j \in \N \}, &\hspace{-0.3em} \alpha > 0.5,
\end{cases}\quad 
d \geq \begin{cases}
\Omega_{\varphi, \alpha}(1)& \alpha \in [0,0.5) \\
\Omega_{\rs,\alpha}(1), & \alpha> 0.5.
\end{cases} 
}
  We define
\eq{
\mathsf{Err}_{\ru}(t_{\mathrm{sc}}) &  \coloneqq \frac{  -  \La_{\mathrm{eff}}  \exp ( - t  \La_{\mathrm{eff}}   ) }{\bs{I}_{\ru} - \exp  (  - t   \La_{\mathrm{eff}}  )} \\[0.8em]
& +   \frac{    \La_{\mathrm{eff}}   \exp (  - 0.5 t  \La_{\mathrm{eff}}  ) }{\bs{I}_{\ru} - \exp ( - t   \La_{\mathrm{eff}} )}  \!   \left(   \frac{  \La_{\mathrm{eff}}\exp(- t  \La_{\mathrm{eff}} )}{\bs{I}_{\ru} -\exp(- t  \La_{\mathrm{eff}})}    +  \G(0)  \right)^{-1} \!\!   \frac{     \La_{\mathrm{eff}}  \exp  (  - 0.5 t  \La_{\mathrm{eff}} ) }{\bs{I}_{\ru} - \exp ( - t   \La_{\mathrm{eff}} )} .
}
$ \mathcal{G}_{\text{init}}$ implies that
\eq{
\frac{\norm{ \mathsf{Err}_{\ru}\! \left(t \mathsf{T}_{\mathrm{eff}}  \right)}_F^2}{\normL^2}  = 1 -   \frac{1}{\normL^2}  \sum_{j= 1}^{\rus} \lambda_j^2   \mathbbm{1} \{   t \mathsf{\kappa}_{\mathrm{eff}} \geq  \tfrac{1}{\lambda_j} \}     + o_d(1).  
}
\end{proposition}

\begin{proof}
We let
\eq{
 \Z_{1:\ru}   \Z_{1:\ru}^\top \coloneqq \begin{bmatrix}
  \Z_{1:\rus}   \Z_{1:\rus}^\top  &  \Z_1   \Z_2^\top \\
   \Z_2   \Z_1^\top &  \Z_2   \Z_2^\top
 \end{bmatrix}, \qquad 
 \La_{\mathrm{eff}}  \coloneqq    \begin{bmatrix}
  \La_{\mathrm{eff},11}   &  0 \\
 0 &    \La_{\mathrm{eff},22} 
 \end{bmatrix},
}
where  $  \La_{\mathrm{eff},11}   \in \R^{ \rus  \times \rus }$,     $\Z_{2}  \in  \R^{ (\ru -  \rus ) \times \rs}$.  Let
 \eq{
  \bs{\Gamma}(t_{\mathrm{sc}})   \coloneqq  \left(   \frac{  \La_{\mathrm{eff}}\exp(- t  \La_{\mathrm{eff}} )}{\bs{I}_{\ru} -\exp(- t  \La_{\mathrm{eff}})}    +    \frac{c}{t} \Z_{1:\ru} \Z_{1:\ru}^\top \right)^{-1}  ~~ \text{and} ~~   \mathsf{Err}(t_{\mathrm{sc}}, \bs{\Gamma} )  \coloneqq \mathsf{Err}_{\ru}(t_{\mathrm{sc}} ). 
 }
By using Proposition \ref{prop:matrixyounginequality} and the events \ref{event:htZub} and \ref{event:ltZub},
\eq{
  \bs{\Gamma}(t_{\mathrm{sc}})   \succeq     \underline{ \bs{\Gamma} }(t_{\mathrm{sc}})  \!\coloneqq
 \!\! \begin{bmatrix}
 \displaystyle \frac{10c}{ t } \frac{\rs}{d}   \bs{I}_{\rus}  +   \frac{   \La_{\mathrm{eff},11}  \exp(- t   \La_{\mathrm{eff},11}  )}{\bs{I}_{\rus} -\exp(- t   \La_{\mathrm{eff},11}  )}   &    \hspace{-8.5em} 0\\[2.5em]
 \hspace{-7em}  0  &  \displaystyle    \hspace{-8.5em} \frac{2 c}{t}  \frac{\rs \log^{2.5} d}{d} \bs{I}_{\ru - \rus} +   \frac{   \La_{\mathrm{eff},22}  \exp(- t   \La_{\mathrm{eff},22}  )}{\bs{I}_{\ru -\rus} -\exp(- t   \La_{\mathrm{eff},22}  )} 
 \end{bmatrix}^{-1} \!\!\!\!\!.
}
For the upper bound,  by using $\varepsilon = \frac{1}{c \log^{3} d}$ in Proposition \ref{prop:schurlowerbound},   and the events \ref{event:htZlb},  \ref{event:ltZlb} and \ref{event:htZub},  \ref{event:ltZub},   we have  for  $t  \leq  (\rus+1)^{\alpha} \log \frac{d}{\rs}$,
\eq{
  \bs{\Gamma}(t_{\mathrm{sc}})  \preceq     \overline{ \bs{\Gamma} }(t_{\mathrm{sc}}) \! \coloneqq \!\!   \begin{bmatrix}
 \displaystyle  \frac{\nicefrac{0.2}{ t } }{   \log^4 d}  \frac{\rs}{d}  \bs{I}_{\rus} \! +  \!   \frac{   \La_{\mathrm{eff},11}  \exp(- t   \La_{\mathrm{eff},11}  )}{\bs{I}_{\rus} -\exp(- t   \La_{\mathrm{eff},11}  )}   & \hspace{-8.5em}  0\\[2.5em]
 \hspace{-8.5em}  0  &   \displaystyle    \hspace{-8.5em}  \frac{\nicefrac{- 1.1}{ t } }{    \log^{\nicefrac{1}{2}} d} \frac{\rs}{d} \bs{I}_{\ru -\rus}  \! + \!     \frac{   \La_{\mathrm{eff},22}  \exp(- t   \La_{\mathrm{eff},22}  )}{\bs{I}_{\ru -\rus} -\exp(- t   \La_{\mathrm{eff},22}  )} 
 \end{bmatrix}^{-1} \!\! \!\! \!\! .
}
Therefore,  for $t  < \frac{(\rus+1)^{\alpha}}{\kappa_{\text{eff}}}$, by Corollary \ref{cor:asymptoticaux}, we have  
\eq{
 \frac{\norm{ \mathsf{Err}(t \mathsf{T}_{\mathrm{eff}},    \bs{\Gamma}  )}_F^2}{\normL^2} \geq \frac{\norm{ \mathsf{Err}(t \mathsf{T}_{\mathrm{eff}},   \underline{ \bs{\Gamma} } )}_F^2}{\normL^2} & =  \frac{1}{\normL^2}  \sum_{j= 1}^{\ru} \lambda_j^2   \mathbbm{1} \{   \tfrac{1 }{\lambda_j} > t \mathsf{\kappa}_{\mathrm{eff}} \}    + o_d(1) \\
 & =  1 -   \frac{1}{\normL^2}  \sum_{j= 1}^{\rus} \lambda_j^2   \mathbbm{1} \{   t \mathsf{\kappa}_{\mathrm{eff}} \geq  \tfrac{1 }{\lambda_j} \}     + o_d(1).
}
On the other hand,   by Corollary \ref{cor:asymptoticaux}, we have  
\eq{
 \frac{\norm{ \mathsf{Err}(t \mathsf{T}_{\mathrm{eff}},    \bs{\Gamma}  )}_F^2}{\normL^2} & \leq \frac{\norm{ \mathsf{Err}(t \mathsf{T}_{\mathrm{eff}},   \overline{ \bs{\Gamma} } )}_F^2}{\normL^2} \\
 & =  \frac{1}{\normL^2}  \sum_{j= 1}^{\rus} \lambda_j^2   \mathbbm{1} \{   \tfrac{1 }{\lambda_j}  > t \mathsf{\kappa}_{\mathrm{eff}} \}      + \frac{1}{\normL^2}  \sum_{j=\rus+ 1}^{\ru} \lambda_j^2     + o_d(1)  \\
 & =  1 -   \frac{1}{\normL^2}  \sum_{j= 1}^{\rus} \lambda_j^2   \mathbbm{1} \{   t \mathsf{\kappa}_{\mathrm{eff}} \geq  \tfrac{1 }{\lambda_j} \}     + o_d(1).
}
\end{proof}

\begin{proposition}
\label{prop:riskmain}
$\mathcal{G}_{\text{init}}$ implies that
\eq{
\frac{\norm{ \mathsf{Err}  \left(t \mathsf{T}_{\mathrm{eff}}  \right)}_F^2}{\normL^2}  = 1 -   \frac{1}{\normL^2}  \sum_{j= 1}^{\rus} \lambda_j^2   \mathbbm{1} \{  t \mathsf{\kappa}_{\mathrm{eff}} \geq  \tfrac{1 }{\lambda_j}    \}   + o_d(1), 
}
for
\eq{
t \in \begin{cases}
(0,  \infty ), & \alpha \in [0,0.5) \\[0.5em]
(0,   \infty ) \setminus \{ j^{\alpha}: j \in \N \}, & \alpha > 0.5,
\end{cases}\quad 
d \geq \begin{cases}
\Omega_{\varphi, \alpha}(1)& \alpha \in [0,0.5) \\
\Omega_{\rs,\alpha}(1), & \alpha> 0.5.
\end{cases} 
}
\end{proposition}
 
\begin{proof}
We recall that
\eq{
\D_1 \coloneqq \frac{  \bs{\Lambda}_{\mathrm{e},11} \exp(- t  \bs{\Lambda}_{\mathrm{e},11} )}{\bs{I}_{\ru} -\exp(- t  \bs{\Lambda}_{\mathrm{e},11}) }, \quad   \D_2 \coloneqq     \frac{  \bs{\Lambda}_{\mathrm{e},22} \exp(- t  \bs{\Lambda}_{\mathrm{e},22} )}{\bs{I}_{d - \ru} -\exp(- t  \bs{\Lambda}_{\mathrm{e},22}) },   \quad  \Z \coloneqq \begin{bmatrix}
\Z_{1:\ru} \\
\Z_2
\end{bmatrix},
}
and
\eq{
 \mathsf{Err}(t_{\mathrm{sc}})   & =     \frac{  -  \Lae \exp  ( - t \Lae ) }{\bs{I}_d - \exp  (  - t  \Lae )}  \\
 & +  \frac{    \Lae  \exp (  - 0.5 t \Lae  ) }{\bs{I}_d - \exp  ( - t \Lae )}     \bigg(  \frac{\normL}{\sqrt{\rs}} \Z \Z^\top +   \frac{    \Lae \exp ( - t \Lae  ) }{\bs{I}_d - \exp  (  - t  \Lae )}    \bigg)^{-1}   \frac{  \Lae  \exp  (   - 0.5 t \Lae  ) }{\bs{I}_d - \exp  (   -   t \Lae  )} \\[0.3em]
 &= \begin{bmatrix}
 \mathsf{Err}_{\ru}(t_{\mathrm{sc}})  &  \frac{- 1}{\sqrt{\rs}}   \bs{G}_{W,12}(t_{\mathrm{sc}})   \\
  \frac{- 1}{\sqrt{\rs}}     \bs{G}_{W,12}(t_{\mathrm{sc}})   &  \frac{- 1}{\sqrt{\rs}}  \bs{G}_{W,22}(t_{\mathrm{sc}})
 \end{bmatrix}.
}
Note that  by Proposition \ref{prop:auxbound},  in the time scale we consider we have  $\frac{1 - o_d(1)}{t \mathsf{T}_{\mathrm{eff}}} \leq \lambda_{\mathrm{min}}(\D_2) \leq \lambda_{\mathrm{max}}(\D_2) \leq \frac{1}{t \mathsf{T}_{\mathrm{eff}}}$. The by using \ref{event:Zbound},
\eq{
 & \mathsf{Err}_{\ru}(t_{\mathrm{sc}})   =     \frac{  -  \La_{\mathrm{eff}}  \exp  ( - t  \La_{\mathrm{eff}}  ) }{\bs{I}_{\ru} - \exp  (  - t  \La_{\mathrm{eff}} )}  \\
 & +  \frac{  \La_{\mathrm{eff}}   \exp (  - 0.5 t \La_{\mathrm{eff}}   ) }{\bs{I}_{\ru} - \exp  ( - t \La_{\mathrm{eff}}  )}     \bigg(  \frac{\Theta(1)}{\frac{\normL}{\sqrt{\rs}} +  t} \Z_{1:\ru} \Z_{1:\ru}^\top +   \frac{    \Lae \exp ( - t \La_{\mathrm{eff}}   ) }{\bs{I}_{\ru} - \exp  (  - t  \La_{\mathrm{eff}}  )}    \bigg)^{-1} \!\!   \frac{  \Lae  \exp  (   - 0.5 t \La_{\mathrm{eff}}   ) }{\bs{I}_{\ru} - \exp  (   -   t  \La_{\mathrm{eff}}   )}.
}
By Propositions \ref{prop:riskoffdiagonal} and \ref{prop:riskhighdimensional}, we have
\eq{
\frac{\norm{ \mathsf{Err}  \left(t \mathsf{T}_{\mathrm{eff}}  \right)}_F^2}{\normL^2}  = 1 -   \frac{1}{\normL^2}  \sum_{j= 1}^{(\rs \wedge r)} \lambda_j^2   \mathbbm{1} \{  t \mathsf{\kappa}_{\mathrm{eff}} \geq  \tfrac{1 }{\lambda_j}    \}   + o_d(1), 
}
for
\eq{
t \in \begin{cases}
\big(0,  \tfrac{(\rus+1)^{\alpha}}{\kappa_{\text{eff}}} \big), & \alpha \in [0,0.5) \\[0.5em]
\big(0,  \tfrac{(\rus+1)^{\alpha}}{\kappa_{\text{eff}}} \big) \setminus \{ j^{\alpha}: j \in \N \}, & \alpha > 0.5. \label{eq:errcharacterization0}
\end{cases} 
}
To extend the limit for $t > \tfrac{(\rus+1)^{\alpha}}{\kappa_{\text{eff}}},$ we observe that
\begin{itemize}[leftmargin=*]
\item   $\norm{ \mathsf{Err}  \left(t \right)}_F^2$  non increasing since it corresponds to the objective under \eqref{eq:gf}.
\item   The global optimum of  \eqref{eq:gf}  and the previous item with  \eqref{eq:errcharacterization0} guarantees that for  $t > \tfrac{(\rus+1)^{\alpha}}{\kappa_{\text{eff}}}$,
\eq{
1  \! -   \!  \frac{1}{\normL^2}  \!\!  \sum_{j= 1}^{(\rs \wedge r)} \!\!  \lambda_j^2   \mathbbm{1} \{  t \mathsf{\kappa}_{\mathrm{eff}} \geq  \tfrac{1 }{\lambda_j}    \}   \! \leq \! \frac{\norm{ \mathsf{Err}  \left(t \mathsf{T}_{\mathrm{eff}}  \right)}_F^2}{\normL^2}  \! \leq  \! 1   \! -   \!  \frac{1}{\normL^2} \!\! \sum_{j= 1}^{(\rs \wedge r)} \!\!  \lambda_j^2   \mathbbm{1} \{  t \mathsf{\kappa}_{\mathrm{eff}} \geq  \tfrac{1 }{\lambda_j}    \} + o_d(1).
}
\end{itemize}
Therefore,  the statement extends to  $t > \tfrac{(\rus+1)^{\alpha}}{\kappa_{\text{eff}}}$.
\end{proof}

\subsection{Proof of Theorem \ref{thm:sgdresult}}
We redefine the time-scale and effective-width as:
\eq{
t_{\mathrm{sc}} = t \sqrt{\rs} \normL, \quad \mathsf{\kappa}_{\mathrm{eff}}  = \begin{cases}
r^{\alpha}/\eta, & \alpha \in [0,0.5) \\
1/\eta, & \alpha > 0.5
\end{cases}, \qquad   
\mathsf{T}_{\mathrm{eff}} =  \mathsf{\kappa}_{\mathrm{eff}}  \sqrt{\rs} \normL \log \nicefrac{d}{\rs}.  \label{def:discretetimescale}
}
and 
\eq{
& \ru = \begin{cases}
r, & \alpha \in [0,0.5) \\
\ceil{\log^{2.5} d}, & \alpha > 0.5
\end{cases}, \qquad 
\rus\coloneqq \begin{cases}
\floor{ \rs (1 - \log^{\nicefrac{-1}{8}} d) \wedge r} , & \alpha \in [0,0.5) \\
\rs, & \alpha > 0.5.
\end{cases} \label{def:effwidthdiscrete}
}
We consider the learning rate and fine-tuning sample size given as
\eq{
\eta \asymp \frac{1}{d} \begin{cases}
\frac{1}{r^{\alpha} \log^{20} (1 + d/\rs)}, & \alpha \in [0,0.5) \\
\frac{1}{\ru^{4  \alpha + 3}  \log^{18} d}, & \alpha > 0.5
\end{cases} ~~ \text{and} ~~ N_{\mathrm{Ft}} \asymp \rs^2 \log^5 d.
}
We define the effective learning rate $\lr$ and the hitting time $\mathcal{T}_{\mathrm{hit}}$ as follows:
\eq{
\lr \coloneqq \frac{\eta/2}{\normL \sqrt{\rs}}, \quad  \mathcal{T}_{\mathrm{hit}} \coloneqq \Bigg \{ t \geq 0 ~ \Big  \vert ~      1 - \frac{ \norm{\La^{\frac{1}{2}} \G_{t  } \La^{\frac{1}{2}}  }_F^2}{\normL^2}   \leq \frac{1}{\normL^2} \sum_{j = (\rs \wedge 1) + 1}^r \lambda_j^2 + \frac{10}{\log^{\frac{1}{8}} d}  \Bigg \}.
}
We note that bounding $\mathcal{T}_{\mathrm{hit}} $ suffices to derive sample complexity since by Proposition  \ref{prop:finetuningappendix}, we have
\eq{
R(\bs{W}_t^{\mathrm{final}}) \leq 1 - \frac{ \norm{\La^{\frac{1}{2}} \G_t \La^{\frac{1}{2}}}_F^2}{\normL^2} + \frac{O(1)}{\log d}.
}
The main statement of this part is as follows:

\begin{proposition}
\label{prop:sgdtheoremproof}
The intersection of the following events hold with probability $1 - o_d(1/d^2) - \Omega(1/\rs^2)$:
\begin{enumerate}[leftmargin=*]
\item We have  
\eq{
\mathcal{T}_{\mathrm{hit}}  \leq  \begin{cases}
  \frac{1}{2 \lr }   \left(    \rs  \big(1 -  \log^{\nicefrac{-1}{8}} d  \big) \wedge r  \right)^{\alpha} \log \left( \frac{ 20 d \log^{\frac{3}{4}} (1 + d/\rs)}{\rs}   \right),  & \alpha \in [0,0.5) \\
 \frac{1}{2 \lr }  \rs^{\alpha} \log \left( 20 \frac{d \log^{\nicefrac{3}{4}} d}{\rs}   \right),   &  \alpha > 0.5.
\end{cases}
}
\item  For $t > 0$, 
\begin{itemize}
 \item $\mathcal{A}(t \mathsf{T}_{\mathrm{eff}}, \bs{\theta}_j  ) = \mathbbm{1} \{ \eta t \mathsf{\kappa}_{\mathrm{eff}} \geq \tfrac{1}{\lambda_j} \} +o_d(1)$ for $t \not = \lim_{d \to \infty}\tfrac{1}{\eta \mathsf{\kappa}_{\mathrm{eff}}  \lambda_j}$ and  $j \leq \rus$.
 \item  $\normL^2 - \norm{\La^{\frac{1}{2}} \G_{t \mathsf{T}_{\mathrm{eff}} } \La^{\frac{1}{2}}  }_F^2 = 1 - \sum_{j = 1}^{\rus} \lambda_j^2 \mathbbm{1} \{ \eta t \mathsf{\kappa}_{\mathrm{eff}} \geq \tfrac{1}{\lambda_j} \} +o_d(\normL^2)$. 
\end{itemize}
\end{enumerate}
\end{proposition}

\begin{proof}
By using Lemma \ref{lem:goodevents} and Corollary \ref{cor:Tbadbound}, we have   with probability $1 - o_d(1/d^2) - \Omega(1/\rs^2)$:
\eq{
\mathcal{T}_{\text{bad}} \geq   \begin{cases}
\frac{1}{2 \lr}  \left( \rs   \big(1 -  \log^{\frac{- 1}{2}} \!\! d   \big) \wedge r \right)^{\alpha}   \log \left( \frac{d\log^{1.5} d}{\rs} \right),  & \alpha \in [0,0.5) \\[0.5em]
\frac{1}{2 \lr}  \rs^{\alpha}   \log \left( \frac{d\log^{1.5} d}{\rs} \right),    &  \alpha  > 0.5,
\end{cases}
}
where  $\mathcal{T}_{\text{bad}}$  is defined in \eqref{def:Tbad}.  Given the lower bound, by Proposition \ref{prop:boundstoppedprocess},  and the third item of Proposition \ref{prop:lbsystemanalysis}, we have the first item. 

\smallskip
For the second item,  by Proposition \ref{prop:boundstoppedprocess}, and  Proposition \ref{prop:lbsystemanalysis} (for the lower bound) and    Proposition \ref{prop:goodeventres}  (for the upper bound),   we have for $\rus \times \rus$ dimensional top left submatrices $\G_{t,11}$ and $\bs{\Lambda}_{11}$,
\eq{
\frac{1 }{  \frac{1.2}{C_{\text{lb}}}  \frac{d}{\rs}  \exp \left( - 2 t \lr  \bs{\Lambda}_{11}  \right)   + 1 }  - o_d(1)   \preceq \G_{t,11}  \preceq \left(  \frac{C_{\text{ub}} \rs}{d}  \exp \left(  2  \lr t  \bs{\Lambda}_{11}  \right) \wedge 1 \right) +    o_d(1), \label{eq:alignmentdiscretedyn}
}
and  
\eq{
& \norm{\La}_F^2 -  \norm{ \La^{\frac{1}{2}} \G_t \La^{\frac{1}{2}} }_F^2   \geq  \sum_{i = 1}^{\ru}  \lambda_i^2 \Big(  1 - \tfrac{ C_{\text{ub}} \rs}{d}   \exp \left(  2  \lr t \lambda_i  \right) \Big)_{+}  -  o_d(\norm{\La}_F^2) \label{eq:risklowerbounddisc} \\
& \norm{\La}_F^2 -  \norm{ \La^{\frac{1}{2}} \G_t \La^{\frac{1}{2}} }_F^2  \leq  \!\!\!\!\! \sum_{i  =  (\rus     \wedge r) + 1 }^r  \!\!\!\!\!  \lambda_i^2 + \sum_{i = 1}^{\rus} \lambda_i^2  \Bigg(   1 -  \frac{1  }{  \frac{1.2}{C_{\text{lb}}}  \frac{d}{\rs}  \exp \left( - 2 t \lr  \lambda_i  \right)   + 1 }    \Bigg)^2 + o_d(\norm{\La}_F^2),  \label{eq:riskupperbounddisc}
  }
for
\eq{
t \leq 
\begin{cases}
  \frac{1}{2 \lr }   \left(    \rs  \big(1 -  \log^{\nicefrac{-1}{8}} d  \big) \wedge r  \right)^{\alpha} \log \left( \frac{d \log^{1.5} d}{\rs}   \right),  & \alpha \in [0,0.5) \\
 \frac{1}{2 \lr }  \rs^{\alpha} \log \left( \frac{d \log^{1.5} d}{\rs}   \right),   &  \alpha > 0.5.
\end{cases} \label{eq:tlimhere}
}
where
\eq{
C_{\text{ub}} =  \begin{cases}
2.5  \left(1 + \frac{1}{\sqrt{\varphi}} \right)^2, & \alpha \in[0,0.5) \\
15, & \alpha > 0.5,
\end{cases} \quad  
C_{\text{lb}} = \frac{1}{15} \begin{cases}
 \log^{\nicefrac{- 1}{2}} d , & \alpha \in [0,0.5) \\
\rs^{-6}, & \alpha > 0.5.
\end{cases}
}
The high-dimensional limits of the alignment and risk up to the time horizon in \eqref{eq:tlimhere} follow from \eqref{eq:alignmentdiscretedyn} (for the alignment), and from \eqref{eq:riskupperbounddisc} (for the risk) by Proposition \ref{prop:asymptoticlimit}. Proposition \ref{prop:stable} then allows us to extend these results beyond the time limit in \eqref{eq:tlimhere}, yielding the full statement.  
\end{proof}

\subsection{Proof of Corollary \ref{cor:asympriskcont} and Corollary \ref{cor:asympriskdis}}

\label{app:scaling-proof}

Finally, we derive the scaling of prediction risk under power-law second-layer coefficients.  Since Corollary \ref{cor:asympriskdis} is a rescaled version of Corollary \ref{cor:asympriskcont}, we will only consider the latter. 

\begin{proof}[Proof of   Corollary \ref{cor:asympriskcont}]
We will prove heavy   and light-tailed cases  separately. 
\paragraph{Heavy-tailed case ($\alpha \in [0,0.5)$):}   
We define $C \coloneqq \Big( \frac{  (1 - \beta) \sqrt{\varphi}}{\sqrt{1 - 2 \alpha}} \Big)^{\frac{1}{\alpha}}$. We first fix a $(C \varphi)^{\alpha} > t > 0$. By Proposition~\ref{prop:riskmain}, for any  $d \geq \Omega_{\varphi, \alpha}(1)$, we have with probability at least $1 - o(1/d^2)$
\eq{
\mathcal{R}(   t r  \log d    ) = \underbrace{ 1 - \frac{1}{\normL^2} \sum_{i = 1}^{\rs} \lambda_j^2 \mathbbm{1}\{ \tfrac{t r^{\alpha}}{C^{\alpha}} \geq \tfrac{ 1 \pm o_d(1)   }{\lambda_j} \} }_{\coloneqq \mathcal{R}_d( (Ct)^{\frac{1}{\alpha}})}  + o_d(1)
}
where we define $\mathcal{R}_d( (Ct)^{\frac{1}{\alpha}})$ to isolate the main term and make the dependence on the ambient dimension explicit.  By using $\lambda_j = j^{-\alpha}$ in the indicator function, we can rewrite
\eq{
\mathcal{R}_d(t ) =  1 -  \frac{1}{\normL^2} \sum_{i = 1}^{\rs} \lambda_j^2 \mathbbm{1}\{  (1 \pm o_d(1)) t  \geq \tfrac{  j     }{r } \}
}
We define a sequence of measures supported on $\{j/r: j \in [r] \}$, where $\mu_d\{\frac{j}{r}\} \propto j^{-2\alpha}$ for $j = 1, \cdots, r$. We observe the following:
\begin{itemize}
    \item $\mu_d$ converges weakly weakly to a limiting probability measure $\mu$   supported on  $[0,1]$, with cumulative distribution function
    \eq{
    \mu\{ [0, c) \} = \begin{cases}
        c^{1 - 2 \alpha}, & c < 1 \\
        1 & x  \geq 1
    \end{cases}
    }
    \item Moreover, the risk can be expressed as
    \eq{
    \mathcal{R}_d(t ) = 1 - (1 \pm o_d(1) ) \E_{X \sim \mu_d}[\mathbbm{1} \{ (1 \pm o_d(1) ) t \wedge  \varphi  \geq X \}   ]
    }
\end{itemize}
By the Portmanteau theorem \cite{Durrett1993ProbabilityTA}, it follows that for any fixed  $t \in (0,\varphi)$,
\eq{
\mathcal{R}_d(t ) \to 1 -t^{1 - 2\alpha}.
}
almost surely as $d \to \infty$.  The almost sure convergence follows from the Borel-Cantelli lemma \cite{Durrett1993ProbabilityTA} applied to the failure probabilities. 

To extend this result to $t \geq   \varphi$, we observe that by \eqref{eq:gf},   $\mathcal{R}_d(t )$  is non-increasing and  $\inf_{t \geq 0} \mathcal{R}_d(t )  \geq (1 - \varphi^{1 - 2\alpha})_{+} - o_d(1)$. Hence,    for all $t > 0$, we obtain
\eq{
\mathcal{R}_d(t) =   (1 -t^{1 - 2\alpha})_+ \vee  (1 -\varphi^{1 - 2\alpha})_+
}
The desired result for a fixed 
$t > 0$ follows by a change of variable. Finally, since the risk curves are continuous in $t$, the almost sure convergence extends to all $t > 0$   pointwise.

\paragraph{Light-tailed case ($\alpha > 0.5$):}  For this part, we consider the probability space conditioned on $\mathcal{G}_{init}$ which holds with probability at least $1 - o(1/\rs^2)$.  We define 
\eq{
\mathcal{Z} \coloneqq \sum_{j = 1}^{\infty} j^{- 2 \alpha}, \quad 
 C \coloneqq  (\rs \mathcal{Z})^{\frac{1}{2\alpha}}. 
 }
 We first fix a $t \in \big( 0, (C \rs)^{\alpha} \big)  \setminus \{ j^{\alpha} : j \in \N\}$. 
By Proposition~\ref{prop:riskmain}, for any  $d \geq \Omega_{\rs, \alpha}(1)$, we have ,
\eq{
\mathcal{R}(   t   \log d    ) = \underbrace{ 1 - \frac{1}{\normL^2} \sum_{i = 1}^{\rs} \lambda_j^2 \mathbbm{1}\{ \tfrac{t}{C^{\alpha}} \geq \tfrac{ 1 \pm o_d(1)   }{\lambda_j} \} }_{\coloneqq \mathcal{R}_d( (Ct)^{\frac{1}{\alpha}})}  + o_d(1).
}
By using $\lambda_j = j^{-\alpha}$ in the indicator function, we   rewrite
\eq{
\mathcal{R}_d(t ) =  1 -  \frac{1}{\normL^2} \sum_{i = 1}^{\rs} \lambda_j^2 \mathbbm{1}\{  (1 \pm o_d(1)) t  \geq    j      \}
}
We define a sequence of measures supported on $\N$, where $\mu_d\{j \} \propto j^{-2\alpha}$ for $j = 1, \cdots, \rs$. We observe the following:
\begin{itemize}
    \item $\mu_d$ converges weakly weakly to a limiting probability measure $\mu$   supported on  $\N$, such that $\mu\{ j\} = \frac{j^{-2\alpha}}{\mathcal{Z}}$.
    \item Moreover, the risk can be expressed as
    \eq{
    \mathcal{R}_d(t ) =    \E_{X \sim \mu_d}[\mathbbm{1} \{ (1 \pm o_d(1) t \vee \rs < X \}    ]
    }
\end{itemize}
Since $t \not \in \N$, we have
\eq{
\R_d(t) \to  \mu([t \vee \rs, \infty)).
}
By observing that  $\mu([t, \infty)) \in \Theta(t^{1 - 2\alpha})$, the result follows for a fixed $t \in \big( 0, (C \rs)^{\alpha} \big)  \setminus \{ j^{\alpha} : j \in \N\}$. Since the limit is piecewise continuous and non increasing, it is sufficient to take a union over $t \in \{ 0.5, 1.5, \cdots, \rs + 0.5 \}$ to extend the result for all $t > 0$.
\end{proof}

\section{Details of the Fine-tuning Step}
\label{sec:finetuning}

In this part, we describe how to efficiently solve the empirical risk minimization problem used in the fine-tuning step of Algorithm~\ref{alg:sgd}. Recall that this step aims to find a rotation matrix $\bs{\Omega} \in \mathbb{R}^{\rs \times \rs}$ that aligns the learned features with the teacher directions by minimizing the empirical loss over $N_{\mathrm{Ft}}$ fresh samples:
\begin{equation}
\bs{\Omega}_* = \argmin_{\bs{\Omega} \in \mathbb{R}^{\rs \times \rs}} \sum_{j = 1}^{N_{\mathrm{Ft}}} \mathcal{L}\big(\bs{W}_t \bs{\Omega}; (\bs{x}_{t+j}, y_{t+j}) \big), \label{eq:finetuningobj}
\end{equation}
where each sample loss is given by
\begin{equation}
\mathcal{L}\big(\bs{W}_t \bs{\Omega}; (\bs{x}_{t+j}, y_{t+j}) \big) = \frac{1}{16} \Big( y_{t+j} - \frac{1}{\sqrt{\rs}} \tr\big(\bs{\Omega} \bs{\Omega}^\top \bs{W}_t^\top (\bs{x}_{t+j} \bs{x}_{t+j}^\top - \bs{I}_d) \bs{W}_t \big) \Big)^2.
\end{equation}

Let us define  
$\bs{A}_j \coloneqq \bs{W}_t^\top (\bs{x}_{t+j} \bs{x}_{t+j}^\top - \bs{I}_d) \bs{W}_t.$
We observe that the loss becomes quadratic in the symmetric matrix positive semidefinite matrix $\bs{S} \coloneqq \bs{\Omega} \bs{\Omega}^\top$. Then, the fine-tuning objective reduces to a standard least squares problem over the cone of symmetric matrix positive semidefinite matrices:
\begin{equation}
\bs{S}_* \coloneqq  \argmin_{\substack{\bs{S} \in \mathbb{R}^{\rs \times \rs} \\ \bs{S} = \bs{S}^\top \!\!,  \bs{S} \succeq 0}} \underbrace{ \frac{1}{2 N_{\mathrm{Ft}}} \sum_{j = 1}^{N_{\mathrm{Ft}}} \Big( \sqrt{\rs}y_{t+j} - \tr(\bs{S} \bs{A}_j) \Big)^2 }_{\coloneqq \mathrm{Ft}(\bs{S})}. \label{eq:least-squares-S}
\end{equation}
For the following, we also define the global minimum of  the least square objective in \eqref{eq:least-squares-S} as:
\begin{equation}
\bs{S}_{\mathrm{glob}} \coloneqq \argmin_{\substack{\bs{S} \in \mathbb{R}^{\rs \times \rs} \\ \bs{S} = \bs{S}^\top}} \mathrm{Ft}(\bs{S}). \label{eq:least-squares-Sglob}
\end{equation}

\subsection{Characterizing the Minimum}
Since the fine-tuning objective reduces to a least squares regression problem over symmetric matrices, we can write   
\eq{
\mathrm{Ft}(\bs{S}) = \mathrm{Ft}(\bs{S}_{\mathrm{glob}}) + \tr \big( (\bs{S} - \bs{S}_{\mathrm{glob}}) \mathsf{L}(\bs{S} - \bs{S}_{\mathrm{glob}}) \big)
}
where $\mathsf{L}$ is defined as the linear operator acting on symmetric matrices via
\begin{equation}
\mathsf{L}(\bs{S}) \coloneqq \frac{1}{2 N_{\mathrm{Ft}}} \sum_{j=1}^{N_{\mathrm{Ft}}} \tr(\bs{S} \bs{A}_j) \bs{A}_j,
\end{equation}
which corresponds to the empirical second moment operator associated with the covariates  $\bs{A}_j$. We note that the operator  $\mathsf{L}$ is self-adjoint and positive semi-definite on the space of symmetric matrices, and we can write the characterization in \eqref{eq:least-squares-S} equivalently
\begin{equation}
\bs{S}_* \coloneqq  \argmin_{\substack{\bs{S} \in \mathbb{R}^{\rs \times \rs} \\ \bs{S} = \bs{S}^\top \!\!,  \bs{S} \succeq 0}}  \tr \big( (\bs{S} - \bs{S}_{\mathrm{glob}}) \mathsf{L}(\bs{S} - \bs{S}_{\mathrm{glob}}) \big). \label{eq:least-squares-S2}
\end{equation}
We define the projection on the cone of symmetric positive semi-definite matrices as:
\begin{equation}
\mathsf{\Pi}(\widetilde{\bs{S}})   \coloneqq  \argmin_{\substack{\bs{S} \in \mathbb{R}^{\rs \times \rs} \\ \bs{S} = \bs{S}^\top \!\!,  \bs{S} \succeq 0}} \norm{ \bs{S} - \widetilde{\bs{S}} }_F^2  . \label{eq:least-squares-S2}
\end{equation}
In the following, we will show that the operator $\mathsf{L}$ is close to the identity, and thus,  $\bs{S}_*$ is close to  $\mathsf{\Pi} \circ \mathsf{L} (\bs{S}_{\mathrm{glob}})$. Before proceeding, we make the following observations: 
\begin{itemize}[leftmargin=*]
\item  We observe that by the first-order optimality condition applied in \eqref{eq:least-squares-Sglob}, we have
\eq{
\mathsf{L}( \bs{S}_{\mathrm{glob}} ) =  \frac{\sqrt{\rs}}{2 N_{\mathrm{Ft}}} \sum_{j = 1}^{ N_{\mathrm{Ft}}} y_{t+j} \bs{A}_j. \label{eq:firstorderoptimality}
}
\item By the generalized Pythagorean theorem \cite[Lemma 3.1]{Bubeck2014ConvexOA}, we have
\eq{
\norm{\bs{S}_* - \mathsf{\Pi} \circ \mathsf{L} (\bs{S}_{\mathrm{glob}})}_F^2 
&\leq \norm{ \bs{S}_* - \mathsf{L} (\bs{S}_{\mathrm{glob}})}_F^2 - \norm{ \mathsf{\Pi} \circ \mathsf{L} (\bs{S}_{\mathrm{glob}}) -   \mathsf{L} (\bs{S}_{\mathrm{glob}})}_F^2 \\
& = \mathrm{Ft}(\bs{S}_*) - \mathrm{Ft}\big(\mathsf{\Pi} \circ \mathsf{L} (\bs{S}_{\mathrm{glob}}) \big)  \\
& -  \tr \big( (\bs{S}_* - \mathsf{\Pi} \circ \mathsf{L} (\bs{S}_{\mathrm{glob}}))(\mathsf{L} - \mathsf{Id})(\bs{S}_* + \mathsf{\Pi} \circ \mathsf{L} (\bs{S}_{\mathrm{glob}})) \big), \label{eq:generalpythagorean}
}
where we use $\mathsf{Id}$ to denote  the identity map on symmetric matrices.
\end{itemize}

\subsection{Computing the Minimum}
We define the approximate solution for \eqref{eq:finetuningobj} as:
\eq{
\hat{\bs{\Omega}} \coloneqq  \Big( \mathsf{\Pi} \circ \mathsf{L}(\bs{S}_{\mathrm{glob}}) \Big)^{\frac{1}{2}}, \label{eq:approximation}
}
where  $\bs{S} \to \bs{S}^{1/2}$ denotes the square root operator on symmetric positive semidefinite matrices.   Note that the approximation in \eqref{eq:approximation} can be computed by taking the spectral decomposition of $\mathsf{L}(\bs{S}_{\mathrm{glob}})$ given in \eqref{eq:firstorderoptimality}, which requires $\tilde{O}(d \rs^3)$ including the computation of $\mathsf{L}(\bs{S}_{\mathrm{glob}})$. This is negligible compared to the feature learning phase, whose complexity scales as $O(T d \rs)$.  The following statement shows that $\hat{\bs{\Omega}}$ is sufficiently close to the fine-tuning solution $\bs{\Omega}^*$:

\begin{proposition}
\label{prop:finetuningappendix}
Suppose $N_{\mathrm{Ft}} \geq \rs^2 \log^5 d$. Then, with probability at least $1 - 2 d^{-3}$,  the final risk incurred by $\bs{W}_t \hat{\bs{\Omega}}$ is close to that of the optimal fine-tuning solution:
\begin{equation}
R(\bs{W}_t \hat{\bs{\Omega}}) \leq R(\bs{W}_t \bs{\Omega}_*) + \frac{1}{\log d} \leq 1 - \frac{\norm{\La^{\frac{1}{2}} \G_t \La^{\frac{1}{2}}}_F^2}{\normL^2}  + \frac{O(1)}{\log d}.
\end{equation}
\end{proposition}

\subsubsection{Proof of Proposition \ref{prop:finetuningappendix}}

We define the operator norm of $\mathsf{L}$ as 
\eq{
\norm{\mathsf{L}}_{2} = \sup_{\substack{\bs{S} \in \mathbb{R}^{\rs \times \rs} \\ \bs{S} = \bs{S}^\top}} \norm{\mathsf{L}(\bs{S} )}_F.
}
We consider the intersection of the following events:
\begin{itemize}
    \item $\norm{\mathsf{L} - \mathsf{Id}}_{2} \leq \frac{6}{\sqrt{ \log d }}$
    \item $\Big \lVert  \frac{1}{2 N_{\mathrm{Ft}}} \sum_{j = 1}^{N_{\mathrm{Ft}}} y_{t+j} \bs{A}_j - \frac{1}{\normL} \W_t^\top \Th \La \Th^\top \W_t  \Big \rVert_F^2 \leq \frac{1}{\log d}.$
\end{itemize}
We note that for $d \geq \Omega(1)$ the first item holds with probability $1 - d^{-3}$ by
Proposition \ref{prop:matrixsensingcondition}, where we choose $C = 5$ and $u = \log d$, and the second item holds follows  with probability $1 - d^{-3}$ by Proposition \ref{prop:steinestimator} where we choose $C = 16$. Given the events, we have
\eq{
R(\bs{W}_t \hat{\bs{\Omega}})
& =   \frac{1}{\rs} \norm*{  \mathsf{\Pi} \circ \mathsf{L} (\bs{S}_{\mathrm{glob}}) - \frac{\sqrt{\rs}}{\normL} \W_t^\top \Th \bs{\Lambda} \Th^\top \W_t }_F^2 + \Big( 1 -  \frac{ \norm{\bs{\Lambda}^{\frac{1}{2}}  \G_t \bs{\Lambda}^{\frac{1}{2}}}_F^2}{\normL^2}  \Big)  \\
& \labelrel\leq{ineqq0:ft} \frac{1}{\rs} \norm*{   \mathsf{L} (\bs{S}_{\mathrm{glob}}) - \frac{\sqrt{\rs}}{\normL} \W_t^\top \Th \bs{\Lambda} \Th^\top \W_t }_F^2 + \Big( 1 -  \frac{ \norm{\bs{\Lambda}^{\frac{1}{2}}  \G_t \bs{\Lambda}^{\frac{1}{2}}}_F^2}{\normL^2}  \Big) \\
& \labelrel\leq{ineqq1:ft}  \frac{1}{\log d} +  R(\bs{W}_t \bs{\Omega}_*). 
}
where we use the convexity of the cone of  symmetric positive semi-definite matrices in \eqref{ineqq0:ft} and the second event above in \eqref{ineqq1:ft}. By using \eqref{eq:generalpythagorean}, we have
\eq{
 \norm*{   \bs{S}_* -  \mathsf{\Pi} \circ \mathsf{L} (\bs{S}_{\mathrm{glob}})  }_F 
  \leq   \norm{\mathsf{L} - \mathsf{Id} }_2   \norm*{   \bs{S}_* +  \mathsf{\Pi} \circ \mathsf{L} (\bs{S}_{\mathrm{glob}})  }_F.
  }
 Therefore,  
 \eq{
 \norm*{   \bs{S}_* }_F   \leq  \frac{1 + \norm{\mathsf{L} - \mathsf{Id} }_2}{ 1 - \norm{\mathsf{L} - \mathsf{Id} }_2} \norm{ \mathsf{\Pi} \circ \mathsf{L} (\bs{S}_{\mathrm{glob}})  }_F, 
 }
and thus,
\eq{
 \norm*{   \bs{S}_* -  \mathsf{\Pi} \circ \mathsf{L} (\bs{S}_{\mathrm{glob}})  }_F 
& \leq  \frac{2 \norm{\mathsf{L} - \mathsf{Id} }_2 \norm{ \mathsf{\Pi} \circ \mathsf{L} (\bs{S}_{\mathrm{glob}})  }_F }{ 1  - \norm{\mathsf{L} - \mathsf{Id} }_2}   \labelrel\leq{ineqq2:ft} \frac{15 \rs }{\sqrt{\log d}}
}
where we followed the reasoning in \eqref{ineqq0:ft}-\eqref{ineqq1:ft} to bound $\norm{ \mathsf{\Pi} \circ \mathsf{L} (\bs{S}_{\mathrm{glob}})  }_F$ in \eqref{ineqq2:ft} . Therefore,
\eq{
R(\bs{W}_t \bs{\Omega}_*)
& =  \frac{1}{\rs} \norm*{   \bs{S}_* - \frac{\sqrt{\rs}}{\normL} \W_t^\top \Th \bs{\Lambda} \Th^\top \W_t }_F^2 + \Big( 1 -  \frac{ \norm{\bs{\Lambda}^{\frac{1}{2}}  \G_t \bs{\Lambda}^{\frac{1}{2}}}_F^2}{\normL^2}  \Big) \\
& \leq \frac{2}{\rs} \norm*{   \bs{S}_* -  \mathsf{\Pi} \circ \mathsf{L} (\bs{S}_{\mathrm{glob}})  }_F^2 + \frac{2}{\rs} \norm*{   \mathsf{\Pi} \circ \mathsf{L} (\bs{S}_{\mathrm{glob}})  - \frac{\sqrt{\rs}}{\normL} \W_t^\top \Th \bs{\Lambda} \Th^\top \W_t }_F^2 \\
& \quad + \Big( 1 -  \frac{ \norm{\bs{\Lambda}^{\frac{1}{2}}  \G_t \bs{\Lambda}^{\frac{1}{2}}}_F^2}{\normL^2}  \Big) \\ 
& \leq  \frac{O(1)}{\log d} + \Big( 1 -  \frac{ \norm{\bs{\Lambda}^{\frac{1}{2}}  \G_t \bs{\Lambda}^{\frac{1}{2}}}_F^2}{\normL^2}  \Big).
}

\section{Deferred Proofs for Online SGD}
\label{sec:onlinesgd}
\subsection{Preliminaries}
We consider
\eq{
y_{t +1} = \frac{1}{\norm{\bs{\Lambda}}_F} \sum_{j = 1}^r  \lambda_j  \big( \inner{\bs{\theta}_j}{\bs{x}_{t+1}}^2 -1 \big) + \epsilon_{t+1}  ~~ \text{and} ~~  \hat{y}(\bs{W}_t; \bs{x}_{t+1}) =   \frac{1}{\sqrt{\rs}} \sum_{j = 1}^{\rs} \inner{\bs{w}_{t,j}}{\bs{x}_{t+1}}^2 - 1,
}
where $\norm{\epsilon_{t+1}}_{\psi_2} \leq \sigma$.
We use $\hat{y}_{t+1} \coloneqq \hat{y} (\bs{W}_t;\bs{x}_{t+1})$ and consider
\begin{itemize}
\item The loss function is  $\mathcal{L}(\bs{W}_t; (\bs{x}_{t+1}, y_{t+1})) = \tfrac{  1}{16} \big(y_{t+1}   - \hat{y}_{t+1} \big)^2$
\item The Euclidean gradient is   $\grad   \mathcal{L}(\bs{W}_t)  =   \frac{-1}{4\sqrt{\rs}} \big(y_{t+1} -   \hat{y}_{t+1} \big)  \bs{x}_{t+1} \bs{x}_{t+1}^\top  \W_t$.  Therefore, we have
\eq{
\grad_{\text{St}}  \mathcal{L}(\bs{W}_t)  =  \frac{- 1/4}{\sqrt{\rs}} \left(  \bs{I}_{d} -   \W_t  \W_t^\top  \right)  \big( y_{t+1}  - \hat{y}_{t+1} \big) \bs{x}_{t+1} \bs{x}_{t+1}^\top    \W_t. 
}
\item  We recall that $\G_t = \Th^\top  \W_t \W_t^\top  \Th$. 
\end{itemize}
Then, \eqref{eq:sgd2} reads
\eq{
\widetilde{\bs{W}}_{t+1} & =  \W_t +     \frac{\eta/4}{\sqrt{\rs}} \underbrace{ \left(  \bs{I}_{d} -   \W_t  \W_t^\top  \right)  \big( y_{t+1}  - \hat{y}_{t+1} \big) \bs{x}_{t+1} \bs{x}_{t+1}^\top    \W_t }_{\coloneqq \grad_{\text{St}}  \bs{L}_{t+1}} \\
\bs{W}_{t+1} & =   \widetilde{\bs{W}}_{t+1} \Bigg(  \bs{I}_{\rs} +  \frac{\eta^2/16}{\rs} \underbrace{ \grad_{\text{St}}  \bs{L}_{t+1}^\top \grad_{\text{St}}  \bs{L}_{t+1} }_{\coloneqq \bs{\cP}_{t+1}} \Bigg)^{-1/2}. \tag{SGD} \label{eq:sgd2}
}
We observe that
\eq{
 \frac{\eta^2/16}{\rs} \bs{\cP}_{t+1} & =   \frac{\eta^2/16}{\rs} \big( y_{t+1}  - \hat{y}_{t+1} \big)^2 \W_t^\top  \bs{x}_{t+1} \bs{x}_{t+1}^\top    \left(  \bs{I}_{d} -   \W_t  \W_t^\top  \right)    \bs{x}_{t+1} \bs{x}_{t+1}^\top    \W_t  \\
& =      \frac{\eta^2/16}{\rs} \big( y_{t+1}  - \hat{y}_{t+1} \big)^2  \norm{ \left(  \bs{I}_{d} -   \W_t  \W_t^\top  \right)  \bs{x}_{t+1}  }_2^2   \W_t^\top \bs{x}_{t+1}  \bs{x}_{t+1}^\top  \W_t.
}
Let 
\eq{
c^2_{t+1} \coloneqq   \frac{\eta^2/16}{\rs} \norm{ \bs{\cP}_{t+1} }_2 =    \frac{\eta^2/16}{\rs} \big( y_{t+1}  - \hat{y}_{t+1} \big)^2  \norm{ \left(  \bs{I}_{d} -   \W_t  \W_t^\top  \right)  \bs{x}_{t+1}  }_2^2  \norm{\W_t^\top \bs{x}_{t+1} }_2^2.
}
We define  $\bs{P}_{t+1} \coloneqq    \left(  \bs{I}_{\rs} +  \frac{\eta^2/16}{\rs}  \bs{\cP}_{t+1} \right)^{-1/2}$  and since $ \bs{\cP}_{t+1}$ is $1$-rank, we have
\eq{
\bs{P}_{t+1}^2 = \bs{I}_{\rs} -  \frac{\eta^2/16}{\rs} \frac{\bs{\cP}_{t+1}}{1 +c^2_{t+1} }.  
}
We let 
\eq{
\bs{M}_{t} \coloneqq \Th^\top \W_t ~~\text{and} ~~ \bs{\hat{M}}_{t+1} \coloneqq \Th^\top \bs{\hat{W}}_{t+1}. 
}
We have
\eq{
\bs{\hat{M}}_{t+1}  = \bs{M}_{t}  + \frac{\eta/4}{\sqrt{\rs}}  \Th^\top \grad_{\text{St}}  \bs{L}_{t+1}.
}
By recalling that  $\G_{t} = \bs{M}_{t} \bs{M}_{t}^\top$, we have
\eq{
\G_{t+1}   & = \bs{\hat{M}}_{t+1}  \bs{\hat{M}}_{t+1}^\top + \bs{\hat{M}}_{t+1} (\bs{P}_{t+1}^2 - \bs{I}_{\rs}) \bs{\hat{M}}_{t+1}^\top \\
& =  \G_t + \frac{\eta/4}{\sqrt{\rs}} \bs{M}_{t} \grad_{\text{St}}  \bs{L}_{t+1}^\top   \Th +    \frac{\eta/4}{\sqrt{\rs}}   \Th^\top  \grad_{\text{St}}  \bs{L}_{t+1}  \bs{M}_{t}^\top   +   \frac{\eta^2}{16 \rs}  \Th^\top \grad_{\text{St}}  \bs{L}_{t+1} \grad_{\text{St}}  \bs{L}_{t+1}^\top   \Th   \\
& -    \frac{\eta^2}{16 \rs}   \frac{ \bs{\hat{M}}_{t+1}  \bs{\cP}_{t+1} \bs{\hat{M}}_{t+1}^\top}{1 +c^2_{t+1} }.
}
We have
\eq{
\grad_{\text{St}}  \bs{L}_{t+1}  = \frac{2}{\normL} \left( \bs{I}_d - \W_t \W_t^\top  \right) \Th \bs{\Lambda} \Th^\top \W_t + \big( \grad_{\text{St}}  \bs{L}_{t+1}  - \E_t \left[ \grad_{\text{St}}  \bs{L}_{t+1}  \right] \big), 
}
Therefore,
\eq{
\Th^\top  \grad_{\text{St}}  \bs{L}_{t+1}  \bs{M}_{t}^\top & =   \frac{2}{\normL} \Th^\top  \left( \bs{I}_d - \W_t \W_t^\top  \right) \Th \bs{\Lambda} \Th^\top \W_t \W_t^\top \Th \\
& \quad +   \Th^\top  \big( \grad_{\text{St}}  \bs{L}_{t+1}  - \E_t \left[ \grad_{\text{St}}  \bs{L}_{t+1}  \right] \big) \bs{M}_{t}^\top \\
& =  \frac{2}{\normL} \left(   \bs{I}_{r} - \G_t   \right)  \bs{\Lambda}  \G_t  +   \Th^\top  \big( \grad_{\text{St}}  \bs{L}_{t+1}  - \E_t \left[ \grad_{\text{St}}  \bs{L}_{t+1}  \right] \big)  \bs{M}_{t}^\top.
}
Hence, we have
\eq{
\G_{t+1} & =  \G_t + \frac{\eta/2}{ \normL \sqrt{\rs}} \left(  \bs{\Lambda} \G_t  +  \G_t  \bs{\Lambda} - 2  \G_t \bs{\Lambda}  \G_t  \right) 
  \\
& + \frac{\eta/2}{\sqrt{\rs}}  \sym \left(  \Th^\top \big( \grad_{\text{St}}  \bs{L}_{t+1}  - \E_t \left[ \grad_{\text{St}}  \bs{L}_{t+1}  \right] \big)    \bs{M}_{t}^\top \right)  \\
&  +   \frac{\eta^2}{16 \rs}  \Th^\top \grad_{\text{St}}  \bs{L}_{t+1} \grad_{\text{St}}  \bs{L}_{t+1}^\top   \Th -  \frac{\eta^2}{16 \rs}   \frac{ \bs{\hat{M}}_{t+1}  \bs{\cP}_{t+1} \bs{\hat{M}}_{t+1}^\top}{1 +c^2_{t+1} }
}
On the other hand,
\eq{
&  \bs{\hat{M}}_{t+1}   \bs{\cP}_{t+1} \bs{\hat{M}}_{t+1}^\top   \! = \! \left(  \bs{M}_{t}  \!+ \!\frac{\eta/4}{\sqrt{\rs}}  \Th^\top \grad_{\text{St}}  \bs{L}_{t+1}  \right) \! \bs{\cP}_{t+1}   \!\left(  \bs{M}_{t}  + \frac{\eta/4}{\sqrt{\rs}}  \Th^\top \grad_{\text{St}}  \bs{L}_{t+1}  \right)^\top \\
&    =    \bs{M}_{t} \bs{\cP}_{t+1} \bs{M}_{t}^\top \! + \! \frac{\eta/2}{\sqrt{\rs}} \sym \left( \Th^\top  \grad_{\text{St}}  \bs{L}_{t+1} \bs{\cP}_{t+1}   \bs{M}_{t}^\top  \right)  \! + \!  \frac{\eta^2}{16 \rs }   \Th^\top  \grad_{\text{St}}  \bs{L}_{t+1} \bs{\cP}_{t+1}   \grad_{\text{St}}  \bs{L}_{t+1}^\top   \Th.
}
We collect the higher order terms in a single term defined as follows: 
\eq{
R_{\text{so}}[\G_t] &\coloneqq   \frac{\eta^2}{16 \rs}  \Th^\top  \E_t \left[  \grad_{\text{St}}  \bs{L}_{t+1} \grad_{\text{St}}  \bs{L}_{t+1}^\top   \right]   \Th 
-  \frac{\eta^2}{16 \rs}      \bs{M}_{t} \E_t \left[ \frac{\bs{\cP}_{t+1} }{1 + c_{t+1}^2} \right] \bs{M}_{t}^\top  \\
& -  \frac{\eta^3}{32 \rs^{3/2}} \sym \left( \Th^\top  \E_t \left[ \frac{ \grad_{\text{St}}  \bs{L}_{t+1} \bs{\cP}_{t+1} }{1 + c_{t+1}^2}  \right]  \bs{M}_{t}^\top  \right)  \\
& - \frac{\eta^4}{256 \rs^2}  \Th^\top   \E_t \left[  \frac{ \grad_{\text{St}}  \bs{L}_{t+1} \bs{\cP}_{t+1}   \grad_{\text{St}}  \bs{L}_{t+1}^\top }{1 + c_{t+1}^2} \right]     \Th. 
}
We collect the noise terms in a single term defined as follows:
\eq{
\frac{\eta/2}{\sqrt{\rs}}  \bs{\nu}_{t+1} & \coloneqq \frac{\eta/2}{\sqrt{\rs}}   \sym \left(  \Th^\top  \Big(  \grad_{\text{St}}  \bs{L}_{t+1} - \E_t \left[  \grad_{\text{St}}  \bs{L}_{t+1} \right] \Big)    \bs{M}_{t}^\top \right)  \\
& -  \frac{\eta^2}{16 \rs}    \bs{M}_{t}  \left( \frac{\bs{\cP}_{t+1}}{1 +  c_{t+1}^2 } - \E_t \left[ \frac{\bs{\cP}_{t+1} }{1 + c_{t+1}^2} \right] \right) \bs{M}_{t}^\top  \\
& + \frac{\eta^2}{16 \rs}   \Th^\top \left( \grad_{\text{St}}  \bs{L}_{t+1} \grad_{\text{St}}  \bs{L}_{t+1}^\top    -    \E_t \left[ \grad_{\text{St}}  \bs{L}_{t+1} \grad_{\text{St}}  \bs{L}_{t+1}^\top \right] \right)  \Th   \\
&  -  \frac{\eta^3}{32 \rs^{3/2}} \sym \left( \Th^\top \left( \frac{ \grad_{\text{St}}  \bs{L}_{t+1} \bs{\cP}_{t+1} }{1 + c_{t+1}^2}   -   \E_t \left[ \frac{ \grad_{\text{St}}  \bs{L}_{t+1} \bs{\cP}_{t+1} }{1 + c_{t+1}^2}  \right]  \right) \bs{M}_{t}^\top  \right) \\
& -   \frac{\eta^4}{256 \rs^2}  \Th^\top  \left(  \frac{ \grad_{\text{St}}  \bs{L}_{t+1} \bs{\cP}_{t+1}   \grad_{\text{St}}  \bs{L}_{t+1}^\top }{1 + c_{t+1}^2} -   \E_t \left[  \frac{ \grad_{\text{St}}  \bs{L}_{t+1} \bs{\cP}_{t+1}   \grad_{\text{St}}  \bs{L}_{t+1}^\top }{1 + c_{t+1}^2} \right]   \right)  \Th.
}
With these definitions in hand, we have
\eq{
\G_{t+1}   =  \G_t + \frac{\eta/2}{ \normL \sqrt{\rs}  } \left(  \bs{\Lambda} \G_t  +  \G_t  \bs{\Lambda} - 2  \G_t \bs{\Lambda}  \G_t \right)  + R_{\text{so}}[\G_t] + \frac{\eta/2}{\sqrt{\rs}}   \bs{\nu}_{t+1} . 
}

\subsection{Including second-order terms and monotone bounds}
\label{sec:monotonesec}
For $C > 1$, we define
\begin{align} 
\Llf  \coloneqq  \bs{\Lambda} - C  \normL   \frac{\eta d}{\sqrt{\rs} } \bs{I}_r    ~~ \text{and} ~~   \Luf   \coloneqq  \bs{\Lambda} +  C \normL     \frac{\eta d}{\sqrt{\rs} }  \bs{I}_r.  \label{eq:Llfs} 
\end{align}
We recall the definition of effective learning rate  $\lr =  \frac{\eta/2}{\normL \sqrt{\rs}}$ . 
By Proposition \ref{prop:soterms}, we have
\eq{
 \G_{t+1}  &  \succeq \G_t  
+ \lr \Big(  \Llf   \G_t 
+ \G_t  \Llf- 2 \G_t  \La  \G_t \Big)   - \frac{C}{2} \lr^2 \normL^2 \rs \bs{I}_r + \frac{\eta/2}{\sqrt{\rs}}  \bs{\nu}_{t+1}, \label{eq:sotermsminus}  \\
 \G_{t+1}  & \preceq \G_t   
+ \lr\Big(    \Luf  \G_t 
+ \G_t   \Luf   - 2 \G_t  \La    \G_t \Big)    +  \frac{C}{2} \lr^2 \normL^2 \rs \bs{I}_r + \frac{\eta/2}{\sqrt{\rs}}    \bs{\nu}_{t+1}.  \label{eq:sotermsplus}
}
 
\subsubsection{Heavy tailed case - $\alpha \in [0,0.5)$}
\begin{proposition}
\label{prop:boundforiterationsheavy}
We consider $\alpha \in [0,0.5)$,  $\frac{\rs}{r} \to (0,\infty]$  and   $\eta \ll  \tfrac{1}{d \log^4 d} \sqrt{\tfrac{\rs}{r}}$. We define
\eq{
\bs{V}^{-}_t \coloneqq  2 \La^{\frac{1}{2}}  \G_t   \La^{\frac{1}{2}}  -   \Llf ~~ \text{and} ~~  
\bs{V}^{+}_t \coloneqq  2 \La^{\frac{1}{2}}  \G_t    \La^{\frac{1}{2}}   -  \Luf .
}
For  $d \geq \Omega(1)$,  we have   
\eq{
\bs{\Lambda} +  \tfrac{0.1 r^{- \alpha}}{\log^4 d}  \bs{I}_r   \succ  \Luf \succ \La \succ   \Llf   \succ \bs{\Lambda} -  \tfrac{0.1 r^{- \alpha}}{\log^4 d} \bs{I}_r \label{eq:boundsforLL0}
}
and
\eq{
& \bs{V}^{-}_{t+1} \succeq    \bs{V}^{-}_t \!\!  \left(   \bs{I}_r  +  \frac{\lr}{1 - 1.1 \lr}   \bs{V}^{-}_t   \right)^{-1}    \!+  \lr   \Llf^{2}  - C \lr^2 \normL^2 \rs \La   +   \frac{\eta}{\sqrt{\rs}}    \La^{\frac{1}{2}}  \bs{\nu}_{t+1}  \La^{\frac{1}{2}} \label{eq:lbiter} \\
&  \bs{V}^{+}_{t+1} \preceq    \bs{V}^{+}_t \!\!   \left(     \bs{I}_r  +   \frac{\lr}{1 + 1.1 \lr}   \bs{V}^{+}_t   \right)^{-1}    \! +  \lr   \Luf^{2}   + C \lr^2 \normL^2 \rs \La     +   \frac{\eta}{\sqrt{\rs}}    \La^{\frac{1}{2}}  \bs{\nu}_{t+1}  \La^{\frac{1}{2}}   \label{eq:ubiter}
} 
where the bounding iterations  are monotone in the sense defined in Proposition \ref{prop:monotoneiteration}.
\end{proposition}

\begin{proof}
We first note that since $\normL \asymp r^{\frac{1}{2} - \alpha}$ for  $\alpha \in [0,0.5)$, we have
\eq{
\normL \frac{\eta d}{ \sqrt{\rs}} \ll   \frac{r^{- \alpha}}{\log^4 d}.
}
Therefore,     \eqref{eq:boundsforLL0} holds for $d \geq \Omega(1)$,  which implies
\eq{
 \norm{ \bs{V}^{-}_t  }_2 \vee \norm{ \bs{V}^{+}_t  }_2 \leq  1 + \frac{0.1  r^{- \alpha} }{\log^4 d} , ~~ \text{for all }~ t \in \N.  \label{eq:boundsforv}
}
Therefore, the monotonicity follows from Proposition \ref{prop:monotoneiteration}.

For the remaining part, we introduce the following notation, $\bs{K}_t \coloneqq   \La^{\frac{1}{2}} \G_t  \La^{\frac{1}{2}}$.
For the lower bound,  by \eqref{eq:sotermsminus},  we have
\eq{
\bs{K}_{t+1} \succeq \bs{K}_t + \frac{\lr}{2} \left(   \Llf^2 - (2\bs{K}_t -   \Llf )^2 \right) - \frac{C}{2} \lr^2 \normL^2 \rs \La   +  \frac{\eta/2}{\sqrt{\rs}}    \La^{\frac{1}{2}}  \bs{\nu}_{t+1}  \La^{\frac{1}{2}}.
}
By multiplying both sides with $2$ and subtracting  $ \Llf$ from both sides, we have
\eq{
 &  \bs{V}^{-}_{t+1}\succeq    \bs{V}^{-}_t   - \lr  ( \bs{V}^{-}_t )^2    + \lr     \Llf^2   - C \lr^2 \normL^2 \rs \La  +    \frac{\eta}{\sqrt{\rs}}    \La^{\frac{1}{2}}  \bs{\nu}_{t+1}  \La^{\frac{1}{2}} \\
& \labelrel\succeq{boundforiterations:eq0}  \bs{V}^{-}_t \!  \! - \! \frac{\lr}{1 -  1.1  \lr }   ( \bs{V}^{-}_t )^2 \left( \bs{I}_r + \frac{\lr}{1 - 1.1  \lr } \bs{V}^{-}_t \right)^{-1}  \!\!   + \lr     \Llf^2  \\
& - C \lr^2 \normL^2 \rs  \La \! + \!     \frac{\eta}{\sqrt{\rs}}    \La^{\frac{1}{2}}  \bs{\nu}_{t+1}  \La^{\frac{1}{2}} \\
&  =   \bs{V}^{-}_t  \left( \bs{I}_r +  \frac{\lr}{1 - 1.1  \lr }    \bs{V}^{-}_t  \right)^{-1}     + \lr     \Llf^2   -  C \lr^2 \normL^2 \rs  \La +    \frac{\eta}{\sqrt{\rs}}    \La^{\frac{1}{2}}  \bs{\nu}_{t+1}  \La^{\frac{1}{2}},
}
where we used  \eqref{eq:boundsforv} for \eqref{boundforiterations:eq0}.

\smallskip
For the upper bound,  by \eqref{eq:sotermsplus},  we have
\eq{
\bs{K}_{t+1} \preceq \bs{K}_t + \frac{\lr}{2} \left(   \Luf^2 - (2\bs{K}_t -   \Luf )^2 \right)  +   \frac{C}{2} \lr^2 \normL^2 \rs  \La +  \frac{\eta/2}{\sqrt{\rs}}    \La^{\frac{1}{2}}  \bs{\nu}_{t+1}  \La^{\frac{1}{2}}.
}
By multiplying both sides with $2$ and subtracting  $ \Luf$ from both sides, we get
\eq{
\bs{V}^{+}_{t+1} & \preceq   \bs{V}^{+}_t   - \lr  ( \bs{V}^{+}_t )^2    + \lr     \Luf^2 +  C \lr^2 \normL^2 \rs  \La +    \frac{\eta}{\sqrt{\rs}}    \La^{\frac{1}{2}}  \bs{\nu}_{t+1}  \La^{\frac{1}{2}} \\
& \labelrel\preceq{boundforiterations:eq1}   \bs{V}^{+}_t   -  \frac{\lr}{1+ 1.1 \lr}  ( \bs{V}^{+}_t )^2 \left(  \bs{I}_r + \frac{\lr}{1+ 1.1 \lr} \bs{V}^{+}_t  \right)^{-1}   + \lr     \Luf^2  \\
& +   C \lr^2 \normL^2 \rs  \La +  \frac{\eta}{\sqrt{\rs}}    \La^{\frac{1}{2}}  \bs{\nu}_{t+1}  \La^{\frac{1}{2}} \\
& = \bs{V}^{+}_t  \left( \bs{I}_r +   \frac{\lr}{1+ 1.1 \lr}   \bs{V}^{+}_t  \right)^{-1}     + \lr     \Luf^2 + C \lr^2 \normL^2 \rs  \La  +    \frac{\eta}{\sqrt{\rs}}    \Lus^{\frac{1}{2}}  \bs{\nu}_{t+1}  \Lus^{\frac{1}{2}}, 
}
where we used  \eqref{eq:boundsforv} for \eqref{boundforiterations:eq1}.
\end{proof}

\subsubsection{Light tailed case - $\alpha > 0.5$}
We introduce the submatrix notation
\eq{
 \G_t \eqqcolon \begin{bmatrix}
\G_{t,11} &  \G_{t,12}  \\
  \G_{t,12}^\top &   \G_{t,22}
\end{bmatrix} ~~ 
\bs{\nu}_{t}  \eqqcolon  \begin{bmatrix}
\bs{\nu}_{t,11}  &  \bs{\nu}_{t,12}  \\
\bs{\nu}_{t,12}^\top  &   \bs{\nu}_{t,22} 
\end{bmatrix} ~~ 
  \La \eqqcolon   \begin{bmatrix}
\bs{\Lambda}_{11}   &  0 \\
0 & \bs{\Lambda}_{22}  
\end{bmatrix}   ~~ 
 \Llf \eqqcolon   \begin{bmatrix}
\bs{\Lambda}_{\ell_1, 11}   &  0 \\
0 & \bs{\Lambda}_{\ell_1, 22}  
\end{bmatrix} , 
}
where $\G_{t,11}  ,  \bs{\nu}_{t,11}  ,  \bs{\Lambda}_{11},    \bs{\Lambda}_{\ell_1, 11}    \in \R^{\ru \times \ru}$ for $\ru < r$. Similarly, we define the block matrices of $\Luf$ as  $\bs{\Lambda}_{u_1, 11} \in   \R^{\ru \times \ru}$ and  $\bs{\Lambda}_{u_1, 22}$. We can write  iterations  \eqref{eq:sotermsminus}  and  \eqref{eq:sotermsplus}  for the left-top submatrix as:
\begin{alignat}{2}
&   \G_{t+1, 11}   &&   \succeq    \G_{t,11}  +  \lr \Big(  \Llftop  \G_{t,11}   + \G_{t,11}  \Llftop - 2   \G_{t,11}   \La_{11}  \G_{t,11}    -  2  \G_{t,12}    \La_{22}   \G_{t,12}^\top  \Big) \\
& \ && \quad - \frac{C}{2} \lr^2 \normL^2 \rs \bs{I}_{\ru}  +   \frac{\eta/2}{\sqrt{\rs}} \bs{\nu}_{t+1,11}  \label{eq:sotermsminusslb}    \\
&   \G_{t+1, 11}  &&  \preceq   \G_{t,11} +  \lr\Big(    \Luftop  \G_{t,11}  +  \G_{t,11}  \Luftop  - 2   \G_{t,11}    \La_{11} \G_{t,11}   - 2  \G_{t,12}   \La_{22}  \G_{t,12}^\top  \Big) \\
& \ && \quad + \frac{C}{2} \lr^2 \normL^2 \rs \bs{I}_{\ru}    + \frac{\eta/2}{\sqrt{\rs}}   \bs{\nu}_{t+1,11}.    \label{eq:sotermsplussub}  
\end{alignat}
The following statement is analogous to Proposition \ref{prop:boundforiterationsheavy} in the case $\alpha > 0.5$.

\begin{proposition}
\label{prop:boundforiterationslight}
We consider $\alpha > 0.5$,  $\rs \asymp 1$,   and
\eq{
 \eta \ll  \frac{1}{d \log^3 d} \frac{1}{\ru^{2 + \alpha}} ~~ \text{and} ~~ \ru = \ceil{\log^{2.5} d}.
}
We define $\bs{V}^{+}_t \coloneqq  2 \La_{11}^{\frac{1}{2}}    \G_{t,11}   \La_{11}^{\frac{1}{2}}   -  \Luftop$   and
\eq{
\bs{V}^{-}_t \coloneqq  2 \left( \bs{\Lambda}_{11} \!  - \!   \tfrac{1}{( \ru + 1 )^{\alpha} }    \bs{I}_{\ru} \right)^{\frac{1}{2}} \G_{t,11}  \left( \bs{\Lambda}_{11} \!  - \!   \tfrac{1}{( \ru + 1 )^{\alpha} }    \bs{I}_{\ru}  \right)^{\frac{1}{2}} -   \left( \bs{\Lambda}_{\ell_1, 11} \! - \! \tfrac{1}{(\ru + 1)^{\alpha}} \bs{I}_{\ru} \right).
}
For  $d \geq \Omega(1)$,  we have   
\eq{
\bs{\Lambda}_{11} +  \tfrac{0.1}{\ru^{2 +  \alpha} \log^3 d} \bs{I}_{\ru}     \succ    \Luftop  \succ \bs{\Lambda}_{11}   \succ    \Llftop    \succ \bs{\Lambda}_{11}   -   \tfrac{0.1}{ \ru^{2 + \alpha} \log^3  d}   \bs{I}_{\ru}  \label{eq:boundsforLL1}
}
and
\eq{
 \bs{V}^{-}_{t+1} & \succeq    \bs{V}^{-}_t \!\!  \left(   \bs{I}_{\ru} \! + \!   \tfrac{\lr}{1 - 1.1 \lr}  \bs{V}^{-}_t  \! \right)^{-1}  \!\! \! \!   +  \lr  \left( \bs{\Lambda}_{\ell_1, 11} \! - \! \tfrac{1}{(\ru + 1)^{\alpha}} \bs{I}_{\ru} \! \right)^2 \!\! - \!  C\lr^2 \normL^2 \rs \!  \left( \bs{\Lambda}_{11}  \!  - \!     \tfrac{1}{( \ru + 1 )^{\alpha} } \bs{I}_{\ru}  \right)  \\
& +   \frac{\eta}{\sqrt{\rs}}    \left( \bs{\Lambda}_{11} \!  - \!   \tfrac{1}{( \ru + 1 )^{\alpha} }    \bs{I}_{\ru} \right)^{\frac{1}{2}}  \bs{\nu}_{t+1,11}  \left( \bs{\Lambda}_{11} \!  - \!   \tfrac{1}{( \ru + 1 )^{\alpha} }    \bs{I}_{\ru} \right)^{\frac{1}{2}} \label{eq:lbitersub} \\[0.3em]
 \bs{V}^{+}_{t+1}&   \preceq   \bs{V}^{+}_t  \left( \bs{I}_{\ru} +   \tfrac{\lr}{1+ 1.1 \lr}   \bs{V}^{+}_t  \right)^{-1}     + \lr     \Luftop^2 + C\lr^2 \normL^2 \rs     \bs{\Lambda}_{11}  +      \frac{\eta}{\sqrt{\rs}}    \La_{11}^{\frac{1}{2}}     \bs{\nu}_{t+1,11}  \La_{11}^{\frac{1}{2}}.   \label{eq:ubitersub}
} 
where the bounding iterations  are monotone in the sense defined in Proposition \ref{prop:monotoneiteration}.
\end{proposition}

\begin{proof}
We first note that since  $\rs \asymp 1$ and $\normL \asymp 1$ for  $\alpha > 0.5$, we have
\eq{
 \normL   \frac{\eta d}{ \sqrt{\rs} } \ll  \frac{1}{\ru^{2+ \alpha} \log^3  d} .
}
Therefore,     \eqref{eq:boundsforLL1} holds for $d \geq \Omega(1)$,  which implies
\eq{
 \norm{ \bs{V}^{-}_t  }_2 \vee \norm{ \bs{V}^{+}_t  }_2 \leq  1 + \frac{0.1    }{\ru^{2 +\alpha} \log^3  d} + \frac{1}{(\ru + 1)^{\alpha}} , ~~ \text{for all }~ t \in \N.  \label{eq:boundsforvlight}
}
For the remaining part, we introduce the following notation,
\eq{
\bs{K}^{-}_t \coloneqq   \left( \bs{\Lambda}_{11} \!  - \!   \tfrac{1}{( \ru + 1 )^{\alpha} }    \bs{I}_{\ru} \right)^{\frac{1}{2}} \G_{t,11}  \left( \bs{\Lambda}_{11} \!  - \!   \tfrac{1}{( \ru + 1 )^{\alpha} }    \bs{I}_{\ru}  \right)^{\frac{1}{2}}     ~~ \text{and} ~~  \bs{K}^{+}_t \coloneqq   \La_{11}^{\frac{1}{2}}   \G_{t,11}   \La_{11}^{\frac{1}{2}}.
}
For the upper bound,  since   $\La_{22}  \succ 0$,  by \eqref{eq:sotermsplussub}  we have
\eq{
\bs{K}^{+}_{t+1} \preceq \bs{K}^{+}_t + \frac{\lr}{2} \left(   \Luftop^2 - (2\bs{K}^{+}_t -   \Luftop )^2 \right)  +  \frac{C}{2} \lr^2 \normL^2 \rs \La_{11}  + \frac{\eta/2}{\sqrt{\rs}}    \La_{11}^{\frac{1}{2}}  \bs{\nu}_{t+1,11}  \La_{11}^{\frac{1}{2}}.
}
By multiplying both sides with $2$ and subtracting  $\Luftop$ from both sides, we get
\eq{
\bs{V}^{+}_{t+1} & \preceq   \bs{V}^{+}_t   - \lr  ( \bs{V}^{+}_t )^2    + \lr     \Luftop^2 +  C \lr^2 \normL^2 \rs \La_{11}  +    \frac{\eta}{\sqrt{\rs}}    \La_{11}^{\frac{1}{2}}  \bs{\nu}_{t+1,11}  \La_{11}^{\frac{1}{2}} \\
& \labelrel\preceq{boundlight:ineqq1}   \bs{V}^{+}_t   -  \frac{\lr}{1+ 1.1 \lr}  ( \bs{V}^{+}_t )^2 \left(  \bs{I}_{\ru} + \frac{\lr}{1+ 1.1 \lr} \bs{V}^{+}_t  \right)^{-1}   + \lr     \Luftop^2  \\
& +  C \lr^2 \normL^2 \rs \La_{11} +    \frac{\eta}{\sqrt{\rs}}    \La_{11}^{\frac{1}{2}} \bs{\nu}_{t+1,11}   \La_{11}^{\frac{1}{2}} \\
& = \bs{V}^{+}_t  \left( \bs{I}_{\ru} +   \frac{\lr}{1+ 1.1 \lr}   \bs{V}^{+}_t  \right)^{-1}     + \lr     \Luftop^2  +  C \lr^2 \normL^2 \rs \La_{11} +    \frac{\eta}{\sqrt{\rs}}    \La_{11}^{\frac{1}{2}} \bs{\nu}_{t+1,11}   \La_{11}^{\frac{1}{2}}, 
}
where we used  \eqref{eq:boundsforvlight} for \eqref{boundlight:ineqq1}.

\smallskip
For the lower bound,  we first observe that  $\G_{t,11} (\bs{I}_{r_u} -  \G_{t,11} ) -    \G_{t,12}   \G_{t,12}^\top \succeq 0$ since it corresponds to the left-top submatrix of $\G_t ( \bs{I}_{r} -  \G_t ).$ Therefore, by  \eqref{eq:sotermsminusslb} 
\eq{
 \G_{t+1, 11} \! \!  &  \succeq    \G_{t,11}  
+\! \lr \Big(  \Llftop    \G_{t,11} \!
+ \! \G_{t,11}  \Llftop \! - \! 2   \G_{t,11}   \La_{11}  \G_{t,11} \!  - \! 2  \G_{t,12}    \La_{22}   \G_{t,12}^\top  \Big) \!     \\
&  - \frac{C}{2} \lr^2 \normL^2 \rs \bs{I}_{\ru} + \! \frac{\eta/2}{\sqrt{\rs}}  \bs{\nu}_{t+1,11} - \frac{2\lr }{(\ru + 1)^\alpha} \left( \G_{t,11} (\bs{I}_{r_u} -  \G_{t,11} ) -    \G_{t,12}   \G_{t,12}^\top \right) \\
&  \labelrel\succeq{boundlight:ineqq2}   \! \G_{t,11} \!
+\! \lr \Big(    \big(\Llftop \! - \! \tfrac{1}{(\ru + 1)^\alpha} \bs{I}_{\ru}  \big)  \G_{t,11} \!
+ \! \G_{t,11}  \big(\Llftop \! - \! \tfrac{1}{(\ru + 1)^\alpha} \bs{I}_{\ru} \big)  \Big)  \\ 
&     -   2 \lr   \G_{t,11}    \big(\La_{11} \! - \! \tfrac{1}{(\ru + 1)^\alpha} \bs{I}_{\ru}  \big)    \G_{t,11}  \!   - \frac{C}{2} \lr^2 \normL^2 \rs \bs{I}_{\ru} + \! \frac{\eta/2}{\sqrt{\rs}}  \bs{\nu}_{t+1,11},  
}
where \eqref{boundlight:ineqq2} follows by  $ \La_{22}    \preceq \tfrac{1}{(\ru +1)^{\alpha} } \bs{I}_{r - \ru}$.  Therefore,  we have
\eq{
\bs{K}^{-}_{t+1} & \succeq \bs{K}^{-}_t + \frac{\lr}{2} \left(  \big(\Llftop - \tfrac{1}{(\ru + 1)^\alpha} \bs{I}_{\ru} \big) ^2 - \big(2\bs{K}^{-}_t -   \big(\Llftop - \tfrac{1}{(\ru + 1)^\alpha} \bs{I}_{\ru} \big)  \big)^2 \right)  \\
& - \frac{C}{2} \lr^2 \normL^2 \rs \big(\La_{11} - \tfrac{1}{(\ru + 1)^\alpha} \bs{I}_{\ru} \big) \\
& +   \frac{\eta/2}{\sqrt{\rs}}    \big(\La_{11} - \tfrac{1}{(\ru + 1)^\alpha} \bs{I}_{\ru} \big)^{\frac{1}{2}}   \bs{\nu}_{t+1,11}  \big(\La_{11} - \tfrac{1}{(\ru + 1)^\alpha}  \bs{I}_{\ru} \big)^{\frac{1}{2}}.
}
By multiplying both sides with $2$ and subtracting  $\Llftop$ from both sides, we get
\eq{
\bs{V}^{-}_{t+1} & \succeq   \bs{V}^{-}_t   - \lr  ( \bs{V}^{-}_t )^2    + \lr    \big(\Llftop - \tfrac{1}{(\ru + 1)^\alpha} \bs{I}_{\ru}  \big)^2  - C \lr^2 \normL^2 \rs \big(\La_{11} - \tfrac{1}{(\ru + 1)^\alpha} \bs{I}_{\ru}  \big) \\
& +   \frac{\eta}{\sqrt{\rs}}    \big(\La_{11} - \tfrac{1}{(\ru + 1)^\alpha} \bs{I}_{\ru}  \big)^{\frac{1}{2}}   \bs{\nu}_{t+1,11}  \big(\La_{11} - \tfrac{1}{(\ru + 1)^\alpha} \bs{I}_{\ru}  \big)^{\frac{1}{2}}  \\
& \labelrel\succeq{boundlight:ineqq3}   \bs{V}^{-}_t   -  \frac{\lr}{1 - 1.1 \lr}  ( \bs{V}^{-}_t )^2 \left(  \bs{I}_{\ru} + \frac{\lr}{1 - 1.1 \lr} \bs{V}^{-}_t  \right)^{-1}     + \lr    \big(\Llftop - \tfrac{1}{(\ru + 1)^\alpha} \bs{I}_{\ru} \big)^2    \\
& - C \lr^2 \normL^2 \rs \big(\La_{11} - \tfrac{1}{(\ru + 1)^\alpha} \bs{I}_{\ru}  \big)  \\
& +     \frac{\eta}{\sqrt{\rs}}    \big(\La_{11} - \tfrac{1}{(\ru + 1)^\alpha}  \bs{I}_{\ru} \big)^{\frac{1}{2}}   \bs{\nu}_{t+1,11}  \big(\La_{11} - \tfrac{1}{(\ru + 1)^\alpha}   \bs{I}_{\ru} \big)^{\frac{1}{2}} \\
& = \bs{V}^{-}_t  \left( \bs{I}_{\ru} +   \frac{\lr}{1 - 1.1 \lr}   \bs{V}^{-}_t  \right)^{-1}     + \lr    \big(\Llftop - \tfrac{1}{(\ru + 1)^\alpha}  \bs{I}_{\ru} \big)^2    \\
& - C \lr^2 \normL^2 \rs \big(\La_{11} - \tfrac{1}{(\ru + 1)^\alpha} \bs{I}_{\ru}  \big)  \\
& +     \frac{\eta}{\sqrt{\rs}}    \big(\La_{11} - \tfrac{1}{(\ru + 1)^\alpha} \bs{I}_{\ru}  \big)^{\frac{1}{2}}   \bs{\nu}_{t+1,11}  \big(\La_{11} - \tfrac{1}{(\ru + 1)^\alpha} \bs{I}_{\ru}  \big)^{\frac{1}{2}}, 
}
where we used  \eqref{eq:boundsforvlight} for \eqref{boundlight:ineqq3}.  The monotonicity of the update follows the same argument in the heavy-tailed case.
\end{proof}

\subsection{Definitions and bounding systems}
\label{sec:defs}
To avoid repetition in the derivations, we introduce the following unified notation:
\eq{
\rk \in \{ r, \ru \}, \quad  \sG_{t} \in \{ \G_t, \G_{t,11} \}, \quad  \upnu_{t} \in \{ \bs{\nu}_t,  \bs{\nu}_{t,11}\}.
}

where each variable will take its first value in the heavy-tailed case and its second value in the light-tailed case.
To avoid repetition in the following sections,   we make the following simplifications by slight abuse of notation:
\eq{
\Llf \leftarrow \begin{cases}
  \Llf, & \alpha \in [0,0.5) \\
  \Llftop - \frac{1}{(\ru + 1)^{\alpha}} \bs{I}_{\ru}, & \alpha > 0.5   
\end{cases}
~~ \text{and} ~~ 
\Lls \leftarrow \begin{cases}
  \La, & \alpha \in [0,0.5) \\
  \La_{11} - \frac{1}{(\ru + 1)^{\alpha}} \bs{I}_{\ru}, & \alpha > 0.5   
\end{cases}
}
and
\eq{
\Luf \leftarrow \begin{cases}
  \Luf, & \alpha \in [0,0.5) \\
  \Luftop, & \alpha > 0.5   
\end{cases}
~~ \text{and} ~~ 
\Lus \leftarrow \begin{cases}
  \La, & \alpha \in [0,0.5) \\
  \La_{11}, & \alpha > 0.5.   
\end{cases}
}
The dimension of each block is  $\ru < r$   for $\alpha > 0.5$ and $r$ for $\alpha \in [0,0.5)$, from which readers can distinguish the light tailed case from the heavy tailed case. Throughout the proof, we will also use constants $\c_d \in o_d(1)$ and $\tilde{C} \in O(1)$ that will be specified later.   Moreover,  we make the following definitions:
\begin{itemize}[leftmargin=*]
\item  \textbf{Noise sequence.} For $\unu{0} = 0$, we define the noise sequence $\unu{t+1} \coloneqq \unu{t}  +  \frac{\eta/2}{\sqrt{\rs}}  \upnu_{t+1}$.
\item   \textbf{Reference sequence.} For   $\T_{0}  = \frac{\c_d \rs}{d} \bs{I}_{\rk}$,  we define  the reference sequence
\eq{
\T_{t+1}  = \T_{t} +  2  (1 -  2 \c_d) \lr  \left(  \Llf   \T_{t}    -      \frac{3 \c_d + 1 }{\c_d (1 - 2 \c_d)}   \Lus  \T^2_{t}   \right). 
}
\item \textbf{Bounding systems.}   We define the lower and upper bounding recursions as  
\eq{
\uV{t+1} \! & =      \uV{t} \left( \bs{I}_{\rk} +  \frac{\lr (1 + 2 \c_d)}{1 - 1.2 \lr}  \uV{t}   \right)^{- 1} \!\!\!   +    \frac{\lr (1 + 2 \c_d)}{1 - 1.2 \lr}  \! \! \left(   \frac{\Llf^{2}}{(1 + 2 \c_d)^2} -   \tilde C   \lr  \normL^2 \rs \Llf  \right),  ~~ ~~ \label{eq:lbsystem}   \\ 
\oV{t+1} & =              \oV{t}  \left( \bs{I}_{\rk}  +  \frac{\lr  (1 - 2 \c_d) }{1 + 1.2 \lr}   \oV{t}  \right)^{- 1} \!\!\!  +   \frac{\lr (1 - 2 \c_d)}{1 + 1.2 \lr}  \! \!  \left(    \frac{\Luf^{2}}{(1 - 2 \c_d)^2} + \tilde C    \lr  \normL^2 \rs \Luf \right). ~~~~   \label{eq:ubsystem}
 } 
where the iterates $\{ \uV{t} \}_{t \in \N}$ and $\{ \oV{t} \}_{t \in \N}$ are functions of the bounding sequences   $\{ \uG{t} \}_{t \in \N}$ and $\{\oG{t} \}_{t \in \N}$   as following:
\begin{alignat}{5}
& \uK{t} \coloneqq   \Lls^{\frac{1}{2}} \uG{t} \Lls^{\frac{1}{2}}  ~~ &&\text{and} ~~ &&\uV{t} = 2 \uK{t} - \frac{\Llf}{1+ 2 \c_d}  ~~ &&\text{and} ~~   &&\uG{0} \preceq  \sG_0 -  \T_0,  \\[0.7em]
& \oK{t} \coloneqq   \Lus^{\frac{1}{2}} \oG{t} \Lus^{\frac{1}{2}}  ~~ &&\text{and} ~~ &&\oV{t} = 2 \oK{t} -  \frac{\Luf}{1 - 2 \c_d} ~~ &&\text{and} ~~   &&\oG{0} \succeq  \sG_0  + \T_0.
\end{alignat} 
\item \textbf{Stopping times.} We define a sequence of events $\{  \mathcal{E}_t \}_{t \geq 0}$
\eq{
\mathcal{E}_{t}  \! \coloneqq \!   \begin{cases}
  \left \{    -    \c_d r^{\frac{-\alpha}{2}}  \T_{t}    \preceq   \unu{t}  \preceq  \c_d  r^{\frac{-\alpha}{2}}   \T_{t} \right \}   \! \cap \!  \left \{    -    \c_d^2 r^{-  \alpha}  \T_{t}    \preceq  \bs{\Lambda}^{\frac{1}{2}}  \unu{t} \bs{\Lambda}^{\frac{1}{2}}  \preceq  \c_d^2 r^{- \alpha}   \T_{t} \right \},  & \hspace{-0.75em}  \alpha \in [0,0.5) \\
  \left \{    -    \c_d \ru^{\frac{-\alpha}{2}}  \T_{t}    \preceq   \unu{t}  \preceq  \c_d  \ru^{\frac{ - \alpha}{2}}   \T_{t} \right \}    \! \cap \!  \left \{      \frac{-   \c_d^2}{4} \ru^{- \alpha}  \T_{t}    \preceq  \bs{\Lambda}_{11}^{\frac{1}{2}}  \unu{t} \bs{\Lambda}_{11}^{\frac{1}{2}}  \preceq  \frac{\c_d^2}{4}  \ru^{-  \alpha }   \T_{t} \right \},    &   \hspace{-0.75em}\alpha > 0.5.
\end{cases}
}
We define the stopping times
\eq{
\mathcal{T}_{\text{noise}}(\omega)   \coloneqq \inf \left \{  t \geq 0  ~ \middle \vert ~  \omega \not \in \mathcal{E}_t  \right \} \wedge d^3  ~~ \text{and} ~~
\mathcal{T}_{\text{bounded}}  \coloneqq \inf \left  \{ t \geq 0  ~ \middle \vert ~ \norm{\uG{t}}_2 \vee \norm{\oG{t}}_2 \geq 1.5   \right  \}, \label{def:Tnoise}
}
and  
\eq{
\mathcal{T}_{\text{bad}} \coloneqq   \mathcal{T}_{\text{noise}}  \wedge  \mathcal{T}_{\text{bounded}}   \wedge \{ t \geq 0: \norm{\T_t}_2 >1.2  \c_d \}.    \label{def:Tbad}
}
\end{itemize}

The main result of this section is the following:
\begin{proposition}
\label{prop:boundstoppedprocess}
Let $\c_d$ satisfy \eqref{eq:cdcond} and consider d large enough so that $\c_d  \leq \tfrac{1}{50}.$   Under the learning rate conditions considered in Propositions  \ref{prop:boundforiterationsheavy} and \ref{prop:boundforiterationslight}, we have for $\tilde C > 1 \vee \Omega(1)$
\eq{
\uG{t \wedge \mathcal{T}_{\text{bad}}}  + \T_{t \wedge\mathcal{T}_{\text{bad}}} +  \unu{t \wedge \mathcal{T}_{\text{bad}}}     \preceq \sG_{t \wedge \mathcal{T}_{\text{bad}}}   \preceq \oG{t \wedge \mathcal{T}_{\text{bad}}}  -   \T_{t \wedge\mathcal{T}_{\text{bad}}}  +   \unu{t \wedge\mathcal{T}_{\text{bad}}}.
}
\end{proposition}

Before starting the proof, we provide an auxiliary statement. 

\begin{lemma}
\label{lem:noisecorollary}
We consider  the learning rate conditions considered in Propositions  \ref{prop:boundforiterationsheavy} and \ref{prop:boundforiterationslight}  with 
\eq{
\c_d \ll \begin{cases}
\tfrac{1}{\log d}, & \alpha \in [0,0.5) \\
\tfrac{\ru^{-1}}{\log d}, & \alpha > 0.5.
\end{cases} \label{eq:cdcond}
}
The event $\mathcal{E}_{t}$ implies for  $d \geq \Omega(1)$ and $t \leq    \mathcal{T}_{\text{bounded}}   \wedge   \{ t : \norm{\T_t}_2 >1.2  \c_d \}$ that
\begin{enumerate}
\item  $- 3  \c_d  \Llf^{\frac{1}{2}} \T_{t}  \Llf^{\frac{1}{2}} \preceq  \Llf  \unu{t} + \unu{t} \Llf $
\item  $ \Luf \unu{t} + \unu{t} \Luf     \preceq  3 \c_d   \Luf^{\frac{1}{2}}  \T_{t}  \Luf^{\frac{1}{2}}$
\item  $  \big( \Lls^{\frac{1}{2}} \unu{t}    \Lls^{\frac{1}{2}} \big)^2   \preceq \frac{\c^2_d}{4}   \Llf  \T_{t}  \Llf$
\item  $  \big( \Lus^{\frac{1}{2}} \unu{t}    \Lus^{\frac{1}{2}} \big)^2   \preceq \frac{\c^2_d}{4}   \Luf  \T_{t}  \Luf$  
\end{enumerate}
\end{lemma}
\begin{proof}
 For notational convenience, we define  $\widetilde{\bs{\nu}}_{t} \coloneqq  \T_t^{\frac{-1}{2}} \unu{t}  \T_t^{\frac{-1}{2}}$.   We initially observe that  $\mathcal{E}_{t}$ implies for $\alpha \in [0.0.5)$ that
\eq{
 \bs{\Lambda} \widetilde{\bs{\nu}}_{t} +  \widetilde{\bs{\nu}}_{t}   \bs{\Lambda}
&\preceq  \c_d^2  \bs{\Lambda} + \frac{1}{\c^2_d} \bs{\Lambda}^{\frac{-1}{2}} \left( \bs{\Lambda}^{\frac{1}{2}} \widetilde{\bs{\nu}}_{t}  \bs{\Lambda} ^{\frac{1}{2}} \right)^2   \bs{\Lambda}^{\frac{-1}{2}} \preceq   \c^2_d \bs{\Lambda} + \c^2_d r^{- \alpha} \bs{I}_{r} \\
\bs{\Lambda}  \widetilde{\bs{\nu}}_{t}  +  \widetilde{\bs{\nu}}_{t}   \bs{\Lambda}
&\succeq  - \c_d^2  \bs{\Lambda} - \frac{1}{\c^2_d} \bs{\Lambda}^{\frac{-1}{2}} \left( \bs{\Lambda}^{\frac{1}{2}}  \widetilde{\bs{\nu}}_{t} \bs{\Lambda} ^{\frac{1}{2}} \right)^2   \bs{\Lambda}^{\frac{-1}{2}} \succeq -   \c^2_d \bs{\Lambda} - \c^2_d r^{- \alpha} \bs{I}_{r}. 
}
For  $\alpha > 0.5$, these bounds become
\eq{
 \bs{\Lambda}_{11} \widetilde{\bs{\nu}}_{t} +  \widetilde{\bs{\nu}}_{t}   \bs{\Lambda}_{11}
&\preceq  \frac{\c_d^2}{4}   \bs{\Lambda}_{11} + \frac{4}{\c^2_d} \bs{\Lambda}_{11}^{\frac{-1}{2}} \left( \bs{\Lambda}_{11}^{\frac{1}{2}} \widetilde{\bs{\nu}}_{t}  \bs{\Lambda}_{11} ^{\frac{1}{2}} \right)^2   \bs{\Lambda}_{11}^{\frac{-1}{2}} 
\preceq   \frac{\c^2_d}{4} \bs{\Lambda}_{11} + \frac{\c^2_d}{4} \ru^{- \alpha} \bs{I}_{\ru} \\
\bs{\Lambda}_{11}  \widetilde{\bs{\nu}}_{t}  +  \widetilde{\bs{\nu}}_{t}   \bs{\Lambda}_{11}
&\succeq  - \frac{\c_d^2}{4} \ru^{-1}  \bs{\Lambda}_{11} - \frac{4}{\c^2_d} \bs{\Lambda}_{11}^{\frac{-1}{2}} \left( \bs{\Lambda}_{11}^{\frac{1}{2}}  \widetilde{\bs{\nu}}_{t} \bs{\Lambda}_{11}^{\frac{1}{2}} \right)^2   \bs{\Lambda}_{11}^{\frac{-1}{2}} \succeq -   \frac{\c^2_d}{4} \ru^{-1}\bs{\Lambda}_{11} - \frac{\c^2_d}{4} \ru^{-  \alpha} \bs{I}_{\ru}. 
}
In the following, we will use these bounds.  For the first item, we have
\eq{
\Llf   \widetilde{\bs{\nu}}_{t}  +  \widetilde{\bs{\nu}}_{t} \Llf 
& \succeq  \begin{cases}
- \c_d \left( \c_d \bs{\Lambda} + \c_d r^{- \alpha} \bs{I}_{r} + r^{ \frac{- \alpha}{2}}  C  \normL   \frac{\eta d}{\sqrt{\rs} } \bs{I}_{r} \right), & \alpha \in [0,0.5) \\[0.3em]
- \c_d \left( \frac{\c_d}{4} \bs{\Lambda}_{11} + \frac{\c_d}{4}  \ru^{- \alpha} \bs{I}_{\ru} +  \ru^{ \frac{ - \alpha}{2}}  C  \normL    \frac{\eta d}{\sqrt{\rs}} \bs{I}_{\ru}     \right), & \alpha > 0.5
\end{cases} \\
& \succeq   - 3 \c_d  \Llf.  
}
For the second item, we have
\eq{
\Luf    \widetilde{\bs{\nu}}_{t}  +   \widetilde{\bs{\nu}}_{t} \Luf  
& \preceq
\begin{cases}
 \c_d \left(   \c_d \bs{\Lambda} + \c_d r^{- \alpha} \bs{I}_{r}  + r^{\frac{- \alpha}{2}}  C  \normL     \frac{\eta d}{\sqrt{\rs}} \bs{I}_{r}  \right) , & \alpha \in [0,0.5) \\[0.3em]
 \c_d \left(   \frac{\c_d}{4} \bs{\Lambda}_{11} + \frac{\c_d}{4} \ru^{-  \alpha} \bs{I}_{\ru}  + \ru^{ \frac{- \alpha}{2}}  C  \normL     \frac{\eta d}{\sqrt{\rs}} \bs{I}_{\ru}   \right) , & \alpha > 0.5
\end{cases} \\
& \preceq 3  \c_d \Luf.  
}
For the third item, we immediately observe that  $\big( \Lls^{\frac{1}{2}} \unu{t}    \Lls^{\frac{1}{2}} \big)^2 \preceq  \Lls^{\frac{1}{2}} \unu{t}^2  \Lls^{\frac{1}{2}}$ . Therefore,
\eq{
\T_t^{\frac{-1}{2}} \Lls^{\frac{1}{2}} \unu{t}^2  \Lls^{\frac{1}{2}} \T_t^{\frac{-1}{2}} \preceq 1.2 \c_d   \Lls^{\frac{1}{2}}  \widetilde{\bs{\nu}}_{t}^2  \Lls^{\frac{1}{2}} \preceq 1.2 \c_d^3  \rk^{- \alpha}  \Lls 
\preceq \frac{\c_d^2}{4}   \Llf ^2.
}
For the fourth item, we  observe that  
\eq{
\big( \Lus^{\frac{1}{2}} \unu{t}    \Lus^{\frac{1}{2}} \big)^2 =  \Lus^{\frac{1}{2}} \unu{t} \Lus  \unu{t} \Lus^{\frac{1}{2}} \preceq  \left(1 + \frac{0.1 \rk^{-\alpha}}{\log^3 d} \right)  \Lus^{\frac{1}{2}} \unu{t}^2  \Lus^{\frac{1}{2}}. 
}
Therefore
\eq{
 \left(1 + \frac{0.1 \rk^{-\alpha}}{\log^3 d} \right) \T_t^{\frac{-1}{2}} \Lus^{\frac{1}{2}} \unu{t}^2  \Lus^{\frac{1}{2}} \T_t^{\frac{-1}{2}} \preceq 1.25 \c_d   \Lus^{\frac{1}{2}}  \widetilde{\bs{\nu}}_{t}^2  \Lus^{\frac{1}{2}} \preceq 1.25 \c_d^3 \rk^{-\alpha}  \Lus \preceq \frac{\c_d^2}{4}   \Luf ^2.
}
\end{proof}

\subsubsection{Proof of Proposition \ref{prop:boundstoppedprocess}}
\begin{proof}
For the proof, we introduce the following notations  
\eq{
\uze{t} \coloneqq 2   \Lls^{\frac{1}{2}} \unu{t}  \Lls^{\frac{1}{2}}   ~~ \text{and} ~~ \oze{t}  \coloneqq  2 \Lus^{\frac{1}{2}} \unu{t}  \Lus^{\frac{1}{2}} ~~ \text{and} ~~
 \uB{t} \coloneqq 2 \Lls^{\frac{1}{2}} \T_t   \Lls^{\frac{1}{2}} ~~ \text{and} ~~   \oB{t} \coloneqq 2 \Lus^{\frac{1}{2}} \T_t   \Lus^{\frac{1}{2}}  .
}
Using this notation, we obtain:
\eq{
\uB{t+1} & \preceq \uB{t} +    \frac{2  (1 - 2 \c_d)   \lr}{1 - 1.1 \lr}   \left(   \Llf \uB{t}   -   \frac{1. 5 \c_d + 0.5 }{\c_d (1 - 2 \c_d)} \uB{t}^2  \right)   + 10  \lr^2 \c_d \Llf   \label{eq:Bund} \\
\oB{t+1}&  \preceq  \oB{t} +   \frac{2  (1 - 2 \c_d)   \lr}{1 + 1.1 \lr}   \left(   \Luf \oB{t}   -   \frac{1. 5 \c_d + 0.5}{\c_d (1 - 2 \c_d)} \oB{t}^2  \right)    + 10  \lr^2 \c_d \Luf  \label{eq:Bov}
}
Before proceeding with the proof, we observe that  the following inequalities hold:
\eq{
\norm{\Lls^{-1} \Llf}_2 \leq 1, ~   \norm{\Llf^{-1} \Lls}_2 \leq  \frac{1}{1 - \frac{0.1}{\log^4  d}} ~~ \text{and} ~~  \norm{\Lus^{-1} \Luf}_2 \leq 1, ~  \norm{\Luf^{-1} \Lus}_2 \leq  \frac{1}{1 - \frac{0.1}{\log^4  d}}. \label{eq:boundsforratios}
}
These bounds will be used in the following whenever we apply Propositions  \ref{prop:monotoneresidual} and \ref{prop:taylorresidual}, without explicitly restating them each time. We will establish the upper and lower bounds simultaneously for $\rk \in \{ r, \ru \}$.
\paragraph{Upper bound proof:}  We will use proof by induction. Specifically, we will show that  for $t <   \mathcal{T}_{\text{bad}}$,
\eq{
\V^{+}_t  \preceq \oV{t} + \oze{t}  - \oB{t} + \frac{2 \c_d \Luf}{1 - 2 \c_d}  ~~ \Rightarrow  ~~    \V^{+}_{t+1} \preceq  \oV{t+1} + \oze{t+1}  - \oB{t+1} + \frac{2 \c_d \Luf}{1 - 2 \c_d}.  \label{eq:inductionstepub}
}
Since the base case holds at $t = 0$ and  $\mathcal{T}_{\text{bad}} > 0$,  it remains to prove  \eqref{eq:inductionstepub}.  By \eqref{eq:ubiter},  we have 
\eq{
 \bs{V}^{+}_{t+1} &   \preceq      \left( \oV{t} + \oze{t} - \oB{t} + \frac{2 \c_d \Luf}{1 - 2 \c_d} \right) \! \left(  \! \bs{I}_{\rk}  \! +  \! \frac{ \lr}{1+ 1.1  \lr}     ( \oV{t} \! + \! \oze{t} \! -\! \oB{t} \! + \! \frac{2 \c_d \Luf}{1 - 2 \c_d})   \right)^{-1}  \\
 &  +   \lr   \Luf^2       +  C \lr^2 \normL^2 \rs \Lus  +  \frac{\eta}{\sqrt{\rs}}     \Lus^{\frac{1}{2}} \upnu_{t+1}  \Lus^{\frac{1}{2}} \\
 & =     \oV{t}   \left(   \bs{I}_{\rk}  +    \frac{ \lr}{1+ 1.1  \lr}  \oV{t}   \right)^{-1}  \!\!\!\!  +  \lr  \Luf^2    +  C \lr^2 \normL^2 \rs \Lus      +  \frac{\eta}{\sqrt{\rs}}   \Lus^{\frac{1}{2}} \upnu_{t+1} \Lus^{\frac{1}{2}}   \\
 &  +     \left(   \bs{I}_{\rk}  +    \frac{ \lr}{1+ 1.1  \lr}\oV{t}   \right)^{-1}  \!\!\!  (  \oze{t} - \oB{t} + \frac{2 \c_d \Luf}{1 - 2 \c_d}   )   \left(   \bs{I}_{\rk}  +   \frac{ \lr}{1+ 1.1  \lr}      (  \oV{t} + \oze{t} - \oB{t} + \frac{2 \c_d \Luf}{1 - 2 \c_d}  ) \right)^{-1} \!\!\!\!  .  \label{eq:upsystem0}
}
  By using Proposition \ref{prop:monotoneresidual},  we have       for $t <  \mathcal{T}_{\text{bad}}$ 
\eq{
& \left(   \bs{I}_{\rk}  +    \frac{ \lr}{1+ 1.1  \lr}\oV{t}   \right)^{-1}   (  \oze{t} - \oB{t} + \frac{2 \c_d \Luf}{1 - 2 \c_d}   )   \left(   \bs{I}_{\rk}  +   \frac{ \lr}{1+ 1.1  \lr}      (  \oV{t} + \oze{t} - \oB{t} + \frac{2 \c_d \Luf}{1 - 2 \c_d}  ) \right)^{-1} \\
 & \preceq     \left( \oze{t} - \oB{t} + \frac{2 \c_d  \Luf}{1 - 2 \c_d} \right)  -    
 \frac{\lr}{1 + 1.1 \lr}    \oV{t}  \left( \oze{t} - \oB{t} +  \frac{2 \c_d  \Luf}{1 - 2 \c_d}  \right)  \\
 & -  \frac{\lr}{1 + 1.1 \lr}  \left( \oze{t} - \oB{t} +  \frac{2 \c_d  \Luf}{1 - 2 \c_d}  \right)   \oV{t}  
-    \frac{\lr}{1 + 1.1 \lr}   \left( \oze{t}  - \oB{t} +  \frac{2 \c_d  \Luf}{1 - 2 \c_d}  \right)^2  +  \frac{ \lr^2 \c_d^2}{(1 + 1.1 \lr)^2} \oV{t}^4  \\
& +    \frac{2 \lr^2/\c_d^2}{(1 + 1.1 \lr)^2}   \left( \oze{t}  - \oB{t} +  \frac{2 \c_d  \Luf}{1 - 2 \c_d}  \right)^2  +\frac{ \lr^2 \c_d^2}{(1 + 1.1 \lr)^2} \oV{t}  \!\! \left( \oze{t}  - \oB{t} +  \frac{2 \c_d  \Luf}{1 -2 \c_d}  \right)^2  \!\!  \oV{t}    \\
&    + \frac{\lr^2}{(1 +  1.1 \lr)^2}    \left( \oze{t} - \oB{t} +   \frac{2 \c_d  \Luf}{1 - 2 \c_d}  \right)  \!  \oV{t}  \!   \left( \oze{t} - \oB{t} +  \frac{2 \c_d  \Luf}{1 - 2 \c_d} \right)  + \lr^3 \tilde C_1 \Luf    \\
&     + \frac{\lr^2}{(1 + 1.1  \lr)^2}    \left( \oze{t} - \oB{t} +  \frac{2 \c_d  \Luf}{1 - 2 c_d}  \right)^3     + \frac{\lr^2}{(1 +1.1  \lr)^2}    \oV{t} \left( \oze{t} - \oB{t} +  \frac{2 \c_d  \Luf}{1 - 2 \c_d} \right)    \oV{t}   
}  
for some  $\tilde C_1 = O(1)$ .     We have the following: First: 
\eq{
\c_d^2  \oV{t}^4   \! + \! \c_d^2 \oV{t}   \left( \oze{t}  -  \oB{t} +  \frac{2 \c_d  \Luf}{1 -2 \c_d}  \right)^2   \oV{t} \! + \!   \oV{t} \left( \oze{t} - \oB{t} +  \frac{2 \c_d  \Luf}{1 - 2 \c_d} \right)    \oV{t}  \! 
&  \labelrel\preceq{bsyms:ineqq0}  4 \c_d \oV{t}^2  \\
 & \labelrel\preceq{bsyms:ineqq1}      \frac{1}{15} \frac{1 + 1.1 \lr}{1 +1.2\lr}  \oV{t}^2,
}
where \eqref{bsyms:ineqq0}  follows by $\norm{ \oV{t}  }_2 \leq 5$ and   $\norm*{   \oze{t}  - \oB{t} +  \frac{2 \c_d  \Luf}{1 - 2 \c_d}  }_2 \leq 2.5 \c_d$,  and  \eqref{bsyms:ineqq1}  follows by  $\c_d  \leq \tfrac{1}{50}$.  Second,
\eq{
& \frac{2}{\c_d^2}  \left( \oze{t}  - \oB{t} +  \frac{2 \c_d  \Luf}{1 - 2 \c_d}  \right)^2 \! + \!  \left( \oze{t} -\oB{t} +   \frac{2 \c_d  \Luf}{1 - 2 \c_d}  \right)  \! \oV{t}  \!  \left( \oze{t} - \oB{t} +  \frac{2 \c_d  \Luf}{1 - 2 \c_d} \right) \!   \\
&  + \!  \left( \oze{t} + \oB{t} +  \frac{2 \c_d  \Luf}{1 - 2 c_d}  \right)^3 \!\!  \labelrel\preceq{bsyms:ineqq2}  \frac{3}{\c_d^2}  \left( \oze{t}  - \oB{t} +  \frac{2 \c_d  \Luf}{1 - 2 \c_d}  \right)^2 \!\!\!  \preceq    \frac{9}{\c_d^2}  \left( \oze{t}^2  + \oB{t}^2 \right) + \frac{36}{(1 - 2\c_d)^2}  \Luf^2,
}
where    \eqref{bsyms:ineqq2}  follows by $\norm{ \oV{t}  }_2 \leq 5$ and   $\norm*{   \oze{t} -  \oB{t} +  \frac{2 \c_d  \Luf}{1 - 2 \c_d}  }_2 \leq 2.5 \c_d$. Third:
\eq{
& - \!  \oV{t} \!  \left( \oze{t} \! - \! \oB{t} \! +  \! \frac{2 \c_d  \Luf}{1 - 2 \c_d}  \right)  \! - \!   \left( \oze{t} \! - \! \oB{t} \! +  \! \frac{2 \c_d  \Luf}{1 - 2 \c_d}  \right)   \oV{t}  \! -  \!   \left( \oze{t} - \oB{t} +  \frac{2 \c_d  \Luf}{1 - 2 \c_d}  \right)^2 \\
& +  \frac{9\lr/\c_d^2}{1 + 1.1\lr}   \left( \oze{t}^2  + \oB{t}^2 \right)  \\
& \preceq - 2 (\oK{t} \oze{t} + \oze{t} \oK{t} ) + 2 (\oK{t} \oB{t} + \oB{t} \oK{t}) - \frac{4 \c_d}{1 - 2 \c_d} (\oK{t} \Luf + \Luf \oK{t}  ) + 3  \left( \oze{t}^2  + \oB{t}^2 \right) \\
& +  (\Luf \oze{t}  + \oze{t} \Luf ) -     (\Luf \oB{t} + \oB{t} \Luf) + \frac{4 \c_d (1 - \c_d)}{(1 - 2 \c_d)^2} \Luf^2  \\ 
& \labelrel\preceq{bsyms:ineqq3}  8 \c_d  \oK{t}^2 - \frac{4 \c_d}{1 - 2 \c_d} (\oK{t} \Luf + \Luf \oK{t}  )  + \frac{4 \c_d (1 - \c_d)}{(1 - 2 \c_d)^2} \Luf^2 \\
& - (2 - 4 \c_d) \Luf \oB{t}  + \left(3 + \frac{1}{\c_d} \right) \oB{t}^2 \\
& = 2 \c_d \oV{t}^2 + \frac{2 \c_d}{1 - 2 \c_d}  \Luf^2  - (2 - 4 \c_d) \Luf \oB{t}  + \left(3 + \frac{1}{\c_d} \right) \oB{t}^2,
}
where we used  Proposition \ref{prop:matrixyounginequality}, and the second and fourth items in Lemma \ref{lem:noisecorollary} in \eqref{bsyms:ineqq3}. Therefore, we have
\eq{
\eqref{eq:upsystem0} &  \preceq   \oV{t}   \left(   \bs{I}_{\rk}  +    \frac{ \lr}{1+ 1.1  \lr}  \oV{t}   \right)^{-1}    + \frac{ 2 \c_d \lr}{1 + 1.1 \lr}   \oV{t}^2 + \frac{1}{10} \frac{(1 - 2 \c_d) \lr^2}{(1 + 1.1 \lr) (1 + 1.2 \lr)} \oV{t}^2 + \frac{2 \c_d  \Luf}{1 -2 \c_d}  \\
& + \frac{\lr}{1 + 1.1 \lr} \left( \frac{\Luf^2}{1 - 2 \c_d} + \tilde C  \lr \normL^2 \rs  \Lus    \right) + \oze{t+1} - \oB{t}  \\
& -  \frac{2 (1 -2  \c_d) \lr}{1 + 1.1 \lr}\left( \Luf \oB{t}  - \frac{1. 5 \c_d + 0.5 }{\c_d (1 - 2 \c_d)} \oB{t}^2  \right) \\
& \labelrel\preceq{bsyms:ineqq4}   \oV{t}   \left(   \bs{I}_{\rk}  +    \frac{\lr (1 - 2 \c_d)}{1+ 1.2  \lr}  \oV{t}   \right)^{-1} \!\!\!   + \frac{\lr }{1 + 1.2 \lr} \left( \frac{\Luf^2}{1 - 2 \c_d} +   \tilde C  \lr \normL^2 \rs  \Luf      \right)  \\
& +  \oze{t+1} + \frac{2 \c_d  \Luf}{1 -2 \c_d}  - \oB{t+1} \\
& \preceq  \oV{t+1}  + \oze{t+1} + \frac{2 \c_d  \Luf}{1 -2 \c_d}  - \oB{t+1},
}
where we used Proposition \ref{prop:taylorresidual} and \eqref{eq:Bov}  in \eqref{bsyms:ineqq4}. 
\paragraph{Lower bound proof:}   Similar to the upper bound proof,  here we will show that  for $t <   \mathcal{T}_{\text{bad}}$,
\eq{
\uV{t} + \uze{t} + \uB{t} - \frac{2 \c_d \Llf}{1 + 2 \c_d}  \preceq  \V^{-}_{t} ~~ \Rightarrow  ~~  \uV{t+1} + \uze{t+1} + \uB{t+1} - \frac{2 \c_d}{1 + 2 \c_d} \Llf  \preceq  \V^{-}_{t+1}. \label{eq:inductionsteplb}
}
Since the base case holds at $t = 0$ and  $\mathcal{T}_{\text{bad}} > 0$,  it remains to prove  \eqref{eq:inductionsteplb}.    By \eqref{eq:ubiter},  we have 
\eq{
& \bs{V}^{-}_{t+1}    \succeq     \left( \uV{t} \! + \! \uze{t} \! + \! \uB{t} \! - \! \frac{2 \c_d  \Llf}{1 + 2 \c_d}  \right)  \! \left(   \bs{I}_{\rk}  +   \frac{ \lr}{1 - 1.1  \lr}     (  \uV{t} \! + \! \uze{t}  \! + \! \uB{t} \!  - \!  \frac{2 \c_d  \Llf}{1 + 2 \c_d}  )   \right)^{-1}   +  \lr   \Llf^2 \\
 & - C 
 \lr^2 \normL^2 \rs \Lls       +   \frac{\eta}{\sqrt{\rs}}    \Lls^{\frac{1}{2}}  \upnu_{t+1}  \Lls^{\frac{1}{2}} \\
 & =     \uV{t}   \left(   \bs{I}_{\rk}  +    \frac{ \lr}{1 - 1.1  \lr}  \uV{t}   \right)^{-1}  \!\!\!  +  \lr  \Llf^2     - C 
 \lr^2 \normL^2 \rs \Lls  +  \frac{\eta}{\sqrt{\rs}}     \Lls^{\frac{1}{2}}  \upnu_{t+1}  \Lls^{\frac{1}{2}} \\  
 &  +     \left(   \bs{I}_{\rk}  +    \frac{ \lr}{1 -1.1  \lr}\uV{t}   \right)^{-1}   \!\!\!\!\!  ( \uze{t} + \uB{t} - \frac{2 \c_d  \Llf}{1 + 2 \c_d}   )   \left(   \bs{I}_{\rk}  +   \frac{ \lr}{1 - 1.1  \lr}      ( \uV{t}  + \uze{t} + \uB{t} -  \frac{2 \c_d  \Llf}{1 + 2 \c_d}  )   \right)^{-1} \!\!\!\!\! \!\! . \label{eq:lbsystem0}
}
By using Proposition \ref{prop:monotoneresidual},  we have       for $t <    \mathcal{T}_{\text{bad}} $
\eq{
&  \left(   \bs{I}_{\rk} \! +  \! \frac{\lr}{1 - 1.1 \lr}    \uV{t}   \right)^{-1} \!\!\! \left ( \uze{t} \! + \! \uB{t} -  \frac{2 \c_d  \Llf}{1 + 2 \c_d} \right ) \!    \left(   \bs{I}_{\rk}  \! + \! \frac{\lr}{1 - 1.1 \lr}   ( \uV{t}  + \uze{t} + \uB{t} - \frac{2 \c_d  \Llf}{1 + 2 \c_d}) \!     \right)^{-1}  \\
& \succeq    \left( \uze{t} + \uB{t} - \frac{2 \c_d  \Llf}{1 + 2 \c_d} \right)  -    
 \frac{\lr}{1 - 1.1 \lr}    \uV{t}  \left( \uze{t} + \uB{t} -  \frac{2 \c_d  \Llf}{1 + 2 \c_d}  \right) -  \frac{ \lr^2 \c_d^2}{(1 - 1.1 \lr)^2} \uVq{t}  \\
 & -  \frac{\lr}{1 - 1.1 \lr}  \left( \uze{t} + \uB{t} -  \frac{2 \c_d  \Llf}{1 + 2 \c_d}  \right)   \uV{t}       -    \frac{\lr}{1 - 1.1 \lr}   \left( \uze{t}  + \uB{t} -  \frac{2 \c_d  \Llf}{1 + 2 \c_d}  \right)^2   \\
 & -    \frac{2 \lr^2/\c_d^2}{(1 - 1.1 \lr)^2}   \left( \uze{t}  + \uB{t} -  \frac{2 \c_d  \Llf}{1 + 2 \c_d}  \right)^2   - \frac{ \lr^2 \c_d^2}{(1 - 1.1 \lr)^2} \uV{t}   \! \left( \uze{t}  + \uB{t} -  \frac{2 \c_d  \Llf}{1 + 2 \c_d}  \right)^2  \!\!  \uV{t}   \\
&     + \frac{\lr^2}{(1 -  1.1 \lr)^2}  \! \!  \left( \uze{t} \! + \! \uB{t} \! -  \! \frac{2 \c_d  \Llf}{1 + 2 \c_d}  \right)  \!  \uV{t}  \!   \left( \uze{t} \! +  \!\uB{t} \! -  \! \frac{2 \c_d  \Llf}{1 + 2 \c_d} \right)   - \lr^3 \tilde C_2 \Llf   \\
&      + \frac{\lr^2}{(1 -1.1  \lr)^2}    \left( \uze{t} + \uB{t} -  \frac{2 \c_d  \Llf}{1 + 2 \c_d}  \right)^3     + \frac{\lr^2}{(1 -1.1  \lr)^2}    \uV{t} \left( \uze{t} + \uB{t} -  \frac{2 \c_d  \Llf}{1 + 2 \c_d} \right)    \uV{t} 
}
for some  $\tilde C_2 = O(1)$ .  We have the following: First: 
\eq{
-  \!   \c_d^2   \uVq{t}   \!  -  \! \c_d^2 \uV{t}  \! \left( \uze{t}  + \uB{t} \! -  \! \frac{2 \c_d  \Llf}{1 + 2 \c_d}  \right)^2  \!   \uV{t} \! + \!    \uV{t}  \! \left( \uze{t} \! + \! \uB{t} \! -  \!  \frac{2 \c_d  \Llf}{1 + 2 \c_d} \right)    \uV{t}   \!\!
&  \labelrel\succeq{bsyms:ineqq5}  \! \! - 3.2 \c_d \uVs{t} \\
 & \labelrel\succeq{bsyms:ineqq6}   \!  -  \frac{1}{15} \frac{1 - 1.1 \lr}{1 - 1.2\lr}  \uVs{t}.
}
where \eqref{bsyms:ineqq5}  follows by $\norm{ \uV{t}  }_2 \leq 5$ and   $\norm*{   \uze{t}  + \uB{t} -  \frac{2 \c_d  \Llf}{1 + 2 \c_d}  }_2 \leq 2.5 \c_d$,  and  \eqref{bsyms:ineqq6}  follows by  $\c_d  \leq \tfrac{1}{50}$. 
Second:
\eq{
& - \frac{2}{\c_d^2}   \!  \left( \uze{t}  \! +  \! \uB{t}  \! -  \! \frac{2 \c_d  \Llf}{1 + 2 \c_d}  \right)^2 \! + \!    \left( \uze{t}  \! +  \!\uB{t}  \! -  \!   \frac{2 c_d  \Llf}{1 + 2 \c_d}  \right)   \! \uV{t}   \!   \left( \uze{t}  \! +  \! \uB{t}  \! -   \! \frac{2 \c_d  \Llf}{1 + 2 \c_d} \right)   \\
&   +  \!  \left( \uze{t}  \! +  \! \uB{t}  \! -  \frac{2 \c_d  \Llf}{1 + 2 \c_d}  \right)^3  \labelrel\succeq{bsyms:ineqq7}  \frac{- 3}{\c_d^2} \left( \uze{t}  + \uB{t} -  \frac{2 \c_d  \Llf}{1 + 2 \c_d}  \right)^2   \succeq  \frac{- 9}{\c_d^2} \left( \uze{t}^2  + \uB{t}^2 \right) - 36   \Llf^2,
}
where    \eqref{bsyms:ineqq7}  follows by $\norm{ \uV{t}  }_2 \leq 5$ and   $\norm*{   \uze{t}  + \uB{t} -  \frac{2 \c_d  \Llf}{1 + 2 \c_d}  }_2 \leq 2.5 \c_d$.
Third:
\eq{
&   -     \uV{t}  \! \left( \uze{t}  \! +  \! \uB{t}  \! -  \!  \frac{2 \c_d  \Llf}{1 + 2 \c_d}  \right)  \! -  \!  \left( \uze{t} + \uB{t} -  \frac{2 \c_d  \Llf}{1 + 2 \c_d}  \right) \!  \uV{t}  \!  -       \! \left( \uze{t}   \! +  \! \uB{t}  \! -  \!  \frac{2 \c_d  \Llf}{1 + 2 \c_d}  \right)^2   \!    \\
& - \! \frac{9 \lr/ \c_d^2}{1 - 1.1 \lr}   \left( \uze{t}^2   \! +  \! \uB{t}^2 \right)  \\ 
& \succeq  -     2   \left(  \uK{t}  \uze{t} + \uze{t}   \uK{t} \right)      -     2   \left(  \uK{t}  \uB{t} + \uB{t}   \uK{t} \right)    +   \frac{4 \c_d}{1 + 2 \c_d} \left(   \uK{t}  \Llf + \Llf \uK{t}  \right) - 3 ( \uze{t}^2 +  \uB{t}^2)  \\
&  +     \left(  \Llf \uze{t} +  \uze{t}  \Llf  \right)  +     \left(  \Llf \uB{t} +  \uB{t} \Llf  \right) -   \frac{4 \c_d (1 + \c_d)  \Llf^2}{ ( 1 + 2 \c_d )^2}   
    \\
&    \labelrel\succeq{bsyms:ineqq8}   - 8 c_d \uK{t}^2    +   \frac{4 \c_d}{1 + 2 \c_d} \left(   \uK{t}  \Llf + \Llf \uK{t}  \right)  -   \frac{4 \c_d (1 +   \c_d)  \Llf^2}{ ( 1 + 2 \c_d )^2}    \\
& +  (2 - 4 \c_d)   \Llf \uB{t}     -  \left(3 + \frac{1}{\c_d} \right)  \uB{t}^2 \\
& =     - 2 \c_d \uVs{t}  -   \frac{2 \c_d   \Llf^2}{ ( 1 + 2 \c_d )}     +  (2 - 4 \c_d)  \uB{t}    \Llf   -  \left(3 + \frac{1}{\c_d} \right)  \uB{t}^2,
}
where we used Proposition \ref{prop:matrixyounginequality},  and the first and third items in Lemma \ref{lem:noisecorollary} in \eqref{bsyms:ineqq8}.  Therefore, we have
\eq{
&  \eqref{eq:lbsystem0} \succeq   \uV{t}   \left(  \! \bs{I}_r  +    \frac{ \lr}{1 - 1.1  \lr}  \uV{t}  \!  \right)^{-1} \!\!\!\!   -    \frac{ 2 \c_d\lr}{1 - 1.1  \lr}  \uVs{t}   -   \frac{(1 + 2 \c_d) \lr^2 / 15 }{(1 - 1.1  \lr)(1 - 1.2 \lr)} \uVs{t}  - \frac{2 \c_d}{1 + 2 \c_d} \Llf  \\
&  +  \frac{\lr}{1 - 1.1 \lr} \left( \frac{\Llf^2}{1 + 2 \c_d} -   \tilde C   \lr  \normL^2 \rs \Lls   \right) +  \uze{t+1} +  \uB{t}  \\
& +   \frac{ 2 (1 - 2 \c_d) \lr}{1 - 1.1  \lr}  \left(      \Llf  \uB{t}   -   \frac{1. 5 \c_d + 0.5 }{\c_d (1 - 2 \c_d)}  \uB{t}^2 \right) \\
&  \labelrel\succeq{bsyms:ineqq9} \!     \uV{t}   \left(   \bs{I}_r  +    \frac{ \lr (1 + 2 \c_d)}{1 - 1.2  \lr}  \uV{t}   \right)^{-1} \!\!\!\!\!\!   +  \frac{\lr}{1 - 1.1 \lr} \left( \frac{\Llf^2}{1 + 2 \c_d} \! - \!  \tilde C   \lr  \normL^2 \rs \Llf \right)  +  \uze{t+1} \\
& +  \uB{t+1}    - \frac{2 \c_d}{1 + 2 \c_d} \Llf \\
& =   \uV{t+1}    +  \uze{t+1} +  \uB{t+1}    - \frac{2 \c_d}{1 + 2 \c_d} \Llf, 
}
where we used Proposition \ref{prop:taylorresidual}  and \eqref{eq:Bund} in \eqref{bsyms:ineqq9}. 
\end{proof}

\subsection{Analysis of the bounding systems}
\subsubsection{Lower bounding system}  
In this section, we consider   \eqref{eq:lbsystem}. For notational convenience, we multiply both sides by the factor $(1 + 2 \c_d)$ and use a generic learning rate $\lr$, i.e.,
\eq{
\uV{t+1} = \uV{t} (\bs{I}_{\rk} + \lr   \uV{t}  )^{-1} + \lr \left(  \Llf^2 - \tilde C   \lr \normL^2 \rs \Llf \right), ~~ \text{where} ~ \uV{t} = 2 \Lls^{\frac{1}{2}}  \uG{t}  \Lls^{\frac{1}{2}} - \Llf.
}
The main result of this section is stated in Proposition~\ref{prop:lbsystemanalysis}. To establish it, we first prove an auxiliary result, Lemma~\ref{lem:dncond}.   For the following,   we define
\eq{
 \hL \coloneqq  \sqrt{   \Llf^2  -   \tilde C   \lr \normL^2 \rs \Llf    } = \text{diag}(\{ \hat{\lambda}_i \}_{i = 1}^r), \quad \D_t \coloneqq   \frac{ \Lls^{-1} \hL   \left(  \frac{\A_{t,11}}{\A_{t,12}}   -  \bs{I}_{\rk} \right) }{2}  -  \frac{1.1 \c_d \rs}{d} \bs{I}_{\rk}.
}
 By  Corollary \ref{cor:discretericcatidyn},  we have
\eq{
 \uG{t} \! = \!   \frac{1}{2}  \left( \frac{\Llf}{\Lls} \! + \! \frac{\A_{t,22}}{\A_{t,12}} \frac{\hL}{\Lls} \right)  
   \! -  \!  \frac{1}{4}    \frac{\A_{t,12}^{-1}  \hL}{\Lls}    \left( \frac{  \frac{\hL}{\Lls} \left(  \frac{\A_{t,11}}{\A_{t,12}}  - \bs{I}_{\rk} \right) }{2} \!  + \!   \frac{ \left(\hL - \Llf  \right)}{2 \Lls} +\uG{0} \! \right)^{-1}  \!\!\!\!   \frac{\hL \A_{t,12}^{-1}}{\Lls},
   \label{eq:uGdyn}
}
where $\A_{t,11}$,   $\A_{t,12}$,  and $ \A_{t,22}$ are defined as in \ref{res:posdef} with   $\hL$.  \emph{For $\alpha = 0$,  we will consider $\{ \uG{t} \}_{t \in \N}$ in the basis of  $\uG{0}$ without writing explicitly, which will imply that $\{ \uG{t} \}_{t \in \N}$ is  diagonal due to the rotational symmetry for  $\alpha = 0$.}

\smallskip
We further decompose $\{ \uG{t} \}_{t \in N}$ and related matrices to isolate their top-left submatrices of dimension  $\rks \in \{ r_{\star}, r_{u_\star}  \}$, where  $r_{\star}  < r$  and $r_{u_\star} < \ru$ which we will denote as $\rks < \rk$.  The decompositions are as follows:
\eq{
\uG{t}   \coloneqq \begin{bmatrix}
\uG{t,11} & \uG{t,12}  \\[0.1em]
\uG{t,12}^\top & \uG{t,22}  
\end{bmatrix},  \quad
\hL \coloneqq  \begin{bmatrix}
\hL_{11} & 0 \\
0 & \hL_{22} 
\end{bmatrix},  \quad
\D_t \coloneqq  \begin{bmatrix}
\D_{t,1} & 0 \\
0 &  \D_{t,2}
\end{bmatrix},  \quad
 \Z_{1:\rk}   \coloneqq \begin{bmatrix}
\Z_{1:\rks}\\
\Z_2
\end{bmatrix},  
}
where $\uG{t,11}, \D_{t,1},  \hL_{11} \in \R^{\rks \times \rks}$.
We define
\eq{
\bs{\Gamma}_{t} \coloneqq  \D_t +   \frac{1}{1.05}   \Z_{1:\rk}  \Z_{1:\rk}^\top  = \begin{bmatrix}
  \frac{1}{1.05}  \Z_{1:\rks}  \Z_{1:\rks}^\top +   \bs{D}_{t,1} &    \frac{1}{1.05 d}   \Z_{1:\rks} \Z_2^\top \\[0.5em]
  \frac{1}{1.05} \Z_2  \Z_{1:\rks}^\top  &     \frac{1}{1.05}   \Z_2 \Z_2^\top +   \bs{D}_{t,2}
\end{bmatrix}  
}
and
\eq{
\bs{\Gamma}_{t}^{-1} \coloneqq   \begin{bmatrix}
(\bs{\Gamma}_{t}^{-1}  )_{11} & ( \bs{\Gamma}_{t}^{-1} )_{12} \\
 (\bs{\Gamma}_{t}^{-\top} )_{12} &  (\bs{\Gamma}_{t}^{-1}  )_{22} 
\end{bmatrix}
}
whenever  $\bs{\Gamma}_{t}$ is invertible.  Lemma \ref{lem:dncond} is stated as follows:
\begin{lemma}
\label{lem:dncond}
We consider the following setting:
\begin{alignat}{4}
&\alpha \in [0, 0.5): \quad 
&& \frac{\rs}{r} \to \varphi \in (0, \infty),  \quad  
&&  \lr \ll \frac{1}{d\, r^{1 - \alpha} \log^4 d},  \quad  
&&  \c_d = \frac{1}{\log^{3.5} d},  
    \\[0.5em]
& \alpha > 0.5: \quad 
&& \rs  \asymp 1, \quad  
&&\lr \ll \frac{1}{d\, \ru^{2+\alpha} \log^3 d}, \quad  
&& \c_d = \frac{1}{\ru \log^{2.5} d}. 
\end{alignat}
$\mathcal{G}_{\text{init}}$ implies the following:
\begin{itemize}[leftmargin = *]
\item  For $\alpha \geq 0$ and $ K \leq \rks \leq \rk$,  we have for $\lr t  \leq  \frac{1}{2}   (K + 1)^{\alpha}   \log \left( \frac{d \log^{1.5} d}{\rs} \right)$,
\eq{
\D_t \! \succeq   \! \frac{   \Lls^{-1} \hL   \left(  \frac{\A_{t,22}}{\A_{t,12}}   -  \bs{I}_{\rk} \right) }{2} \! - \!  \frac{1.2 \c_d \rs}{d}  \bs{I}_{\rk}, \quad \bs{D}_{t,2} \! \succeq  \!  \frac{\log^3 d - 1}{\log^3 d} \! \left(  \frac{0.5  \rs}{d \log^{1.5} d} \right)^{\left(  \frac{K+1}{\rks + 1} \right)^{\alpha}} \!\!\!\!\!  \bs{I}_{\rk - \rks}.
}
\item For   $r_{\star}= \floor{  \rs \big(1 -  \log^{\frac{-1}{2}} \! d  \big) \wedge r}$ and  $r_{u_{\star}} = \rs$,  and  $\lr t \leq   \frac{1}{2}  (\rks  + 1)^{\alpha}  \log \left( \frac{d \log^{1.5} d}{\rs}   \right)$,  we have
\eq{
\bs{\Gamma}_{t} \succeq    \frac{ \Lls^{-1} \hL   \left(  \frac{\A_{t,22}}{\A_{t,12}}   -  \bs{I}_{\rk} \right) }{2}  + \begin{bmatrix}
  \frac{C_1  \rs }{d \log^{4.5} d} \bs{I}_{\rks} & 0 \\
0 &  - \frac{C_2   \rs}{d \log^2 d} \bs{I}_{\rk - \rks}
\end{bmatrix}    \succ 0,  \label{eq:mnlb}
}
where
\eq{
C_1 = \begin{cases}
\frac{1}{10}, & \alpha \in [0,0.5) \\[0.25em]
\left( \frac{1}{1.1 \rs^6} - \frac{1.3}{\sqrt{\log d}}  \right) , & \alpha > 0.5
\end{cases}   ~~ \text{and} ~~ C_2 = \begin{cases}
2.1 \left(1 + \frac{1}{\sqrt{\varphi}} \right)^2, & \alpha \in [0,0.5) \\[0.25em]
2, & \alpha > 0.5.
\end{cases} 
}
Within the same time interval, we have
\eq{
\bs{\Gamma}_{t}^{-1} \preceq \left(     \frac{    \Lls^{-1} \hL   \left(  \frac{\A_{t,22}}{\A_{t,12}}   -  \bs{I}_{\rk} \right) }{2}  + \begin{bmatrix}
  \frac{C_1 \rs}{d \log^{4.5} d} \bs{I}_{\rks} & 0 \\
0 &    \frac{- C_2 \rs}{d \log^2 d}\bs{I}_{\rk - \rks}
\end{bmatrix}  \right)^{-1}  .  \label{eq:mninvub}
}
\item  For   $\alpha > 0$,  we have 
\eq{
( \bs{\Gamma}_{t}^{-1}  )_{11}  \preceq \left(\D_{t,1} + \frac{1}{2} \Z_{1:\rks} \Z_{1:\rks}^\top \right)^{-1},
}
for   $0.001 >  \delta \geq \log^{\frac{-1}{4}} \! d$
\eq{
\begin{cases}
r_{\star} = \floor{  \rs \big(1 -  \delta  \big) \wedge r } ~~ \text{and} ~~\lr t \leq   \frac{1}{2}   \left(    \rs  \big(1 -  \sqrt{\delta}  \big) \wedge r  \right)^{\alpha} \log \left( \frac{d \log^{1.5} d}{\rs}   \right),  & \alpha \in (0,0.5) \\[0.25em]
r_{u_\star} = \rs  ~~ \text{and} ~~\lr t \leq   \frac{1}{2}  \rs^{\alpha} \log \left( \frac{d \log^{1.5} d}{\rs}   \right),   &  \alpha > 0.5
\end{cases}
}
provided that
\eq{
\begin{cases}
d \geq \Omega_{\beta}(1) \vee \exp \big(  2.5 \alpha^{-8} \big) ,  & \alpha \in (0,0.5) \\[0.4em]
d \geq \Omega_{\rs}(1)  ,  & \alpha > 0.5.
\end{cases}
}
\end{itemize}
\end{lemma}

\begin{proof}
For the first part of the first item,  by   \ref{res:symdet},  we have
\eq{
\D_t  & =   \frac{   \Lls^{-1} \hL   \left(  \frac{\A_{t,22}}{\A_{t,12}}   -  \bs{I}_{\rk} \right) }{2} - \frac{\lr}{2}   \Lls^{-1} \hL^2  -    \frac{1.1 \c_d \rs}{d}  \bs{I}_{\rk}   \labelrel\succeq{mnlb:ineqq99} \frac{   \Lls^{-1} \hL   \left(  \frac{\A_{t,22}}{\A_{t,12}}   -  \bs{I}_{\rk} \right) }{2}  -    \frac{1.2 \c_d \rs}{d}  \bs{I}_{\rk},
}
where  \eqref{mnlb:ineqq99} follows $\lr \ll \tfrac{\c_d \rs}{d}$.
Moreover,  since    $\Lls^{-1} \hL \succeq ( 1 - \tfrac{1}{\log^3 d} )\bs{I}_{\rk}$, by \ref{res:ratiobound}, we have
\eq{
\D_{t,2}  &  \succeq \left(1 - \frac{1}{\log^3 d} \right)  \left[   \frac{\left(\bs{I}_{\rk - \rks} - \lr \hL_{22} \right)^{t}}{\left(\bs{I}_{\rk - \rks} + \lr \hL_{22} \right)^t - \left(\bs{I}_{\rk - \rks} - \lr \hL_{22} \right)^t} -     \frac{1.3 \c_d \rs}{d}  \bs{I}_{\rk - \rks}  \right]. \label{eq:dn2lb}
}
We observe that
\eq{
 \frac{\left(\bs{I}_{\rk - \rks} - \lr \hL_{22} \right)^{t}}{\left(\bs{I}_{\rk - \rks} + \lr \hL_{22} \right)^t \!\! - \! \left(\bs{I}_{\rk - \rks} - \lr \hL_{22} \right)^t}  
 & \succeq  \frac{ \bs{I}_{\rk - \rks}  }{ \exp \left( \frac{2 t\lr \hL_{22} }{ 1 - \lr \hL_{22}}  \right)- \bs{I}_{\rk - \rks}} \\
& \labelrel\succeq{mnlb:ineqq0}   \left( \frac{\rs}{d \log^{1.5} d} \right)^{(1 + 2\lr \rks^{-\alpha}) \left(\frac{K+1}{\rks + 1} \right)^{\alpha}}      \bs{I}_{\rk - \rks}  \\
& \labelrel\succeq{mnlb:ineqq1}  \left(  \frac{0.9  \rs}{d \log^{1.5} d} \right)^{\left(\frac{K+1}{\rks + 1} \right)^{\alpha}}       \bs{I}_{\rk - \rks},
}
where  we use   $\lr t \leq   \frac{1}{2} (K + 1)^{\alpha}   \log \left( \frac{d \log^{1.5} d}{\rs}   \right)$   in \eqref{mnlb:ineqq0} and $d \geq \Omega(1)$ in  \eqref{mnlb:ineqq1}.  By \eqref{eq:dn2lb}, the first item follows.
 
For the second item,   by Proposition \ref{prop:schurlowerbound} with  $\varepsilon = \tfrac{1}{\log^2 d}$ for $\alpha \in [0,0.5)$ and $\varepsilon = \tfrac{1}{\ru \log^2 d}$ for  $\alpha> 0.5$, we have
\eq{
\bs{\Gamma}_{t} & \succeq  \frac{  \Lls^{-1} \hL   \left(  \frac{\A_{t,22}}{\A_{t,12}}   -  \bs{I}_{\rk} \right) }{2}  - \frac{1.2 \c_d \rs}{d} \bs{I}_{\rk}    
+\frac{1}{1.05}  \begin{bmatrix}
  \varepsilon   \Z_{1:\rks} \Z_{1:\rks}^\top & 0 \\
   0 & \frac{- \varepsilon}{1 - \varepsilon}  \Z_2 \Z_2^\top
   \end{bmatrix}.
}
For   $\alpha \in [0,0.5)$,  since  $\c_d = \tfrac{1}{\log^{3.5} d}$,  by \ref{event:htZlb}, we have 
\eq{
   \frac{\varepsilon}{1.05}  \Z_{1:r_{\star}} \Z_{1:r_{\star}}^\top - \frac{1.2 \c_d \rs}{d} \bs{I}_{r_{\star}}
\succeq  \frac{\rs}{d \log^3 d} \frac{1}{6.25} \bs{I}_{r_{\star}} - \frac{1.2 \c_d \rs}{d}  \bs{I}_{r_{\star}} \succ   \frac{1}{10}  \frac{\rs}{d \log^3 d} \bs{I}_{r_{\star}} .
}
Similarly  by  \ref{event:htZub},  we have 
\eq{
\frac{-  \varepsilon}{1 - \varepsilon} \frac{1}{1.05}  \Z_2 \Z_2^\top  - \frac{1.2 \c_d \rs}{d} \bs{I}_{r - r_{\star}} \succeq -   \left(1 + \frac{1}{\sqrt{\varphi}} \right)^2    \frac{ 2.1 \rs}{d \log^2 d} \bs{I}_{r -  r_{\star}}.
}
For   $\alpha > 0.5$,  since $\c_d = \tfrac{1}{\ru  \log^{2.5} d}$ and $\ru = \ceil{\log^{2.5} d} $,  by \ref{event:ltZlb},  we have 
\eq{
   \frac{\varepsilon}{1.05}  \Z_{1:\rus} \Z_{1:\rus}^\top - \frac{1.2 \c_d \rs}{d} \bs{I}_{r_{u_\star}}
& \succ   \frac{\rs}{d \log^{4.5} d} \left( \frac{1}{1.1 \rs^6} - \frac{1.3}{\sqrt{\log d}}  \right) \bs{I}_{r_{u_\star}} .
}
Similarly  by  \ref{event:ltZub},
\eq{
\frac{-  \varepsilon}{1 - \varepsilon} \frac{1}{1.05 d}  \Z_2 \Z_2^\top  - \frac{ 1.2 \c_d \rs}{d} \bs{I}_{\ru - r_{u_\star}} \succeq        \frac{ - 2 \rs}{d \log^2 d} \bs{I}_{\ru -  r_{u_\star}}.
}

Therefore,  we have   \eqref{eq:mnlb}.   By Proposition \ref{prop:monotonecharac}, we have  \eqref{eq:mninvub}.

For the last item, we have
\eq{
(\bs{\Gamma}_{t}^{-1}  )_{11}  
& = \left( \D_{t,1}  + \frac{1}{1.05} \Z_{1:\rks}  \left( \bs{I}_{\rs} +  \frac{1}{1.05} \Z_2^\top   \bs{D}_{t,2}^{-1}   \Z_2  \right)^{-1} \Z_{1:\rks}^\top   \right)^{-1}. \label{eq:gninvtop}
}
For $\alpha \in (0,0.5),$  if $r_{\star} = r$, the statement follows. If not by the first item,  for  $K =   \floor{  \big(1 - \sqrt{\delta}  \big) \rs}$  and $r_{\star} =   \floor{  \rs \big(1 -  \delta  \big)}$, we have
\eq{
\frac{K+1}{r_{\star} + 1} \leq  \frac{1 - \sqrt{\delta}}{1 - \delta} + \frac{2}{\rs}
\leq 1 - 0.9 \sqrt{\delta} ~ \Rightarrow  ~ \left( \frac{K+1}{r_{\star}  + 1} \right)^{\alpha} \leq 1 -    \alpha 0.9 \sqrt{\delta}.
}
Therefore,
\eq{
 \bs{D}_{t,2} \succeq   \left(  \frac{0.5\rs}{d  \log^{1.5} d } \right)^{ 1 -    \alpha 0.9 \sqrt{\delta} } \bs{I}_{r - r_{\star}} 
\labelrel\succeq{mnlb:ineqq2}    \frac{0.5\rs}{d  \log^{1.5} d }    \left(  \frac{d}{\rs}   \right)^{\log^{\! \nicefrac{- 1}{4}} \! d}      \bs{I}_{r - r_{\star}} 
\labelrel\succeq{mnlb:ineqq3}  \frac{  \rs \log d }{d} \bs{I}_{r - r_{\star}},
}
where we used $d \geq \Omega(1) \vee \exp(2.5 \alpha^{-8})$ in \eqref{mnlb:ineqq2} and  $ d \geq \Omega_{\beta}(1)$ in \eqref{mnlb:ineqq3}. By \eqref{eq:gninvtop} and \ref{event:htZub}, we have the statement for $\alpha \in (0,0.5)$.  For $\alpha > 0.5$,    $K =  r_{\star} =   \rs$, we have
\eq{
 \left( \frac{K+1}{r_{\star}  + 1} \right)^{\alpha} \leq  \left( 1 + \frac{1}{\rs + 1} \right)^{0.5} \leq 1 - \frac{1}{2 (\rs + 1)}.
 }
 Therefore,
 \eq{
 \bs{D}_{t,2} \succeq   \left(  \frac{0.5\rs}{d  \log^{1.5} d } \right)^{1 - \frac{1}{2(\rs + 1)} } \bs{I}_{\ru - r_{u_\star}} \succeq  \frac{  \rs \log^8 d }{d} \bs{I}_{r - r_{\star}}
}
for $d \geq \Omega_{\rs}(1)$.  By \eqref{eq:gninvtop} and   \ref{event:ltZub}, we have the statement for $\alpha > 0.5$.  
\end{proof}

\begin{proposition}
\label{prop:lbsystemanalysis}
Let
\eq{
\uG{0} = 
(1 + 2\c_d) \left( \sG_0 - \frac{\c_d \rs}{d} \bs{I}_{\rk}\right),
}
Under the parameter choice in Lemma \ref{lem:dncond},
$\mathcal{G}_{\text{init}}$ guarantees  that:
\begin{itemize}[leftmargin=*]
\item We have $\Omega \big(  -  \log^{\frac{-1}{2}} \! d \big) \bs{I}_{\rk} \preceq \uG{t}$ whenever
\eq{
\lr t \leq    
\begin{cases}
\frac{1}{2}  \left( \rs   \big(1 -  \log^{\frac{- 1}{2}} \!\! d   \big) \wedge r \right)^{\alpha}   \log \left( \frac{d\log^{1.5} d}{\rs} \right),  & \alpha \in [0,0.5) \\[0.5em]
\frac{1}{2}  \rs^{\alpha}   \log \left( \frac{d\log^{1.5} d}{\rs} \right),    &  \alpha  > 0.5
\end{cases} .
}
\item Let $\bs{\Lambda}_{11} $be the $\rks \times \rks$ dimensional top-left sub-matrix of $\bs{\Lambda}$. Given  $0.001 \geq \delta \geq \log^{\frac{-1}{4}} \! d$ and  $\rks = \Big \{ r_{\star} = 
\floor{  \rs \big(1 -  \delta  \big) \wedge r },
 r_{u_\star} =  \rs \Big \}$,  we have  
\eq{
 \uG{t,11}   \succeq       \frac{1 - \frac{10}{\log^3 d}}{  \frac{1.2}{C_{\text{lb}}}  \frac{d}{\rs}  \exp \left( - 2 \lr t  \bs{\Lambda}_{11}  \right)   + 1 }
}
and
\eq{ 
\norm{\hL}_F^2 -  \norm{ \hL_1^{\frac{1}{2}} \uG{t,11} \hL_1^{\frac{1}{2}} }_F^2  \leq  \!\!\!\!\! \sum_{i  =  (\rks     \wedge r) + 1 }^r  \!\!\!\!\!  \hat{\lambda}_i^2 + \sum_{i = 1}^{\rks} \hat{\lambda}_i^2  \bigg(   1 -    \frac{1 - \frac{10}{\log^3 d}}{  \frac{1.2}{C_{\text{lb}}}  \frac{d}{\rs}  \exp \left( - 2 \lr t  \lambda_i  \right)   + 1 }      \bigg)^2, 
  }
for
\eq{
C_{\text{lb}} = \frac{1}{15} \begin{cases}
\delta^2, & \alpha \in [0,0.5) \\
\frac{1}{\rs^6}, & \alpha > 0.5
\end{cases} ~~ \text{and} ~~
\lr t \leq 
\begin{cases}
  \frac{1}{2}   \left(    \rs  \big(1 -  \sqrt{\delta}  \big) \wedge r  \right)^{\alpha} \log \left( \frac{d \log^{1.5} d}{\rs}   \right),  & \alpha \in [0,0.5) \\
 \frac{1}{2}  \rs^{\alpha} \log \left( \frac{d \log^{1.5} d}{\rs}   \right),   &  \alpha > 0.5.
\end{cases}
}
\item  For   $\delta =  \log^{\frac{- 1}{4}} \!\! d$,  we define
\eq{
\mathcal{T}_{\text{lb}} \coloneqq \inf \left \{  n \geq 0 ~ \middle \vert ~ \norm{\hL}_F^2 -  \norm{ \hL_1^{\frac{1}{2}} \uG{t,11} \hL_1^{\frac{1}{2}} }_F^2  \leq  \sum_{j =  (\rs     \wedge r) + 1 }^r  \lambda_j^2 +   \frac{3  \norm{\hL}_F^2 }{\log^{\frac{1}{8}} d} 
\right \}.
}
Then,  
\eq{
\mathcal{T}_{\text{lb}}  \leq   \begin{cases}
\frac{1}{2 \lr}    \left(   \rs \big(1 - \log^{\frac{- 1}{8}} \! d  \big) \wedge r  \right)^{\alpha} \log \left( \frac{  20 d \log^{\frac{3}{4}}   (1 + d/\rs)}{\rs}   \right),  & \alpha \in [0,0.5) \\ 
 \frac{1}{2 \lr}  \rs^{\alpha} \log \left( \frac{  20 d \log^{\frac{3}{4}} \! d}{\rs}   \right), & \alpha > 0.5.
\end{cases}
}
\end{itemize}
\end{proposition}
 
\begin{proof}
For $\alpha > 0.5$,  we assume that $d$ is large enough to guarantee that $\left( \frac{1}{1.1 \rs^6} - \frac{1.3}{\sqrt{\log d}}  \right)  > 0$.  We observe that  
\eq{
 \frac{\Lls^{-1}  \hL \left(  \frac{\A_{t,11}}{\A_{t,12}}  - \bs{I}_{\rk} \right) }{2} & + \frac{  \Lls^{-1}  \left(\hL -  \Llf  \right)}{2}  + \uG{0}   \succeq   \bs{\Gamma}_t,
} 
where we used \ref{event:Zbound}, $\Lls \succeq \Llf$ and $\lr \normL^2 \rs \bs{I}_{\rk} \ll \tfrac{\c_d \rs}{d} \Llf$.

\smallskip
For the first item,  by using $\rks = \big \{ r_{\star} = 
\floor{  \rs \big(1 -  \log^{\frac{-1}{2}} \! d  \big) \wedge r}, 
 r_{u_\star} =  \rs \big \}$,  we define  
\eq{
& \D_{lb}   \coloneqq    \begin{bmatrix}
 \tfrac{C_1 \rs}{d \log^{4.5} d} \bs{I}_{\rks} & 0 \\
0 &   \tfrac{-  C_2 \rs}{d \log^2 d} \bs{I}_{\rk - \rks}
\end{bmatrix} ~~ \text{and} ~~  \tilde{\D}_{lb} \coloneqq  \frac{\Lls}{\hL}   \D_{lb} .
}
We introduce submatrix notation for block-diagonal matrices. Specifically, we write
\eq{
\tilde{\D}_{lb} =   \begin{bmatrix} 
\tilde{\D}_{lb,1} & 0 \\
0 & \tilde{\D}_{lb,2}
\end{bmatrix} ~~ \text{and} ~~ 
\frac{\A_{t,22}}{\A_{t,12}} \pm \bs{I}_{\rk} =   \begin{bmatrix} 
\left( \frac{\A_{t,22}}{ \A_{t,12} } \pm  \bs{I}_{\rk} \right)_{11} & 0 \\
0 &  \left( \frac{\A_{t,22}}{ \A_{t,12} } \pm \bs{I}_{\rk} \right)_{22}
\end{bmatrix},  \label{eq:submatrixlb}
}
where the block dimensions of each submatrix match those of $\D_{lb}$.  We start with proving the lower bound part.  By the second item in Lemma \ref{lem:dncond},  we have
\eq{
\uG{t} &\succeq \frac{1}{2} \sqrt{  \frac{\hL}{\Lls} } \left(  \left(  \frac{\A_{t,22}}{\A_{t,12}} + \bs{I}_{\rk}  \right)  -     \A_{t,12}^{-1}    \left(  \left(  \frac{\A_{t,22}}{\A_{t,12}}     - \bs{I}_{\rk} \right)  + 2  \tilde{\D}_{lb} \right)^{-1}     \A_{t,12}^{-1}    \right)   \sqrt{  \frac{\hL}{\Lls} } \\
& + \frac{\frac{\Llf}{\Lls}  - \frac{\hL}{\Lls}}{2}   \\
&   \succeq   \frac{1}{2} \sqrt{  \frac{\hL}{\Lls} } \left(  \left(  \frac{\A_{t,22}}{\A_{t,12}} + \bs{I}_{\rk}  \right)  -     \A_{t,12}^{-1}    \left(  \left(  \frac{\A_{t,22}}{\A_{t,12}}     - \bs{I}_{\rk} \right)  + 2   \tilde{\D}_{lb}\right)^{-1}    \A_{t,12}^{-1}    \right)   \sqrt{  \frac{\hL}{\Lls} } ,  \label{eq:lb4lbsys0}
}
where we used   $\Llf \succ \hL$ in the second line.
We have
\eq{
\Big(   \frac{\A_{t,22}}{\A_{t,12}}  \!+ \!  \bs{I}_{\rk}  \Big)   \! - \!    \A_{t,12}^{-1}  &  \Big(  \Big(  \frac{\A_{t,22}}{\A_{t,12}}     -  \bs{I}_{\rk} \Big)     +   2  \tilde{\D}_{lb} \Big)^{-1}    \A_{t,12}^{-1}  \\
  & = \frac{(\A_{t,22} + \A_{t,12}) ( \A_{t,22} - \A_{t,12} + 2  \tilde{\D}_{lb}  \A_{t,12}  )  - \bs{I}_{\rk} }{ \A_{t,12}  ( \A_{t,22} - \A_{t,12} + 2 \tilde{\D}_{lb} \A_{t,12}  ) } \\
  & \labelrel\succeq{lbcondd:ineqq0}   \frac{  2 \tilde{\D}_{lb}   \left( \frac{\A_{t,22}}{ \A_{t,12} } + \bs{I}_{\rk} \right) }{   \frac{\A_{t,22}}{ \A_{t,12} } - \bs{I}_{\rk}  + 2 \tilde{\D}_{lb}  } \\[0.5em]
  &=  \begin{bmatrix}
    \frac{  2  \tilde{\D}_{lb,1}   \left( \frac{\A_{t,22}}{ \A_{t,12} } + \bs{I}_{\rk} \right)_{11} }{  \left( \frac{\A_{t,22}}{ \A_{t,12} } - \bs{I}_{\rk} \right)_{11}  + 2  \tilde{\D}_{lb,1}   } & 0 \\
    0 &      \frac{  2 \tilde{\D}_{lb,2}    \left( \frac{\A_{t,22}}{ \A_{t,12} } + \bs{I}_{\rk} \right)_{22} }{  \left( \frac{\A_{t,22}}{ \A_{t,12} } - \bs{I}_{\rk} \right)_{22} + 2  \tilde{\D}_{lb,2}   } 
\end{bmatrix},   \label{eq:lbmatrix}
}
where we used   $\A_{t,22}^2 - \A_{t,12}^2 \succ \bs{I}_{\rk}$ (by \ref{res:symdet}) and   $\A_{t,22} - \A_{t,12} + 2 \tilde{\D}_{lb} \A_{t,12} \succ 0$ (by \eqref{eq:mnlb})  in \eqref{lbcondd:ineqq0}.
Since    $ \tfrac{\A_{t,22}}{ \A_{t,12} } - \bs{I}_{\rk}  \succ 0$ and  $ \tilde{\D}_{lb,1} \succ 0$,  it is enough to look at the bottom-right  submatrix in \eqref{eq:lbmatrix} for the lower bound part.  We have  
\eq{
  \frac{  2  \tilde{\D}_{lb,2}  \left( \frac{\A_{t,22}}{ \A_{t,12} } + \bs{I}_{\rk} \right)_{22} }{  \left( \frac{\A_{t,22}}{ \A_{t,12} } - \bs{I}_{\rk} \right)_{22} + 2 \tilde{\D}_{lb,2}   }   =  \frac{  2  \tilde{\D}_{lb,2}     \left( \frac{\A_{t,22}}{ \A_{t,12} } + \bs{I}_{\rk} \right)_{22} }{  \left( \frac{\A_{t,22}}{ \A_{t,12} } + \bs{I}_{\rk} \right)_{22} - 2  \bs{I}_{\rk- \rks}  + 2  \tilde{\D}_{lb,2}   }.  \label{eq:lb4lbsys}
}
Note that by  \ref{res:ratiobound},
\eq{
\left( \frac{\A_{t,22}}{ \A_{t,12} } + \bs{I}_{\rk} \right)_{2}  & \succeq \frac{2 (\bs{I}_{\rk - \rks} + \lr \hL_2  )^t }{ (\bs{I}_{\rk - \rks} + \lr \hL_2  )^t -  (\bs{I}_{\rk - \rks} - \lr \hL_2  )^t}   \\
&  \succeq 2 \bs{I}_{\rk - \rks}  +  \frac{ 2 (\bs{I}_{\rk - \rks} - \lr^2 \hL^2_2  )^t \exp \left( - 2 t \lr \hL_2  \right)  }{ \bs{I}_{\rk - \rks} -  (\bs{I}_{\rk - \rks} - \lr^2 \hL^2_2  )^t \exp \left( - 2 t \lr \hL_2  \right)} \\
&\labelrel\succeq{lbcondd:ineqq1}  \left( 2  + \frac{0.9 \rs}{d \log^{1.5} d}  \right) \bs{I}_{\rk -  \rks}, 
}
where we use $\lr t \leq \frac{1}{2} (\rks + 1)^{\alpha} \log \left( \tfrac{d \log^{1.5} d}{\rs}  \right)$ in \eqref{lbcondd:ineqq1}.  Hence, for $d \geq \Omega(1)$
\eq{
 \eqref{eq:lb4lbsys} \succeq   \frac{ 2  \left( 2  + \frac{0.9 \rs}{d \log^{1.5} d}  \right)   \tilde{\D}_{lb,2}   }{       \frac{0.9 \rs}{d \log^{1.5} d}   \bs{I}_{\rk - \rks}  +   \tilde{\D}_{lb,2}  } \labelrel\succeq{lbcondd:ineqq05}    \frac{ 12   \tilde{\D}_{lb,2}   }{       \frac{\rs}{d \log^{1.5} d}   \bs{I}_{\rk - \rks}    }  \labelrel\succeq{lbcondd:ineqq075}   \frac{- 15 C_2}{\log^{0.5} d}   \bs{I}_{\rk - \rks},  \label{eq:lbb}
}
where we used   $  \tilde{\D}_{lb,2}  \succeq \frac{- 1.1 C_2 \rs}{d \log^2 d} \bs{I}_{\rk}$ in \eqref{lbcondd:ineqq05} and \eqref{lbcondd:ineqq075}.  The  first item follows from \eqref{eq:lbb}.

\smallskip
For the second and third items,  let  $\big( \tfrac{\A_{t,22}}{ \A_{t,12} } \pm \bs{I}_{\rk} \big)_{11}$ denote the $\rks \times \rks$ dimensional top-left submatrices with  $\rks = \big \{ r_{\star} =  \floor{  \rs \big(1 - \delta   \big) \wedge r },  r_{u_\star} =  \rs \big \}$.   By using the third item in Lemma \ref{lem:dncond}, we immediately observe that  for $\alpha > 0$,    $\uG{t,11} \succeq 0$ and
 \eq{
 \uG{t,11}   \labelrel\succeq{lbcondd:ineqq2}   \left (1 - \frac{10}{\log^3 d} \right)   \frac{1}{2} \frac{   \frac{2 C_{\text{lb}} \rs}{d}  \left( \tfrac{\A_{t,22}}{ \A_{t,12} } + \bs{I}_{\rk} \right)_{11} }{ \left( \tfrac{\A_{t,11}}{ \A_{t,12} } - \bs{I}_{\rk} \right)_{11} + \frac{2 C_{\text{lb}} \rs}{d} \bs{I}_{\rks}   }, \label{eq:gntoplb}
}
 for 
\eq{
C_{\text{lb}} = \frac{1}{15} \begin{cases}
\delta^2, & \hspace{-0.5em} \alpha \in (0,0.5) \\
\frac{1}{\rs^6}, &  \hspace{-0.5em} \alpha > 0.5
\end{cases} ~ \text{and} ~ 
 \lr t \leq \begin{cases}
 \frac{1}{2} \left(   \rs \big(1 - \sqrt{\delta}  \big) \wedge r \right)^{\alpha} \log \left( \tfrac{d \log^{1.5} d}{\rs}  \right),  &  \hspace{-0.5em} \alpha \in (0,0.5)   \\
 \frac{1}{2} \rs^{\alpha}  \log \left( \tfrac{d \log^{1.5} d}{\rs}  \right) ,  &  \hspace{-0.5em} \alpha > 0.5, 
\end{cases}  \label{eq:etacond}
}
where we  used   $ \Lls \succeq \hL  \succeq   \left(1 - \tfrac{0.5}{\log^4 d} \right) \Lls$, and  followed the steps in \eqref{eq:lb4lbsys0}- \eqref{eq:lbmatrix} with  \ref{event:htZlb} and \ref{event:ltZlb}  to obtain \eqref{lbcondd:ineqq2}. Then,  by \ref{res:ratiobound}, we have
\eq{
 \frac{1}{2} \frac{   \frac{2C_{\text{lb}}   \rs}{d}  \left( \tfrac{\A_{t,22}}{ \A_{t,12} } + \bs{I}_{\rk} \right)_{11} }{ \left( \tfrac{\A_{t,11}}{ \A_{t,12} } - \bs{I}_{\rk} \right)_{11} + \frac{2 C_{\text{lb}}  \rs}{d} \bs{I}_{\rks}   }  
& \succeq \frac{\bs{I}_{\rks}}{\left( \frac{1}{C_{\text{lb}}} \frac{d}{\rs} - 1  \right) \frac{(\bs{I}_{\rks} - \lr \hL_1)^t}{ (\bs{I}_{\rks} + \lr \hL_1)^t  } + \bs{I}_{\rks} }   \\
& \succeq \frac{\bs{I}_{\rks}}{  \frac{1.1}{C_{\text{lb}}}  \frac{d}{\rs}  \exp \left( - 2  \lr t \hL_1  \right)   + \bs{I}_{\rks} }. \label{eq:lbuG}
}
Consequently, by observing $\hL \succeq \Llf - \tilde{C} \lr \bs{I}_{\rk}$ and using the lower bounds for $\Llf$ in Propositions \ref{prop:boundforiterationsheavy} and \ref{prop:boundforiterationslight}, we have
\eq{
  \uG{t,11}   \succeq       \frac{1 - \frac{10}{\log^3 d}}{  \frac{1.2}{C_{\text{lb}}}  \frac{d}{\rs}  \exp \left( - 2 \lr t  \bs{\Lambda}_{11}  \right)   + 1 },
}
where $\bs{\Lambda}_{11}$ denotes the $\rks \times \rks$ dimensional top-left sub-matrix of $\bs{\Lambda}$.
Therefore, 
\eq{
 \norm{\hL}_F^2 -  \norm{ \hL_1^{\frac{1}{2}} \uG{t,11} \hL_1^{\frac{1}{2}} }_F^2  \leq  \!\!\!\!\! \sum_{i  =  (\rks     \wedge r) + 1 }^r  \!\!\!\!\!  \hat{\lambda}_i^2 +
  \sum_{i = 1}^{\rks} ~\hat{\lambda}_i^2~  \Bigg(   1 -  \frac{1 - \frac{10}{\log^3 d}}{  \frac{1.2}{C_{\text{lb}}}  \frac{d}{\rs}  \exp \left( - 2   \lr t \lambda_i  \right)   + 1 }    \Bigg)^2,   \label{eq:uburisk}
 }
which proves the second item for $\alpha > 0$.  Moreover, since \eqref{eq:uGdyn} is in the eigenbasis of $\uG{0}$, the arguments in \eqref{eq:lb4lbsys0}-\eqref{eq:lbmatrix} and the condition in \ref{event:htZlbiso}  extend \eqref{eq:gntoplb} to $\alpha = 0$ in the eigenbasis of $\uG{0}$ for $d \geq \Omega(1)$.   Given \eqref{eq:gntoplb}, we can extend  \eqref{eq:uburisk} to $\alpha = 0$ as the Frobenious norm is basis independent.

\smallskip 
For  the third item,  for $\alpha>0.5$ and  $t \geq   \frac{1}{2 \lr}  \rs^{\alpha} \log \big( \tfrac{  20 d \log^{\frac{3}{4}} d}{\rs}   \big)$,  we have 
\eq{
\eqref{eq:uburisk} &  \leq    \sum_{i  =  (\rs    \wedge r) + 1 }^r \hat{\lambda}_i^2 +   \frac{\norm{\hL}_F^2 }{\log^{\frac{1}{2}} \! d} 
\leq    \sum_{i  =  (\rs    \wedge r) + 1 }^r  \lambda_i^2 +   \frac{\norm{\hL}_F^2 }{\log^{\frac{1}{2}} \! d},
}
which gives us the corresponding bound for $\mathcal{T}_{lb}$.    

For  $\alpha \in [0,0.5)$ and  $t \geq   \frac{1}{2 \lr}    \big(   \rs \big(1 - \log^{\frac{- 1}{8}} \!\! d  \big) \wedge r  \big)^{\alpha} \log \big( \tfrac{  20 d \log^{\frac{3}{4}} (1 +d/\rs)}{\rs}   \big)$,  we have 
\eq{
\eqref{eq:uburisk} & \leq   \sum_{i  =  (\rs    \wedge r) + 1 }^r \hat{\lambda}_i^2 +  \sum_{i  =  \floor{  \rs \big(1 - \log^{\frac{- 1}{8}} \!\! d  \big) \wedge r} + 1 }^{\rs \wedge r} \hspace{-2em} \hat{\lambda}_i^2    \\
& +
 \sum_{i = 1}^{\floor{  \rs \big(1 - \log^{\frac{- 1}{8}} \!\! d  \big) \wedge r} }  \hat{\lambda}_i^2  \hspace{0.1em} \Bigg(1 -  \frac{1 - \frac{10}{\log^3 d}}{  \frac{1.2}{C_{\text{lb}}}  \frac{d}{\rs}  \exp \left( - 2   \lr t \lambda_i  \right)   + 1 }  \Bigg)^2   \\
& \leq   \sum_{i  =  (\rs    \wedge r) + 1 }^r \lambda_i^2 + \frac{3 \norm{\hL}_F^2 }{\log^{\frac{1}{8}} \! d}, \label{eq:lburiskheavy}
}
which gives us its bound for $\mathcal{T}_{lb}$.  
\end{proof}

\subsubsection{Upper bounding system}  

In this section, we consider   \eqref{eq:ubsystem}. For notational convenience, we multiply both sides by the factor $(1 - 2 \c_d)$ and use a generic learning rate $\lr$, i.e.,
\eq{
\oV{t+1} = \oV{t} (\bs{I}_{\rk} + \lr   \oV{t}  )^{-1} + \lr \left(  \Luf^2 + \tilde C   \lr \normL^2 \rs \Luf \right), ~~ \text{where} ~~ \oV{t} = 2 \Lus^{\frac{1}{2}} \oG{t}  \Lus^{\frac{1}{2}} - \Luf.
}
The main result of this section is stated in Proposition~\ref{prop:goodeventres}. To establish it, we first prove an auxiliary result:

\begin{lemma}
\label{lem:ubsgo}
The following statement holds:
\begin{itemize}
\item  The reference sequence satisfies $\T_t \succeq \frac{\c_d \rs}{d} \bs{I}_{\rk}$  and   $\{ t \geq 0: \norm{\T_t}_2 >1.2  \c_d \}  = \infty$.
\item For $r_{u_\star} = 2 \rs$, we have
\eq{
\begin{cases}
\oG{0} = \frac{2.2 \left( 1 + \frac{1}{\sqrt{\varphi}}  \right)^2 \rs}{d} \bs{I}_r  \succeq (1 - 2 \c_d) \left(  \G_0 + \frac{\c_d \rs}{d} \bs{I}_r \right),  & \alpha \in [0,0.5) \\[0.8em]
\oG{0} = \frac{5.5}{d}  \left[ \begin{smallmatrix}
2 \rs  \bs{I}_{r_{u_\star}} & 0 \\
0 & \ru  \bs{I}_{r - r_{u_\star}}
\end{smallmatrix} \right]   \succeq (1 - 2 \c_d) \left( \G_{0,11} + \frac{\c_d \rs}{d} \bs{I}_{\ru}\right) ,   & \alpha > 0.5
\end{cases}  \label{eq:initialuG}
}
provided that $\mathcal{G}_{\text{init}}$ holds.
\end{itemize}
\end{lemma}

For the following,  we introduce $\hat{\T}_t \coloneqq \tfrac{\Lus}{\Llf} \tfrac{ \left( 3 \c_d +  1 \right)}{\c_d (1 - \c_d)} \T_t$.  Note that  for $d \geq \Omega(1)$, we have  
\eq{
\hat{\T}_{t+1}  = \hat{\T}_t  +  2 (1 -   2 \c_d)  \lr   \Llf  \hat{\T}_t \Big( \bs{I}_{\rk}  -  \hat{\T}_t   \Big) ~~ \text{and} ~~  \frac{\c_d}{1.1} \frac{\Llf}{\Lus} \preceq \frac{\T_t}{\hat{\T}_t} \preceq \c_d \bs{I}_{\rk}. \label{eq:auxtdyn}
}
By  Proposition \ref{prop:1dsystemaux}, we have
\eq{
1.1 \wedge  \hat{\T}_{0, ii}  \exp \left(  2 \lr t  \lambda_i    \right)   
&\geq \hat{\T}_{t,ii} \\
& \geq  \frac{1}{2}   \begin{cases}
1 \wedge  \hat{\T}_{0, ii}  \exp \Bigg(  \frac{  (1 -   2 \c_d)       2 \lr t   \big(\lambda_i - \frac{0.1 r^{-\alpha} }{ \log^{4} \! d} \big)   }{1 + 2  (1 -2 \c_d) \lr \lambda_i } \Bigg),  & \hspace{-0.5em} \alpha \in [0,0.5) \\[0.9em]
1 \wedge  \hat{\T}_{0, ii}  \exp \Bigg(  \frac{  (1 -  2 \c_d)       2 \lr  t \big(\lambda_i - \frac{1}{(\ru + 1)^{\alpha}} - \frac{0.1 }{\ru^{2 + \alpha} \log^{4} \! d} \big)   }{1  + 2  (1 - 2 \c_d) \lr \lambda_i } \Bigg),  &  \hspace{-0.5em}\alpha > 0.5.
\end{cases}
 \label{eq:auxtdyn2}
}

\begin{proof}[Proof of Lemma \ref{lem:ubsgo}]
For the first item,  by   Proposition \ref{prop:1dsystemaux}, we have   
\eq{
\hat{\T}_t \succeq  \hat{\T}_0 \labelrel\Rightarrow{ubsgo:if0}  \T_t \succeq   \T_0 = \frac{\c_d \rs}{d} \bs{I}_{\rk},
}
where we multiplied each side with $ \tfrac{\c_d (1 - \c_d)}{ 3 \c_d +  1} \tfrac{\Llf}{\Lus}$ for \eqref{ubsgo:if0}. Moreover,  by \eqref{eq:auxtdyn}-\eqref{eq:auxtdyn2}, we have
\eq{
 \T_t  \preceq \c_d   \hat{\T}_t   \preceq 1.1  \bs{I}_{\rk} \Rightarrow  \{ t \geq 0: \norm{\T_t}_2 >1.2  \c_d \}  = \infty.
}
The second item    follows \ref{event:Gbound} and \ref{event:htZub} (for $ \alpha \in [0,0.5)$) and \ref{event:ltZub} (for $ \alpha > 0.5$).
\end{proof}

\begin{proposition}
\label{prop:goodeventres}
We consider $\rk \in \{ r, \ru \}$,  where $\ru = \ceil{\log^{2.5} d}$,  and
\begin{alignat}{4}
&\alpha \in [0, 0.5): \quad 
&& \frac{\rs}{r} \to \varphi \in (0, \infty),  \quad  
&&  \lr \ll \frac{1}{d\, r^{1 - \alpha} \log^4 d},  \quad  
&&  \c_d = \frac{1}{\log^{3.5} d},  
    \\[0.5em]
& \alpha > 0.5: \quad 
&& \rs  \asymp 1, \quad  
&&\lr \ll \frac{1}{d\, \ru^{2+\alpha} \log^3 d}, \quad  
&& \c_d = \frac{1}{\ru \log^{2.5} d}. 
\end{alignat}
If   $\oG{0}$ are taken as  in  \eqref{eq:initialuG}, we have the following:
\begin{itemize}[leftmargin = *]
\item  $\{ \oG{t} \}_{n \in \N}$ is diagonal  and satisfies
\eq{
\frac{\rs}{d} \bs{I}_{\rk}  \preceq \oG{t+1}  \preceq     \oG{t} +  \lr \big( (1 + \c_d)  \Luf \oG{t} + (1 + \c_d) \oG{t} \Luf - 2   \oG{t} \Lus \oG{t} \big)   \preceq 1.1 \bs{I}_{\rk}.
}
\item   For $\alpha \in [0,0.5)$ and $d \geq \Omega(1)$, we have for $t\leq \frac{1}{2\lr} r^{\alpha} \log \left( \frac{d \log^{1.5} d}{\rs}  \right):$
\begin{itemize}[leftmargin=*]
\item   $\T_t^{- \frac{1}{2}} \oG{j} \T_t^{- \frac{1}{2}} \preceq  \frac{11  \left(1 + \frac{1}{\sqrt{\varphi}} \right)^2 }{\c_d}   \bs{I}_r$ for $0 \leq j \leq t$.
\item  $\T_t^{- \frac{1}{2}} \left( \lr \sum_{j = 1}^t \oG{j - 1} \right) \T_t^{- \frac{1}{2}} \preceq \frac{5.5  \left(1 + \frac{1}{\sqrt{\varphi}} \right)^2}{\c_d} (2 \lr t\vee r^{\alpha})   \bs{I}_r$.
\item  $ \oG{t}   \preceq \left(  1.1 \oG{0}  \exp \left(  2  \lr t  \bs{\Lambda}  \right) \wedge  \bs{I}_r \right) +    o_d(1)$
\item $\normL^2 - \tr(\bs{\Lambda} \oG{t} \bs{\Lambda} ) \geq \sum_{i = 1}^r  \lambda_i^2 \Big(  1 - \tfrac{2.5  \left(1 + \frac{1}{\sqrt{\varphi}} \right)^2  \rs}{d}   \exp \left(  2  \lr t \lambda_i  \right) \Big)_{+}  -   o_d(1)$.
\end{itemize}
\item For $\alpha> 0.5$ and $d \geq \Omega_{\rs}(1)$,  we have   for  $t\leq \frac{1}{2\lr} \rs^{\alpha} \log \left( \frac{d \log^{1.5} d}{\rs}  \right):$
\begin{itemize}[leftmargin=*]
\item   $\T_t^{- \frac{1}{2}} \oG{j} \T_t^{- \frac{1}{2}} \preceq  \frac{26.4 \ru}{\c_d}   \bs{I}_{\ru}$ for $0 \leq j \leq t$.
\item  $\T_t^{- \frac{1}{2}} \left( \lr \sum_{j = 1}^t \oG{j - 1} \right) \T_n^{- \frac{1}{2}} \preceq \frac{15 \ru}{\c_d} (2 \rs)^{\alpha} \log d \bs{I}_{\ru}$.
\item  $   \oG{t}     \preceq \left(  1.1 \oG{0}  \exp \left(  2  \lr t  \bs{\Lambda}_{11}  \right) \wedge  \bs{I}_{r_u} \right) +    o_d(1)$.
\item $\norm{ \bs{\Lambda}_{11}}_F^2  - \tr(\bs{\Lambda}_{11} \oG{t,11} \bs{\Lambda}_{11} ) \geq   \sum_{i = 1}^{\rs} \! \lambda_i^2 \Big(  1 \! - \! \tfrac{12.1  \rs}{d}   \exp \left(  2  \lr t  \lambda_i \right)   \Big)_{+} \!\!\! +   \sum_{i = \rs + 1}^{\ru} \! \lambda_i^2  \! -  \!  o_d(1).$  
\end{itemize}
\end{itemize}
\end{proposition}

\begin{proof}
Given that   $\tfrac{\rs}{d} \bs{I}_{\rk} \preceq \oG{t} \preceq 1.1 \bs{I}_{\rk}$,  we have
\eq{
\oG{t+1} & \labelrel\preceq{gnbound:ineqq1} \oG{t} +  \lr  \left(  \Luf \oG{t} + \oG{t} \Luf - 2  \oG{t} \Lus \oG{t} \right) + 1.1 \tilde C \lr^2 \normL^2 \rs \bs{I}_{\rk} \\
  & \labelrel\preceq{gnbound:ineqq2}  \oG{t} +  \lr  \left( (1 + \c_d)  \Luf \oG{t} + (1 +  \c_d) \oG{t} \Luf - 2   \oG{t} \Lus \oG{t} \right),   \\[0.5em]
  \oG{t+1} & \labelrel\succeq{gnbound:ineqq3} \oG{t} +  \lr  \left(  \Luf \oG{t} + \oG{t} \Luf - 2  \oG{t} \Lus \oG{t} \right) - 1.1 \tilde C \normL^2 \rs \lr^2 \bs{I}_{\rk} \\
  & \labelrel\succeq{gnbound:ineqq4}  \oG{t} +  \lr  \left( (1 - \c_d)  \Luf \oG{t} + (1 - \c_d) \oG{t} \Luf - 2   \oG{t} \Lus \oG{t} \right) 
}
where we use  $- 2 \Lus \preceq  \oV{t}^3 \big(\bs{I}_r +   \lr  \oV{t} \big)^{-1} \preceq 2 \Lus$   in  \eqref{gnbound:ineqq1} and    \eqref{gnbound:ineqq3},  and  we use   $\lr  \normL^2 \rs \ll \c_d \rk^{- \alpha}\frac{\rs}{d}$   in  \eqref{gnbound:ineqq2} and  \eqref{gnbound:ineqq4}.   By  Proposition \ref{prop:1dsystemaux}, we have   $\frac{\rs}{d} \bs{I}_{\rk} \preceq \oG{t+1} \preceq 1.1 \bs{I}_{\rk}$, hence, the induction hypothesis holds. Therefore, we have the first item. 

\smallskip
By using the first item and  Proposition \ref{prop:1dsystemaux}, we can write for the given time horizons in second and third items that
\eq{
\oG{t} & \preceq    \Big(   \oG{0} \exp \big(   2 (1 + \c_d) \lr t  \Luf  \big)  \wedge     \left(   1 +    (1 + \c_d)^2 \lr^2  \Luf^2  \right) \bs{I}_{\rk}    \Big) \\
& \preceq  \begin{cases}  
 \displaystyle 
\left(  1.1 \oG{0}   \exp \left(  2  \lr t  \bs{\Lambda}  \right) \wedge \bs{I}_{\rk} \right) +    2 \lr^2 \bs{I}_{\rk}  \\[0.2em]
 \displaystyle 
1.2   \left( \oG{0}   \exp \left(  2  \lr t  \bs{\Lambda}  \right)  \wedge \bs{I}_{\rk} \right),
\end{cases}  \label{eq:ubuG}
}
where both upper bounds in \eqref{eq:ubuG} are valid and will be used in different parts of the proof.  The third sub-items immediately follow from the first bound.

\smallskip
For $\alpha \in [0,0.5)$, we have
\eq{
\hat{\T}_{t, ii}  \geq \frac{1}{3} \big( 1 \wedge  \hat{\T}_{0, ii}  \exp \left(  2 \lr t  \lambda_i    \right) \big) 
& \Rightarrow \T_{t, ii}  \geq \frac{1}{4} \big( \c_d \wedge  \T_{0, ii}  \exp \left(  2 \lr t  \lambda_i    \right) \big) \\
& \labelrel \Rightarrow {gnbound:if5}\T_{t, ii}  \geq \frac{\c_d}{4}  \big( 1 \wedge  \frac{\rs}{d}  \exp \left(  2 \lr t  \lambda_i    \right) \big),
}
where we used $\T_0 = \tfrac{\c_d \rs}{d} \bs{I}_{\rk}$ in \eqref{gnbound:if5}.
Therefore by \eqref{eq:initialuG} and   the second bound in \eqref{eq:ubuG},   we have  for $j \leq t$
\eq{
 \frac{\oG{j, ii} }{ \T_{t,ii}  } \leq  \frac{  1.2 \left( 1 \wedge  \left(1 + \frac{1}{\sqrt{\varphi}} \right)^2 \frac{2.2  \rs}{d} \exp \left(  2 \lr t  \lambda_i  \right)  \right) }{0.25 \c_d \big(1 \wedge   \frac{\rs}{d} \exp \left(   2 \lr   t  \lambda_i  \right)   \big)  }   \leq \frac{11}{\c_d}   \left(1 + \frac{1}{\sqrt{\varphi}} \right)^2 .
}
On the other hand,    by using the second bound in \eqref{eq:ubuG}, 
\eq{
\frac{\lr \sum_{j = 0}^{t - 1} \oG{j, ii} }{ \T_{t,ii} } & \leq  \frac{11 \left( 1  + \frac{1}{\sqrt{\varphi}} \right)^2 }{\c_d} \frac{ \lr \left( t \wedge   \frac{\rs}{d} \sum_{j = 0}^{t - 1} \exp(2\lr j \lambda_i)  \right)}{ \big(1 \wedge   \frac{\rs}{d} \exp \left(   2 \lr  t   \lambda_i     \right)   \big) }  \\
&  \leq    \frac{5.5  \left( 1  + \frac{1}{\sqrt{\varphi}} \right)^2 }{\c_d}   \begin{cases}
\frac{1}{\lambda_i}, & 2  \lr t \leq \frac{\log \frac{d}{\rs}}{  \lambda_i} \\[0.4em]
2 \lr t,  &  2 \lr t > \frac{\log \frac{d}{\rs}}{ \lambda_i}
\end{cases} \\
& \leq    \frac{5.5 \left( 1  + \frac{1}{\sqrt{\varphi}} \right)^2}{\c_d}   (2  \lr t \vee r^{\alpha}).
}
Lastly,       by using the first bound in \eqref{eq:ubuG},  we get
\eq{
\normL^2 - \tr(\bs{\Lambda} \oG{t} \bs{\Lambda} ) \geq \sum_{i = 1}^r  \lambda_i^2 \Big(  1 - \tfrac{2.5  \left(1 + \frac{1}{\sqrt{\varphi}} \right)^2  \rs}{d}   \exp \left(  2  \lr t  \lambda_i  \right)  \Big)_{+}  -   2 \lr^2 \normL^2.
}
For $\alpha > 0.5$,  we have for  $i \leq 2 \rs \log^{\frac{1}{\alpha}} \! d$ and $d \geq \Omega_{\rs}(1)$,
\eq{
\hat{\T}_{t, ii}  \geq \frac{1}{3} \big( 1 \wedge  \hat{\T}_{0, ii}  \exp \left(  2 \lr t  \lambda_i    \right) \big) \Rightarrow \T_{t, ii}  \geq \frac{1}{4} \big( \c_d \wedge  \T_{0, ii}  \exp \left(  2 \lr t \lambda_i    \right) \big).
}
Therefore,  we have
\eq{
\T_{t, ii}  \geq  \c_d
\begin{cases}
0.25  \big( 1 \wedge  \frac{\rs}{d}  \exp \left(  2 \lr t  \lambda_i    \right) \big),  & i \leq  2 \rs \log^{\frac{1}{\alpha}} \! d \\
\frac{\rs}{d},  & i > 2 \rs \log^{\frac{1}{\alpha}} \! d.
\end{cases}
}
On the other hand, for $\lr t \leq \frac{1}{2} \rs^{\alpha} \log (\frac{d \log^{1.5} d}{\rs} )$ and  $i > 2 \rs \log^{\frac{1}{\alpha}} \! d$, we have for $d \geq \Omega(1)$.
\eq{
\oG{t, ii} \leq  1.2   \left( \oG{0, ii}   \exp \left(  2  \lr t  \lambda_i  \right)  \wedge 1 \right) \leq    1.2   \left( \oG{0, ii}   \exp \left(  \frac{\rs^{\alpha} \log (\frac{d \log d}{\rs} ) }{2^{\alpha} \rs^{\alpha} \log d } \right)  \wedge 1 \right) \leq 1.5 \oG{0,  ii}.  
}
Therefore,  for  $\lr t \leq \frac{1}{2} \rs^{\alpha} \log (\frac{d \log^{1.5} d}{\rs} )$, 
\eq{
 \frac{\oG{j, ii} }{ \T_{t,ii}  } & \leq   \begin{cases}
 \displaystyle 
 \frac{  1.2 \left( 1 \wedge  \frac{5.5 \ru  \rs}{d} \exp \left(  2 \lr t  \lambda_i  \right)  \right) }{0.25 \c_d \big(1 \wedge   \frac{\rs}{d} \exp \left(   2 \lr   t  \lambda_i  \right)   \big)  }  , & i \leq 2 \rs \log^{\frac{1}{\alpha}} \! d \\[1.3em]
 \displaystyle 
  \frac{d}{ \c_d \rs}\frac{8.25 \ru  \rs}{d},  & i >  2 \rs \log^{\frac{1}{\alpha}} \! d
\end{cases} \\[0.2em]
& \leq \frac{26.4 \ru}{\c_d}.
}
Moreover,  for  $\lr t \leq \frac{1}{2} \rs^{\alpha} \log (\frac{d \log^{1.5} d}{\rs} )$,
\eq{
\frac{\lr \sum_{j = 0}^{t - 1} \oG{j, ii} }{ \T_{t,ii} } & \leq 
\begin{cases}
 \displaystyle
 \frac{26.4 \ru }{\c_d} \frac{ \lr \left( t \wedge   \frac{\rs}{d} \sum_{j = 0}^{t - 1} \exp(2\lr j \lambda_i  )  \right)}{ \big(1 \wedge   \frac{\rs}{d} \exp \left(   2 \lr t   \lambda_i  \right)   \big) } , & i \leq 2 \rs \log^{\frac{1}{\alpha}} \! d  \\[1.3em]
  \displaystyle \frac{d}{\c_d \rs}
   \frac{8.25 \ru  \rs}{d} \lr t  ,  & i >2 \rs \log^{\frac{1}{\alpha}} \! d
\end{cases} \\
&  \leq    \frac{13.2  \ru }{\c_d}   \begin{cases}
\frac{1}{\lambda_i}, & 2  \lr t \leq \frac{\log \frac{d}{\rs}}{  \lambda_i} ~ \text{and} ~  i \leq 2 \rs \log^{\frac{1}{\alpha}} \! d \\[0.4em]
2 \lr t,  &  \text{otherwise}   
\end{cases}  \\
&  \leq   \frac{13.2  \ru }{\c_d}   (2 \lr t \vee (2\rs)^{\alpha} \log d) \leq    \frac{15   \ru }{\c_d}  (2\rs)^{\alpha} \log d  .
}
Finally for  $\lr t \leq \frac{1}{2} \rs^{\alpha} \log (\frac{d \log^{1.5} d}{\rs} )$ and $d \geq \Omega_{\rs}(1)$,   by using the first bound in \eqref{eq:ubuG},
\eq{
& \norm{\bs{\Lambda}_{11}}_F^2    - \tr(\bs{\Lambda}_{11} \oG{t,11} \bs{\Lambda}_{11} ) \geq \sum_{i = 1}^{\rs}  \lambda_i^2 \Big(  1 - \tfrac{12.1  \rs}{d}   \exp \left(  2  \lr t  \lambda_i \right)  \Big) _{+} 
  \\
& + \sum_{i = \rs + 1}^{2 \rs}  \!\! \lambda_i^2 \Big(  1 - \tfrac{12.1  \rs}{d}   \exp \left(  2  \lr t   \lambda_i  \right)  \Big)_{+}  \!\! + \!\!\sum_{i = 2\rs + 1}^{\ru} \!\!  \lambda_i^2 \Big(  1 - \tfrac{6.05 \ru \rs}{d}   \exp \left(  2  \lr t  \lambda_i  \right)  \Big)_{+}   
\!\! -   2 \lr^2 \normL^2 \\
& \labelrel\geq{gnbound:ineqq5}  \sum_{i = 1}^{\rs}  \lambda_i^2 \Big(  1 - \tfrac{12.1  \rs}{d}   \exp \left(  2  \lr t  \lambda_i \right)   \Big)_{+}   + \left(1 - \frac{1}{\log d} \right) \sum_{i = \rs + 1}^{2 \rs}  \lambda_i^2    \\
& +\left (1 - 6.05 \ru \log^{\frac{1.5}{\sqrt{2}}} d \left( \frac{\rs}{d} \right)^{1 - \frac{1}{\sqrt{2}}}  \right) \sum_{i = 2\rs + 1}^{\ru}  \lambda_i^2  
-   2 \lr^2 \normL^2 \\
& \geq   \sum_{i = 1}^{\rs}  \lambda_i^2 \Big(  1 - \tfrac{12.1  \rs}{d}   \exp \left(  2  \lr t  \lambda_i \right)  \Big)_{+}  + \left (1 - 6.05 \ru \log^{\frac{1.5}{\sqrt{2}}} d \left( \frac{\rs}{d} \right)^{1 - \frac{1}{\sqrt{2}}}  \right) \sum_{i = \rs + 1}^{\ru}  \lambda_i^2  \\
& -    \frac{(\rs + 1)^{1 + 2\alpha}}{\log d} -   2 \lr^2 \normL^2,
}  
where we used the bounds for $t, d$ in  \eqref{gnbound:ineqq5}.
\end{proof}

\subsection{Bounds for the second-order terms}
We recall 
\eq{
R_{\text{so}}[\G_t] &=  \frac{\eta^2}{16\rs}  \Th^\top  \E_t \left[  \grad_{\text{St}}  \bs{L}_{t+1} \grad_{\text{St}}  \bs{L}_{t+1}^\top   \right]   \Th \\
& -  \frac{\eta^2}{16\rs}      \bs{M}_{t} \E_t \left[ \frac{\bs{\cP}_{t+1} }{1 + c_{t+1}^2} \right] \bs{M}_{t}^\top  -  \frac{\eta^3}{32 \rs^{3/2}} \sym \left( \Th^\top  \E_t \left[ \frac{ \grad_{\text{St}}  \bs{L}_{t+1} \bs{\cP}_{t+1} }{1 + c_{t+1}^2}  \right]  \bs{M}_{t}^\top  \right)  \\
& -   \frac{\eta^4}{256\rs^2 }  \Th^\top   \E_t \left[  \frac{ \grad_{\text{St}}  \bs{L}_{t+1} \bs{\cP}_{t+1}   \grad_{\text{St}}  \bs{L}_{t+1}^\top }{1 + c_{t+1}^2} \right]     \Th   \label{eq:secondorderterms}.
}

\begin{proposition}
\label{prop:soterms}
For $\eta \ll d^{-1/2}$, there exists a universal constant $C > 0$ such that
\eq{
- C \Big(  \frac{\eta^2 d}{\rs} \G_t +  &   \eta^2    \bs{I}_r    \Big) \preceq R_{so}[\G_t]    \preceq C \left(  \frac{\eta^2 d}{\rs } \G_t +  \eta^2    \bs{I}_r    \right).
}
\end{proposition}

\begin{proof} 
We bound each term in  \eqref{eq:secondorderterms}.    In the following,   $\bs{v}$   denotes a generic unit norm vector with proper dimensionality.  For the first term,
\eq{
& \Th^\top    \E_t \Big[ \Lst  \Lst^\top \Big]  \Th \\
&= \Th^\top  \left( \bs{I}_d - \W_t \W_t^\top \right)  \E_t \left[ (y_{t+1} - \hat{y}_{t+1})^2 \norm{\W_t^\top \bs{x}_{t+1}}_2^2 \bs{x}_{t+1} \bs{x}_{t+1}^\top \right]   \left( \bs{I}_d - \W_t \W_t^\top \right)    \Th.  
}
We have
\eq{
 \E_t \left[ (y_{t+1} - \hat{y}_{t+1})^2  \norm{\W_t^\top \bs{x}_{t+1}}_2^2  \inner{\bs{v}}{\bs{x}_{t+1}}^2 \right]  \leq C \rs.
}
Therefore,
\eq{
0  \preceq  \frac{\eta^2}{16 \rs}   \Th^\top   \E_t \left[ \Lst  \Lst^\top \right]  \Th   \preceq   C \eta^2  (\bs{I}_r - \G_t).
}
For the second term,
\eq{
& \bs{M}_t  \E_t \left[ \frac{ \Pst}{1+ c_{t+1}^2}  \right]  \bs{M}_t^\top \\
& =  \bs{M}_t  \E_t \left[ \frac{ \big(y_{t+1} - \hat{y}_{t+1}\big)^2  \norm{  \left( \bs{I}_d - \W_t \W_t^\top \right) \bs{x}_{t+1}  }_2^2 \W_t^\top \bs{x}_{t+1}  \bs{x}_{t+1}^\top  \W_t  }{1+ c_{t+1}^2}  \right]  \bs{M}_t^\top. 
}
We have
\eq{
\E_t \left[ \frac{ \big(y_{t+1} - \hat{y}_{t+1}\big)^2  \norm{  \left( \bs{I}_d - \W_t \W_t^\top \right) \bs{x}_{t+1}  }_2^2  \inner{\bs{v}}{\W_t^\top \bs{x}_{t+1}}^2 }{1+ c_{t+1}^2}  \right] \leq C d.
}
Therefore,
\eq{
0  \preceq   \frac{\eta^2}{16 \rs} \bs{M}_t  \E_t \left[ \frac{ \Pst}{1+ c_{t+1}^2}  \right]  \bs{M}_t^\top  \preceq    C   \frac{\eta^2 d}{\rs}    \G_t.
}
For the third term by using Proposition \ref{prop:matrixyounginequality},
\eq{
\frac{\eta^3}{32 \rs^{3/2}} &  \sym \left( \Th^\top   \E_t \left[ \frac{ \Lst \Pst }{1 + c_{t+1}^2}  \right]  \bs{M}_{t}^\top  \right) \\
&  \preceq C \left(  \frac{\eta^4}{\rs^2 d}  \Th^\top  \E_t \left[ \frac{ \Lst \Pst }{1 + c_{t+1}^2}  \right]  \E_t \left[ \frac{ \Pst \Lst^\top }{1 + c_{t+1}^2}  \right]  \Th + \frac{\eta^2 d }{\rs}  \G_t  \right)
}
We have
\eq{
 & \Th^\top    \E_t \left[ \frac{ \Lst \Pst }{1 + c_{t+1}^2}  \right]    \\
& =   \Th^\top  \!\!  \left( \bs{I}_d - \W_t \W_t^\top \right)   \\
& \qquad \times \E_t \left[ \frac{  \big(y_{t+1} - \hat{y}_{t+1}\big)^3 }{1 + c_{t+1}^2}   \norm{  \left( \bs{I}_d - \W_t \W_t^\top \right) \bs{x}_{t+1}  }_2^2  \norm{\W_t^\top \bs{x}_{t+1}}_2^2    \bs{x}_{t+1}  \bs{x}_{t+1}^\top \W_t \right] .
}
Then, by using Cauchy-Schwartz inequality, we can show that
\eq{
\norm*{  \E_t \left[ \frac{ \big(y_{t+1} - \hat{y}_{t+1}\big)^3 }{1 + c_{t+1}^2}   \norm{  \left( \bs{I}_d - \W_t \W_t^\top \right) \bs{x}_{t+1}  }_2^2  \norm{\W_t^\top \bs{x}_{t+1}}_2^2    \bs{x}_{t+1}  \bs{x}_{t+1}^\top \W_t \right]  }_2 \leq C d \rs.
}
Therefore,
\eq{
\frac{\eta^4}{\rs^2 d}  \Th^\top  \E_t \left[ \frac{ \Lst \Pst }{1 + c_{t+1}^2}  \right]  \E_t \left[ \frac{ \Pst \Lst^\top }{1 + c_{t+1}^2}  \right]  \Th \leq C  \eta^4 d (\bs{I}_r - \G_t).
}
We get
\eq{
\frac{\eta^3}{32 \rs^{3/2}} & \sym \left( \Th^\top   \E_t \left[ \frac{ \Lst \Pst }{1 + c_{t+1}^2}  \right]  \bs{M}_{t}^\top  \right) \preceq C \left( \eta^4 d  (\bs{I}_r - \G_t)+ \frac{\eta^2 d }{\rs}  \G_t  \right).
}
By repeating the argument with the lower bound in Proposition \ref{prop:matrixyounginequality}, we can also show
\eq{
\frac{\eta^3}{32 \rs^{3/2}} & \sym \left( \Th^\top   \E_t \left[ \frac{ \Lst \Pst }{1 + c_{t+1}^2}  \right]  \bs{M}_{t}^\top  \right) \succeq - C \left( \eta^4 d  (\bs{I}_r - \G_t)   + \frac{\eta^2 d }{\rs}  \G_t  \right).
}
For the last term, we write
\eq{
&\Th^\top  \E_t  \left[ \frac{  \Lst \Pst   \Lst^\top}{1+ c_{t+1}^2}   \right]  \Th   \\
& =  \Th^\top \!  \left( \bs{I}_d - \W_t \W_t^\top \right)   \\
& \quad \times \E_t \! \left[ \frac{ \big(y_{t+1} - \hat{y}_{t+1}\big)^4 }{1+ c_{t+1}^2}  \norm{\W_t^\top \bs{x}_{t+1}}_2^4 \norm{ \left( \bs{I}_d \! - \! \W_t \W_t^\top \right)   \bs{x}_{t+1}}_2^2   \bs{x}_{t+1}  \bs{x}_{t+1}^\top   \right]  \left( \bs{I}_d - \W_t \W_t^\top \right)  \Th.
}
We have
\eq{
 \E_t \left[ \frac{  \big(y_{t+1} - \hat{y}_{t+1}\big)^4 }{1+ c_{t+1}^2}  \norm{\W_t^\top \bs{x}_{t+1}}_2^4 \norm{ \left( \bs{I}_d - \W_t \W_t^\top \right)   \bs{x}_{t+1}}_2^2  \inner{\bs{v}}{\bs{x}_{t+1} }^2   \right]     \leq C d \rs^2.
}
Therefore,
\eq{
0 \preceq \frac{\eta^4}{\rs^2} \Th^\top  \E_t  \left[ \frac{  \Lst \Pst   \Lst^\top}{1+ c_{t+1}^2}   \right]  \Th   \preceq  C  \eta^4 d   (\bs{I}_r - \G_t).
}
By using $\G_t \succeq 0$ and $\eta \ll d^{-1/2}$, the result follows.
\end{proof}

\subsection{Noise characterization}
To prove the noise concentration bound for both the heavy-tailed and light-tailed cases simultaneously, we introduce some new notation. Specifically, we define the submatrix notation:
\eq{
\Th \eqqcolon \begin{bmatrix}
\Th_1 & \Th_2
\end{bmatrix}
 ~~ \text{and} ~~
\bs{M}_t \eqqcolon
  \begin{bmatrix}
 \bs{M}_{t,1}   \\
 \bs{M}_{t,2}
\end{bmatrix}  =
\begin{bmatrix}
\bs{\Theta}_1^\top \W_t \\
\bs{\Theta}_2^\top \W_t
\end{bmatrix},
}
where $\bs{\Theta}_1 \in \R^{d \times \ru}$ and   $\bs{M}_{t,1}  \in \R^{\ru \times \rs}$.  We note that $\G_{t,11} =   \bs{M}_{t,1}  \bs{M}_{t,1}^\top$.    To unify the treatment of the heavy-tailed and light-tailed cases, we use the following notation to represent both cases:
\eq{
\sTh \coloneqq \{ \Th, \Th_1 \}    \quad  \sM_{t} \coloneqq \{ \bs{M}_t,  \bs{M}_{t,1} \} .
}
With the new notation,  we have
 \eq{
\frac{\eta/2}{\sqrt{\rs}} \upnu_{t+1} & = \frac{\eta/2}{\sqrt{\rs}}   \sym \left( \sTh^\top  \Big(  \grad_{\text{St}}  \bs{L}_{t+1} - \E_t \left[  \grad_{\text{St}}  \bs{L}_{t+1} \right] \Big)    \sM_{t}^\top \right)  \\
& -  \frac{\eta^2}{16 \rs}     \sM_{t}  \left( \frac{\bs{\cP}_{t+1}}{1 +  c_{t+1}^2 } - \E_t \left[ \frac{\bs{\cP}_{t+1} }{1 + c_{t+1}^2} \right] \right)  \sM_{t}^\top  \\
& + \frac{\eta^2}{16 \rs}    \sTh^\top \left( \grad_{\text{St}}  \bs{L}_{t+1} \grad_{\text{St}}  \bs{L}_{t+1}^\top    -    \E_t \left[ \grad_{\text{St}}  \bs{L}_{t+1} \grad_{\text{St}}  \bs{L}_{t+1}^\top \right] \right)  \sTh  \\
&  -  \frac{\eta^3}{32 \rs^{3/2}} \sym \left(  \sTh^\top \left( \frac{ \grad_{\text{St}}  \bs{L}_{t+1} \bs{\cP}_{t+1} }{1 + c_{t+1}^2}   -   \E_t \left[ \frac{ \grad_{\text{St}}  \bs{L}_{t+1} \bs{\cP}_{t+1} }{1 + c_{t+1}^2}  \right]  \right) \sM_{t}^\top  \right) \\
& -   \frac{\eta^4}{256 \rs^2} \sTh^\top  \left(  \frac{ \grad_{\text{St}}  \bs{L}_{t+1} \bs{\cP}_{t+1}   \grad_{\text{St}}  \bs{L}_{t+1}^\top }{1 + c_{t+1}^2} -   \E_t \left[  \frac{ \grad_{\text{St}}  \bs{L}_{t+1} \bs{\cP}_{t+1}   \grad_{\text{St}}  \bs{L}_{t+1}^\top }{1 + c_{t+1}^2} \right]   \right)  \sTh. \label{eq:noiserecall}
}
For  $\rk \in \{ r, \ru \}$ and $\bs{T}_1,  \bs{T}_2 \in \R^{\rk \times \rk}$ be a deterministic symmetric positive definite matrices, we define
\eq{
& \mathcal{A}_{t+1} \left(    \bs{T}_1 ,   \bs{T}_2  \right) \equiv \left \{   \norm*{  \bs{T}_1   \sTh^\top    \grad_{\text{St}}  \bs{L}_{t+1}  \sM_{t}^\top     \bs{T}_2  }_2  \leq   \frac{L^2}{2}   \sqrt{   \tr \big(  \bs{T}^2_1   (\bs{I}_{\rk} - \sG_t)   \big)   \tr(  \bs{T}^2_2  \sG_t)   }         \right \} \\
& \mathcal{B}_{t+1} \left(    \bs{T}_1 ,   \bs{T}_2  \right)  \equiv \left \{   \norm*{  \bs{T}_1    \sM_{t} \bs{\cP}_{t+1}  \sM_{t}^\top  \bs{T}_2  }_2  \leq   \frac{L^4 d}{2}     \sqrt{  \tr(  \bs{T}_1^2 \sG_t  ) \tr(  \bs{T}_2^2 \sG_t  ) }    \right \} \\
& { \small \mathcal{C}_{t+1} \left(    \bs{T}_1 ,   \bs{T}_2  \right) \! \equiv \! \left \{   \norm*{  \bs{T}_1     \sTh^\top  \grad_{\text{St}}  \bs{L}_{t+1} \grad_{\text{St}}  \bs{L}_{t+1}^\top \sTh   \bs{T}_2  }_2 \! \leq \! \!   \frac{L^4 \rs}{2} \!   \sqrt{  \tr\big(\T_1^2 (\bs{I}_{\rk} \! - \!\sG_t) \big)  \tr\big(\T_2^2 (\bs{I}_{\rk} \! - \! \sG_t) \big) }   \right \} } \\
& \mathcal{D}_{t+1} \left(    \bs{T}_1 ,   \bs{T}_2  \right)  \equiv \left \{   \norm*{  \bs{T}_1  \sTh^\top    \grad_{\text{St}}  \bs{L}_{t+1} \bs{\cP}_{t+1}   \sM_{t}^\top     \bs{T}_2}_2  \leq   \frac{L^6  d  \rs}{2}  \sqrt{ \tr \big(\bs{T}_1^2 (\bs{I}_{\rk} - \sG_t) \big)    \tr \big(\bs{T}^2_2  \sG_t \big)   }  \right \} \\
& { \footnotesize \mathcal{F}_{t+1} \left(    \bs{T}_1 ,   \bs{T}_2  \right)  \!  \equiv \! \left \{   \norm*{  \bs{T}_1  \sTh^\top   \grad_{\text{St}}  \bs{L}_{t+1} \bs{\cP}_{t+1}   \grad_{\text{St}}  \bs{L}_{t+1}^\top  \sTh    \bs{T}_2 }_2 \! \leq \! \!   \frac{L^8  d  \rs^2 }{2} \!  \sqrt{  \tr\big(\T_1^2 (\bs{I}_{\rk} \! - \! \sG_t) \big)  \tr\big(\T_2^2 (\bs{I}_{\rk} \! - \! \sG_t) \big) }    \right \} } .
}  
We start with the following statement:  
\begin{proposition}
\label{prop:noisevarbounds}
 Let $e_{t+1} \coloneqq (y_{t+1} - \hat{y}_{t+1})$. There exists a universal constant $C > 0$ such that for $L \geq 2  e (\sqrt{8} + \sigma)$, the following statements hold:
\begin{enumerate}[leftmargin=*]
\item 
We have  $\mpr_t \left[  \mathcal{A}_{t+1} \left(    \bs{T}_1 ,   \bs{T}_2  \right) \cap  \{ \abs{e_{t+1}} \leq L \} \right] \geq  1- 2 e^{\frac{- L/e}{(\sqrt{8} + \sigma)}}$.  Moreover,   
\eq{
&  \E_t \left[   \big(    \sym \left( \bs{T}_1   \sTh^\top    \grad_{\text{St}}  \bs{L}_{t+1}  \sM_{t}^\top   \bs{T}_2  \right)    \big)^2  \right]  \\
&  \hspace{6em} \preceq  C \Big( \tr(  \bs{T}^2_2 \sG_t )    \bs{T}_1  (\bs{I}_{\rk} - \sG_t)   \bs{T}_1 +    \tr \big(  \bs{T}^2_1    (\bs{I}_{\rk} - \sG_t)   \big)  \bs{T}_2  \sG_t  \bs{T}_2  \Big).
}  
\item  We have  $\mpr_t \left[  \mathcal{B}_{t+1} \left(    \bs{T}_1 ,   \bs{T}_2  \right)  \cap  \{ \abs{e_{t+1}} \leq L \}  \right] \geq  1- 2 e^{\frac{- L/e}{(\sqrt{8} + \sigma)}}$.  Moreover,   
\eq{
&  \E_t \left[   \big(   \sym \left( \bs{T}_1      \sM_{t} \bs{\cP}_{t+1}  \sM_{t}^\top  \bs{T}_2 \right)   \big)^2  \right] \\
& \hspace{6em} \preceq  C d^2 \left( \tr(  \bs{T}_2^2 \sG_t )  \bs{T}_1 \sG_t  \bs{T}_1 +  \tr(  \bs{T}_1^2 \sG_t )  \bs{T}_2 \sG_t \bs{T}_2   \right). 
}
\item  We have   $\mpr_t \left[  \mathcal{C}_{t+1} \left(    \bs{T}_1 ,   \bs{T}_2  \right)   \cap  \{ \abs{e_{t+1}} \leq L \}  \right] \geq   1- 2 e^{\frac{- L/e}{(\sqrt{8} + \sigma)}}$.  Moreover,   
\eq{
&  \E_t \Big[   \big( \sym \left(   \bs{T}_1    \sTh^\top \grad_{\text{St}}  \bs{L}_{t+1} \grad_{\text{St}}  \bs{L}_{t+1}^\top  \bs{T}_2 \right)   \big)^2  \Big] \\ 
&  \hspace{4em}  \preceq   C \rs^2 \Big( \tr\big(  \bs{T}_2^2 (\bs{I}_{\rk} -  \sG_t ) \big)  \bs{T}_1 (\bs{I}_{\rk} -  \sG_t )   \bs{T}_1 + \tr\big(  \bs{T}_1^2 (\bs{I}_{\rk} -  \sG_t ) \big)  \bs{T}_2 (\bs{I}_{\rk} -  \sG_t )   \bs{T}_2    \Big).
}
\item We have  $\mpr_t \left[ \mathcal{D}_{t+1} \left(    \bs{T}_1 ,   \bs{T}_2  \right)   \cap  \{ \abs{e_{t+1}} \leq L \}  \right] \geq  1- 2 e^{\frac{- L/e}{(7/2 + \sigma)}}$.  Moreover,   
\eq{
&  \E_t \left[   \big(      \sym \left( \bs{T}_1  \sTh^\top    \grad_{\text{St}}  \bs{L}_{t+1} \bs{\cP}_{t+1}   \sM_{t}^\top  \bs{T}_2 \right)   \big)^2  \right] \\
& \hspace{3em} \preceq  C d^2 \rs^2  \Big( \tr(  \bs{T}_2^2 \sG_t )    \bs{T}_1  (\bs{I}_{\rk} - \sG_t)   \bs{T}_1 +    \tr \big(  \bs{T}_1^2   (\bs{I}_{\rk} - \sG_t)   \big)  \bs{T}_2 \sG_t  \bs{T}_2  \Big). 
}
\item We have  $\mpr_t \left[  \mathcal{F}_{t+1} \left(    \bs{T}_1 ,   \bs{T}_2  \right)  \cap  \{ \abs{e_{t+1}} \leq L \}   \right] \geq  1- 2 e^{\frac{- L/e}{ (4 \sqrt{2} + \sigma)}}$.  Moreover,  
\eq{
&  \E_t \left[  \big(  \sym \left(   \bs{T}_1    \sTh^\top   \grad_{\text{St}}  \bs{L}_{t+1} \bs{\cP}_{t+1}   \grad_{\text{St}}  \bs{L}_{t+1}^\top  \sTh     \bs{T}_2  \right)  \big)^2  \right] \\
& \hspace{2em}\preceq C d^2 \rs^4 \Big(  \tr \big(\bs{T}_2^2 (\bs{I}_{\rk} - \sG_t) \big)     \bs{T}_1 (\bs{I}_{\rk} - \sG_t)  \bs{T}_1  +  \tr \big(\bs{T}_1^2 (\bs{I}_{\rk} - \sG_t) \big)     \bs{T}_2 (\bs{I}_{\rk} - \sG_t)  \bs{T}_2  \Big).
}
\end{enumerate}
\end{proposition}

\begin{proof}
First, we derive a concentration bound for $\abs{e_{t+1}}.$ By Corollary  \ref{cor:hc1} and Proposition  \ref{prop:basicconcentration} we have   
\eq{
\E_t \left[   \abs{e_{t+1}}^p \right] \leq \big( \E_t \left[   \abs{y_{t+1}}^p \right]^{\frac{1}{p}} + \E_t \left[   \abs{\hat{y}_{t+1}}^p \right]^{\frac{1}{p}} \big)^p
\leq (\sqrt{8} + \sigma)^{p} p^p   ~  \text{for} ~   p \geq 2, 
}
which implies $\mpr_t \left[ \abs{e_{t+1}} \geq u   \right] \leq  e^{\frac{- u/e}{ (\sqrt{8} + \sigma) } }$ for $u \geq 2  e (\sqrt{8} + \sigma)$. In the following,  we prove each item separately.     

\smallskip
\textbf{First item. } We define
\eq{
\bs{T}_1  \sTh^\top   \grad_{\text{St}}  \bs{L}_{t+1}     \sM_{t}^\top   \bs{T}_2 =   \underbrace{   \bs{T}_1     \sTh^\top (\bs{I}_d - \W_t \W_t^\top) e_{t+1} \bs{x}_{t+1} }_{\coloneqq \bs{u}_{t+1}}    \underbrace{  \bs{x}_{t+1}^\top \W_t \W_t^\top \sTh \bs{T}_2 }_{\coloneqq \bs{v}^\top_{t+1}}. 
}
For $u, L > 0$
\eq{
& \mpr_t \Big[ \norm*{    \bs{u}_{t+1}  \bs{v}_{t+1}^\top  }_2 \geq   u L   \sqrt{  \tr \big(  \bs{T}^2_1   (\bs{I}_{\rk} - \sG_t)   \big)  \tr(   \bs{T}^2_2 \sG_t)  }        ~~ \text{or} ~~ \abs{e_{t+1}} \geq L  \Big] \\
& \leq  \mpr_t \Big[  \norm*{    \bs{u}_{t+1}  \bs{v}_{t+1}^\top  }_2 \geq   u L   \sqrt{  \tr \big(  \bs{T}^2_1   (\bs{I}_{\rk} - \sG_t)   \big)  \tr(   \bs{T}^2_2 \sG_t)  }        ~~ \text{and} ~~ \abs{e_{t+1}} \leq L  \Big] 
+  \mpr_t\Big[ \abs{e_{t+1}} \geq L   \Big] \\
& \leq   \mpr_t \Big[   \norm*{   \indic{\abs{e_{t+1}} \leq L  }  \bs{u}_{t+1}  \bs{v}_{t+1}^\top  }_2 \geq   u L   \sqrt{  \tr \big(  \bs{T}^2_1   (\bs{I}_{\rk} - \sG_t)   \big)  \tr(   \bs{T}^2_2 \sG_t)  }   \Big]   +  \mpr_t \Big[ \abs{e_{t+1}} \geq L  \Big].
}
We have for $p \geq 2$
\eq{
&  \E_t \left[   \norm*{    \indic{\abs{e_{t+1}} \leq L  }  \bs{u}_{t+1}  \bs{v}_{t+1}^\top  }_2^p  \right]  \\
& \qquad \leq   L^p   \E_t \left[    \norm{  \bs{T}_1 \sTh^\top   (\bs{I}_d - \W_t \W_t^\top)   \bs{x}_{t+1} }_2^p \right]   \E_t \left[ \norm{  \bs{T}_2  \sTh^\top \W_t \W_t^\top  \bs{x}_{t+1} }_2^p   \right]  \\
& \qquad  \labelrel\leq{nb:ineqq0} L^p \left( \frac{p}{2} \right)^p \Big(3 \tr \big(  \bs{T}^2_1   (\bs{I}_{\rk} - \sG_t)   \big)    \tr(  \bs{T}^2_2  \sG_t)   \Big)^{\frac{p}{2}},
}
where we used  Corollary \ref{cor:hc2} in  \eqref{nb:ineqq0}. By Proposition \ref{prop:basicconcentration},  we have for $u \geq 2e$
\eq{
 \mpr_t \Big[  \norm*{   \indic{\abs{e_{t+1}} \leq t  } \bs{u}_{t+1}  \bs{v}_{t+1}^\top  }_2  \geq    u L      \sqrt{   \tr \big(  \bs{T}^2_1   (\bs{I}_{\rk} - \sG_t)   \big)   \tr(  \bs{T}^2_2  \sG_t)   }        \Big] \leq   e^{- \frac{u}{e}}.
}
By choosing $u = \frac{L}{2}$, we have the probability bound.

\smallskip
For the variance bound, we have
\eq{
\E_{t} \left[      \sym \left(   \bs{T}_1   \sTh^\top   \grad_{\text{St}}  \bs{L}_{t+1}     \sM_{t}^\top   \bs{T}_2  \right)^2 \right] = \E_t \left[      \sym \left(   \bs{u}_{t+1}  \bs{v}_{t+1}^\top  \right)^2 \right].
}
By using Proposition \ref{prop:matrixyounginequality}, we have
\eq{
& \E_t \left[      \sym \left(   \bs{u}_{t+1}  \bs{v}_{t+1}^\top  \right)^2 \right] \\
& \preceq  \bs{T}_1   \sTh^\top (\bs{I}_d - \W_t \W_t^\top)   \E_t \left[ e_{t+1}^2 \norm{ \bs{T}_2 \sTh^\top \W_t \W_t^\top  \bs{x}_{t+1} }_2^2      \bs{x}_{t+1}  \bs{x}_{t+1}^\top  \right]   (\bs{I}_d - \W_t \W_t^\top) \sTh \bs{T}_1  \\
&  +   \bs{T}_2 \sTh^\top \W_t \W_t^\top     \E_t \left[ e_{t+1}^2 \norm{ \bs{T}_1 \sTh^\top   (\bs{I}_d - \W_t \W_t^\top)   \bs{x}_{t+1} }_2^2      \bs{x}_{t+1}  \bs{x}_{t+1}^\top  \right]     \W_t \W_t^\top \sTh  \bs{T}_2  \\
&  \labelrel\preceq{nb:ineqq1}   C \Big( \tr(  \bs{T}^2_2 \sG_t )    \bs{T}_1  (\bs{I}_{\rk} - \sG_t)   \bs{T}_1 +    \tr \Big(  \bs{T}^2_1    (\bs{I}_{\rk} - \sG_t)   \Big)  \bs{T}_2  \sG_t  \bs{T}_2  \Big),
}
where we used the Cauchy-Schwartz inequality in  \eqref{nb:ineqq1}.
 
\textbf{Second item. }  We define
\eq{
\bs{T}_1  \sM_t  \bs{\cP}_{t+1} \sM_t^\top \bs{T}_2   = \underbrace{ e_{t+1}^2  \norm{ \left(  \bs{I}_{d} -   \W_t  \W_t^\top  \right)  \bs{x}_{t+1}  }_2^2 \bs{T}_1  \sTh^\top \W_t \W_t^\top \bs{x}_{t+1}  }_{\coloneqq \bs{u}_{t+1}}  \underbrace{ \bs{x}_{t+1}^\top  \W_t  \W_t^\top \sTh \bs{T}_2 }_{\coloneqq  \bs{v}_{t+1}^\top}.
}
We have for $p \geq 2$
\eq{
& \E_t \left[   \norm*{    \indic{\abs{e_{t+1}} \leq L  } \bs{u}_{t+1}  \bs{v}_{t+1}^\top  }_2^p  \right]    \\
& \leq   L^{2p}   \E_t \left[      \norm{ \left(  \bs{I}_{d} -   \W_t  \W_t^\top  \right)  \bs{x}_{t+1}  }_2^{2p}  \right]   \E_t \left[  \norm{\bs{T}_1  \sTh^\top \W_t \W_t^\top \bs{x}_{t+1}}_2^{2p}    \right]^{\frac{1}{2}}  \\
& \quad \times \E_t \left[  \norm{\bs{T}_2  \sTh^\top \W_t \W_t^\top \bs{x}_{t+1}}_2^{2p}    \right]^{\frac{1}{2}}   \\
& \labelrel\leq{nb:ineqq2}   L^{2p} p^{2p} \Big(3 d  \sqrt{  \tr(  \bs{T}_1^2 \sG_t  ) \tr(  \bs{T}_2^2 \sG_t  ) } \Big)^{p},  
}
where we use   Corollary \ref{cor:hc2} in \eqref{nb:ineqq2}.   By Proposition \ref{prop:basicconcentration},  we have for $u \geq (2e)^2$
\eq{
 \mpr_t \Big[  \norm*{    \indic{\abs{e_{t+1}} \leq L  } \bs{u}_{t+1}  \bs{v}_{t+1}^\top  }_2   \geq    u L^2    3 d  \sqrt{  \tr(  \bs{T}_1^2 \sG_t  ) \tr(  \bs{T}_2^2 \sG_t  ) }      \Big] \leq e^{- \frac{u^{1/2}}{e}}.
}
By choosing $u = \tfrac{L^2}{6}$, we have the probability bound.   For the variance bound, we have
\eq{
\E_t \left[     \big(  \sym \left( \bs{T}_1  \sM_t  \bs{\cP}_{t+1} \sM_t^\top \bs{T}_2  \right) \big)^2 \right] = \E_t \left[      \sym \left(   \bs{u}_{t+1}  \bs{v}_{t+1}^\top  \right)^2 \right].
}
By using Proposition \ref{prop:matrixyounginequality}, we have
\eq{
& \E_t \left[      \sym \left(   \bs{u}_{t+1}  \bs{v}_{t+1}^\top  \right)^2 \right] \\
& \preceq     \bs{T}_1 \sTh^\top \W_t \W_t^\top   \\
& \quad \times \E_t \left[ e_{t+1}^4  \norm{ \left(  \bs{I}_{d} -   \W_t  \W_t^\top  \right)  \bs{x}_{t+1}  }_2^4 \norm{ \bs{T}_2 \sTh^\top   \W_t \W_t^\top   \bs{x}_{t+1} }_2^2      \bs{x}_{t+1}  \bs{x}_{t+1}^\top  \right]    \W_t \W_t^\top \sTh  \bs{T}_1  \\
& +   \bs{T}_2 \sTh^\top \W_t \W_t^\top  \\
& \quad \times \E_t \left[ e_{t+1}^4  \norm{ \left(  \bs{I}_{d} -   \W_t  \W_t^\top  \right)  \bs{x}_{t+1}  }_2^4 \norm{ \bs{T}_1 \sTh^\top   \W_t \W_t^\top   \bs{x}_{t+1} }_2^2      \bs{x}_{t+1}  \bs{x}_{t+1}^\top  \right]    \W_t \W_t^\top \sTh  \bs{T}_2  \\
& \labelrel\preceq{nb:ineqq3}   C d^2 \left( \tr(  \bs{T}_2^2 \sG_t )  \bs{T}_1 \sG_t  \bs{T}_1 +  \tr(  \bs{T}_1^2 \sG_t )  \bs{T}_2 \sG_t  \bs{T}_2   \right) ,
}
where we use  the Cauchy-Schwartz inequality in  \eqref{nb:ineqq3}. 

\smallskip
\textbf{Third item. }  We define
\eq{
& \bs{T}_1 \sTh^\top  \grad_{\text{St}}  \bs{L}_{t+1}  \grad_{\text{St}}  \bs{L}_{t+1}^\top \sTh   \bs{T}_2  \\
& = \underbrace{ e_{t+1}^2   \norm{\W_t^\top \bs{x}_{t+1}}_2^2  \bs{T}_1 \sTh^\top   (\bs{I}_d - \W_t \W_t^\top) \bs{x}_{t+1} }_{\coloneqq \bs{u}_{t+1}} \underbrace{ \bs{x}_{t+1}^\top    (\bs{I}_d - \W_t \W_t^\top)   \sTh  \bs{T}_2  }_{\coloneqq \bs{v}_{t+1}^\top}.
}
We have for $p \geq 2$
\eq{
&  \E_t \left[   \norm*{    \indic{\abs{e_{t+1}} \leq L  } \bs{u}_{t+1}  \bs{v}_{t+1}^\top  }_2^p  \right]     \\
& \leq L^{2p}  \E_t \Big[  \norm{\W_t^\top \bs{x}_{t+1}}_2^{2p}     \Big]  \E_t \Big[ \norm{\bs{T}_1  \sTh^\top  (\bs{I}_d - \W_t \W_t^\top) \bs{x}_{t+1}}_2^{2p}   \Big]^{\frac{1}{2}} \\
& \times \E_t \Big[ \norm{\bs{T}_2  \sTh^\top  (\bs{I}_d - \W_t \W_t^\top) \bs{x}_{t+1}}_2^{2p}   \Big]^{\frac{1}{2}} \\
& \labelrel\leq{nb:ineqq31}  L^{2p} p^{2p} \left(3 \rs \sqrt{  \tr\big(\T_1^2 (\bs{I}_{\rk} - \sG_t) \big)  \tr\big(\T_2^2 (\bs{I}_{\rk} - \sG_t) \big) } \right)^p,
}
where we use   Corollary \ref{cor:hc2} in \eqref{nb:ineqq31}.  
By Proposition \ref{prop:basicconcentration},  we have for $u \geq (2e)^2$
\eq{
 \mpr_t \Big[  \norm*{    \indic{\abs{e_{t+1}} \leq L  } \bs{u}_{t+1}  \bs{v}_{t+1}^\top   }_2   \geq    u L^2   3 \rs \sqrt{  \tr\big(\T_1^2 (\bs{I}_{\rk} - \sG_t) \big)  \tr\big(\T_2^2 (\bs{I}_{\rk} - \sG_t) \big) }  \Big] \leq e^{- \frac{u^{1/2}}{e}}.
}
By choosing $u = \tfrac{L^2}{6}$, we have the probability bound.      For the variance bound, we have
\eq{
\E_t \left[   \big(  \sym \left(  \bs{T}_1  \sTh^\top      \grad_{\text{St}}  \bs{L}_{t+1}  \grad_{\text{St}}  \bs{L}_{t+1}^\top \sTh   \bs{T}_2   \right) \big)^2 \right] = \E_t \left[      \sym \left(   \bs{u}_{t+1}  \bs{v}_{t+1}^\top  \right)^2 \right].
}
By using Proposition  \ref{prop:matrixyounginequality}, we have
\eq{
& \E_t \left[      \sym \left(   \bs{u}_{t+1}  \bs{v}_{t+1}^\top  \right)^2 \right] \\
& \preceq  \!   \bs{T}_1  \sTh^\top \!\! \left( \bs{I}_d  \! - \!  \W_t \W_t^\top    \right) \\
& \quad \times \E_t \!  \big[ e_{t+1}^4  \norm{  \W_t^\top  \bs{x}_{t+1}  }_2^4 \norm{ \bs{T}_2 \sTh^\top   (\bs{I}_d \! - \! \W_t \W_t^\top)   \bs{x}_{t+1} }_2^2      \bs{x}_{t+1}  \bs{x}_{t+1}^\top  \big] \left( \bs{I}_d -  \W_t \W_t^\top    \right)   \sTh  \bs{T}_1  \\
& + \! \bs{T}_2  \sTh^\top \!\! \left( \bs{I}_d -  \W_t \W_t^\top    \right)  \\
& \quad \times  \E_t \! \big[ e_{t+1}^4  \norm{  \W_t^\top  \bs{x}_{t+1}  }_2^4 \norm{ \bs{T}_1 \sTh^\top   (\bs{I}_d \! - \! \W_t \W_t^\top)   \bs{x}_{t+1} }_2^2      \bs{x}_{t+1}  \bs{x}_{t+1}^\top  \big]    \left( \bs{I}_d \! -  \! \W_t \W_t^\top    \right)   \sTh  \bs{T}_2  \\
& \labelrel\preceq{nb:ineqq32}   C \rs^2 \big( \tr\big(  \bs{T}_2^2 (\bs{I}_{\rk} -  \sG_t ) \big)  \bs{T}_1 (\bs{I}_{\rk} -  \sG_t )   \bs{T}_1 + \tr\big(  \bs{T}_1^2 (\bs{I}_{\rk} -  \sG_t ) \big)  \bs{T}_2 (\bs{I}_{\rk} -  \sG_t )   \bs{T}_2    \big),
}
where we used  the Cauchy-Schwartz inequality in  \eqref{nb:ineqq32}.
 
\textbf{Fourth item. }  We define
\eq{
& \bs{T}_1  \sTh^\top    \grad_{\text{St}}  \bs{L}_{t+1} \bs{\cP}_{t+1}   \sM_{t}^\top   \bs{T}_2   \\
& = \underbrace{ e_{t+1}^3 \norm{ (\bs{I}_d - \W_t \W_t^\top) \bs{x}_{t+1}  }_2^2 \norm{\W_t^\top \bs{x}_{t+1}}_2^2   \bs{T}_1  \sTh^\top   (\bs{I}_d - \W_t \W_t^\top) \bs{x}_{t+1} }_{\coloneqq \bs{u}_{t+1}} \underbrace{ \bs{x}_{t+1}^\top \W_t \W_t^\top \sTh \bs{T}_2 }_{\coloneqq \bs{v}_{t+1}^\top}.
}
We have for $p \geq 2$
\eq{
& \E_t \left[   \norm*{    \indic{\abs{e_{t+1}} \leq L  } \bs{u}_{t+1}  \bs{v}_{t+1}^\top   }_2^p  \right]    \\
& \leq  L^{3p}   \E_t \left[      \norm{ \left(  \bs{I}_{d} -   \W_t  \W_t^\top  \right)  \bs{x}_{t+1}  }_2^{2p}  \norm{\bs{T}_1  \sTh^\top  (\bs{I}_d - \W_t \W_t^\top) \bs{x}_{t+1}}_2^{p}  \right]   \\
& \quad  \times \E_t \left[   \norm{\W_t^\top \bs{x}_{t+1}}_2^{2p} \norm{\bs{T}_2  \sTh^\top \W_t \W_t^\top \bs{x}_{t+1}}_2^{p}    \right]  \\
& \leq  L^{3p} (2p)^p (\sqrt{3} d)^p p^{\frac{p}{2}} \left( \sqrt{3} \tr \big(\bs{T}^2_1 (\bs{I}_{\rk} - \sG_t) \big)    \right)^{\frac{p}{2}} (2p)^p (\sqrt{3} \rs)^p p^{\frac{p}{2}} \left( \sqrt{3} \tr \big(\bs{T}^2_2   \sG_t \big)    \right)^{\frac{p}{2}}  \\
& =    L^{3p} (12 \sqrt{3})^p  p^{3p}  \left( d \rs  \sqrt{  \tr \big(\bs{T}^2_1 (\bs{I}_{\rk} - \sG_t) \big)     \tr \big(\bs{T}^2_2   \sG_t \big)  }   \right)^{p} . 
}
By Proposition \ref{prop:basicconcentration},  we have for $u \geq (2e)^3$
\eq{
 \mpr_t \Big[  \norm*{     \indic{\abs{e_{t+1}} \leq L  } \bs{u}_{t+1}  \bs{v}_{t+1}^\top }_2   \geq    u L^3   12 \sqrt{3}  d \rs \sqrt{ \tr \big(\bs{T}_1^2 (\bs{I}_{\rk} - \sG_t) \big)    \tr \big(\bs{T}^2_2  \sG_t \big)   }      \Big] \leq  e^{- \frac{u^{1/3}}{e}}.
}
By choosing $u = \tfrac{L^3}{24 \sqrt{3}}$, we have the probability bound.     For the variance bound, we have
\eq{
\E_t \left[  \big( \sym   \left(  \bs{T}_1  \sTh^\top   \grad_{\text{St}}  \bs{L}_{t+1} \bs{\cP}_{t+1}   \sM_t^\top   \bs{T}_2   \right) \big)^2 \right] = \E_t \left[      \sym \left(   \bs{u}_{t+1}  \bs{v}_{t+1}^\top  \right)^2 \right].
}
By using Proposition \ref{prop:matrixyounginequality}, we have
\eq{
& \E_t \left[      \sym \left(   \bs{u}_{t+1}  \bs{v}_{t+1}^\top  \right)^2 \right] \\
& \preceq  \bs{T}_1 \sTh^\top \!  (\bs{I}_d \! - \! \W_t \W_t^\top) \\
& \quad \times \E_t   \left[  e_{t+1}^6 \norm{(\bs{I}_d \! - \! \W_t \W_t^\top) \bs{x}_{t+1} }_2^4 \norm{\W_t^\top \bs{x}_{t+1} }_2^4 \norm{\bs{T}_2 \sTh^\top \W_t \W_t^\top \bs{x}_{t+1}  }_2^2  \bs{x}_{t+1} \bs{x}_{t+1}^\top  \right]  \\
& \qquad \times (\bs{I}_d \! - \! \W_t \W_t^\top) \sTh  \bs{T}_1 \\
& +  \bs{T}_2 \sTh^\top  \W_t \W_t^\top  \\
& \quad \times  \E_t \left[  e_{t+1}^6 \norm{(\bs{I}_d - \W_t \W_t^\top) \bs{x}_{t+1} }_2^4 \norm{\W_t^\top \bs{x}_{t+1} }_2^4 \norm{\bs{T}_1  \sTh^\top (\bs{I}_d - \W_t \W_t^\top)   \bs{x}_{t+1}  }_2^2  \bs{x}_{t+1} \bs{x}_{t+1}^\top  \right]   \\
& \qquad \times  \W_t \W_t^\top  \sTh  \bs{T}_2 \\
& \preceq  C d^2 \rs^2  \Big( \tr(  \bs{T}_2^2 \sG_t )    \bs{T}_1  (\bs{I}_{\rk} - \sG_t)   \bs{T}_1 +    \tr \big(  \bs{T}_1^2   (\bs{I}_{\rk} - \sG_t)   \big)  \bs{T}_2 \sG_t  \bs{T}_2  \Big)  .
}

\smallskip
\textbf{Fifth item. }  We define
\eq{
   \bs{T}_1  \sTh^\top   \grad_{\text{St}} &  \bs{L}_{t+1} \bs{\cP}_{t+1}   \grad_{\text{St}}  \bs{L}_{t+1}^\top  \sTh   \bs{T}_2 \\
& = \underbrace{e_{t+1}^4 \norm{(\bs{I}_d - \W_t \W_t^\top) \bs{x}_{t+1} }_2^2 \norm{\W_t^\top \bs{x}_{t+1}}_2^4  \bs{T}_1   \sTh^\top (\bs{I}_d - \W_t \W_t^\top) \bs{x}_{t+1} }_{\coloneqq \bs{u}_{t+1}} \\
& \quad \times \underbrace{ \bs{x}_{t+1}^\top   (\bs{I}_d - \W_t \W_t^\top) \sTh  \bs{T}_2  }_{\coloneqq \bs{v}_{t+1}^\top}.
}
We have for $p \geq 2$
\eq{
& \E_t \left[   \norm*{     \indic{\abs{e_{t+1}} \leq L } \bs{u}_{t+1}  \bs{v}_{t+1}^\top   }_2^p  \right]    \\
& \leq  L^{4p}  \E_t \Big[ \norm{ \left(  \bs{I}_{d} \! - \!  \W_t  \W_t^\top  \right)  \bs{x}_{t+1}  }_2^{2p}   \norm{\W_t^\top \bs{x}_{t+1}}_2^{4p}   \\
& \hspace{5em} \times \norm{\bs{T}_1  \sTh^\top  \! (\bs{I}_d \!-\! \W_t \W_t^\top) \bs{x}_{t+1}}_2^{p}  \norm{\bs{T}_1  \sTh^\top \!  (\bs{I}_d\!  - \! \W_t \W_t^\top) \bs{x}_{t+1}}_2^{p} \Big]   \\
& \leq   L^{4p}  \E_t \Big[ \norm{ \left(  \bs{I}_{d} -   \W_t  \W_t^\top  \right)  \bs{x}_{t+1}  }_2^{4p} \Big]^{\frac{1}{2}} \E_t \Big[  \norm{\bs{T}_1 \sTh^\top  (\bs{I}_d - \W_t \W_t^\top) \bs{x}_{t+1}}_2^{4p}\Big]^{\frac{1}{4}}   \\
& \quad  \times   \E_t \Big[  \norm{\bs{T}_2 \sTh^\top  (\bs{I}_d - \W_t \W_t^\top) \bs{x}_{t+1}}_2^{4p}\Big]^{\frac{1}{4}}  \E_t \Big[ \norm{\W_t^\top \bs{x}_{t+1}}_2^{4p}     \Big] \\
& \leq  L^{4p} (2p)^{p} (\sqrt{3} d)^p (2 p)^p \Big( 3 \tr \big(\bs{T}_1^2 (\bs{I}_{\rk} - \sG_t) \big) \tr \big(\bs{T}_2^2   (\bs{I}_{\rk} - \sG_t) \big)  \Big)^{\frac{p}{2}} (2p)^{2p} (\sqrt{3} \rs)^{2p} \\
& =   L^{4p} (2 \sqrt{3})^{4p} p^{4p} \left( d \rs^2  \sqrt{ \tr \big(\bs{T}_1^2 (\bs{I}_{\rk} - \sG_t) \big) \tr \big(\bs{T}_2^2   (\bs{I}_{\rk} - \sG_t)  \big)  } \right)^p.
}
By Proposition \ref{prop:basicconcentration},  we have for $u \geq (2e)^4$
\eq{
& \mpr_t \Big[  \norm*{   \indic{\abs{e_{t+1}} \leq L  } \bs{u}_{t+1}  \bs{v}_{t+1}^\top  }_2   \! \geq  \!   u L^4   (2 \sqrt{3})^4  d \rs^2    \sqrt{ \tr \big(\bs{T}_1^2 (\bs{I}_{\rk} \! - \! \sG_t) \big) \tr \big(\bs{T}_2^2   (\bs{I}_{\rk} \! - \! \sG_t)  \big)  }      \Big]  \! \leq  \! 2 e^{- \frac{u^{1/4}}{e}}.
}
By choosing  $u = \tfrac{L^4}{2 (2 \sqrt{3})^4}$,  we have the probability bound.  For the variance bound, we have
\eq{
 \E_t \left[  \big( \sym  \left(   \bs{T}_1    \sTh^\top   \grad_{\text{St}}  \bs{L}_{t+1} \bs{\cP}_{t+1}   \grad_{\text{St}}  \bs{L}_{t+1}^\top  \sTh   \bs{T}_2 \right) \big)^2 \right] = \E_t \left[      \sym \left(   \bs{u}_{t+1}  \bs{v}_{t+1}^\top  \right)^2 \right].
}
By using Proposition \ref{prop:matrixyounginequality}, we have
\eq{
& \E_t \left[      \sym \left(   \bs{u}_{t+1}  \bs{v}_{t+1}^\top  \right)^2 \right] \\
& \preceq    \bs{T}_1  \sTh^\top   (\bs{I}_d - \W_t \W_t^\top)         \\
&  \quad   \times  \E_t \left[ e_{t+1}^8  \norm{  \! \left(  \bs{I}_{d} -   \W_t  \W_t^\top  \right) \!  \bs{x}_{t+1}  }_2^4  \norm{ \W_t^\top  \bs{x}_{t+1}   }_2^8  \norm{ \bs{T}_2 \sTh^\top   (\bs{I}_d - \W_t \W_t^\top)   \bs{x}_{t+1} }_2^2      \bs{x}_{t+1}  \bs{x}_{t+1}^\top  \right]   \\
&  \qquad \times (\bs{I}_d - \W_t \W_t^\top) \sTh  \bs{T}_1  \\
& +    \bs{T}_2  \sTh^\top   (\bs{I}_d - \W_t \W_t^\top)         \\
&  \quad   \times  \E_t \left[ e_{t+1}^8  \norm{ \! \left(  \bs{I}_{d} -   \W_t  \W_t^\top  \right) \! \bs{x}_{t+1}  }_2^4  \norm{ \W_t^\top  \bs{x}_{t+1}   }_2^8  \norm{ \bs{T}_1 \sTh^\top   (\bs{I}_d - \W_t \W_t^\top)   \bs{x}_{t+1} }_2^2      \bs{x}_{t+1}  \bs{x}_{t+1}^\top  \right] \\
& \qquad \times \! (\bs{I}_d - \W_t \W_t^\top) \sTh  \bs{T}_2  \\
&  \preceq   C d^2 \rs^4 \Big(  \tr \big(\bs{T}_2^2 (\bs{I}_{\rk} - \sG_t) \big)     \bs{T}_1 (\bs{I}_{\rk} - \sG_t)  \bs{T}_1  +  \tr \big(\bs{T}_1^2 (\bs{I}_{\rk} - \sG_t) \big)     \bs{T}_2 (\bs{I}_{\rk} - \sG_t)  \bs{T}_2  \Big).
}
\end{proof}

By recalling the definitions  $\{ \T_t \}_{t \in \N}$, $  \bs{\Lambda}$, $ \bs{\Lambda}_{11}$ in Sections \ref{sec:monotonesec} and  \ref{sec:defs}, we define the event:
\eq{
\mathcal{A}_{t +1} \coloneqq   \begin{cases}
\mathcal{A}_{t +1} \left( \T_t^{\frac{-1}{2}},  \T_t^{\frac{-1}{2}} \right) \cap   \mathcal{A}_{t +1} \left(  \bs{\Lambda}^{\frac{1}{2}} \T_t^{\frac{-1}{2}},  \T_t^{\frac{-1}{2}}  \bs{\Lambda}^{\frac{1}{2}} \right) \cap \{ e_{t+1} \leq L \} , & \alpha \in [0,0.5) \\[0.35em]
\mathcal{A}_{t +1} \left( \T_t^{\frac{-1}{2}},  \T_t^{\frac{-1}{2}} \right) \cap   \mathcal{A}_{t +1} \left(  \bs{\Lambda}_{11}^{\frac{1}{2}} \T_t^{\frac{-1}{2}},  \T_t^{\frac{-1}{2}}  \bs{\Lambda}_{11}^{\frac{1}{2}} \right) \cap \{ e_{t+1} \leq L \} , & \alpha > 0.5.  
\end{cases}
}
We define the events  $ \mathcal{B}_{t+1}$, $\mathcal{C}_{t+1}$,   $\mathcal{D}_{t+1}$, and  $\mathcal{F}_{t+1}$ in the same way.  Based on these events, we define the clipped versions of the noise matrices: 
\eq{
 & \mathsf{A}_{t+1}  \coloneqq \sym \left(  \sTh^\top    \grad_{\text{St}}  \bs{L}_{t+1} \sM_{t}^\top  \indic{\mathcal{A}_{t+1} } - \E_t \left[ \sTh^\top \grad_{\text{St}}  \bs{L}_{t+1} \sM_{t}^\top \indic{\mathcal{A}_{t+1} } \right]   \right)  \\[0.7em]
 & \mathsf{B}_{t+1}   \coloneqq     \frac{  \sM_{t} \bs{\cP}_{t+1}  \sM_{t}^\top \indic{\mathcal{B}_{t+1} } }{1 +  c_{t+1}^2 } - \E_t \left[ \frac{  \sM_{t} \bs{\cP}_{t+1}  \sM_{t}^\top  \indic{\mathcal{B}_{t+1} } }{1 + c_{t+1}^2} \right]  \\[0.7em]
  & \mathsf{C}_{t+1}   \coloneqq    \sTh^\top  \grad_{\text{St}}  \bs{L}_{t+1} \grad_{\text{St}}  \bs{L}_{t+1}^\top   \sTh  \indic{\mathcal{C}_{t+1} }    -    \E_t \left[ \sTh^\top \grad_{\text{St}}  \bs{L}_{t+1} \grad_{\text{St}}  \bs{L}_{t+1}^\top  \sTh  \indic{\mathcal{C}_{t+1} }   \right]  \\[0.7em]
 & \mathsf{D}_{t+1}   \coloneqq   \sym \left( \frac{  \sTh^\top \grad_{\text{St}}  \bs{L}_{t+1} \bs{\cP}_{t+1}  \sM_{t}^\top \indic{\mathcal{D}_{t+1} } }{1 + c_{t+1}^2}   -   \E_t \left[ \frac{  \sTh^\top \grad_{\text{St}}  \bs{L}_{t+1} \bs{\cP}_{t+1}  \sM_{t}^\top \indic{\mathcal{D}_{t+1} } }{1 + c_{t+1}^2}  \right]     \right) \\[0.7em]
  & \mathsf{F}_{t+1}   \coloneqq   \frac{   \sTh^\top \grad_{\text{St}}  \bs{L}_{t+1} \bs{\cP}_{t+1}   \grad_{\text{St}}  \bs{L}_{t+1}^\top \sTh \indic{\mathcal{F}_{t+1} } }{1 + c_{t+1}^2} -   \E_t \left[  \frac{   \sTh^\top \grad_{\text{St}}  \bs{L}_{t+1} \bs{\cP}_{t+1}   \grad_{\text{St}}  \bs{L}_{t+1}^\top \sTh \indic{\mathcal{F}_{t+1} } }{1 + c_{t+1}^2} \right] .    \label{eq:clippedversions}
}
Let  $ \mathsf{X} \in \left \{  \tfrac{\eta/2}{\sqrt{\rs}} \mathsf{A},    \tfrac{\eta^2/16}{\rs} \mathsf{B},    \tfrac{\eta^2/16}{\rs} \mathsf{C},    \tfrac{\eta^3/32}{\rs^{3/2}} \mathsf{D},  \tfrac{\eta^4/256}{\rs^2}  \mathsf{F}  \right \}$ and
\eq{
\mathsf{\Gamma}_1 \coloneqq \begin{cases}
 \bs{I}_{r}, & \alpha \in [0,0.5) \\
 \bs{I}_{\ru} & \alpha > 0.5 
\end{cases}  \qquad   \mathsf{\Gamma}_2 \coloneqq  \begin{cases}
\bs{\Lambda}^{\frac{1}{2}}, & \alpha \in [0,0.5) \\
\bs{\Lambda}_{11}^{\frac{1}{2}}, & \alpha > 0.5.
\end{cases}
}
For  $\ell \in \{1,2\}$,  we define:
\begin{align} 
& \text{Quad}_{k,t}^{(\ell)}( \mathsf{X}) \coloneqq     \sum_{j = 1}^k  \E_{j - 1} \left[ \Big( \mathsf{\Gamma}_\ell  \T_t^{\frac{-1}{2}} \mathsf{X}_{j}  \T_t^{\frac{-1}{2}}  \mathsf{\Gamma}_\ell \Big)^2 \right].
\end{align}
We have the following corollary.
\begin{corollary}  
\label{cor:boundsqv}
Let  $\rk \in \{ r, \ru \}$,  $\rks \in \{ r, \rs\}$ and
\eq{
\mathsf{S}_t \coloneqq \lr \sum_{j =1}^t \sG_{j-1}    ~~ \text{and} ~ \lr = \frac{\eta}{\sqrt{\rs} \normL} ~~ \text{and} ~~ \ru = \ceil{\log^{2.5} d}.
}
Assume the following conditions hold:
\begin{itemize}
\item $\T_t \succeq \frac{\c_d \rs}{d} \bs{I}_{\rk}$, 
\item  $\T_t^{\frac{-1}{2}}  \mathsf{S}_t \T_t^{\frac{-1}{2}}  \preceq   \frac{C_{\text{ub}}}{\c_d}  \begin{cases} (2  \lr t \vee r^{\alpha})   \bs{I}_{r}, & \alpha \in [0,0.5) \\
\rs^{\alpha} \log d \bs{I}_{\ru} ,   & \alpha > 0.5, 
   \end{cases}$
\item  $\T_t^{\frac{-1}{2}} \sG_{j} \T_t^{\frac{-1}{2}} \preceq  \frac{C_{\text{ub}}}{\c_d} \bs{I}_{\rk}$ for $j  \leq t -1$.
\end{itemize}
 Let 
\eq{
\begin{cases} 
\mathsf{p}_1 = 1   ~~ \text{and} ~~  \mathsf{p}_2 = 1 - \alpha, & \alpha \in [0,1) \\
\mathsf{p}_1 = 1  ~ ~ \text{and} ~~  \mathsf{p}_2 = \frac{2 \log \log \ru}{\log \ru}  & \alpha = 1 \\
\mathsf{p}_1 = 1  ~ ~ \text{and} ~~  \mathsf{p}_2 = 0  & \alpha > 1.
\end{cases}
}
For  $\lr t \leq \tfrac{1}{2} \rks^{\alpha} \log \left( \tfrac{d \log^{1.5} d}{\rs} \right)$,  the following results hold:  
\paragraph{(a) Quadratic variation bounds.} We have: 
\eq{
\norm{ \text{Quad}_{t,t}^{(\ell)}( \mathsf{X}) }_2 \leq   \frac{C \log d}{\c_d^2} \frac{  \normL  \rk^{\mathsf{p}_\ell} \rks^{\alpha} }{\rs^{3/2}}    \begin{cases}
 C_{\text{ub}} \eta d      & \mathsf{X} =  \tfrac{\eta/2}{\sqrt{\rs}} \mathsf{A}  \\[0.35em]
 C_{\text{ub}}^2 \eta^3 d^2 ,    & \mathsf{X} =  \tfrac{\eta^2/16}{\rs}  \mathsf{B} \\[0.35em] 
\eta^3 d^2 ,     & \mathsf{X} = \tfrac{\eta^2/16}{\rs}   \mathsf{C} \\[0.35em] 
 C_{\text{ub}} \eta^5 d^3 ,     & \mathsf{X} = \tfrac{\eta^3/32}{\rs^{3/2}}  \mathsf{D} \\[0.35em] 
\eta^7 d^2 ,    & \mathsf{X} = \tfrac{\eta^4/256}{\rs^{2}}  \mathsf{F}. 
\end{cases}
}
\paragraph{(b) Operator norm bounds.} For $L \geq 8 \sqrt{2} e$,  there exists $C > 0$ such that
\eq{
 \norm*{   \mathsf{\Gamma}_\ell \T_t^{\frac{-1}{2}} \mathsf{X}_{j}  \T_t^{\frac{-1}{2}}  \mathsf{\Gamma}_\ell}_2 &   \! \leq \!  r_{j, t}^{(\ell)}(\mathsf{X})  \! \coloneqq  \!   \begin{cases}
 \frac{\eta L^2}{2\sqrt{\rs}}  \sqrt{   \tr \big(    \mathsf{\Gamma}_\ell^2  \bs{T}_t^{-1}  \big)   \tr(    \mathsf{\Gamma}_\ell^2  \bs{T}_t^{\frac{-1}{2}} \sG_{j-1}   \bs{T}_t^{\frac{-1}{2}}) } ,       & \mathsf{X} =  \tfrac{\eta/2}{\sqrt{\rs}} \mathsf{A}  \\[0.35em]
\frac{\eta^2 L^4}{16 \rs}  d       \tr(   \mathsf{\Gamma}_\ell^2  \bs{T}_t^{\frac{-1}{2}} \sG_{j - 1}   \bs{T}_t^{\frac{-1}{2}} )    ,     & \mathsf{X} =  \tfrac{\eta^2/16}{\rs}  \mathsf{B} \\[0.35em] 
\frac{\eta^2 L^4}{16}       \tr\big(  \mathsf{\Gamma}_\ell^2  \bs{T}_t^{-1} \big),    & \mathsf{X} =  \tfrac{\eta^2/16}{\rs}  \mathsf{C} \\[0.35em] 
\frac{\eta^3 L^6}{32 \sqrt{\rs}}    d  \sqrt{   \tr \big(    \mathsf{\Gamma}_\ell^2  \bs{T}_t^{-1}  \big)   \tr(    \mathsf{\Gamma}_\ell^2  \bs{T}_t^{\frac{-1}{2}} \sG_{j-1}   \bs{T}_t^{\frac{-1}{2}}) }  ,   & \mathsf{X} = \tfrac{\eta^3/32}{\rs^{3/2}}  \mathsf{D} \\[0.35em] 
 \tfrac{\eta^4 L^8}{ 256 }       d      \tr\big(  \mathsf{\Gamma}_\ell^2  \bs{T}_t^{-1} \big),   & \mathsf{X} = \tfrac{\eta^4/256}{\rs^{2}}  \mathsf{F}  
\end{cases}  \\
& \leq   \frac{C}{\c_d} \frac{\rk^{\mathsf{p}_\ell} }{\rs}    \begin{cases}
L^2 \sqrt{C_{\text{ub}}} \eta \sqrt{d},       & \mathsf{X} =  \tfrac{\eta/2}{\sqrt{\rs}} \mathsf{A}  \\[0.35em]
L^4  C_{\text{ub}} \eta^2 d  ,     & \mathsf{X} =  \tfrac{\eta^2/16}{\rs}  \mathsf{B} \\[0.35em] 
L^4 \eta^2 d  ,    & \mathsf{X} =  \tfrac{\eta^2/16}{\rs}  \mathsf{C} \\[0.35em] 
L^6 \sqrt{C_{\text{ub}}} \eta^3 d^{3/2}  ,   & \mathsf{X} = \tfrac{\eta^3/32}{\rs^{3/2}}  \mathsf{D} \\[0.35em] 
L^8 \eta^4 d^2 ,   & \mathsf{X} = \tfrac{\eta^4/256}{\rs^{2}}  \mathsf{F} . 
\end{cases} 
}
\end{corollary}

 \begin{proof}
\textit{Quadratic variation bounds.} We will use the variance bounds given  in Proposition \ref{prop:noisevarbounds}.   For $\mathsf{X} =  \tfrac{\eta/2}{\sqrt{\rs}} \mathsf{A}$,  we have
\eq{
 \text{Quad}_{t,t}^{(\ell)}(  \tfrac{\eta/2}{\sqrt{\rs}} \mathsf{A} )
&  \preceq \frac{C \eta \normL}{\sqrt{\rs}}     \left(  \tr \Big( \mathsf{\Gamma}_\ell^2 \bs{T}_t^{\frac{-1}{2}} \mathsf{S}_t \bs{T}_t^{\frac{-1}{2}} \Big) \mathsf{\Gamma}_\ell    \bs{T}_t^{-1} \mathsf{\Gamma}_\ell   +    \tr \big(   \mathsf{\Gamma}_\ell^2 \bs{T}_t^{-1}     \big)  \mathsf{\Gamma}_\ell  \bs{T}_t^{\frac{-1}{2}} \S_t \bs{T}_t^{\frac{-1}{2}} \mathsf{\Gamma}_\ell \right) \\
&  \preceq   \frac{C  \eta \normL}{\sqrt{\rs}}    \frac{C_{\text{ub}}  \rk^{\mathsf{p}_\ell} \rks^{\alpha} \log d}{\c_d^2}  \frac{d}{\rs} \bs{I}_{\rk}.
}
For $\mathsf{X} =  \tfrac{\eta^2/16}{\rs}  \mathsf{B}$,  we have
\eq{
 \text{Quad}_{t,t}^{(\ell)}(    \tfrac{\eta^2/16}{\rs}  \mathsf{B} )
& \preceq    \frac{C C_{\text{ub}} \eta^3 d^2  \normL}{  \rs^{3/2}}  \sup_{j  \leq t} \left(\tr \big( \mathsf{\Gamma}_\ell^2   \bs{T}_t^{\frac{-1}{2}}  \sG_{j-1} \bs{T}_t^{\frac{-1}{2}}     \big) \right)   \mathsf{\Gamma}_\ell   \bs{T}_t^{\frac{-1}{2}}  \mathsf{S}_t \bs{T}_t^{\frac{-1}{2}}  \mathsf{\Gamma}_\ell \\
&  \preceq    \frac{C \eta^3 d^2  \normL}{  \rs^{3/2}} \frac{C_{\text{ub}}^2 \rk^{\mathsf{p}_\ell} \rks^{\alpha} \log d}{\c_d^2} \bs{I}_{\rk}.
}
For $\mathsf{X} =  \tfrac{\eta^2/16}{\rs}  \mathsf{C}$,  we have
\eq{
 \text{Quad}_{t,t}^{(\ell)}(    \tfrac{\eta^2/16}{\rs}  \mathsf{C} )
\preceq    C \eta^4 t  \tr(  \mathsf{\Gamma}_\ell^2 \bs{T}_t^{-1}  )   \mathsf{\Gamma}_\ell \bs{T}_t^{-1} \mathsf{\Gamma}_\ell \preceq \frac{C \eta^3 d^2 \normL }{\rs^{3/2}} \frac{\rk^{\mathsf{p}_\ell} \rks^{\alpha} \log d}{\c_d^2} \bs{I}_{\rk}.
}
For $ \mathsf{X} = \tfrac{\eta^3/32}{\rs^{3/2}}  \mathsf{D}$,  we have
\eq{
 \text{Quad}_{t,t}^{(\ell)}(  \tfrac{\eta^3/32}{\rs^{3/2}}  \mathsf{D}) &  \! \preceq \!   \frac{C \eta^5 d^2  \normL}{\sqrt{\rs}} \!  \left(  \tr(   \mathsf{\Gamma}_\ell^2 \bs{T}_t^{\frac{-1}{2}} \mathsf{S}_t \bs{T}_t^{\frac{-1}{2}} )     \mathsf{\Gamma}_\ell \bs{T}_t^{-1}  \mathsf{\Gamma}_\ell   +    \tr \big(    \mathsf{\Gamma}_\ell^2 \bs{T}_t^{-1}   \big)  \mathsf{\Gamma}_\ell  \bs{T}_t^{\frac{-1}{2}}  \mathsf{S}_t   \bs{T}_t^{\frac{-1}{2}}   \mathsf{\Gamma}_\ell \right) \\
 & \preceq     \frac{C \eta^5 d^2  \normL}{\sqrt{\rs}} \frac{C_{\text{ub}} \rk^{\mathsf{p}_\ell} \rks^{\alpha} \log d}{\c_d^2}  \frac{d}{\rs} \bs{I}_{\rk}.
 }
For $ \mathsf{X} = \tfrac{\eta^4/256}{\rs^{2}}  \mathsf{F}  $,  we have
\eq{
 \text{Quad}_{t,t}^{(\ell)}(  \tfrac{\eta^4/256}{\rs^{2}}  \mathsf{F} ) \preceq      C \eta^8 d^2  t   \tr(   \mathsf{\Gamma}_\ell^2 \bs{T}_t^{-1}  )   \mathsf{\Gamma}_\ell   \bs{T}_t^{-1}  \mathsf{\Gamma}_\ell \preceq \frac{C \eta^7 d^2 \normL}{\rs^{3/2}} \frac{\rk^{\mathsf{p}_\ell} \rks^{\alpha} \log d}{\c_d^2} \bs{I}_{\rk}.
} 
\textit{Operator Norm Bounds. }  We will use the events defined  in Proposition \ref{prop:noisevarbounds}.    For $\mathsf{X} =  \tfrac{\eta/2}{\sqrt{\rs}} \mathsf{A}$,  we have
\eq{
r_{j, t}^{(\ell)}( \tfrac{\eta/2}{\sqrt{\rs}} \mathsf{A}) =   \frac{\eta/2}{\sqrt{\rs}} L^2  \sqrt{ \tr(   \mathsf{\Gamma}_\ell^2 \T_t^{-1})  \tr(   \mathsf{\Gamma}_\ell^2 \T_t^{ \frac{-1}{2}} \sG_{j-1} \T_t^{ \frac{-1}{2}})}
  \leq  \frac{C  L^2}{\c_d} \frac{\sqrt{C_{\text{ub}}} \eta \sqrt{d} \rk^{\mathsf{p}_\ell} }{\rs}.
}
For $\mathsf{X} =  \tfrac{\eta^2/16}{\rs}  \mathsf{B}$,  we have
\eq{
r_{j, t}^{(\ell)}(\tfrac{\eta^2/16}{\rs}  \mathsf{B}) =   \frac{\eta^2}{16 \rs}  L^4 d  \,  \tr(\mathsf{\Gamma}_\ell^2 \T_t^{ \frac{-1}{2}} \sG_{j-1} \T_t^{ \frac{-1}{2}})
 \leq  \frac{C L^4}{\c_d} \frac{C_{\text{ub}}  \eta^2 d \rk^{\mathsf{p}_\ell} }{\rs}.
}  
For $\mathsf{X} =  \tfrac{\eta^2/16}{\rs}  \mathsf{C}$,  we have
\eq{
r_{j, t}^{(\ell)}(\tfrac{\eta^2/16}{\rs}  \mathsf{C}) =   \frac{\eta^2}{16 \rs} L^4   \tr( \mathsf{\Gamma}_\ell^2  \T_t^{-1})
\leq   \frac{C L^4}{\c_d} \frac{  \eta^2 d \rk^{\mathsf{p}_\ell} }{\rs}.
}
For $ \mathsf{X} = \tfrac{\eta^3/32}{\rs^{3/2}}  \mathsf{D}$,  we have
\eq{
r_{j, t}^{(\ell)}( \tfrac{\eta^3/32}{\rs^{3/2}}  \mathsf{D})  = \frac{\eta^3}{32 \sqrt{\rs}} L^6 d      \sqrt{ \tr(\mathsf{\Gamma}_\ell^2  \T_t^{ \frac{-1}{2}} \sG_{j-1} \T_t^{ \frac{-1}{2}}) \tr(\mathsf{\Gamma}_\ell^2  \T_t^{-1})} 
\leq  \frac{C L^6}{\c_d} \frac{\sqrt{C_{\text{ub}}} \eta^3 d^{3/2} \rk^{\mathsf{p}_\ell}}{\rs}.
}
For $ \mathsf{X} = \tfrac{\eta^4/256}{\rs^2}  \mathsf{F}$,  we have
\eq{
r_{j, t}^{(\ell)}(  \tfrac{\eta^4/256}{\rs^2}  \mathsf{F} )  = \frac{\eta^4}{256} L^8 d    \tr(\mathsf{\Gamma}_\ell^2 \T_t^{-1}) \leq  \frac{C L^8}{\c_d} \frac{\eta^4 d^2 \rk^{\mathsf{p}_\ell}}{\rs}.
 } 
 \end{proof}

\begin{proposition}
\label{prop:friedmanineq}
Let $\{  \mathsf{Y}_t,  ~ t =  1,2 \cdots \}$  be a symmetric-matrix martingale with difference sequence  $\{  \mathsf{X}_{t} \coloneqq \mathsf{Y}_{t+1} - \mathsf{Y}_t ,  ~ t = 1,2 \cdots \}$,  whose values are symmetric matrices with dimension $r \leq d$.  Let    $\{ \T_t, ~ t =  1,2,\cdots \}$ be a deterministic sequence,   whose values are positive semi-definite matrices with the same dimensionality.  
Assume that the difference sequence is uniformly bounded in the sense that for a predictable triangular sequence $\{ r_{j, t} \}_{j \leq t}$, we have
\eq{
\lambda_{\text{max}} (   \T_t^{\frac{-1}{2}} \mathsf{X}_j  \T_t^{\frac{-1}{2}} ) \leq r_{j, t} ~~ \text{for} ~~ j = 1, 2, \cdots, t. 
}
 Define the predictable uniform bound and quadratic variation process of the martingale:
\eq{
R_{k,t} \coloneqq \max_{j \leq k} r_{j,t} ~~ \text{and} ~~ \text{Quad}_{k,t}(\mathsf{X}) \coloneqq  \sum_{j = 1}^k \E_{j -1} \left[ \left( \T_t^{\frac{-1}{2}} \mathsf{X}_{j}  \T_t^{\frac{-1}{2}} \right)^2   \right]  ~~ \text{for} ~~ k \leq t = 1, 2, \cdots. 
}
Let $\mathcal{T} \leq d^3$ be a bounded stopping time.  Then,  for any deterministic $\sigma^2, \widetilde{L} > 0$
\eq{
& \mpr  \left[  \exists t \leq  \mathcal{T},   \mathsf{Y}_t \not \preceq u  \T_t  ~~ \text{and} ~~   \max_{t \leq  \mathcal{T} }  \norm{ \text{Quad}_{t,t} (\mathsf{X}) }_2 \leq \sigma^2  ~~ \text{and} ~~  \max_{t \leq  \mathcal{T} }R_{t,t} \leq \widetilde{L} \right] \\
& \hspace{25em} \leq d^4 \cdot \exp \left(  \frac{- u^2/2}{\sigma^2 + \widetilde{L} u/3} \right).
}
\end{proposition} 

\begin{proof}
We have
\eq{
 \mathcal{E}_{\text{target}}  \equiv  \big \{  \exists t \leq  \mathcal{T},   & \mathsf{Y}_t   \not \preceq u  \T_t  ~ \text{and}  ~    \max_{t \leq  \mathcal{T} }  \norm{ \text{Quad}_{t,t} (\mathsf{X}) }_2  \leq \sigma^2  ~ \text{and}  ~  \max_{t \leq  \mathcal{T} }R_{t,t} \leq \widetilde{L}   \big \}  \\
 & \subseteq \bigcup_{t = 0}^{\mathcal{T}}  \big \{ \exists k \leq t,      \mathsf{Y}_k  \not \preceq u  \T_t  ~ \text{and}  ~  \norm{ \text{Quad}_{t,t} (\mathsf{X})   }_2 \leq \sigma^2 ~ \text{and} ~ R_{t,t} \leq \widetilde{L} \big \}.
}
Therefore, we have
\eq{
\mpr \big[  \mathcal{E}_{\text{target}}   \big] \leq \sum_{n = 1}^{d^3} \mpr \left[  \exists k \leq t, ~   \mathsf{Y}_k  \not \preceq u  \T_t  ~~ \text{and} ~~  \norm{ \text{Quad}_{t,t}(\mathsf{X}) }_2 \leq \sigma^2 ~ \text{and} ~ R_{t,t} \leq \widetilde{L} \right]. \label{eq:unionbound}
}
In the following, we will bound the each term in the right hands-side of  \eqref{eq:unionbound}.   By \cite[
Lemma 6.7]{Tropp2010UserFriendlyTB}, we have for $k = 1, \cdots, t$ and $\theta > 0$,
\eq{
\indic{R_{k,t} \leq \widetilde{L}} \E_{k -1} \Big[ e^{ \frac{\theta}{L}   \T_t^{\frac{-1}{2}} \mathsf{X}_k  \T_t^{\frac{-1}{2}} } \Big] \preceq \indic{R_{k,t} \leq \widetilde{L}}  \exp \left( \frac{e^{\theta} - \theta - 1}{\widetilde{L}^2} \E_{k - 1} \left[  (  \T_t^{\frac{-1}{2}} \mathsf{X}_{k}  \T_t^{\frac{-1}{2}})^2  \right]    \right).  \label{eq:cgfbound}
}
For notational convenience call $g(\theta) \coloneqq  e^{\theta} - \theta - 1$.
We define a super martingale such that for $0< k \leq t$,
\eq{
S_k  \coloneqq \tr \left( \exp \left( \frac{\theta}{L} \T_t^{\frac{-1}{2}}  \mathsf{Y}_k \T_t^{\frac{-1}{2}} - \frac{g(\theta)}{\widetilde{L}^2}  \text{Quad}_{k,t}(\mathsf{X})  \right) \right) \indic{R_{k,t} \leq \widetilde{L}},
} 
with initial values   $R_{0,t} = 0$, $\mathsf{Y}_0 =  \text{Quad}_{0,t} = 0$, and thus,  $S_0 = r$.  Note that by \eqref{eq:cgfbound} and \cite[Corollary 3.3]{Tropp2010UserFriendlyTB}, we can show that $\E_{k-1} S_k \leq S_{k-1}$. We define a stopping time and an event
\eq{
& \mathcal{T}_{\text{hit}} \coloneqq   \{ k \geq 0 ~ \vert ~  \lambda_{\text{max}} ( \T_t^{\frac{-1}{2}}  \mathsf{Y}_k  \T_t^{\frac{-1}{2}} )  \geq u      \} \wedge t, \\
& \mathcal{E}_{\text{hit}} \coloneqq \{ \mathcal{T}_{\text{hit}} \leq t \} \cap \{   \norm{ \text{Quad}_{t,t}(\mathsf{X})}_2 \leq \sigma^2 ~ \text{and} ~ R_{t,t} \leq \widetilde{L}  \} .
}
We have
\eq{
\indic{ \mathcal{E}_{\text{hit}}} S_{\mathcal{T}_{\text{hit}}} \geq  \indic{ \mathcal{E}_{\text{hit}}} \exp \left( \frac{\theta}{\widetilde{L}} u - \frac{g(\theta)}{\widetilde{L}^2} \sigma^2   \right) & \labelrel\Rightarrow{eq:matrixfreedman} r \geq \mpr \left[   \mathcal{E}_{\text{hit}} \right]  \exp \left(\frac{\theta}{\widetilde{L}} u - \frac{g(\theta)}{\widetilde{L}^2} \sigma^2   \right) \\
& \Rightarrow r \inf_{\theta > 0}  \exp \left(- \theta  \frac{u}{\widetilde{L}}+ g(\theta) \frac{\sigma^2}{\widetilde{L}^2}   \right)  \geq \mpr \left[   \mathcal{E}_{\text{hit}} \right], 
}
where we use Doob's optional sampling theorem in \eqref{eq:matrixfreedman}.  Since the infimum is attained at $\theta > 0$ and the convex conjugate of $g(\theta)$ is $g^\star(\theta) = (\theta +1) \log(\theta +1) -  \theta$, we have
\eq{
\mpr \left[   \mathcal{E}_{\text{hit}} \right] \leq r \cdot \exp \left( - \frac{\sigma^2}{\widetilde{L}^2} g^\star \left (\frac{u \widetilde{L}}{\sigma^2} \right) \right) \leq r \cdot  \exp \left( \frac{- u^2/2}{\sigma^2 + \widetilde{L} u/3}  \right),
}
where we used  $g^\star(\theta) \geq \frac{u^2/2}{1 + u/3}$ in the last step. By $r \leq d$ and  \eqref{eq:unionbound}, we have the statement.
\end{proof}

\begin{proposition}
\label{prop:Tbadboundaux}
Let  $\mpr_0$ denote the conditional probability conditioned on $\W_0$. We consider  $\ru = \ceil{\log^{2.5} d}$,  and 
\begin{alignat}{5}
&\alpha \in [0, 0.5): ~
&& \frac{\rs}{r} \to \varphi,  \quad  
&&  \eta \asymp \tfrac{1}{d r^{\alpha} \log^{20}(1+ d/\rs)},  \quad  
&&  \c_d = \frac{1}{\log^{3.5} d}, \quad
&& C_{\text{ub}} = 12 \left( 1 + \tfrac{1}{\sqrt{\varphi}} \right)^2 \\[0.5em]
& \alpha > 0.5: ~
&& \rs  \asymp 1, \quad  
&&\eta \asymp \tfrac{1}{d\,  \ru^{4 \alpha + 3} \log^{18} d}, \quad  
&& \c_d = \frac{1}{\ru \log^{2.5} d}, \quad
&& C_{\text{ub}} = 2^{\alpha} 30 \ru.
\end{alignat}
For $\alpha \in [0, 0.5)$,  we define $\mathcal{T} \coloneqq  \mathcal{T}_{\text{bad}} \wedge \frac{1}{2\lr}  \left( \rs   \big(1 -  \log^{\frac{- 1}{2}} \!\! d   \big) \wedge r \right)^{\alpha}   \log \left( \frac{d\log^{1.5} d}{\rs} \right).$ We have for $d \geq \Omega_{\alpha, \varphi, \beta}(1)$
\eq{
& \mpr_0  \left[ \sup_{t \leq \mathcal{T}}  \norm{ \T_{t}^{^{\frac{-1}{2}}} \unu{t} \T_{t}^{^{\frac{-1}{2}}}  }_2 \vee r^{\frac{\alpha}{2}} \norm{ \bs{\Lambda}^{\frac{1}{2}} \T_{t}^{^{\frac{-1}{2}}} \unu{t} \T_{t}^{^{\frac{-1}{2}}}  \bs{\Lambda}^{\frac{1}{2}}  }_2 \geq \c_d  r^{\frac{- \alpha}{2}} ~~ \text{and} ~~ \mathcal{G}_{\text{init}} \right] \\
& \hspace{25em} \leq 20  d^4 \exp ( - \log^2 d).
}
For $\alpha> 0.5$, we set $\mathcal{T} \coloneqq  \mathcal{T}_{\text{bad}} \wedge \frac{1}{2\lr}  \rs^{\alpha}   \log \left( \frac{d\log^{1.5} d}{\rs} \right)$. We have for $d\geq \Omega_{\alpha, \rs}(1)$
\eq{
& \mpr_0 \left[ \sup_{t \leq \mathcal{T}}  \norm{ \T_{t}^{^{\frac{-1}{2}}} \unu{t} \T_{t}^{^{\frac{-1}{2}}}  }_2 \vee \ru^{\frac{\alpha}{2}} \norm{ \bs{\Lambda}_{11}^{\frac{1}{2}} \T_{t}^{^{\frac{-1}{2}}} \unu{t} \T_{t}^{^{\frac{-1}{2}}}  \bs{\Lambda}_{11}^{\frac{1}{2}}  }_2 \geq \c_d  \ru^{\frac{- \alpha}{2}} ~~ \text{and} ~~  \mathcal{G}_{\text{init}}   \right] \\
& \hspace{25em} \leq 20  d^4 \exp ( - \log^2 d).
}
 \end{proposition}
 
\begin{proof}
For notational convenience, we introduce $\mathcal{X} \coloneqq \big \{  \tfrac{\eta/2}{\sqrt{\rs}} \mathsf{A},    \tfrac{\eta^2/16}{\rs} \mathsf{B},    \tfrac{\eta^2/16}{\rs} \mathsf{C},    \tfrac{\eta^3/32}{\rs^{3/2}} \mathsf{D},  \tfrac{\eta^4/256}{\rs^2}  \mathsf{F} \big \}$.  For both cases, we will set the clip threshold to $L = \log^2 d$.  We introduce the notation  $R_{t}^{(\ell)} \coloneqq \max_{\mathsf{X} \in\mathcal{X} } \max_{j \leq t} r_{j,t}^{(\ell)}( \mathsf{X})$ and  $\norm{ \text{Quad}^{(\ell)}_{t}  }_2 \coloneqq \max_{\mathsf{X} \in\mathcal{X} }  \norm{\text{Quad}^{(\ell)}_{t,t}(\mathsf{X}) }_2$ for $\ell = 1,2$.

\smallskip
For $\alpha \in [0,0.5)$, we can write for all $\mathsf{X} \in \mathcal{X} $,
\eq{
 \mpr_0 & \left[ \sup_{t \leq \mathcal{T}}  \norm{ \T_{t}^{\frac{-1}{2}} \unu{t}\T_{t}^{\frac{-1}{2}}  }_2 \vee r^{\frac{\alpha}{2}} \norm{ \bs{\Lambda}^{\frac{1}{2}} \T_{t}^{\frac{-1}{2}} \unu{t} \T_{t}^{\frac{-1}{2}}  \bs{\Lambda}^{\frac{1}{2}}  }_2 \geq \c_d  r^{\frac{- \alpha}{2}} ~~ \text{and} ~~ \mathcal{G}_{\text{init}}\right]	\\
&  \leq \sum_{\mathsf{X} \in \mathcal{X} } \mpr_0 \Big[ \sup_{t \leq \mathcal{T}}  \Big \lVert \T_{t}^{\frac{-1}{2}} \big( \sum_{j \leq t}  \mathsf{X}_j \big) \T_{t}^{\frac{-1}{2}}  \Big \rVert_2   \geq \frac{\c_d  r^{\frac{- \alpha}{2}} }{10} ~~ \text{and} ~~  \mathcal{G}_{\text{init}} \Big]	 \\
& + \sum_{\mathsf{X} \in \mathcal{X} } \mpr_0 \Big[ \sup_{t \leq \mathcal{T}}      \Big \lVert  \bs{\Lambda}^{\frac{1}{2}} \T_{t}^{\frac{-1}{2}} \big( \sum_{j \leq t}  \mathsf{X}_j \big) \T_{t}^{\frac{-1}{2}}  \bs{\Lambda}^{\frac{1}{2}} \Big \rVert_2   \geq \frac{\c_d r^{- \alpha} }{10}  ~~ \text{and} ~~ \mathcal{G}_{\text{init}}   \Big].	
}
By Propositions \ref{prop:boundstoppedprocess} and \ref{prop:goodeventres} and Corollary \ref{cor:boundsqv},  $\mathcal{G}_{\text{init}}$ implies the events
\begin{alignat}{4}
& \mathcal{E}_{\mathrm{ht},1} \equiv \bigg \lbrace &&\max_{t \leq \mathcal{T}}  \norm{ \text{Quad}^{(1)}_{t}}_2 \leq \frac{O_{\alpha,\beta, \varphi} \left( r^{- \alpha} \right)}{\log^{12} d} ~~&& \text{and} ~~   
&&\max_{t \leq \mathcal{T}}  R_{t}^{(1)}   \leq  \frac{O_{\alpha,\beta, \varphi} \left( r^{- \alpha} \right)}{\sqrt{d} \log^{12.5} d}  \bigg \rbrace     \\[0.25em]
&  \mathcal{E}_{\mathrm{ht},2} \equiv \bigg \lbrace && \max_{t \leq \mathcal{T}}  \norm{ \text{Quad}^{(2)}_{t}}_2 \leq \frac{O_{\alpha,\beta, \varphi} \left( r^{- 2\alpha} \right)}{\log^{12} d} ~~&& \text{and} ~~ 
&& \max_{t \leq \mathcal{T}}  R_{t}^{(2)}    \leq  \frac{O_{\alpha,\beta, \varphi} \left( r^{- 2 \alpha} \right)}{\sqrt{d} \log^{12.5} d}  \bigg \rbrace   . 
\end{alignat}
Therefore, by using Proposition \ref{prop:friedmanineq}
\eq{
  \mpr_0 \Big[ \sup_{t \leq \mathcal{T}}  \Big \lVert \T_{t}^{\frac{-1}{2}} \big( \sum_{j \leq t}  \mathsf{X}_j \big) \T_{t}^{\frac{-1}{2}}  \Big \rVert_2 &  \geq \frac{\c_d  r^{\frac{- \alpha}{2}} }{10} ~~ \text{and} ~~  \mathcal{G}_{\text{init}} \Big]   \\
 & \leq  \mpr_0 \Big[ \sup_{t \leq \mathcal{T}}  \Big \lVert \T_{t}^{\frac{-1}{2}} \big( \sum_{j \leq t}  \mathsf{X}_j \big) \T_{t}^{\frac{-1}{2}}  \Big \rVert_2   \geq \frac{\c_d  r^{\frac{- \alpha}{2}} }{10}  ~ \text{and}  ~   \mathcal{E}_{\mathrm{ht},1} \Big]\\
 & \leq 2 d^4 \exp \left( - \log^2 d \right). 
}
Similarly,
\eq{
  & \mpr_0 \Big[ \sup_{t \leq \mathcal{T}}  \Big \lVert  \bs{\Lambda}^{\frac{1}{2}} \T_{t}^{\frac{-1}{2}} \big( \sum_{j \leq t}  \mathsf{X}_j \big) \T_{t}^{\frac{-1}{2}}   \bs{\Lambda}^{\frac{1}{2}} \Big \rVert_2   \geq \frac{\c_d  r^{-\alpha } }{10}  ~~ \text{and} ~~  \mathcal{G}_{\text{init}} \Big]   \\
 & \quad  \leq  \mpr_0 \Big[ \sup_{t \leq \mathcal{T}}  \Big \lVert  \bs{\Lambda}^{\frac{1}{2}} \T_{t}^{\frac{-1}{2}} \big( \sum_{j \leq t}  \mathsf{X}_j \big) \T_{t}^{\frac{-1}{2}}  \bs{\Lambda}^{\frac{1}{2}} \Big \rVert_2   \geq \frac{\c_d  r^{- \alpha}}{10} ~ \text{and}  ~  \mathcal{E}_{\mathrm{ht},2} \Big]\\
 &  \quad  \leq 2  d^4 \exp \left( - \log^2 d \right). 
}
For $\alpha > 0.5$, we can write for all $\mathsf{X} \in \mathcal{X} $, 
\eq{
& \mpr_0 \left[ \sup_{t \leq \mathcal{T}}  \norm{ \T_{t}^{\frac{-1}{2}} \unu{t}\T_{t}^{\frac{-1}{2}}  }_2 \vee \ru^{\frac{\alpha}{2}} \norm{ \bs{\Lambda}_{11}^{\frac{1}{2}} \T_{t}^{\frac{-1}{2}} \unu{t} \T_{t}^{\frac{-1}{2}}  \bs{\Lambda}_{11}^{\frac{1}{2}}  }_2 \geq \c_d  \ru^{\frac{- \alpha}{2}} ~~ \text{and} ~~ \mathcal{G}_{\text{init}}\right]	\\
&  \leq \sum_{\mathsf{X} \in \mathcal{X} } \mpr_0 \Big[ \sup_{t \leq \mathcal{T}}  \Big \lVert \T_{t}^{\frac{-1}{2}} \big( \sum_{j \leq t}  \mathsf{X}_j \big) \T_{t}^{\frac{-1}{2}}  \Big \rVert_2   \geq \frac{\c_d  \ru^{\frac{- \alpha}{2}} }{10} ~~ \text{and} ~~  \mathcal{G}_{\text{init}} \Big]	 \\
& + \sum_{\mathsf{X} \in \mathcal{X} } \mpr_0 \Big[ \sup_{t \leq \mathcal{T}}      \Big \lVert  \bs{\Lambda}_{11}^{\frac{1}{2}} \T_{t}^{^{\frac{-1}{2}}} \big( \sum_{j  \leq t}  \mathsf{X}_j  \big) \T_{t}^{\frac{-1}{2}}  \bs{\Lambda}_{11}^{\frac{1}{2}} \Big \rVert_2   \geq \frac{\c_d \ru^{- \alpha} }{10}  ~~ \text{and} ~~ \mathcal{G}_{\text{init}}   \Big] 
}
By Propositions \ref{prop:boundstoppedprocess} and \ref{prop:goodeventres} and Corollary \ref{cor:boundsqv},  $\mathcal{G}_{\text{init}}$ implies  the events
\begin{alignat}{4}
& \mathcal{E}_{\mathrm{lt},1} \equiv \bigg \lbrace &&\max_{t \leq \mathcal{T}}  \norm{ \text{Quad}^{(1)}_{t}}_2 \leq \frac{O_{\alpha, \rs}( \ru^{- 4 \alpha})}{\log^{12} d} ~~&& \text{and} ~~   
&& \max_{t \leq \mathcal{T}}   R_{t}^{(1)}     \leq  \frac{ O_{\rs}( \ru^{- 4 \alpha - 1}) }{\sqrt{d} \log^{11.5} d}   \bigg \rbrace  \\[0.25em]
&  \mathcal{E}_{\mathrm{lt},2} \equiv \bigg \lbrace && \max_{t \leq \mathcal{T}}  \norm{ \text{Quad}^{(2)}_{t}}_2 \leq \frac{O_{\alpha, \rs}(\ru^{- 4 \alpha - (\alpha \wedge 1)}  )}{\log^{12} d  \log^{-2} \ru}  ~~&& \text{and} ~~ 
&& \max_{t \leq \mathcal{T}}   R_{t}^{(2)}    \leq   \frac{ O_{\rs}( \ru^{- 4 \alpha - 1}  \ru^{- (\alpha \wedge 1)})  }{\sqrt{d} \log^{11.5} d  \log^{-2} \ru}   \bigg \rbrace . 
\end{alignat}
Therefore, by using Proposition \ref{prop:friedmanineq}
\eq{
 \mpr_0 \Big[ \sup_{t \leq \mathcal{T}}  \Big \lVert \T_{t}^{\frac{-1}{2}} \big( \sum_{j \leq t}  \mathsf{X}_j \big) \T_{t}^{\frac{-1}{2}}  \Big \rVert_2 &   \geq \frac{\c_d  \ru^{\frac{- \alpha}{2}} }{10} ~~ \text{and} ~~  \mathcal{G}_{\text{init}} \Big]    \\
 & \leq  \mpr_0 \Big[ \sup_{n \leq \mathcal{T}}  \Big \lVert \T_{t}^{\frac{-1}{2}} \big( \sum_{j \leq t}  \mathsf{X}_j \big) \T_{t}^{\frac{-1}{2}}  \Big \rVert_2   \geq \frac{\c_d  \ru^{\frac{- \alpha}{2}} }{10} ~ \text{and}  ~  \mathcal{E}_{\mathrm{lt},1}      \Big]\\
 &  \leq  2 d^4 \exp \left( - \log^2 d \right). 
}
Similarly,
\eq{
   & \mpr_0 \Big[ \sup_{t \leq \mathcal{T}}  \Big \lVert  \bs{\Lambda}_{11}^{\frac{1}{2}} \T_{t}^{\frac{-1}{2}} \big( \sum_{j \leq t}  \mathsf{X}_j \big) \T_{t}^{\frac{-1}{2}}   \bs{\Lambda}_{11}^{\frac{1}{2}} \Big \rVert_2   \geq \frac{\c_d  \ru^{-\alpha } }{10} ~~ \text{and} ~~  \mathcal{G}_{\text{init}} \Big]  \\
 & \quad \leq  \mpr_0 \Big[ \sup_{t \leq \mathcal{T}}  \Big \lVert  \bs{\Lambda}_{11}^{\frac{1}{2}} \T_{t}^{\frac{-1}{2}} \big( \sum_{j \leq t}  \mathsf{X}_j \big) \T_{t}^{\frac{-1}{2}}   \bs{\Lambda}_{11}^{\frac{1}{2}} \Big \rVert_2   \geq \frac{\c_d  \ru^{- \alpha}}{10} ~~ \text{and} ~~ \mathcal{E}_{\mathrm{lt},2}   \Big]\\
 & \quad  \leq 2 d^4 \exp \left( - \log^2 d \right). 
}
\end{proof} 

\begin{corollary}
\label{cor:Tbadbound}
Consider  $\rks =   \{ r_{\star} = 
\floor{  \rs \big(1 -  \log^{\nicefrac{-1}{2}} d \big) \wedge r }, r_{u_\star} =  \rs  \}$ and the parameters in Proposition   \ref{prop:Tbadboundaux}. We have
\eq{
\mpr_0 \left[  \mathcal{T}_{\text{bad}} \geq   \tfrac{1}{2\lr} \rks  \log \left( \tfrac{d\log^{1.5} d}{\rs} \right)   ~ \text{and} ~   \mathcal{G}_{\text{init}} \right] \geq1 - 20 d^4 \exp(- \log^2 d).
}
\end{corollary} 
 
\begin{proof}
By using the first items in Proposition \ref{prop:lbsystemanalysis} and \ref{prop:goodeventres} ,and Lemma \ref{lem:ubsgo},  $\mathcal{G}_{\text{init}}$ implies that
\eq{
\mathcal{T}_{\text{bad}} \geq  \mathcal{T}_{\text{noise}} \wedge    \tfrac{1}{2\lr} \rks  \log \left( \tfrac{d\log^{1.5} d}{\rs} \right).
}
On the other hand,  within the (negation) of the events given in Proposition \ref{prop:Tbadboundaux}, we have
\eq{
\mathcal{T}_{\text{noise}} > \mathcal{T}_{\text{bad}}   \wedge    \tfrac{1}{2\lr} \rks  \log \left( \tfrac{d\log^{1.5} d}{\rs} \right).
}
Therefore, the statement follows.
\end{proof}

\subsection{Stability near minima}
In this section, we will establish that given \eqref{eq:sgd2} is near global minimum it will stay near global minimum.  For the statement, we (re)introduce the block matrix notation: $\rks =   \{ r_{\star} = 
\floor{  \rs \big(1 -  \log^{\nicefrac{-1}{8}} d \big) \wedge r }, r_{u_\star} =  \rs  \}$, we have 
\eq{
\G_t = \begin{bmatrix}
\G_{t,11} &  \G_{t,12} \\
\G_{t,12}^\top & \G_{t,22}
\end{bmatrix} ~~
\bs{\nu}_t = \begin{bmatrix}
\bs{\nu}_{t,11} &  \bs{\nu}_{t,12} \\
\bs{\nu}_{t,12}^\top &  \bs{\nu}_{t,22}
\end{bmatrix} ~~
 \bs{\Lambda} = \begin{bmatrix}
 \bs{\Lambda}_{11} & 0 \\
0 &   \bs{\Lambda}_{22}
\end{bmatrix}, ~~ 
 \bs{\Lambda}_{\ell_j}= \begin{bmatrix}
 \bs{\Lambda}_{\ell_j, 11} & 0 \\
0 &   \bs{\Lambda}_{\ell_j, 22}
\end{bmatrix}, \
}
where   $\bs{G}_{t,11} ,   \bs{\nu}_{t,11},  \bs{\Lambda}_{11},  \bs{\Lambda}_{\ell_j, 11}  \in \R^{\rks  \times \rks }$ and $ \bs{\Lambda}_{\ell_j}$ is introduced   \eqref{eq:Llfs}.  We define the following iterations:
\begin{itemize}[leftmargin=*]
\item Given $ \uG{0} = \bs{I}_{\rks} - \frac{1}{\log d }$ diagonal and $\uV{t}  =    2 \Llstop^{\frac{1}{2}} \uG{t}  \Llstop^{\frac{1}{2}} -   \Llftop$,  we define
\eq{
 \uV{t+1} & =    \uV{t}  \left(   \bs{I}_{\rks}  +  \frac{\lr}{1 - 1.1 \lr}  \uV{t}    \right)^{-1} \\
 & \qquad +   \frac{\lr}{1 - 1.1 \lr} \left(  \Llftop^2  - \frac{8.1}{\rks^{\alpha} \log d}   \Llstop -   \frac{O(1)}{\log^2 d}   \Llstop^{2}     \right).
}
 \item For  $\unu{0} = 0$,  $\unu{t+1} =   \unu{t}   +  \tfrac{\eta/2}{\sqrt{\rs}}   \bs{\nu}_{t+1,11}$.
\item We define a sequence of events $\{  \mathcal{E}_t \}_{t \geq 0}$
\eq{
\mathcal{E}_{t}  \coloneqq   
  \left \{      \frac{ -  \rks^{-\frac{\alpha}{2}}}{\log^2 d} \bs{I}_{\rks}   \preceq   \unu{t}  \preceq   \frac{\rks^{-\frac{\alpha}{2}} }{\log^2 d}   \bs{I}_{\rks}   \right \}  \cap  \left \{        \frac{- \rks^{-  \alpha} }{\log^4 d}   \bs{I}_{\rks}    \preceq  \bs{\Lambda}_{11}^{\frac{1}{2}}  \unu{t} \bs{\Lambda}_{11}^{\frac{1}{2}}  \preceq    \frac{\rks^{- \alpha}  }{\log^4 d} \bs{I}_{\rks}  \right \},  
}
We define the stopping times
\eq{
\mathcal{T}_{\text{noise}}(\omega)   \coloneqq \inf \left \{  t \geq 0  ~ \middle \vert ~  \omega \not \in \mathcal{E}_t  \right \} \wedge d^3, \quad 
 \mathcal{T}_{\mathrm{bounded}} \coloneqq \inf \left \{  t \geq 0 ~ \middle \vert ~  \G_t  \not \succeq  \bs{I}_{\rks}  - \frac{2}{\log d }    \bs{I}_{\rks}  \right \}.
} 
and  $\mathcal{T}_{\text{stable}} \coloneqq   \mathcal{T}_{\text{noise}}  \wedge  \mathcal{T}_{\text{bounded}} $.
\end{itemize}
We have the following statement:
\begin{proposition}
\label{prop:stable}
Consider  the parameters in Proposition   \ref{prop:Tbadboundaux}. 
\eqref{eq:sgd2} guarantees that if $\G_{0,11} \succeq \bs{I}_{\rks} - \frac{1}{\log d },$  we have   $\G_{t,11} \succeq  \bs{I}_{\rks}  - \frac{2}{\log d }$ for $t \leq  \frac{\rks^{\alpha} \log^2 d}{\eta}$ with probability $1 - d^4 \exp(- \log^2 d)$.
\end{proposition}

\begin{proof}
We define $\uze{t}   \coloneqq 2  \Llstop^{\frac{1}{2}} \unu{t}  \Llstop^{\frac{1}{2}} $.  We make the following observations observations:
\begin{itemize}[leftmargin=*]
\item Since  $\G_{t,11}^2 + \G_{t,12} \G_{t,12}^\top \preceq \bs{I}_{\rks}$  for   $t \leq   \mathcal{T}_{\text{bounded}}$, we have
\eq{
\G_{t,12} \G_{t,12}^\top \preceq \frac{4}{\log d} \bs{I}_{\rks}
}
Therefore, by  using   \eqref{eq:sotermsminusslb}, we have for $t \leq \mathcal{T}_{\text{bounded}}$ 
\eq{
 \G_{t+1, 11}    & \succeq    \G_{t,11}  +  \lr \Big(  \Llftop  \G_{t,11}   + \G_{t,11}  \Llftop - 2   \G_{t,11}   \Llstop  \G_{t,11}  \Big)  -  \frac{4 \lr}{\rks^{\alpha} \log d} \bs{I}_{\rks}  \\
&- C \lr^2  \normL^2 \rs \bs{I}_{\rks}   +    \frac{\eta/2}{\sqrt{\rs}} \bs{\nu}_{t+1,11}
}
Then, if we define $\bs{V}_{t}^{-} \coloneqq   2 \Llstop^{\frac{1}{2}} \G_{t,11}  \Llstop^{\frac{1}{2}} -   \Llftop$, we have
\eq{
\bs{V}_{t+1}^{-} & \succeq  \bs{V}_{t}^{-}  - \lr ( \bs{V}_{t+1}^{-} )^2 + \lr   \Llftop^2 - \Big( \frac{8 \lr}{\rks^{\alpha} \log d} + C \lr^2 \normL^2\rs \Big) \Llstop \\
& +  \frac{\eta}{\sqrt{\rs}}  \Llstop^{\frac{1}{2}} \bs{\nu}_{t+1,11}  \Llstop^{\frac{1}{2}}  \\
& \succeq  \bs{V}^{-}_t  \left( \bs{I}_r +  \frac{\lr}{1 - 1.1  \lr }    \bs{V}^{-}_t  \right)^{-1} \!\!\!\!  + \lr \left(  \Llftop^2  - \frac{8.1}{\rks^{\alpha} \log d} \Llstop   \right)   + \Llstop^{\frac{1}{2}} \bs{\nu}_{t+1,11}  \Llstop^{\frac{1}{2}} 
}
\item To derive an upper-bound for $\uG{t}$,  assuming $\uG{t} \preceq 1.1 \bs{I}_{\rks}$,  we have
\eq{
 \uV{t+1} & =    \uV{t}  -  \frac{\lr}{1 - 1.1 \lr}  \uVs{t}  +    \frac{\lr^2}{(1 - 1.1 \lr)^2}  \uVc{t}  \left(   \bs{I}_{\rks} +  \frac{\lr}{1 - 1.1 \lr}  \uV{t}    \right)^{-1} \\   
 &\quad  + \frac{\lr}{1 - 1.1 \lr} \left(  \Llftop^2  - \frac{8.1}{\rks^{\alpha} \log d}   \Llstop -   \frac{O(1)}{\log^2 d}   \Llstop^{2}     \right) \\
& \preceq      \uV{t}  -  \frac{\lr}{1 - 1.1 \lr}  \uVs{t}  +  \frac{\lr}{1 - 1.1 \lr}    \Llftop^2 .
}
Then,  by Proposition   \ref{prop:1dsystemaux}, we have  $\uG{t+1} \preceq 1.1 \bs{I}_{\rks}$.  Since the bound holds for $t = 0$, it holds for all $t \in \N$. 
\item  To derive a lower-bound,  we first observe that by monotonicity $\uV{0} \succ 0$ . Therefore, assuming $\uV{t} \succ \uV{0}$
\eq{
 \uV{t+1} &\succeq     \uV{t}  -   \frac{\lr/2}{1 - 1.1 \lr}  (   2 \Llstop^{\frac{1}{2}} \uG{t}  \Llstop^{\frac{1}{2}} -   \Llftop  )^2 \\
 & +  \frac{\lr/2}{1 - 1.1 \lr} \left(   \Llftop^2 - \frac{8.1}{\rks^{\alpha} \log d}   \Llstop -   \frac{O(1)}{\log^2 d}   \Llstop^{2}     \right)  \\
&  \succeq     \uV{t} \!  -  \!  \frac{\lr/2}{1 - 1.1 \lr} \left(   2 \Llstop^{\frac{1}{2}} \uG{t}  \Llstop^{\frac{1}{2}} \! -  \!   \sqrt{\left(   \Llftop^2 - \tfrac{8.1}{\rks^{\alpha} \log d}   \Llstop -   \tfrac{O(1)}{\log^2 d}   \Llstop^{2}     \right) }  \right)^2  \\
  & +  \frac{\lr/2}{1 - 1.1 \lr} \left(  \Llftop^2 -    \frac{8.1}{\rks^{\alpha} \log d}   \Llstop -   \frac{O(1)}{\log^2 d}   \Llstop^{2}     \right) .
}
Then,  by Proposition   \ref{prop:1dsystemaux}, we have  $\uV{t} \succeq \uV{0}$. 
\end{itemize}

We start our proof by showing that $\bs{V}_{t}^{-}  \succeq \uV{t} + \uze{t}$ for $t \leq \mathcal{T}_{\text{stable}}$. Assuming the statement holds for $t \in \N$, we have
\eq{
\bs{V}_{t+1}^{-}  & \succeq  (\uV{t} + \uze{t})  \left( \bs{I}_r +  \frac{\lr}{1 - 1.1  \lr }  (\uV{t} + \uze{t})  \right)^{-1}    + \lr \left(  \Llftop^2  - \frac{8.1}{\rks^{\alpha} \log d}   \Llstop  \right) \\
& + \Llstop^{\frac{1}{2}} \bs{\nu}_{t+1,11}  \Llstop^{\frac{1}{2}}  \\
& =  \uV{t}  \left( \bs{I}_r +  \frac{\lr}{1 - 1.1  \lr }  \uV{t} \right)^{-1}  +     \lr \left(  \Llftop^2  - \frac{8.1}{\rks^{\alpha} \log d}   \Llstop  \right)  \\
& + \left( \bs{I}_r +  \frac{\lr}{1 - 1.1  \lr }  \uV{t} \right)^{-1}  \uze{t}  \left( \bs{I}_r +  \frac{\lr}{1 - 1.1  \lr }  ( \uV{t} +   \uze{t} ) \right)^{-1} 
+ \Llstop^{\frac{1}{2}} \bs{\nu}_{t+1,11}  \Llstop^{\frac{1}{2}} 
}
We have for $t \leq \mathcal{T}_{\text{stable}}$
\eq{
& \left( \bs{I}_r +  \frac{\lr}{1 - 1.1  \lr }  \uV{t} \right)^{-1}   \uze{t}  \left( \bs{I}_r +  \frac{\lr}{1 - 1.1  \lr }  ( \uV{t} +   \uze{t} ) \right)^{-1}  \\
&  =    \uze{t} -    \frac{\lr}{1 - 1.1  \lr }  \uV{t}  \uze{t} -    \frac{\lr}{1 - 1.1  \lr }   \uze{t} \uV{t}  -     \frac{\lr}{1 - 1.1  \lr }   \uze{t}^2 \\
& -    \frac{\lr^2}{(1 - 1.1  \lr)^2 } \uV{t}  \left( \bs{I}_r +  \frac{\lr}{1 - 1.1  \lr }  \uV{t} \right)^{-1}  \!\! \uze{t}  \!\!  \left( \bs{I}_r +  \frac{\lr}{1 - 1.1  \lr }  ( \uV{t} +   \uze{t} ) \right)^{-1} \!\! ( \uV{t} +   \uze{t} ) \\
& \succeq    \uze{t} -    \frac{\lr}{1 - 1.1  \lr }    \frac{1}{\log^2 d} \uVs{t} - \frac{\lr}{1 - 1.1  \lr }  ( 1 +  \log^2 d ) \uze{t} ^2  -    \frac{\lr^2}{(1 - 1.1  \lr)^2 }    \frac{O(\rks^{-\alpha})}{\log^4 d}  \bs{I}_{\rks}  \\
& \succeq  \uze{t} -    \frac{\lr}{1 - 1.1  \lr }    \frac{O(1)}{\log^2 d} \Llstop^{2} .  
}
Since  $\bs{V}_{0}^{-} = \uV{0} + \uze{0},$ the claim follows. Then,  by the third item above,  we have  for $t \leq \mathcal{T}_{\text{stable}}$
\eq{
\G_t \succeq  \uG{0} + \unu{t} \succeq  \bs{I}_{\rks} - \frac{1}{\log d}  \bs{I}_{\rks} - \frac{\rks^{\frac{-\alpha}{2}}}{\log^2 d} \bs{I}_{\rks} \succeq  \bs{I}_{\rks} - \frac{2}{\log d}  \bs{I}_{\rks} \Rightarrow   \mathcal{T}_{\text{noise}} \leq   \mathcal{T}_{\text{bounded}}.
}
In the following, we will bound   $ \mathcal{T}_{\text{noise}}$. We have
\eq{
\E_t \left[ \unu{t+1}^2 \right]  \labelrel\preceq{pstable:ineqq3}  \unu{t}^2       + O (\eta^2) \bs{I}_{\rks} \Rightarrow   \E \left[ \unu{t}^2 \right] \preceq O(\eta^2 t)  \bs{I}_{\rks}
}
By clipping strategy we used with $L = \log^2 d$ in  \eqref{eq:clippedversions}, and defining $\mathsf{\Gamma}_1 \coloneqq \bs{I}_{\rks}$,  $\mathsf{\Gamma}_2   \coloneqq \bs{\Lambda}_{11}^{\frac{1}{2}}, $ and
\begin{align} 
& \text{Quad}_{k,t}^{(\ell)}( \mathsf{X}) \coloneqq     \sum_{j = 1}^k  \E_{j - 1} \left[ \Big( \mathsf{\Gamma}_{\ell}  \T_t^{\frac{-1}{2}} \mathsf{X}_{j}  \T_t^{\frac{-1}{2}}  \mathsf{\Gamma}_\ell \Big)^2 \right], ~ \ell \in \{1,2\},
\end{align}
we can show that the following events hold: For any $T \in \N$,
\begin{alignat}{4}
& \widehat{\mathcal{E}}_{\mathrm{ht},1} \equiv \bigg \lbrace &&\max_{t \leq T}  \norm{ \text{Quad}^{(1)}_{t,t}}_2 \leq O(\eta^2 T) ~~&& \text{and} ~~   
&&\max_{t \leq T}  R_{t,t}^{(1)}   \leq  O(\eta \rks^{\frac{1}{2}} \log^2 d)  \bigg \rbrace     \\[0.25em]
&  \widehat{\mathcal{E}}_{\mathrm{ht},2} \equiv \bigg \lbrace && \max_{t \leq T}  \norm{ \text{Quad}^{(2)}_{t,t}}_2 \leq O_{\alpha}(\eta^2 T \rks^{\mathsf{p}_2 -1} )~~&& \text{and} ~~ 
&& \max_{t \leq T}  R_{t,t}^{(2)}    \leq   O_{\alpha}(\eta \rks^{\mathsf{p}_2 -\frac{1}{2}} \log^2 d )  \bigg \rbrace   . 
\end{alignat}
where $\mathsf{p}_2$  is defined in Corollary \ref{cor:boundsqv}. By using Proposition \ref{prop:friedmanineq}, we can show that with probability $d^4 \exp(- \log^2 d)$,   $\mathcal{T}_{\text{noise}} \geq \frac{\rks^{\alpha} \log^2 d}{\eta}$.
\end{proof}

\section{Auxiliary Statements}
\label{sec:auxstats}
\subsection{Matrix bounds}
\begin{proposition}
\label{prop:matrixyounginequality}
For $\bs{A}, \bs{B} \in \R^{d \times r},$ we have
\eq{
- \bs{A}^\top \bs{A} - \bs{B}^\top \bs{B} \preceq \bs{A}^\top \bs{B} + \bs{B}^\top \bs{A} \preceq \bs{A}^\top \bs{A} + \bs{B}^\top \bs{B}.
}
If $r = d$, then $(\bs{A} + \bs{A}^\top)^2 \preceq 2 \bs{A}^\top \bs{A} + 2 \bs{A} \bs{A}^\top$.   Moreover, if $\bs{A}_1, \cdots, \bs{A}_k$ are symmetric matrices,
\eq{
\left( \sum_{i = 1}^k  \bs{A}_i \right)^2 \preceq k   \sum_{i = 1}^k  \bs{A}_i^2.
}
\end{proposition}

\begin{proof}
We have
\eq{
(\bs{A} - \bs{B})^\top (\bs{A} - \bs{B}) \succeq 0 \Rightarrow \bs{A}^\top \bs{A} + \bs{B}^\top \bs{B} \succeq  \bs{A}^\top \bs{B} + \bs{B}^\top \bs{A}.
}
By using $\bs{A} \leftarrow - \bs{A}$, we obtain the left inequality too. For the second inequality, we have
\eq{
(\bs{A} + \bs{A}^\top)^2 = \bs{A}^\top \bs{A} + \bs{A} \bs{A}^\top + \bs{A} \bs{A} + \bs{A}^\top \bs{A}^\top \\
(\bs{A} - \bs{A}^\top)^\top (\bs{A} - \bs{A}^\top) =   \bs{A}^\top \bs{A} + \bs{A} \bs{A}^\top - \bs{A} \bs{A} - \bs{A}^\top \bs{A}^\top 
}
Therefore,  $(\bs{A} + \bs{A}^\top)^2  \preceq  2 \left(  \bs{A}^\top \bs{A} + \bs{A} \bs{A}^\top \right)$.  For the last statement,
\eq{
\left( \sum_{i = 1}^k  \bs{A}_i \right)^2 = \sum_{i = 1}^k   \bs{A}_i^2 + \sum_{i = 1}^k \sum_{j = i +1}^k      \bs{A}_i  \bs{A}_j +  \sum_{i = 1}^k  \sum_{j = i +1}^k        \bs{A}_j \bs{A}_i \preceq   k   \sum_{i = 1}^k  \bs{A}_i^2,
}
where we use the first statement in the last inequality.
\end{proof}

\begin{proposition}
\label{prop:blockmatrixposdef}
Consider a symmetric square matrix with block partition
\eq{
\bs{M} = \begin{bmatrix}
\bs{A} & \bs{B} \\
\bs{B}^\top & \bs{C}
\end{bmatrix}.
}
If $\bs{A}$ is invertible,  then  $\bs{M} \succ 0$ if and only if $\bs{A} \succ 0$ and $\bs{C} - \bs{B}^\top \bs{A}^{-1} \bs{B} \succ 0$.
\end{proposition}

\begin{proof}
If $\bs{A}$ is invertible, we have
\eq{
\begin{bmatrix}
\bs{A} & \bs{B} \\
\bs{B}^\top & \bs{C}
\end{bmatrix} = 
\begin{bmatrix}
\bs{I} & 0 \\
\bs{B}^\top  \bs{A}^{-1} & \bs{I}
\end{bmatrix}
\begin{bmatrix}
\bs{A} & 0 \\
0 & \bs{C} - \bs{B}^\top \bs{A}^{-1} \bs{B}
\end{bmatrix}
\begin{bmatrix}
\bs{I} &\bs{A}^{-1} \bs{B} \\
0 & \bs{I}
\end{bmatrix}.
}
Note that
\eq{
\begin{bmatrix}
\bs{I} & 0 \\
\bs{B}^\top  \bs{A}^{-1} & \bs{I}
\end{bmatrix}^{-1} =  \begin{bmatrix}
\bs{I} & 0 \\
- \bs{B}^\top  \bs{A}^{-1} & \bs{I}
\end{bmatrix}.
}
Therefore,  the statement follows.
\end{proof}

\begin{proposition}
\label{prop:schurlowerbound}
Let $\ru < r$ and $\Z \in \R^{r \times \rs}$ such that
\eq{
\Z = \begin{bmatrix}
\Z_1 \\
\Z_2
\end{bmatrix},   ~~ \text{where} ~~ \Z_1 \in \R^{\ru \times \rs},  \Z_2 \in \R^{r - \ru \times \rs}. 
}
For ant  $0 \leq  \varepsilon < 1$
\eq{
\Z \Z^\top & \succeq  \varepsilon    \begin{bmatrix}
\Z_1 \Z_1^\top & 0 \\
0 & 0
\end{bmatrix}
+ (1 - \varepsilon)  \begin{bmatrix}
\Z_1 \Z_1^\top - \Z_1 \Z_2^\top (\Z_2 \Z_2^\top)^{+} \Z_2 \Z_1^\top  & 0  \\
0 & 0
\end{bmatrix} \\
&
- \frac{\varepsilon}{1 - \varepsilon}   \begin{bmatrix}
0 & 0 \\
0 & \Z_2 \Z_2^\top
\end{bmatrix},
}
where $\A \to \A^+$ denotes the pseudo inverse operator.
\end{proposition}

\begin{proof}
We will denote $\bs{x} \in \R^r$ as
\eq{
 \bs{x}  = \begin{bmatrix}
\bs{x}_1 \\
\bs{x}_2
\end{bmatrix}  ~~ \text{where} ~~ \bs{x}_1 \in \R^{\ru},  \bs{x}_2 \in \R^{r - \ru}.
}
We have
\eq{
\bs{x}^\top  \Z \Z^\top   \bs{x} & =  \left( \bs{x}_1^\top  \Z_1 \Z_1^\top   \bs{x}_1  +  2 \bs{x}_1^\top  \Z_1 \Z_2^\top   \bs{x}_2 + \frac{1}{1 - \varepsilon} \bs{x}_2^\top  \Z_2 \Z_2^\top   \bs{x}_2 \right) - \frac{\varepsilon}{1 - \varepsilon} \bs{x}_2^\top  \Z_2 \Z_2^\top   \bs{x}_2 \\
& \labelrel\geq{schurbound:ineq0}   \left( \bs{x}_1^\top  \Z_1 \Z_1^\top   \bs{x}_1 - (1 - \varepsilon)   \bs{x}_1^\top  \Z_1  \Z_2^\top (\Z_2 \Z_2^\top)^{+} \Z_2 \Z_1^\top   \bs{x}_1 \right)  - \frac{\varepsilon}{1 - \varepsilon} \bs{x}_2^\top  \Z_2 \Z_2^\top   \bs{x}_2   ,
}
where we minimized the first term in the first line over $\bs{x}_2$ in \eqref{schurbound:ineq0}.  Since  \eqref{schurbound:ineq0} holds for all $\bs{x}$, the statement follows,
\end{proof}

\begin{proposition}
\label{prop:monotonecharac}
Let $\bs{A} \in \R^{r \times r}$ be a symmetric matrix.  For $\bs{S} \succ - \bs{A}$,   $\bs{S} \to   - (\bs{S} + \bs{A})^{-1}$ is monotone.
\end{proposition}

\begin{proof}
Let $\bs{S}_1 \succ \bs{S}_2 \succ - \bs{A}$ . We have
\eq{
 - (\bs{S}_1 + \bs{A})^{-1} +  (\bs{S}_2 + \bs{A})^{-1} \! = \!   (\bs{S}_2 + \bs{A})^{-1} (  (\bs{S}_1 - \bs{S}_2)^{-1}  +  (\bs{S}_2 + \bs{A})^{-1} )^{-1}  (\bs{S}_2 + \bs{A})^{-1} \succ 0. \label{eq:monotoneproof}
}
 For  $\bs{S}_1 \succeq \bs{S}_2,$ we can use    $\bs{S}_1 + \varepsilon \bs{I}_r$  in \eqref{eq:monotoneproof} and take $\varepsilon \downarrow 0$
\end{proof}
\subsubsection{Additional bounds for continuous-time analysis}

\begin{proposition}
\label{prop:offdiagonal}
For a symmetric positive definite $\D_1$, $\D_2$, and $C > 0$, we have
\eq{
\D_1 \Big(\D_1   +   \Z_1 \big( C \bs{I}_{\rs} +    \Z_2^\top    \D_2^{-1}   \Z_2 \big)^{-1} &  \Z_1^\top  \Big)^{-1}  \Z_1 \Z_2^\top  \left(   \Z_2 \Z_2^\top   +  C \D_2 \right)^{-1} \D_2 \\
& =  \Z_1 \big( C \bs{I}_{\rs} +   \Z_2 ^\top   \D_2^{-1} \Z_2 + \Z_1^\top \D_1^{-1} \Z_1   \big)^{-1}   \Z_2^\top.
}
\end{proposition}

\begin{proof}
We have
\eq{
  \big(\D_1   & +   \Z_1 \big(C \bs{I}_{\rs} +    \Z_2^\top \D_2^{-1}   \Z_2 \big)^{-1}   \Z_1^\top  \big)^{-1}  \\
  & = \D_1^{-1} - \D_1^{-1}   \Z_1 \big( C \bs{I}_{\rs} +    \Z_2^\top   \D_2^{-1}  \Z_2 + \Z_1^\top \D_1^{-1} \Z_1 \big)^{-1}    \Z_1^\top   \D_1^{-1}. 
}
Therefore,
\eq{
  \D_1 \big(\D_1  & +   \Z_1 \big(\bs{I}_{\rs} +   \Z_2^\top \D_2^{-1}  \Z_2 \big)^{-1} \!\!  \Z_1^\top  \big)^{-1}  \Z_1   \\
& =    \Z_1  \Big(\bs{I}_{\rs}  - \big( C \bs{I}_{\rs} \! +   \!  \Z_2^\top \D_2^{-1}  \Z_2 + \Z_1^\top \D_1^{-1} \Z_1 \big)^{-1}  \!\!   \Z_1^\top   \D_1^{-1} \Z_1 \Big) \\
& =   \Z_1    \big( C \bs{I}_{\rs} +    \Z_2^\top \D_2^{-1}  \Z_2 + \Z_1^\top \D_1^{-1} \Z_1 \big)^{-1}     \big(C  \bs{I}_{\rs} +   \Z_2^\top \D_2^{-1}   \Z_2 \big)
}
Then,
\eq{
  \Z_1    \big(C  \bs{I}_{\rs} +    \Z_2^\top \D_2^{-1}  \Z_2 + \Z_1^\top \D_1^{-1} \Z_1 \big)^{-1}   &  \big(  C \bs{I}_{\rs} +   \Z_2^\top \D_2^{-1}   \Z_2 \big) \Z_2^\top  \left(    \Z_2 \Z_2^\top   +  C  \D_2\right)^{-1} \D_2  \\
& = \Z_1 \left(  C \bs{I}_{\rs} +  \Z_2 ^\top \D_2^{-1} \Z_2 + \Z_1^\top \D_1^{-1} \Z_1   \right)^{-1}   \Z_2^\top .
}
\end{proof}

\begin{proposition}
\label{prop:offdiagonalfrob}
For some diagonal positive definite   $\bs{A} \coloneqq \mathrm{diag}(\{a_j\}_{j = 1}^{\ru})$ and $\bs{B} \coloneqq \mathrm{diag}(\{b_j\}_{j = 1}^{d -\ru})$, we let 
\eq{
\D_1 \coloneqq   \frac{  \bs{A}  \exp(- t    \bs{A} ) }{\bs{I}_{\ru} -\exp(- t  \bs{A} ) }, \quad   \D_2 \coloneqq     \frac{  \bs{B}  \exp(- t    \bs{B} )}{\bs{I}_{d - \ru} -\exp(- t    \bs{B} ) },
}
For some $\Z_1  \in \R^{\ru \times \rs}$,    $\Z_2  \in \R^{(d - \ru) \times \rs}$,  and $C > 0$, we define
\eq{
\M \coloneqq  \exp(0.5t  \bs{A}  ) \Z_1 \left( C \bs{I}_{\rs} +  \Z_2^\top \D_2^{-1} \Z_2 +  \Z_1^\top \D_1^{-1}   \Z_1  \right)^{-1}   \Z_2^\top \exp(0.5t  \bs{B}  ).
}
We have 
\eq{
\norm{\M}_F^2 \leq  \tilde{C}
\sum_{i = 1}^{\ru \wedge \rs}     \left(  \lambda_{\mathrm{max}}( \Z_1 \Z_1^\top )  \exp(t (a_i + b_i) ) \wedge    \left(C +   \tfrac{\lambda_{\mathrm{min}}(\Z_2^\top \Z_2)}{\lambda_{\mathrm{max}}(\D_2)} \right)  \frac{  a_i     \exp(t (a_i +b_i) ) }{  \exp(t a_i)  - 1 } \right) 
}
where 
\eq{
\tilde{C} = \frac{\lambda_{\mathrm{max}}(\Z_2^\top \Z_2)}{   \left( C +   \tfrac{\lambda_{\mathrm{min}}(\Z_2^\top \Z_2)}{\lambda_{\mathrm{max}}(\D_2)} \right)^2  }
}
\end{proposition}

\begin{proof}
For convenience, we will use 
\begin{alignat}{2}
& \tilde{\D}_{1} \coloneqq   \frac{   \bs{A}  }{\bs{I}_{\ru} -\exp(- t   \bs{A}  ) },  \quad  && \tilde{\D}_{2} =   \frac{ \bs{B} }{\bs{I}_{d - \ru} -\exp(- t  \bs{B}  ) },  \\
& \tilde{\Z}_1  \coloneqq    \exp(0.5t   \bs{A}   ) \Z_1 , \quad  && \tilde{\Z}_2 \coloneqq   \exp(0.5t   \bs{B}  ) \Z_2.
\end{alignat}
We let
\eq{
& \M_1  \coloneqq   \tilde{\Z}_1  \left( C  \bs{I}_{\rs} + \tilde{\Z}_2^\top  \tilde{\D}_{2}^{-1}  \tilde{\Z}_2 +    \tilde{\Z}_1^\top  \tilde{\D}_{1}^{-1}    \tilde{\Z}_1   \right)^{\frac{-1}{2}}, \\
&
\M_2 \coloneqq   \left(  C \bs{I}_{\rs} + \tilde{\Z}_2^\top  \tilde{\D}_{2}^{-1}  \tilde{\Z}_2 +    \tilde{\Z}_1^\top    \tilde{\D}_{1}^{-1}    \tilde{\Z}_1   \right)^{\frac{-1}{2}}  \tilde{\Z}_2^\top
}
We observe that
\eq{
\norm{\M}_F^2 = \tr(   \M_1^\top  \M_1  \M_2 \M_2^\top ) \leq \sum_{i =1}^{\ru \wedge \rs} \lambda_i(  \M_1 \M_1^\top )    \lambda_i(  \M_2^\top \M_2  )
}
where we used that  $\mathrm{rank}(\M_1 \M_1^\top)  \leq \ru \wedge \rs$ and  Von Neumann's trace inequality in the last part. We have
\eq{
 \M_2^\top \M_2 & \preceq    \exp(0.5t \bs{B} ) \Z_2^\top   \big(  C \bs{I}_{\rs} +  \tfrac{1}{\lambda_{\mathrm{max}}(\D_2)}  \Z_2^\top  \Z_2  \big)^{-1}  \Z_2  \exp(0.5t  \bs{B}  )  \\
& \preceq  \frac{\lambda_{\mathrm{max}}(\Z_2^\top \Z_2)}{C + \frac{\lambda_{\mathrm{max}}(\Z_2^\top \Z_2)}{\lambda_{\mathrm{max}}(\D_2)}}   \exp( t  \bs{B}  ) 
}
On the other hand,
\eq{
& \M_1 \M_1^\top   \preceq     \tilde{\Z}_1  \left(  \left(C +   \frac{\lambda_{\mathrm{min}}(\Z_2^\top \Z_2)}{\lambda_{\mathrm{max}}(\D_2)} \right) \bs{I}_{\rs}  +    \tilde{\Z}_1^\top \tilde{\bs{D}}_1^{-1}    \tilde{\Z}_1   \right)^{-1} \tilde{\Z}_1^\top \\
& = \frac{1}{   C +   \tfrac{\lambda_{\mathrm{min}}(\Z_2^\top \Z_2)}{\lambda_{\mathrm{max}}(\D_2)}  }    \tilde{\Z}_1  \left(   \bs{I}_{\rs}  +    \tilde{\Z}_1^\top  \left( \left(C +   \tfrac{\lambda_{\mathrm{min}}(\Z_2^\top \Z_2)}{\lambda_{\mathrm{max}}(\D_2)} \right) \tilde{\bs{D}}_1 \right)^{-1}    \tilde{\Z}_1   \right)^{-1} \tilde{\Z}_1^\top \\
 & = \frac{1}{   C +   \tfrac{\lambda_{\mathrm{min}}(\Z_2^\top \Z_2)}{\lambda_{\mathrm{max}}(\D_2)}  }    \tilde{\Z}_1  \tilde{\Z}_1^\top     \left(   \left(C +   \tfrac{\lambda_{\mathrm{min}}(\Z_2^\top \Z_2)}{\lambda_{\mathrm{max}}(\D_2)} \right) \tilde{\bs{D}}_1  +   \tilde{\Z}_1  \tilde{\Z}_1^\top \right)^{-1}   \left( \left(C +   \tfrac{\lambda_{\mathrm{min}}(\Z_2^\top \Z_2)}{\lambda_{\mathrm{max}}(\D_2)} \right) \tilde{\bs{D}}_1  \right)
}
We have the following at the same time:
\begin{itemize}[leftmargin=*]
\item $  \tilde{\Z}_1  \tilde{\Z}_1^\top     \left(   \left(C \! + \!   \tfrac{\lambda_{\mathrm{min}}(\Z_2^\top \Z_2)}{\lambda_{\mathrm{max}}(\D_2)} \right) \tilde{\bs{D}}_1  +   \tilde{\Z}_1  \tilde{\Z}_1^\top \right)^{-1}  \!\!\!   \left(C +   \tfrac{\lambda_{\mathrm{min}}(\Z_2^\top \Z_2)}{\lambda_{\mathrm{max}}(\D_2)} \right) \tilde{\bs{D}}_1    \preceq   \left(C +   \tfrac{\lambda_{\mathrm{min}}(\Z_2^\top \Z_2)}{\lambda_{\mathrm{max}}(\D_2)} \right) \tilde{\bs{D}}_1   $  
\item    $ \tilde{\Z}_1  \tilde{\Z}_1^\top     \left(   \left(C +   \tfrac{\lambda_{\mathrm{min}}(\Z_2^\top \Z_2)}{\lambda_{\mathrm{max}}(\D_2)} \right) \tilde{\bs{D}}_1  +   \tilde{\Z}_1  \tilde{\Z}_1^\top \right)^{-1}  \!\!\! \left(C +   \tfrac{\lambda_{\mathrm{min}}(\Z_2^\top \Z_2)}{\lambda_{\mathrm{max}}(\D_2)} \right) \tilde{\bs{D}}_1    \preceq       \lambda_{\mathrm{max}}( \Z_1 \Z_1^\top ) \exp(t \bs{A} )$
\end{itemize}
Therefore,  for $i \leq r \wedge \rs,$ we have
\eq{
\lambda_i(\M_1 \M_1^\top ) \leq    \frac{1}{   C +   \tfrac{\lambda_{\mathrm{min}}(\Z_2^\top \Z_2)}{\lambda_{\mathrm{max}}(\D_2)}  }     \left(  \lambda_{\mathrm{max}}( \Z_1 \Z_1^\top )  \exp(t a_i ) \wedge    \left(C +   \tfrac{\lambda_{\mathrm{min}}(\Z_2^\top \Z_2)}{\lambda_{\mathrm{max}}(\D_2)} \right)  \tfrac{  a_i     \exp(t a_i) }{   \exp(t a_i)  - 1} \right) 
}
Therefore,
\eq{
\norm{\M}_F^2 \leq  \tilde{C}
\sum_{i = 1}^{\ru \wedge \rs}     \left(  \lambda_{\mathrm{max}}( \Z_1 \Z_1^\top )  \exp(t (a_i + b_i) ) \wedge    \left(C +   \tfrac{\lambda_{\mathrm{min}}(\Z_2^\top \Z_2)}{\lambda_{\mathrm{max}}(\D_2)} \right)  \frac{  a_i     \exp(t (a_i +b_i) ) }{  \exp(t a_i)  - 1 } \right) .
}
\end{proof}

\subsubsection{Additional bounds for discrete-time analysis}
\begin{proposition}
\label{prop:residualhigherorder} 
For some positive definite diagonal matrices    $\bs{D}_0,   \bs{D}_1  \in \R^{r \times r}$ and symmetric matrices $\G, \bs{\nu} \in \R^{r \times r}$, we let
\eq{
\bs{V} \coloneqq 2  \bs{D}_0^{\frac{1}{2}}  \G  \bs{D}_0^{\frac{1}{2}} -  \bs{D}_1  ~~ \text{and} ~~  \bs{\zeta} \coloneqq \bs{D}_0^{\frac{1}{2}}  \bs{\nu}   \bs{D}_0^{\frac{1}{2}}  ~~ \text{and} ~~   \grave{\bs{V}}  \coloneqq  \bs{V}  +  \bs{\zeta}, 
}
where
\begin{itemize}
\item $\norm{\G}_2 \leq L_{G}$ and $\norm{\bs{\nu}}_2 \leq L_{\nu}$  and  $\norm{\bs{D}_0 }_2 \leq L_0$.
\item  $ \norm{ \bs{D}_0^{-1}  \bs{D}_1 }_2  \leq  L_{1/0}$ and  $ \norm{ \bs{D}_0  \bs{D}_1^{-1} }_2  \leq  L_{0/1}$.
\item For notational convenience,  let $L_F \coloneqq  2 L_{G} +  L_{1/0}$ and $L_{\grave{F}} \coloneqq  2(L_{G} + L_{\nu}) +  L_{1/0}$ . 
\end{itemize}
For   $0 \leq \eta < \tfrac{1}{  L_{\grave{F}}  L_0}$,  we have that $(\bs{I}_r + \eta  \bs{V})$ and   $(\bs{I}_r + \eta  \grave{\bs{V}})$ are invertible and  the following bounds holds:
{\small
\begin{alignat}{2}
&   -  C_1 \bs{D}_1   \preceq \bs{V}^2  \left( \bs{I}_r + \eta   \bs{V} \right)^{-1}  \!\!  \bs{\zeta}     \grave{\bs{V}}   \left( \bs{I}_r + \eta \grave{\bs{V}} \right)^{-1} \!\! \preceq C_1 \bs{D}_1, \quad  &&   - C_2 \bs{D}_1  \! \preceq \! \bs{V} \bs{\zeta}     \grave{\bs{V}}^2   \left( \bs{I}_r + \eta \grave{\bs{V}} \right)^{-1}  \!\!  \preceq C_2 \bs{D}_1   ~~~~    \label{eq:hobound1} \\
&     -   C_3 \bs{D}_1  \preceq \bs{V}^3  \left( \bs{I}_r + \eta   \bs{V} \right)^{-1}  \!\! \bs{\zeta} \preceq C_3 \bs{D}_1, \quad 
&& -  C_4 \bs{D}_1 \! \preceq  \! \bs{\zeta}  \grave{\bs{V}}^3  \left( \bs{I}_r + \eta \grave{\bs{V}} \right)^{-1}    \preceq C_4 \bs{D}_1,  ~~~~  \label{eq:hobound2}
\end{alignat} }
where
\begin{alignat}{2}
& C_1 =  \frac{ L_{\nu} L_{0/1} L_F^2  L_{\grave{F}}  L_0^3}{ \big( 1 - \eta  L_{F} L_{0} \big) \big(1 - \eta  L_{\grave{F}}  L_{0} \big)}, \quad  &&    C_2 =  \frac{ L_{\nu} L_{0/1}  L_F  L_{\grave{F}} ^2 L_0^3}{1 - \eta  L_{\grave{F}}   L_{0}} \\
&C_3 =  \frac{ L_{\nu} L_{0/1}    L_F^3  L_0^3}{1 - \eta   L_F  L_{0}},    
&&  C_4 = \frac{ L_{\nu} L_{0/1}  L_{\grave{F}}^3  L_0^3}{1 - \eta  L_{\grave{F}}  L_{0}}. 
\end{alignat}
\end{proposition}

\begin{proof}
Note that $\norm{\bs{V}}_2  \vee \norm{\grave{\bs{V}}}_2 \leq  L_{\grave{F}}  L_0$,  therefore,  if   $0 \leq \eta < \frac{1}{L_{\grave{F}} L_0}$,  $(\bs{I}_r + \eta  \bs{V})$ and   $(\bs{I}_r + \eta  \grave{\bs{V}})$ are invertible.   For the following, we introduce the notation
\eq{
\grave{\G} \coloneqq \G + \bs{\nu}  ~~ \text{and} ~~ \F = 2  \G +  \bs{D}_0^{-1}  \bs{D}_1 ~~ \text{and} ~~ \grave{\F} = 2 \grave{\G} +  \bs{D}_0^{-1}  \bs{D}_1.
}
Note that we have $\norm{  \F }_2 \leq     L_F$   and  $ \norm{  \grave{\F} }_2 \leq     L_{\grave{F}}$.  For the left  part of   \eqref{eq:hobound1}, we write
\eq{
 \bs{D}_0^{- \frac{1}{2}}\bs{V}^2  & \left( \bs{I}_r + \eta   \bs{V} \right)^{-1}   \bs{\zeta}     \grave{\bs{V}}   \left( \bs{I}_r + \eta \grave{\bs{V}} \right)^{-1} \bs{D}_0^{- \frac{1}{2}} \\
& = \F   \bs{D}_0 \F   \bs{D}_0  \left( \bs{I}_r + \eta   \F   \bs{D}_0  \right)^{-1} \bs{\nu}  \bs{D}_0\grave{\F} \left( \bs{I}_r + \eta      \bs{D}_0 \grave{\F}  \right)^{-1}.
}
Therefore, we have
\eq{
\Big \lVert  \F   \bs{D}_0 \F   \bs{D}_0  \big( \bs{I}_r  + \eta   \F   \bs{D}_0  \big)^{-1} \bs{\nu} &  \bs{D}_0\grave{\F} \big( \bs{I}_r + \eta      \bs{D}_0 \grave{\F}  \big)^{-1}  \Big \rVert_2 \\
&  \leq \frac{\norm{\F}^2_2 \norm{\grave{\F}}_2  \norm{\D_0}_2^3 \norm{\bs{\nu}}_2}{\big( 1 - \eta \norm{\F }_2 \norm{\D_0}_2   \big) \big( 1 - \eta \norm{\grave{\F} }_2 \norm{\D_0}_2   \big) } \\
& \leq \frac{L_F^2 L_{\grave{F}} L_0^3 L_{\nu}}{ (1 - \eta L_{F} L_{0}) (1 - \eta L_{\grave{F}} L_{0})}.
}
Therefore, we have the  bound.  For the right  part of   \eqref{eq:hobound1}, we write
\eq{
 \bs{D}_0^{- \frac{1}{2}} \bs{V} \bs{\zeta}     \grave{\bs{V}}^2   \left( \bs{I}_r + \eta \grave{\bs{V}} \right)^{-1}   \bs{D}_0^{- \frac{1}{2}}
= \F   \bs{D}_0 \bs{\nu} \D_0 \grave{\F}    \D_0 \grave{\F}   \left( \bs{I}_r + \eta  \D_0 \grave{\F}   \right)^{-1} 
}
Therefore, we have
\eq{
\norm*{   \F   \bs{D}_0 \bs{\nu} \D_0 \grave{\F}    \D_0 \grave{\F}   \left( \bs{I}_r + \eta  \D_0 \grave{\F}   \right)^{-1}   }_2    \leq \frac{\norm{\F}_2 \norm{\grave{\F}}^2_2  \norm{\D_0}_2^3 \norm{\bs{\nu}}_2}{1 - \eta \norm{\grave{\F} }_2 \norm{\D_0}_2 }  \leq \frac{L_F L_{\grave{F}}^2 L_0^3 L_{\nu}}{1 - \eta L_{\grave{F}} L_0},
}
which gives us the bound.  For  the left  part of   \eqref{eq:hobound2}, we write
\eq{
 \bs{D}_0^{- \frac{1}{2}}   \bs{V}^3  \left( \bs{I}_r + \eta   \bs{V} \right)^{-1} \bs{\zeta}  \bs{D}_0^{- \frac{1}{2}} 
 = (\F \D_0)^3  \left(  \bs{I}_r + \eta   \F   \D_0  \right)^{-1} \bs{\nu}
}
 Therefore, we have
 \eq{
\norm*{    (\F \D_0)^3  \left(  \bs{I}_r + \eta   \F   \D_0  \right)^{-1} \nu  }_2 &  \leq \frac{\norm{\F}^3_2   \norm{\D_0}_2^3  \norm{\bs{\nu}}_2 }{ 1 - \eta \norm{\F }_2 \norm{\D_0}_2  }  \leq \frac{L_{\nu} L_F^3  L_0^3}{1 - \eta L_F L_{0}} ,
}
which gives us the  bound.  The  the right  part of   \eqref{eq:hobound2} can be derived similarly. 
\end{proof}
 
\begin{proposition}
\label{prop:monotoneresidual}
Let $\bs{V},  \grave{\bs{V}}\in \R^{r \times r}$ be symmetric matrices such that  $\grave{\bs{V}} =  \bs{V} +  \bs{\zeta}$ .  We have
\eq{
\grave{\bs{V}} \left( \bs{I}_r + \eta \grave{\bs{V}}   \right)^{-1} -   \bs{V}  \left( \bs{I}_r + \eta  \bs{V}  \right)^{-1} = \left( \bs{I}_r + \eta   \bs{V} \right)^{-1}   \bs{\zeta}  \left( \bs{I}_r + \eta  \grave{\bs{V}}  \right)^{-1}.
}
Moreover,  given that
\eq{
\M \coloneqq      \bs{\zeta} - \eta   \bs{V} \bs{\zeta}  - \eta  \bs{\zeta}  \bs{V}  - \eta  \bs{\zeta}^2  +    \eta^2  \bs{\zeta}   \bs{V}  \bs{\zeta}  +    \eta^2  \bs{\zeta} ^3   + \eta^2    \bs{V}   \bs{\zeta}    \bs{V} 
}
under the conditions of  Proposition \ref{prop:residualhigherorder}, we have for any $\c_d > 0$,
\eq{
  -     \frac{2}{\c_d^2}  \eta^2  \bs{\zeta} ^2  \!  - \!   \eta^2 \c_d^2 \bs{V}^4  \!  -  \!     \eta^2 \c_d^2 \bs{V}  \bs{\zeta}^2 \bs{V}    \!  -  \!    C   \eta^3  \D_1 & \preceq 
 \left( \bs{I}_r + \eta   \bs{V} \right)^{-1}    \bs{\zeta}  \left( \bs{I}_r + \eta  \grave{\bs{V}}  \right)^{-1}  -  \M \\
 & \preceq    \frac{2}{\c_d^2}  \eta^2  \bs{\zeta}^2 +  \eta^2 \c_d^2  \bs{V}^4 +     \eta^2 \c_d^2 \bs{V}  \bs{\zeta}^2 \bs{V} +  C \eta^3  \D_1
}
where $C = C_1 + C_2 + C_3 + C_4$, i.e., the sum of the constants given in   Proposition \ref{prop:residualhigherorder}.
\end{proposition}

\begin{proof}
We write
\eq{
\grave{\bs{V}}   \big( \bs{I}_r + \eta \grave{\bs{V}}   \big)^{-1}   & -   \big( \bs{I}_r + \eta  \bs{V}  \big)^{-1}    \bs{V}   \\
& =   \big( \bs{I}_r + \eta  \bs{V}  \big)^{-1}  \Big(   \big( \bs{I}_r + \eta  \bs{V}  \big) (  \bs{V} +  \bs{\zeta} ) -    \bs{V}    \big( \bs{I}_r + \eta \grave{\bs{V}}   \big) \Big) \big( \bs{I}_r + \eta \grave{\bs{V}}   \big)^{-1}  \\
& =   \left( \bs{I}_r + \eta   \bs{V} \right)^{-1}   \bs{\zeta}  \left( \bs{I}_r + \eta  \grave{\bs{V}}  \right)^{-1}.
}
For the second part, we write
\eq{
 & \left( \bs{I}_r + \eta   \bs{V} \right)^{-1}   \bs{\zeta}  \left( \bs{I}_r + \eta  \grave{\bs{V}}  \right)^{-1}   \\
& = \left(  \bs{I}_r - \eta \bs{V}  \left( \bs{I}_r + \eta   \bs{V} \right)^{-1} \right)  \bs{\zeta} \left(  \bs{I}_r - \eta  \grave{\bs{V}}  \left( \bs{I}_r + \eta    \grave{\bs{V}} \right)^{-1} \right) \\
& =   \bs{\zeta}  -   \eta \bs{V}   \left(  \bs{I}_r - \eta \bs{V} + \eta^2 \bs{V}^2      \left( \bs{I}_r + \eta   \bs{V} \right)^{-1} \right)   \bs{\zeta}   \\  
& - \eta  \bs{\zeta} \grave{\bs{V}}  \left(  \bs{I}_r - \eta \grave{\bs{V}} + \eta^2  \grave{\bs{V}}^2   \left( \bs{I}_r + \eta   \bs{V} \right)^{-1} \right)  
 +   \eta^2 \bs{V}  \left( \bs{I}_r + \eta   \bs{V} \right)^{-1}   \bs{\zeta} \grave{\bs{V}}  \left( \bs{I}_r + \eta    \grave{\bs{V}} \right)^{-1} \\
& =   \bs{\zeta} - \eta   \bs{V} \bs{\zeta}  - \eta  \bs{\zeta}  \bs{V}  - \eta  \bs{\zeta}^2 + \  \eta^2  \bs{V}^2    \bs{\zeta} +  \eta^2  \bs{\zeta}  \grave{\bs{V}} ^2  - \eta^3 \bs{V}^3  \left( \bs{I}_r + \eta \bs{V}   \right)^{-1} \bs{\zeta}   \\
&  -   \eta^3  \bs{\zeta}  \grave{\bs{V}} ^3  \left( \bs{I}_r + \eta \grave{\bs{V}}   \right)^{-1} + \eta^2    \bs{V}   \left( \bs{I}_r + \eta \bs{V}   \right)^{-1} \bs{\zeta}  \grave{\bs{V}}   \left( \bs{I}_r + \eta \grave{\bs{V}}   \right)^{-1}  \label{eq:residual}
}
We have
\eq{
 \eta^2  \bs{V}^2    \bs{\zeta} +  \eta^2  \bs{\zeta}  \grave{\bs{V}} ^2 
=   \eta^2  \bs{V}^2    \bs{\zeta} +  \eta^2  \bs{\zeta}  (  \bs{V} + \bs{\zeta} )^2 =      \underbrace{  \eta^2  \bs{V}^2    \bs{\zeta}  + \eta^2   \bs{\zeta}    \bs{V}^2 + \eta^2  \bs{\zeta}^2  \bs{V}  }_{\coloneqq \M_1}   +    \eta^2  \bs{\zeta}   \bs{V}  \bs{\zeta}  +   \eta^2  \bs{\zeta} ^3 .
}
Moreover,
\eq{
\eta^2    \bs{V} &  \left( \bs{I}_r + \eta \bs{V}   \right)^{-1} \bs{\zeta}  \grave{\bs{V}}   \left( \bs{I}_r + \eta \grave{\bs{V}}   \right)^{-1} \\
& =  \eta^2    \bs{V}   \bs{\zeta}  \grave{\bs{V}}   \left( \bs{I}_r + \eta \grave{\bs{V}}   \right)^{-1} -  \eta^3    \bs{V}^2   \left( \bs{I}_r + \eta \bs{V}   \right)^{-1} \bs{\zeta}  \grave{\bs{V}}   \left( \bs{I}_r + \eta \grave{\bs{V}}   \right)^{-1} \\
& =   \eta^2    \bs{V}   \bs{\zeta}  \grave{\bs{V}}    -  \eta^3   \bs{V}   \bs{\zeta}   \grave{\bs{V}}^2   \left( \bs{I}_r + \eta \grave{\bs{V}}   \right)^{-1} 
-  \eta^3    \bs{V}^2   \left( \bs{I}_r + \eta \bs{V}   \right)^{-1} \bs{\zeta}  \grave{\bs{V}}   \left( \bs{I}_r + \eta \grave{\bs{V}}   \right)^{-1} \\
& =  \eta^2    \bs{V}   \bs{\zeta}    \bs{V}  +  \underbrace{ \eta^2   \bs{V}   \bs{\zeta}^2 }_{\coloneqq \M_2}    -  \eta^3   \bs{V}   \bs{\zeta}   \grave{\bs{V}}^2   \left( \bs{I}_r + \eta \grave{\bs{V}}   \right)^{-1}  -  \eta^3    \bs{V}^2   \left( \bs{I}_r + \eta \bs{V}   \right)^{-1} \bs{\zeta}  \grave{\bs{V}}   \left( \bs{I}_r + \eta \grave{\bs{V}}   \right)^{-1} 
}
By Proposition \ref{prop:matrixyounginequality}, we have
\eq{
 -  2 \eta^2  \bs{\zeta} ^2   -  \eta^2  \bs{V}^4 -     \eta^2 \bs{V}  \bs{\zeta}^2 \bs{V}  \preceq \M_1 + \M_2 \preceq   2 \eta^2  \bs{\zeta}^2 +  \eta^2  \bs{V}^4 +     \eta^2 \bs{V}  \bs{\zeta}^2 \bs{V}.
}
Therefore by Proposition \ref{prop:residualhigherorder}, we have
\eq{
  -  \frac{2}{\c_d^2} \eta^2  \bs{\zeta} ^2  \!  - \!  \eta^2 \c_d^2 \bs{V}^4 \! -   \!  \eta^2 \c_d^2 \bs{V}  \bs{\zeta}^2 \bs{V}   \! -  \! C   \eta^3  \D_1 & \preceq 
  \left( \bs{I}_r + \eta   \bs{V} \right)^{-1}    \bs{\zeta}  \left( \bs{I}_r + \eta  \grave{\bs{V}}  \right)^{-1}  -  \M \\
 & \preceq    \frac{2}{\c_d^2} \eta^2  \bs{\zeta}^2 +  \eta^2 \c_d^2  \bs{V}^4 +     \eta^2 \c_d^2 \bs{V}  \bs{\zeta}^2 \bs{V} +  C \eta^3  \D_1.
}
\end{proof} 

\begin{proposition}
\label{prop:taylorresidual}
By using the notation  in Proposition \ref{prop:residualhigherorder},   we consider
\eq{
\eta < \frac{1}{L_F L_0} ~~ \text{and} ~~   0 < \varepsilon <  \frac{0.5/\eta}{L_F L_0} -1 \label{eq:taylorresidualparams}
}
Then,
\eq{
& \bs{V} \left(  \bs{I}_r + \eta  \bs{V} \right)^{-1} -   \varepsilon \eta  \bs{V}^2  \succeq  \bs{V} \left(  \bs{I}_r + \eta (1 +\varepsilon )  \bs{V} \right)^{-1}  -  2.5 \varepsilon \eta^2 C \bs{D}_1  \label{eq:resstatement1} \\
& \bs{V} \left(  \bs{I}_r + \eta  \bs{V} \right)^{-1} +    \varepsilon \eta  \bs{V}^2  \preceq  \bs{V} \left(  \bs{I}_r + \eta (1 - \varepsilon )  \bs{V} \right)^{-1}  +1.5 \varepsilon \eta^2 C \bs{D}_1, \label{eq:resstatement2}
}
where   $C = \frac{  L_{0/1}    L_F^3  L_0^3}{1 - \eta   L_F  L_{0}}$.
\end{proposition}

\begin{proof}
For the lower bound, we have
\eq{
& \bs{V} \left(  \bs{I}_r + \eta  \bs{V} \right)^{-1}   -   \varepsilon \eta  \bs{V}^2  \\
 & =  \V - (1 + \varepsilon) \eta \V^2 + \eta^2 \V^3 (\bs{I}_r + \eta \V )^{-1} \\
& =    \V - (1 + \varepsilon) \eta \V^2  + (1 + \varepsilon)^2 \eta^2 \V^3 (\bs{I}_r + (1 + \varepsilon) \eta \V )^{-1}  \\
&  -  (2 \varepsilon +  \varepsilon^2 )  \eta^2 \V^3 (\bs{I}_r +  \eta \V )^{-1}  + (1 + \varepsilon)^2  \varepsilon \eta^3 \V^4 (\bs{I}_r  +  (1+ \varepsilon) \eta \V )^{-1}  \left(  \bs{I}_r + \eta  \bs{V} \right)^{-1}   \\
& \succeq   \V \left(  \bs{I}_r + \eta (1 +\varepsilon )  \V \right)^{-1}  - 2.5 \varepsilon \eta^2 C \D_1,
}
where we used $C_3$  with $L_{\nu} = 1$ in Proposition \ref{prop:residualhigherorder} in the last step.  For the upper bound,
\eq{
 & \bs{V} \left(  \bs{I}_r + \eta  \bs{V} \right)^{-1}  +   \varepsilon \eta  \bs{V}^2  \\
 &  =  \V - (1 - \varepsilon) \eta \V^2 + \eta^2 \V^3 (\bs{I}_r + \eta \V )^{-1} \\
& =    \V - (1 - \varepsilon) \eta \V^2  + (1 - \varepsilon)^2 \eta^2 \V^3 (\bs{I}_r + (1 - \varepsilon) \eta \V )^{-1}  \\
&  +  (2 \varepsilon -  \varepsilon^2 )  \eta^2 \V^3 (\bs{I}_r +  \eta \V )^{-1}  - (1 - \varepsilon)^2  \varepsilon \eta^3 \V^4 (\bs{I}_r  +  (1- \varepsilon) \eta \V )^{-1}  \left(  \bs{I}_r + \eta  \bs{V} \right)^{-1}   \\
& \preceq   \V \left(  \bs{I}_r + \eta (1 +\varepsilon )  \V \right)^{-1}  + 1.5 \varepsilon \eta^2 C \D_1.
}
\end{proof}

\begin{lemma}
\label{lem:expanti2}
For any $\eta \in \R$ and $t\in \N$, we have
\eq{
\begin{bmatrix}
\bs{I}_{r} &  \eta  \bs{I}_{r} \\
\eta \bs{\Lambda}^2 & \bs{I}_{r}
\end{bmatrix}  ^t =  \begin{bmatrix}
\frac{ \left(\bs{I}_r + \eta  \bs{\Lambda} \right)^t +  \left(\bs{I}_r - \eta  \bs{\Lambda} \right)^t}{2}&  \bs{\Lambda}^{-1} \frac{ \left(\bs{I}_r + \eta  \bs{\Lambda} \right)^t -  \left(\bs{I}_r - \eta  \bs{\Lambda} \right)^t}{2} \\[0.5em]
\bs{\Lambda} \frac{ \left(\bs{I}_r + \eta  \bs{\Lambda} \right)^t -  \left(\bs{I}_r - \eta  \bs{\Lambda} \right)^t}{2} & \frac{ \left(\bs{I}_r + \eta  \bs{\Lambda} \right)^t +  \left(\bs{I}_r  - \eta  \bs{\Lambda} \right)^t}{2}
\end{bmatrix} .  \label{eq:expanti}
}
\end{lemma}

\begin{proof}
We observe that
\eq{
\bs{A}  \coloneqq \begin{bmatrix}
0 &   \bs{I}_{r} \\
 \bs{\Lambda}^2 & 0
\end{bmatrix} ~ \Rightarrow  ~ \eqref{eq:expanti}  = \sum_{k = 0}^t \binom{t}{k} \eta^k   \bs{A}^k.
}
Note that 
\eq{
\bs{A}^{2k} = \begin{bmatrix}
\bs{\Lambda}^{2k} & 0 \\
0 &  \bs{\Lambda}^{2k} 
\end{bmatrix} ~~ \text{and} ~~
\bs{A}^{2k +1} = \begin{bmatrix}
0 & \bs{\Lambda}^{2k} \\
 \bs{\Lambda}^{2k + 2} & 0
\end{bmatrix}. 
}
Therefore,
\eq{
 \sum_{k = 0}^t \binom{t}{k} \eta^k   \bs{A}^k & = \begin{bmatrix}
  \sum_{\substack{k = 0 \\ \text{k even}}}^t \binom{t}{k} \eta^k   \bs{\Lambda}^k &    \sum_{\substack{k = 0 \\ \text{k odd}}}^t \binom{t}{k} \eta^k   \bs{\Lambda}^{k-1} \\[1em]
  \sum_{\substack{k = 0 \\ \text{k odd}}}^t \binom{t}{k} \eta^k   \bs{\Lambda}^{k+1} &   \sum_{\substack{k = 0 \\ \text{k even}}}^t \binom{t}{k} \eta^k   \bs{\Lambda}^k
 \end{bmatrix} \\[0.5em]
& =  \begin{bmatrix}
\frac{ \left(\bs{I}_r + \eta  \bs{\Lambda} \right)^t +  \left(\bs{I}_r - \eta  \bs{\Lambda} \right)^t}{2}&  \bs{\Lambda}^{-1} \frac{ \left(\bs{I}_r + \eta  \bs{\Lambda} \right)^t -  \left(\bs{I}_r - \eta  \bs{\Lambda} \right)^t}{2} \\[0.5em]
\bs{\Lambda} \frac{ \left(\bs{I}_r +  \eta  \bs{\Lambda} \right)^t -  \left(\bs{I}_r - \eta  \bs{\Lambda} \right)^t}{2} & \frac{ \left(\bs{I}_r + \eta  \bs{\Lambda} \right)^t +  \left(\bs{I}_r - \eta  \bs{\Lambda} \right)^t}{2}
\end{bmatrix} . 
}
\end{proof}

\subsection{Some moment bounds and concentration inequalities}

\begin{lemma}[Hypercontractivity]
\label{lem:hypercontractivity}
Let $P_k : \R^d \to \R$ be a polynomial of degree-$k$ and $\bs{x} \sim \cN(0,\bs{I}_{d})$.  For $q \geq 2$, we have $
\E  \left[ P_k(\bs{x})^q \right]^{1/q} \leq (q-1)^{k/2}  \E  \left[ P_k(\bs{x})^2 \right]^{1/2} .$
\end{lemma}

\begin{lemma}
\label{lem:hansonwright}
Let  $\bs{x} \sim \cN(0,\bs{I}_{d})$ and $\bs{S} \in \R^{d \times d}$ be a symmetric matrix.  For $u > 0$,
\eq{
\mpr \left[  \abs{ \bs{x}^\top  \bs{S}  \bs{x} - \tr(\bs{S})  } \geq 2 \norm{\bs{S}}_F u + 2 \norm{\bs{S}}_2 u^2 \right] \leq 2 e^{- u^2}.
}
\end{lemma}

\begin{proof}
We note that   $\bs{x}^\top  \bs{S}  \bs{x}  - \tr(\bs{S})$ has the same distribution with $\sum_{i = 1}^d \lambda_i(\bs{S}) (Z_i^2 - 1)$, where   $Z_i \sim_{iid} \cN(0,1)$.  By using the Laurent-Massart lemma \cite{LaurentMassart}, we have the result.
\end{proof}

\begin{corollary}
\label{cor:hc1}
Let  $y =   \bs{x}^\top \bs{S} \bs{x} - \tr(\bs{S})$  and $\bs{x} \sim \cN(0,\bs{I}_{d})$.  For $p \geq 2$,  we have $\E [ \abs{y}^p ]^{\frac{1}{p}} \leq (p-1) \sqrt{2} \norm{\bs{S}}_F$.
\end{corollary}

\begin{proof}
By observing that  $\E [ \abs{y}^2 ] = 2 \norm{\bs{S}}_F^2$, we have the result.
\end{proof}

\begin{corollary}
\label{cor:hc2}
For $\bs{A} \in \R^{d \times r}$, $p \geq 2$ and $\bs{x} \sim \cN(0,\bs{I}_{d})$, we have $ \E[\norm{\bs{A}^\top \bs{x}}_2^{2p}]^{\frac{1}{p}} \leq  \sqrt{ 3 }  (p-1)  \tr(\bs{A}^\top \bs{A})$.
\end{corollary}

\begin{proof}
By Lemma \ref{lem:hypercontractivity}, we have $ \E[\norm{\bs{A}^\top \bs{x}}_2^{2p}]^{\frac{1}{p}} \leq  (p-1)   \E[\norm{\bs{A}^\top \bs{x}}_2^{4}]^{\frac{1}{2}}$.   For $\bs{S} = \bs{A} \bs{A}^\top$, we have
\eq{
\E[\norm{\bs{A}^\top \bs{x}}_2^{4}] = \E [ (\bs{x}^\top \bs{S} \bs{x} )^2 ] = \tr( \E [ (\bs{x}^\top \bs{S} \bs{x} ) \bs{x} \bs{x}^\top ]  \bs{S}  ).
}
We have
\eq{
 \E [ (\bs{x}^\top \bs{S} \bs{x} ) \bs{x} \bs{x}^\top ]  = \tr(\bs{S}) \bs{I}_d + 2 \bs{S}   \Rightarrow   \E[\norm{\bs{A}^\top \bs{x}}_2^{4}] =   \tr(\bs{S})^2 + 2 \norm{\bs{S}}_F^2 \labelrel\leq{hc2:ineqq0} 3 \tr(\bs{S})^2,   
}
where  \eqref{hc2:ineqq0} follows that $\bs{S}$ is positive semi-definite.  Since $\tr(\bs{S}) = \tr(\bs{A}^\top \bs{A})$, we have the statement.
\end{proof}

\begin{proposition}
\label{prop:matrixsensingcondition}
Let $\bs{x}_j \sim_{i.i.d} \cN(0, \bs{I}_r)$, for $j \in [N]$. There exists a constant $c > 0$ such that for  $\delta =  \frac{u ( r + \sqrt{Cr \log d} +  C \log d) }{\sqrt{ N } }$, we have
\eq{
\mpr \Bigg[   \sup_{ \substack{ \bs{S} \in \R^{r  \times r} \\ \norm{\bs{S}}_F = 1 } } \abs*{ \frac{1}{N} \sum_{j = 1}^N  \tfrac{1}{2} \tr \big(\bs{S} (\bs{x}_j \bs{x}_j^\top - \bs{I}_r) \big)^2 -  1  }  \geq \max \{ 2 \delta, \delta^2 \}  & +  10 d^{-C/2} \Bigg] \\
& \leq d^2 \exp(- c u^2)   +  2 N d^{-C}.
}
\end{proposition}

\begin{proof}
We observe that
\eq{
\frac{1}{2} \norm{ \bs{x}_j \bs{x}_j^\top - \bs{I}_r }_F \leq  \frac{1}{2} \left( \norm{\bs{x}_j}_2^2 + r \right).
}
By using Lemma \ref{lem:hansonwright}, we can derive 
\eq{
\mpr \big[ \underbrace{ \norm{\bs{x}_j}_2^2 \leq r + 2 \sqrt{r} \sqrt{ C \log d }  + 2 C \log d }_{\eqqcolon \mathcal{E}_j}  \big] \geq 1 - 2 d^{-C}.
}
We have
\eq{
\Big \lvert \E \left[    \tfrac{1}{2} \tr \big(\bs{S} (\bs{x}_j \bs{x}_j^\top - \bs{I}_r) \big)^2 \indic{ \mathcal{E}_j } \right] - 1 \Big \rvert  & =   \tfrac{1}{2} \E \big[ \tr \big(\bs{S} (\bs{x}_j \bs{x}_j^\top - \bs{I}_r) \big)^2 \indic{ \mathcal{E}^c_j } \big]   \\
& \leq   \tfrac{1}{2} \E \left[ \tr \big(\bs{S} (\bs{x}_j \bs{x}_j^\top - \bs{I}_r) \big)^4 \right]^{1/2} \sqrt{2} d^{-C/2} \\
 & \leq 9 \sqrt{2} d^{-C/2}. 
}
By using \cite[Theorem 5.41]{vershynin2010introduction},  for $\delta =  \tfrac{u ( r + \sqrt{Cr \log d} +  C \log d) }{\sqrt{ N } }$,  we have
\eq{
& \mpr \Bigg[  \sup_{ \substack{ \bs{S} \in \R^{r  \times r} \\ \norm{\bs{S}}_F = 1 } } \abs*{ \frac{1}{N} \sum_{j = 1}^N  \tfrac{1}{2} \tr \big(\bs{S} (\bs{x}_j \bs{x}_j^\top - \bs{I}_r) \big)^2 -  1  } \geq  \max \{ 2 \delta, \delta^2 \}   + 10 d^{-C/2}  \Bigg] \\
& \leq 
\mpr \Bigg[ \!  \sup_{ \substack{ \bs{S} \in \R^{r  \times r} \\ \norm{\bs{S}}_F = 1 } } \abs*{ \frac{1}{N} \sum_{j = 1}^N  \tfrac{1}{2} \tr \big(\bs{S} (\bs{x}_j \bs{x}_j^\top\! - \! \bs{I}_r) \big)^2  \indic{ \mathcal{E}_j } \! - \!  \E \left[  \tfrac{1}{2} \tr \big(\bs{S} (\bs{x}_j \bs{x}_j^\top \! - \! \bs{I}_r) \big)^2  \indic{ \mathcal{E}_j } \right] }  \! \geq \! \max \{ 2 \delta, \delta^2 \}  \! \Bigg] \\
& +  2 N d^{-C} \\
& \leq  d^2 \exp(- c u^2)   +  2 N d^{-C} .
}
\end{proof}

\begin{proposition}
\label{prop:steinestimator}
Let $\bs{x}_j \sim_{i.i.d} \cN(0, \bs{I}_d)$, for $j \in [N]$, and $\W \in \R^{d \times r}$ be an orthonormal matrix. For a fixed $\bs{S} \in \R^{d \times d}$, $C \geq 16$ and $N \geq C r \log d$, we have
\eq{
\mpr \Big[  \Big \lVert \frac{1}{N}  \sum_{j = 1}^N \frac{1}{2} \tr \big(\bs{S} (\bs{x}_j \bs{x}_j^\top \!  - \!   \bs{I}_d) \big) \W^\top  (\bs{x}_j \bs{x}_j^\top \! - \! \bs{I}_d)  \W  \!  -  \!   \W ^\top  \bs{S}  \W  \Big \rVert_2  \! \geq &    24 e \norm{\bs{S}}_F \Big(   \sqrt{ \tfrac{C r}{N} }  \! + \!    d^{\frac{-C}{2}} \Big) \Big] \\
&\leq  2 e^{\frac{-Cr}{8}} + 2N d^{-C}.
}
\end{proposition}

\begin{proof}
    Without loss of generality, we assume $\norm{\bs{S}}_F = 1$. By using Lemma \ref{lem:hansonwright}, we have
    \eq{
    \mpr \big[ \underbrace{ \abs{  \tr \big(\bs{S} (\bs{x}_j \bs{x}_j^\top - \bs{I}_d) \big)  } \leq 4 \sqrt{C \log d} }_{\eqqcolon \mathcal{E}_j } \big] \geq 1 - 2 d^{-C}.
    }
    For the following, we fix a  $\bs{v} \in S^{d-1}$. First, to bound the bias due to clipping, we write:
    \eq{
    & \abs*{ \E \left[  \tr \big(\bs{S} (\bs{x}_j \bs{x}_j^\top - \bs{I}_d) \big)   (\inner{\bs{v}}{\bs{x}}^2 - 1) \indic{\mathcal{E}_j^c} \right] } \\
    & \leq \E \left[  \tr \big(\bs{S} (\bs{x}_j \bs{x}_j^\top - \bs{I}_d) \big)^4 \right]^{\frac{1}{4}} \E \left[ (\inner{\bs{v}}{\bs{x}}^2 - 1)^4\right]^{\frac{1}{4}} \sqrt{2} d^{-C/2}  \leq 18 \sqrt{2} d^{- C/2}.
    }
    On the other hand, to bound the moments of the clipped random variable, we have for $p \geq 2$,
    \eq{
    & \E \left[  \abs{\tr \big(\bs{S} (\bs{x}_j \bs{x}_j^\top - \bs{I}_d) \big) (\inner{\bs{v}}{\bs{x}}^2 - 1) }^p  \indic{\mathcal{E}_j}  \right] \\
     & \leq (4 \sqrt{C \log d})^{p -2} \E \left[  \tr \big(\bs{S} (\bs{x}_j \bs{x}_j^\top - \bs{I}_d) \big)^2 \abs{(\inner{\bs{v}}{\bs{x}}^2 - 1) }^p    \right] \leq (12 e)^2  (8 \sqrt{2} e \sqrt{C \log d})^{p -2}  \frac{p!}{2}.
    }
    By using $\varepsilon$-cover argument, we can derive
    \eq{
    & \mpr \left[ \Big \lVert \frac{1}{N}  \sum_{j = 1}^N \frac{1}{2} \tr \big(\bs{S} (\bs{x}_j \bs{x}_j^\top   -  \bs{I}_d) \big) \W^\top (\bs{x}_j \bs{x}_j^\top - \bs{I}_d)  \W   -     \W ^\top \bs{S}  \W  \Big \rVert_2 \geq 24 e u  + 18 \sqrt{2} d^{- C/2}  \right] \\
    &\leq 2 \cdot 9^r \exp \Big( \frac{- N u^2/2}  { 1 +  u \sqrt{C \log d}  } \Big) + 2Nd^{-C}.  
    }
    By using $u = \sqrt{C r / N}$, we have the result.
\end{proof}

\begin{proof}
Without loss of generality, we assume $\norm{ \bs{S} }_F = 1$. We have
\eq{
 \Big \lVert  \sum_{j = 1}^N y_j  (\bs{x}_j \bs{x}_j^\top - \bs{I}_d) -  \bs{S}  \Big  \rVert_F^2
 \leq   \sup_{ \substack{ \bs{S} \in \R^{r  \times r} \\ \norm{\bs{S}}_F = 1 } } \abs*{ \frac{1}{N} \sum_{j = 1}^N  \tfrac{1}{2} \tr \big(\bs{S} (\bs{x}_j \bs{x}_j^\top - \bs{I}_r) \big)^2 -  1  }^2.
}
 Hence, by considering the event in Proposition \ref{prop:matrixsensingcondition}, we have the statement.
\end{proof}

\begin{proposition}
\label{prop:basicconcentration}
Let $X \in \R$ be a random variable such that for some $K,C > 0$, $\E[ \abs{X}^p ] \leq C K^p p^{p c}$ for some $c > 0$ and $p \geq k$. Then, 
$\mpr \left[ \abs{X} \geq K u  \right] \leq C e^{- \frac{u^{1/c}}{e}}$ for $u \geq (k e)^{c}$.
\end{proposition}

\begin{proof}
Use Markov inequality with $p = \tfrac{u^{1/c}}{e}$.
\end{proof}

\subsection{Miscellaneous}
\begin{proposition}
\label{prop:1dsystemaux}
We consider $\eta  \leq \tfrac{1}{10}$. The following statements holds:
\begin{itemize}[leftmargin=*]
\item For $0.2 \geq \delta > 0$, let
\eq{
u_{t+1} = u_t + \eta u_t (1 - u_t), ~~ 1 + \delta  \geq u_0 \geq 0.
}
We have   $1 + \big( \delta \vee \tfrac{\eta^2}{4} \big)  \geq \sup_{t} u_t \geq 0$.  Moreover, $t^* = \inf \{ t : u_t  \geq 1 \}$, we have $u_{t+1} \geq u_t$ for $t < t^*$ and $ u_{t^*} \geq u_t \geq 1$ for $t \geq t^*$.
\item For   $0.5 > \varepsilon > 0$ and $1.1> \overline{u}_0 \geq u_0 \geq \underline{u}_0  > 0$,  let
\eq{
 \overline{u}_{t+1}  =  \overline{u}_t + \eta  (1 + \varepsilon)   \overline{u}_t (1 -   \overline{u}_t) ~~ \text{and} ~~   \underline{u}_{t+1} =  \underline{u}_t + \eta    \underline{u}_t (1 -  \underline{u}_t).
}
and
\eq{
& u_t + \eta   u_t (1 - u_t) \leq u_{t+1} \leq  u_t + \eta  (1 + \varepsilon)   u_t (1 - u_t).
}
We have  
\eq{
\frac{1}{2} \left( 1 \wedge \underline{u}_0 e^{\frac{\eta t}{1 + \eta}} \right) \leq  \underline{u}_t\leq u_t \leq  \overline{u}_t \leq   \left(  1.1 \wedge \overline{u}_0 e^{\eta (1 + \varepsilon) t} \right).
}
\end{itemize}
\end{proposition}

\begin{proof}
If $t \geq t^*$,    by monotonicity of the update, we have  $1 \leq u_{t+1} \leq u_t \leq u_{t^*}$.  If $t^* > 0$, then for $t < t^*$,  we have $1 \geq u_t \geq 0$ and $u_t (1 - u_t) \geq 0$,  and thus,  we have $u_{t+1} \geq u_t \geq 0$.  Next, we observe that  $u_{t^*} \leq 1 + 0.25 \eta$ and  by monotonicity of the update for $t \geq t^*$,   we have $1 \leq u_t \leq u_{t*}$.
Hence, it is sufficient to bound $u_{t^*}$ to bound  $\sup_{t} u_t$.  Note that,  we have $1 \geq u_{t^* - 1} \geq 1 - 0.25 \eta$, and thus,
\eq{
\frac{u_{t^*}}{u_{t^* - 1}} = 1 +  \eta (1 - u_{t^* - 1}) \leq 1 + \eta^2 \Rightarrow u_{t^* } \leq 1 + \frac{\eta^2}{4}.
}
For the second item,   by monotonicity, we have  $0 < \underline{u}_{t} \leq u_t \leq \overline{u}_{t} < 1.1$.   Moreover,   by  \cite[Lemma A.2]{arous2024high},  we have for $t < t_u \coloneqq \inf \{ t : \underline{u}_t \geq 0.5\}$ 
\eq{
\frac{ \underline{u}_0  e^{\frac{\eta t}{1 + \eta} } }{1   + \underline{u}_0  e^{\frac{\eta t }{1 + \eta} } } \leq \underline{u}_t
\Rightarrow   \frac{\underline{u}_0 }{2 }  e^{\frac{\eta t}{1 + \eta} }  \leq \underline{u}_0 .
}
For $t \geq t_u,$  by the first item,  we have  $\underline{u}_t \geq 0.5$. Therefore, we have
\eq{
\frac{1}{2} \left( \underline{u}_0 e^{\frac{\eta t}{1 + \eta} }  \wedge 1 \right) \leq \underline{u}_t.
}
On the other hand,   for all $t \in \N$,  we have $ \overline{u}_t  \leq \overline{u}_{0} e^{\eta (1 + \varepsilon) t}$. By the first item, we have   $ \overline{u}_t  \leq \overline{u}_{0} e^{\eta (1 + \varepsilon) t} \wedge 1.1$.
\end{proof}

\begin{proposition}
\label{prop:auxbound}
For $t, \lambda > 0$, we have
\eq{
\frac{1}{t \exp(t \lambda)} \leq \frac{\lambda}{\exp(t \lambda) - 1} \leq \frac{1}{t}.
}
\end{proposition}

\begin{proof}
The upper bound follows  $\exp(t \lambda) - 1 \geq t \lambda$. For the lower bound,
\eq{
\frac{1}{t} -  \frac{\lambda}{\exp(t \lambda) - 1}  = \frac{\exp(t \lambda) - t \lambda - 1}{t \big( \exp(t \lambda) - 1 \big) }. \label{eq:diff}
}
We have
\eq{
\exp(t \lambda) - t \lambda - 1 \leq \sum_{k = 2}^{\infty} \frac{(t \lambda)^k}{k !} & =  t \lambda  \sum_{k = 1}^{\infty} \frac{(t \lambda)^k}{(k+1) !}  \leq  t \lambda  \sum_{k = 1}^{\infty} \frac{(t \lambda)^k}{k !} =   t \lambda \big( \exp(t \lambda) - 1 \big).
}
Therefore,
\eq{
\eqref{eq:diff} \leq  \lambda \Rightarrow  \frac{1}{t} \leq   \frac{\lambda}{\exp(t \lambda) - 1} + \lambda \Rightarrow  \frac{1}{t \exp(t \lambda)} \leq \frac{\lambda}{\exp(t \lambda) - 1}.
}
\end{proof}

\begin{lemma}
\label{lem:asympkernel}
Let $\rs \asymp d^{\gamma}$, $\gamma \in [0,1)$, and $\log^{-1} d \ll C_d \ll \log^{10} d $.  We define   $F_d, G_d, H_d$ as
\eq{
& F_d(u)  \coloneqq  \Bigg( 1 - \frac{1}{ 1  + \left( \frac{d C_d  }{\rs}  \frac{1}{u}    - 1 \right) \left( \frac{d}{\rs} \right)^{- \frac{1}{u} } }  \Bigg)^2, \qquad   G_d(u) \coloneqq   1 - \frac{1}{ 1  + \left( \frac{d C_d }{\rs}     - 1 \right) \left( \frac{d}{\rs} \right)^{- \frac{1}{u} } }, \\[0.6em]
&   H_d(u) \coloneqq \left( 1 - C_d\left( \tfrac{d}{\rs} \right)^{ \frac{1}{u} -1} \right)_+,  
}
  We have
\begin{itemize}
\item For any $C > 0$,   $ \sup_{u \leq \log^{C} \! d}  \abs{ F_d(u)  } \leq 1$ for $d \geq \Omega_{C}(1)$.
\item  $ \sup_{u}  \abs{ G_d(u)  } \vee  \abs{ H_d(u)  } \leq 1$.
\item For any   $\delta \in (0,0.5)$, let  $\mathcal{C}_{\delta} \coloneqq \{  u \geq 0 : ~ \abs{u - 1} < \delta  \}$.   For any compact $\mathcal{K} \subset   (0,\infty]  \setminus  \mathcal{C}_{\delta}$, we have $F_d(u) \xrightarrow{d \to \infty} \mathbbm{1} \{ u > 1 \}$
 uniformly on $\mathcal{K}.$ 
\item   For any compact $\mathcal{K} \subseteq  [0,\infty]  \setminus    \mathcal{C}_{\delta}$, we have
\eq{
G_d(u) , H_d(u )  \xrightarrow{d \to \infty} \mathbbm{1} \{ u > 1 \},  \quad  G^2_d(u) \xrightarrow{d \to \infty} \mathbbm{1} \{ u > 1 \}  
}
all uniformly on $\mathcal{K}.$ 
\end{itemize}
\end{lemma}

\begin{proof}
For the first item,  if $u \leq \log^{C} d$, for $d \geq \Omega_{C}(1)$ 
\eq{
\frac{d C_d u}{\rs}      - 1  \geq    \frac{d}{\rs}   \frac{t}{\log^{C+1} d}    - 1  > 0 ~ .
}
Therefore,  $\abs{F_d(u) } \leq 1$.  For the second item, since  $\tfrac{d C_d }{\rs}  > 1$ for $d   \geq \Omega(1)$,  the item follows.

For the third item,  since  $E \coloneqq [0,\infty)  \setminus  \mathcal{C}_{\delta}$ is closed in $ [0,\infty)$,  it suffices to establish the result on small open intervals around each point of  of $E$ within  $[0,\infty)$.  Fix $u_0 \in E$ and choose $\epsilon \in (0,\delta/2)$.  Since  $B(u_0, \epsilon) \coloneqq (u_0 - \epsilon, u_0 + \epsilon) \cap [0,\infty)$ is convex it can be either in   $P_{>} \coloneqq \{ u : ~ u  > 1 + \delta/2   \}$  or  $P_{<} \coloneqq  \{ u  : ~   u  < 1  -\delta/2 \}$. Without loss of generality let us assume it is in  $P_{<}$. Then,
\eq{
\sup_{ u \in B(u_0, \epsilon)  \subset P_{<}}  \abs{F_d(u)}  \leq  1 - \frac{1}{ 1  + \left( O_{\delta}(C_d)  - \left( \frac{d}{\rs} \right)^{- 1 } \right) \left( \frac{d}{\rs} \right)^{- O_{\delta}(1) } }     \to 0.
}
A similar step can be repeated if   $B_{\epsilon} \subset P_{<}$.

\smallskip
For the last item,  we first observe that uniform convergence of   $G_d(u)$  implies the uniform convergence of  $G^2_d(u) $. Therefore, we will only prove the first result.  Since  $E \coloneqq [0,\infty]  \setminus  \mathcal{C}_{\delta}$,  is compact, and thus, $P_{>} \cap E$ and $P_{<} \cap E$ are also compact,   we can directly use these sets. Without loss of generality let us use  $P_{<} \cap E$. Then,
\eq{
\sup_{ u \in P_{<} \cap E}  \abs{G_d(u)} \vee \abs{H_d(u)} \leq  \left(1 -C_d \tfrac{d}{\rs}^{ O_{\delta}(1) } \right)_+    \to 0.     
}
A similar step can be repeated if  $P_{>} \cap E$. Therefore, the statement follows.
\end{proof} 

 \begin{proposition}
 \label{prop:asymptoticlimit}
 Let $\ru \leq r$ and
\eq{
 t \in \begin{cases}
(0,\infty), & \alpha \in [0,0.5) \\
(0,\infty) \setminus \{ j^{\alpha}: j \in \N \}, & \alpha > 0.5,
\end{cases}  \quad  \mathsf{\kappa}_{\mathrm{eff}}  \coloneqq \begin{cases}
r^{\alpha}, & \alpha \in [0,0.5) \\
1, & \alpha > 0.5.
\end{cases} 
}
We have  
\begin{itemize}
    \item For $\mathsf{K} \in \{  G, H \}$ and $t \not = \lim_{d \to \infty} \frac{1}{\lambda_j \mathsf{\kappa}_{\mathrm{eff}}}$, we have
    \eq{
    \mathsf{K}_d( \tfrac{1}{\lambda_j t   \mathsf{\kappa}_{\mathrm{eff}}} )  -   \mathbbm{1} \{   \tfrac{1 }{\lambda_j}  > t \mathsf{\kappa}_{\mathrm{eff}} \}      = o_d(1). \label{eq:alignmentasymptform}
    }
    \item For $\mathsf{K} \in \{ F, G, H \}$, 
    \eq{
\frac{1}{\normL^2} \sum_{j= 1}^{\ru} \lambda^2_j \left( \mathsf{K}_d( \tfrac{1}{\lambda_j t   \mathsf{\kappa}_{\mathrm{eff}}})  -   \mathbbm{1} \{   \tfrac{1 }{\lambda_j}  > t \mathsf{\kappa}_{\mathrm{eff}} \}     \right)  = o_d(1).   \label{eq:riskasymptform}
}
\end{itemize}
 \end{proposition}

\begin{proof}
The first item immediately follows Lemma \ref{lem:asympkernel}. In the following, we will prove the second item for the heavy and light tailed cases separately.
\paragraph{For $\alpha \in [0,0.5)$:} We define a sequence of measures $\mu_d \{  \nicefrac{j}{r}   \} \propto j^{- 2 \alpha}$,  $j \leq [r]$.  We observe that
\begin{itemize}
\item We have $\mu_d \to \mu$ weakly such that $\mu$ is supported on $[0,1]$ and $\mu([0,\tau]) =   \tau^{1 - 2\alpha}$ for $\tau \in [0,1]$.
\item Moreover,  $  \eqref{eq:riskasymptform} =   \E_{X \sim \mu_d} \big[ ( \mathsf{K}_d( X^{\alpha}/t ) -  \mathbbm{1} \{ X^{\alpha}  >  t  \}) \mathbbm{1} \{ X \leq \tfrac{\ru}{r}  \}  \}    \big]$.
\end{itemize}
By using the   $\mathcal{C}_{\delta}$ definition in Lemma \ref{lem:asympkernel}:
\eq{
& \abs*{   \E_{X \sim \mu_d} \big[ ( \mathsf{K}_d( X^{\alpha}/t ) -  \mathbbm{1} \{ X^{\alpha}  >  t  \}) \mathbbm{1} \{ X \leq \tfrac{\ru}{r}  \}  \}    \big]  } \\
& \leq    \E_{X \sim \mu_d} \left[ \abs{  \mathsf{K}_d( X^{\alpha}/t ) -  \mathbbm{1} \{ X^{\alpha}  >  t   \}  } \mathbbm{1} \{ X^{\alpha}  \in [0,1] \setminus \mathcal{C}_{\delta} \}  \right] \\ 
 & +   \E_{X \sim \mu_d} \left[ \abs{  \mathsf{K}_d( X^{\alpha}/ t ) -  \mathbbm{1} \{ X^{\alpha}  >  t   \}  } \mathbbm{1} \{ X^{\alpha}  \in \mathcal{C}_{\delta} \}  \right]  \\
 & \labelrel\leq{riskasymp:ineqq0}  o_d(1) + \mpr_{X \sim \mu}[ X^{\alpha}  \in C_{\delta} ],
}  
where we used the second item in Lemma \ref{lem:asympkernel} for \eqref{riskasymp:ineqq0}.  Since     $\mpr_{X \sim \mu}[ X^{\alpha}  \in C_{\delta} ] \xrightarrow{\delta \to 0} 0$, we have the first result.

\smallskip
\paragraph{For $\alpha > 0.5$:} We define a sequence of measures $\mu_d \{  j  \} \propto j^{- 2 \alpha}$,  $j \leq [r]$.  We observe that
\begin{itemize}
\item We have $\mu_d \to \mu$ weakly such that $ \mu\{ j \} \propto j^{- 2 \alpha}$ for $j \in \N$.
\item Moreover,  $\eqref{eq:riskasymptform}  =  \E_{X \sim \mu_d} \big[  ( \mathsf{K}_d( X^{\alpha}/t ) -  \mathbbm{1} \{ X^{\alpha}  >  t   \}  ) \mathbbm{1} \{ X \leq \ru \}  \big]$.
\end{itemize}
Let  $t  \in \big( (j-1)^{\alpha}, j^{\alpha} \big)$ for some $j  \in \N$.  For small enough $\delta > 0$,  we have
\eq{
&  \abs*{ \E_{X \sim \mu_d} \big[  ( \mathsf{K}_d( X^{\alpha}/t ) -  \mathbbm{1} \{ X^{\alpha}  >  t   \}  ) \mathbbm{1} \{ X \leq \ru \}  \big]  }    \\
&  =   \E_{X \! \sim \mu_d} \! \left[ \abs{  \mathsf{K}_d( X^{\alpha}/t ) -  \mathbbm{1} \{ X^{\alpha}   \! >  \!  t   \}  } \mathbbm{1} \{ X   \! \in  \! [0,\ru] \}    \mathbbm{1} \{ X^{\alpha}  \!\not \in  \! \mathcal{C}_{\delta} \}  \right]    \labelrel={riskasymp:ineqq1} o_d(1), 
}  
where we used   both items in Lemma \ref{lem:asympkernel} for \eqref{riskasymp:ineqq1}.
\end{proof}

\begin{corollary}
\label{cor:asymptoticaux}
For  $1 \geq c_d \gg \log^{-5} d $, we define
\eq{
g_d(\lambda, t)\coloneqq \frac{- \lambda  \exp(- t \lambda) }{1 - \exp(- t \lambda)} + \frac{\lambda^2 \exp(- t \lambda)}{(1 - \exp(-t\lambda))^2} \Big(  \frac{c_d}{t} \frac{\rs}{d} + \frac{\lambda \exp(-t \lambda)}{1 - \exp(- t \lambda)}   \Big)^{-1}
}
Let $\ru \leq r$ and
\eq{
 \mathsf{\kappa}_{\mathrm{eff}}  \coloneqq \begin{cases}
r^{\alpha}, & \alpha \in [0,0.5) \\
1, & \alpha > 0.5
\end{cases}, \qquad   \mathsf{T}_{\mathrm{eff}} \coloneqq  \mathsf{\kappa}_{\mathrm{eff}} \log \nicefrac{d}{\rs}.
}
We have 
\eq{
\frac{1}{\normL^2} \left( \sum_{j= 1}^{\ru} g^2_d(\lambda_j; t   \mathsf{T}_{\mathrm{eff}} ) -   \sum_{j= 1}^{\ru} \lambda_j^2   \mathbbm{1} \{   \tfrac{1 }{\lambda_j}  > t \mathsf{\kappa}_{\mathrm{eff}} \}     \right)  = o_d(1)  
}
for any fixed
\eq{
 t \in \begin{cases}
(0,\infty), & \alpha \in [0,0.5) \\
(0,\infty) \setminus \{ j^{\alpha}: j \in \N \}, & \alpha > 0.5.
\end{cases} 
}
\end{corollary}

\begin{proof}
We observe that
\eq{
g_d( \lambda; t) =  \lambda  \left( 1 - \frac{1}{ 1 - \exp(- t \lambda)  + \frac{d}{\rs} \frac{\lambda t}{c_d}  \exp(- t \lambda) }  \right)   }
Therefore, we have  $g^2_d(\lambda; t  \mathsf{T}_{\mathrm{eff}}) =  \lambda^2 F_d( \tfrac{1}{\lambda_j  t \mathsf{\kappa}_{\mathrm{eff}}})$.
Then,  by Proposition  \ref{prop:asymptoticlimit}
\eq{
\frac{1}{\normL^2} \Bigg( \sum_{j= 1}^{\ru} g^2_d(\lambda_j; t  \mathsf{T}_{\mathrm{eff}})  -  &   \sum_{j= 1}^{\ru} \lambda_j^2\mathbbm{1} \{ \tfrac{1}{\lambda_j} \geq  t \mathsf{\kappa}_{\mathrm{eff}} \}   \Bigg)  \\
& =   \frac{1}{\normL^2}  \sum_{j= 1}^{\ru}    \lambda_j^2 \Big(   F_d( \tfrac{1}{\lambda_j  t \mathsf{\kappa}_{\mathrm{eff}}}) -  \mathbbm{1} \{ \tfrac{1}{\lambda_j} >  t \mathsf{\kappa}_{\mathrm{eff}} \}   \Big) = o_d(1).
}
\end{proof}

\end{document}